\documentclass[11pt]{article}
\pdfoutput=1

\usepackage{mathrsfs}
\usepackage{amsmath,amssymb}
\usepackage{bm}
\usepackage{natbib}
\usepackage[usenames]{color}
\usepackage{amsthm}

\usepackage{multirow} 
\usepackage{enumitem}
\usepackage{wrapfig}

\usepackage[colorlinks,
linkcolor=red,
anchorcolor=blue,
citecolor=blue
]{hyperref}

\newcommand{\Sb}{\mathbf{S}}

\newcommand{\E}{\mathbb{E}}
\newcommand{\R}{\mathbb{R}}
\newcommand{\wstar}{\wb^\star}
\newcommand{\dd}{\,\mathrm{d}}

\newcommand{\ip}[2]{\left\langle#1,#2\right\rangle}
\newcommand{\FLSA}{F_{\mathrm{LSA}}}

\usepackage{mylatexstyle}

\usepackage{setspace}
\usepackage[left=1in, right=1in, top=1in, bottom=1in]{geometry}

\usepackage{xcolor}

\ifdefined\final
\usepackage[disable]{todonotes}
\else
\usepackage[textsize=tiny]{todonotes}
\fi
\allowdisplaybreaks[4]

\title{\huge Transformers with CoT Discover Heavy-Ball Acceleration Through Training}

\author
{
Yuxuan Cai\thanks{
\scriptsize School of Computing \& Data Science, 
	The University of Hong Kong;  {\tt chriscai@connect.hku.hk}}
\quad and \quad 
Yuan Cao\thanks{
\scriptsize School of Computing \& Data Science, 
	The University of Hong Kong;  {\tt yuancao@hku.hk}}
}

\date{}

\begin{document}

\maketitle

\begin{abstract}
Modern transformers often solve complex tasks by generating and refining
intermediate reasoning steps. To develop a theoretical understanding of
chain-of-thought (CoT) reasoning, a notable recent study by
\citet{huang2025cotgd} viewed transformers as algorithm executors, and showed
that when transformers' CoT steps are supervised by a gradient-descent
teacher, they can learn to implement gradient descent for in-context weight
prediction in linear regression. In this paper, we study the same in-context
weight prediction task but ask whether CoT reasoning can go beyond merely
imitating the teacher algorithm. We show that, after initial supervision
from a gradient-descent teacher, further trajectory-level supervision using
the true regression weights leads the model to implement the Heavy-Ball
method, thereby achieving acceleration over the original teacher. Notably,
this accelerated algorithm emerges without supervision from any Heavy-Ball
teacher. Our results provide a concrete theoretical example in which
transformer reasoning can discover new algorithms through training, rather
than merely imitate the algorithm used for supervision.
\end{abstract}

\section{Introduction}
Transformers have demonstrated remarkable success across a wide range of applications, including natural language processing \citep{vaswani2017attention} and computer vision \citep{dosovitskiy2020image}. These models have also exhibited complex reasoning abilities, particularly in mathematical problem solving and coding \citep{chowdhery2023palm}. Chain-of-thought (CoT) reasoning has emerged as a key approach for enhancing such reasoning capabilities, enabling models to generate intermediate reasoning steps before producing the final answer \citep{wei2022chain,kojima2022large}.

\begin{wrapfigure}{R}{0.5\textwidth}
\vskip -0.2in
     \centering
     \includegraphics[width=0.48\columnwidth]{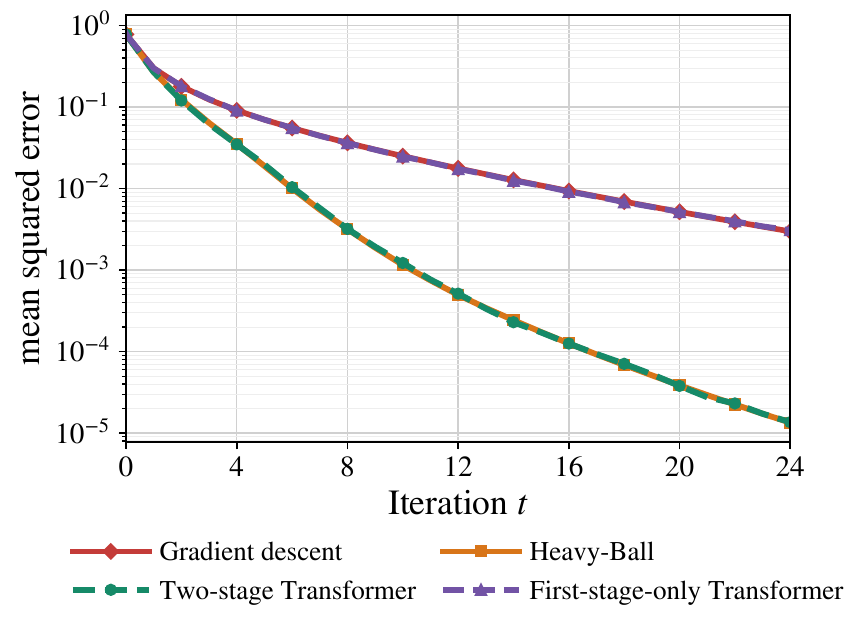}
      \vskip -0.2in
    \caption{Comparison of gradient descent, Heavy-Ball, a transformer trained solely based on gradient-descent supervision, and a transformer trained first using the gradient-descent teacher and then using ground-truth regression weights.}
 \label{fig:intro-four-methods}
 \vspace{-5pt}
\end{wrapfigure}

To theoretically understand the capabilities of transformers, recent works have viewed transformers as ``algorithm executors.'' In particular, transformers have been shown to implement optimization algorithms for in-context learning tasks, where the context is modeled as the training dataset for a learning problem. \citet{akyurek2022what,vonoswald2023transformers} empirically demonstrated that transformers can perform in-context linear regression via gradient descent. \citet{zhang2023trained} theoretically showed that one-layer transformers can be trained to perform in-context linear regression, while \citet{ahn2023linear} showed that deep transformers can be trained to implement preconditioned gradient descent. \citet{guo2023representations,bai2023transformers} further demonstrated the in-context learning capabilities of transformers by constructing multi-head ReLU transformers that implement a variety of learning algorithms. \citet{gatmiry2024looped} showed that looped transformers can implement gradient descent for in-context linear regression. More recently, \citet{huang2025cotgd} showed that, when supervised by a gradient descent teacher, transformers can learn to use CoT reasoning for in-context weight prediction, with each reasoning step implementing one gradient descent update.

In this work, we consider a setting similar to that of \citet{huang2025cotgd}, where transformers use CoT to perform in-context weight prediction. However, rather than asking whether transformers can mimic a gradient-descent teacher, we ask a different question: \emph{Can transformers with CoT go beyond mimicking the teacher algorithm and discover new algorithms through training?} To investigate this question, we consider a two-stage training procedure. In Stage 1, we supervise the transformer's CoT trajectory using a gradient-descent teacher, following the setup of \citet{huang2025cotgd}. In Stage 2, we replace the teacher supervision with supervision from the true regression weights. We find that, when trained via this two-stage procedure, transformers with CoT can learn the Heavy-Ball method, achieving acceleration over the gradient descent teacher in Stage 1 (see Figure~\ref{fig:intro-four-methods}). We also establish rigorous theoretical guarantees to further demonstrate and explain this phenomenon.

The main contributions of our work are as follows:
\begin{itemize}[leftmargin = *,nosep]
\item We identify an interesting phenomenon: as illustrated in Figure~1, a
transformer trained solely with gradient-descent supervision learns to
implement multi-step gradient descent, consistent with previous results
\citep{huang2025cotgd}. However, when further trained under trajectory supervision of the ground-truth regression weights, the transformer surprisingly learns to generate CoT trajectories that approximate the Heavy-Ball method, achieving acceleration. Notably, the Heavy-Ball method emerges without supervision from a Heavy-Ball teacher, demonstrating the capability of CoT to discover ``new'' algorithms unseen during training.

\item We provide a theoretical explanation for how transformers with CoT learn the Heavy-Ball method. We show that, despite the nonconvex loss landscape, both stages of training succeed, with the corresponding objective functions attaining near-optimal values. We further identify the algorithm implemented by the trained transformer after the two-stage procedure and prove that its learned CoT iterates approximate the Heavy-Ball method. This establishes that the model not only optimizes the prescribed objectives, but also internalizes an accelerated optimization procedure. Our result therefore gives a rigorous demonstration of how transformers with CoT can go beyond mimicking a gradient descent teacher and discover a faster algorithm through training.



\item The setting in this paper goes beyond classic in-context learning formulations. First, to distinguish the Heavy-Ball method from gradient descent, we consider a setting with a Marchenko–Pastur spectral limit, in contrast to prior settings with well-concentrated sample covariance matrices \citep{zhang2023trained,huang2025cotgd}. Second, the second-stage training is without teacher forcing, and the loss lacks a clean analytical form due to the autoregressive nature of the outputs. To overcome these difficulties, we introduce auxiliary parameters to establish training convergence and develop recursively defined matrix polynomials to analyze the second-stage loss. The matrix polynomials also serve as a central tool for characterizing the connection between the learned algorithm and the Heavy-Ball method. These new proof techniques may be of independent interest.

\end{itemize}



\section{Preliminaries}
\label{sec:preliminaries}
In this section, we introduce the in-context weight prediction task, the setting of one-layer linear self-attention
with CoT, and the two-stage training procedure considered in
this work.
\subsection{In-Context Weight Prediction with CoT}\label{subsec:setting1}
\noindent\textbf{In-context weight prediction.} A line of recent works
\citep{ahn2023linear,bai2023transformers,zhang2023trained} has studied
in-context learning  under a theoretical setting where the goal is to
predict an unknown response \(y_{\rm query}\), given a query
\(\xb_{\rm query}\) and a collection of context data pairs
\(\{(\xb_i,y_i)\}_{i\in[n]}\). Beyond directly outputting the query response, 
\citet{huang2025cotgd,zhang2026normalized} further studied in-context weight prediction,
where the goal is to predict the regression model weights.

In our work, we study the same in-context weight prediction task as in
\citet{huang2025cotgd}. Specifically, the context is given as a training
data set \(\{(\xb_i,y_i)\}_{i\in[n]}\) of a linear regression task. We assume
that each context is associated with a ground-truth linear regression vector
\(\wstar\), randomly generated according to the standard Gaussian distribution
\(\wstar\sim N(0,\Ib_d)\). Given \(\wstar\), for each \(i\in[n]\), we independently generate  
\(\xb_i\sim N(0,\Ib_d)\) and set
\(y_i=\xb_i^\top\wstar\). In the in-context weight prediction task, independent
copies of \((\{(\xb_i,y_i)\}_{i\in[n]},\wstar)\) are available for training,
and the goal is to obtain a model that can predict the corresponding \(\wstar\) based on new
contexts.

\noindent\textbf{Linear self-attention architecture.} 
We consider a simplified linear
self-attention (LSA) model by removing
the softmax nonlinearity, merging the projection and value matrices into \(\Vb\in\R^{d_e\times d_e}\), and merging
the key and query matrices into  \(\Wb\in\R^{d_e\times d_e}\),  which follows the
settings in recent works \citep{zhang2023trained,ahn2023linear,gatmiry2024looped,huang2025cotgd}. Here $d_e$ denotes the dimension of the token embeddings.
For an input matrix \(\Zb\), we denote
\begin{equation}
 F_{\mathrm{LSA}}(\Zb;\Vb,\Wb)=\Zb+\Vb\Zb\cdot \Zb^\top \Wb\Zb.
 \label{eq:lsa}
\end{equation}
\noindent\textbf{Chain-of-thought reasoning.} 
In CoT reasoning, a transformer model autoregressively generates
reasoning steps before reaching the final result. For in-context weight
prediction, we follow a similar setup to that considered in
\citet{huang2025cotgd}, where the reasoning steps correspond to iterations of
an optimization algorithm, and consider the initial input of the transformer
model as follows:
\begin{align}
 \Zb_0:=
 \begin{bmatrix}
 \Xb & \zero_{d\times1}\\
 \yb & 0\\
 \zero_{d\times n} & \wb_0\\
 \zero_{1\times n} & 1\\
 \zero_{T\times n} & \eb_0
 \end{bmatrix}
 \in\R^{d_e\times(n+1)},
 ~~~
\begin{split}
\Xb
&:=\frac{1}{\sqrt n}
\begin{bmatrix}
\xb_1 & \xb_2 & \cdots & \xb_n
\end{bmatrix}
\in\mathbb{R}^{d\times n},\\
\yb
&:=\frac{1}{\sqrt n}
\begin{bmatrix}
y_1 & y_2 & \cdots & y_n
\end{bmatrix}
\in\mathbb{R}^{1\times n}.
\end{split}
 \label{eq:prompt-initial}
\end{align}
where
\(\wb_0\) is
the initialization of the optimization algorithm. Following
\citet{huang2025cotgd}, we assume \(\wb_0=\zero_d\), and also assume that
the training data set \(\{(\xb_i,y_i)\}_{i\in[n]}\) and the weight vectors are
encoded in separate dimensions in the token embeddings. However, different
from \citet{huang2025cotgd}, we include factors \(1/\sqrt n\) in the
definitions of \(\Xb\) and \(\yb\). This difference is minor,
since \citet{huang2025cotgd} instead include an additional factor of \(1/n\)
in their definition of linear attention. Consequently, the well-trained model
analyzed in their work is equivalent to one in which the \(1/n\) factor in
linear attention is omitted and the factors of \(1/\sqrt n\) are included in
\(\Xb\) and \(\yb\). A more important difference is that we introduce a one-hot positional encoding
\(\eb_0=[1,0,\ldots,0]^\top\in\R^T\) for the initialization token, where
\(T\) denotes the maximum number of CoT steps.
Notably, incorporating these positional encodings requires us to specify a maximum number of reasoning steps before defining the model. Accordingly, each token has total embedding dimension \(d_e:=2d+T+2\).

Now we can define the input of the
\(t\)-th reasoning step recursively. Specifically, we denote the model prediction as
\(\wb_{t+1}:=\FLSA(\Zb_t;\Vb,\Wb)_{[d+2:2d+1,-1]}\) and append it to \(\Zb_t\) so that
\[\Zb_{t+1}=\left[\Zb_t,(\zero_d^\top,0,\wb_{t+1}^\top,1,\eb_{t+1}^\top)^\top\right],\quad 0\le t<T-1,\]
where 
\(\eb_t=\bigl(\underbrace{0,\ldots,0}_{t},1,\underbrace{0,\ldots,0}_{T-t-1}\bigr)^\top\in\R^T\) is the one-hot positional encoding.
\subsection{Two-Stage Training}
We consider a two-stage training framework.
Stage 1 introduces supervision from a gradient-descent teacher, while in Stage 2, we supervise the CoT trajectory with
the true regression weights.

\noindent\textbf{Stage 1.} In Stage 1, following \citet{huang2025cotgd}, we consider supervision from a
gradient-descent teacher with teacher forcing. The input to the transformer
at iteration \(t\) is given as follows:
\[
 \Zb_t^{\mathrm{GD}}=
 \begin{bmatrix}
 \Xb&\zero_{d\times1}&\zero_{d\times1}&\cdots&\zero_{d\times1}\\
 \yb&0&0&\cdots&0\\
 \zero_{d\times n}&\wb_0&\wb_1^{\mathrm{GD}}&\cdots&\wb_t^{\mathrm{GD}}\\
 \zero_{1\times n}&1&1&\cdots&1\\
 \zero_{T\times n}&\eb_0&\eb_1&\cdots&\eb_t
 \end{bmatrix}
 \in\R^{d_e\times(n+t+1)},\qquad 0\le t<T,
\]
where, for \(t\ge1\),
\(\wb_t^{\mathrm{GD}}=\wb_{t-1}^{\mathrm{GD}}-\eta_0 \Xb(\Xb^\top \wb_{t-1}^{\mathrm{GD}}-\yb^\top)\) is the teacher's
iterate after \(t\) gradient steps. The learning rate
\(\eta_0\) is fixed across all steps and tasks. Given
\(\Zb_t^{\mathrm{GD}}\), we aim to train the transformer to predict the next gradient-descent iterate.
We consider gradient flow over the population loss:
\begin{equation}
 \cL_{\rm GD}(\Vb,\Wb)
 =\frac1{2T}\E
 \left[\sum_{t=0}^{T-1}
 \left\|\FLSA(\Zb_t^{\mathrm{GD}};\Vb,\Wb)_{[d+2:2d+1,-1]}
 -\wb_{t+1}^{\mathrm{GD}}\right\|^2\right].
 \label{eq:GD-guide-loss}
\end{equation}
The gradient flow dynamics of the parameters are
\begin{equation}
 \frac{\dd\btheta}{\dd\tau}=-\nabla_{\btheta}\cL_{\rm GD}(\btheta),
 \qquad \btheta:=(\Vb,\Wb).
 \label{eq:GD-gradient-flow}
\end{equation}
We train the model until $\cL_{\rm GD}$ is minimized below a threshold.

\noindent\textbf{Stage 2.} We then consider further training the transformer on ground-truth solutions \(\wstar\) without supervision from any classical algorithm. At the
\(t\)-th iteration, the transformer makes predictions based on its own autoregressively generated sequence $\wb_0,\wb_1,\ldots,\wb_t$. Therefore, we define
\[
 \Zb_t=
 \begin{bmatrix}
 \Xb&\zero_{d\times1}&\zero_{d\times1}&\cdots&\zero_{d\times1}\\
 \yb&0&0&\cdots&0\\
 \zero_{d\times n}&\wb_0&\wb_1&\cdots&\wb_t\\
 \zero_{1\times n}&1&1&\cdots&1\\
 \zero_{T\times n}&\eb_0&\eb_1&\cdots&\eb_t
 \end{bmatrix}
 \in\R^{d_e\times(n+t+1)},\qquad 0\le t<T,
\]
and define the training loss in Stage 2 as
\begin{equation}
 \cL(\Vb,\Wb)
 =\frac1{2T}\E\left[\sum_{t=0}^{T-1}
 \left\|\FLSA(\Zb_t;\Vb,\Wb)_{[d+2:2d+1,-1]}-\wstar\right\|^2\right].
 \label{eq:loss}
\end{equation}
We use exact population gradient flow on the raw parameters:
\begin{equation}
 \frac{\dd\btheta}{\dd\tau}=-\nabla_{\btheta}\cL(\btheta).
 \label{eq:outcome-gradient-flow}
\end{equation}
Notably, without teacher forcing, Equation~\eqref{eq:loss} does not have a clean 
analytic form because \(\Zb_t\) is generated autoregressively. During training,
\(\wb_1,\ldots,\wb_T\) can be very complex, which differs substantially from Stage 1, as well as
the setting of \citet{huang2025cotgd}. 

We note that Stage 1 is analogous to pretraining with next-token prediction and to supervised fine-tuning: the input contains teacher weights rather than the model's own predictions. The loss function in Stage 2 can be viewed as an objective resembling reinforcement learning post-training with discount factor 1. However, because our setting of in-context linear regression is simplified, these analogies serve only as high-level motivation for considering such a training procedure. They do not imply that our analysis directly yields conclusions about general pretraining or post-training.


\section{Main Result}
\label{sec:dynamics}

In this section, we  prove the convergence of the two-stage training procedure and show that the final solution approximates the time-varying Heavy-Ball method,  defined as follows. 

Let \(\{(\alpha_t,\beta_t)\}_{t=0}^{T-1}\) be any sequence with
\(\beta_0=0\). Initialize \(\wb_0^{\rm HB}=\zero_d\) and for $1\le t <T$, define
\begin{equation}
 \begin{aligned}
  \wb_1^{\rm HB}=\alpha_0 \Xb\yb^\top, 
  \quad\wb_{t+1}^{\rm HB}
  =\wb_t^{\rm HB}-\alpha_t(\Xb\Xb^\top \wb_t^{\rm HB}-\Xb\yb^\top)
   +\beta_t(\wb_t^{\rm HB}-\wb_{t-1}^{\rm HB}).
 \end{aligned}
 \label{eq:learned-hb-update}
\end{equation}
We consider the best coefficients under the same loss objective:
\[
 \{(\alpha_t^\star,\beta_t^\star)\}_{t=0}^{T-1}
 \in\argmin_{\substack{\{(\alpha_t,\beta_t)\}_{t=0}^{T-1},
                                      \beta_0=0}}
 \frac1{2T}\sum_{t=1}^T\E\|\wb_t^{\rm HB}-\wstar\|^2.
\]

In the following, we introduce two technical assumptions. We first note that the transformer input \eqref{eq:prompt-initial} is given as a block matrix, motivating us to consider the weight matrices as block matrices:
\[
\begin{aligned}
 \Vb&=\begin{bmatrix}
 \Vb_{11}&\Vb_{12}&\Vb_{13}&\Vb_{14}&\Vb_{15}\\
 \Vb_{21}&\Vb_{22}&\Vb_{23}&\Vb_{24}&\Vb_{25}\\
 \Vb_{31}&\Vb_{32}&\Vb_{33}&\Vb_{34}&\Vb_{35}\\
 \Vb_{41}&\Vb_{42}&\Vb_{43}&\Vb_{44}&\Vb_{45}\\
 \Vb_{51}&\Vb_{52}&\Vb_{53}&\Vb_{54}&\Vb_{55}
 \end{bmatrix},\qquad
 \Wb=\begin{bmatrix}
 \Wb_{11}&\Wb_{12}&\Wb_{13}&\Wb_{14}&\Wb_{15}\\
 \Wb_{21}&\Wb_{22}&\Wb_{23}&\Wb_{24}&\Wb_{25}\\
 \Wb_{31}&\Wb_{32}&\Wb_{33}&\Wb_{34}&\Wb_{35}\\
 \Wb_{41}&\Wb_{42}&\Wb_{43}&\Wb_{44}&\Wb_{45}\\
 \Wb_{51}&\Wb_{52}&\Wb_{53}&\Wb_{54}&\Wb_{55}
 \end{bmatrix}.
\end{aligned}
\]
where the block row and column sizes are \((d,1,d,1,T)\). We make the following assumptions.
\begin{assumption}
\label{ass:init}
We assume that the
initialization of the parameters satisfies
\begin{equation}
 \Vb(0)=
 \begin{pmatrix}
  \zero&\zero&\zero&\zero&\zero\\
  \zero&0&\zero&0&\zero\\
  \zero&\zero&v_{33}(0)\Ib_d&\zero&\zero\\
  \zero&0&\zero&0&\zero\\
  \zero&\zero&\zero&\zero&\zero
 \end{pmatrix},\qquad
 \Wb(0)=
 \begin{pmatrix}
  \zero&\zero&\zero&\zero&\zero\\
  \zero&0&\zero&\sigma&\zero\\
  \zero&\zero&\zero&\zero&\zero\\
  \zero&0&\zero&0&\zero\\
  \zero&\zero&\zero&\zero&\zero
 \end{pmatrix}.
 \label{eq:warm-init}
\end{equation}
Here \(v_{33}(0)\) and \(\sigma\) are any constants satisfying
\(v_{33}(0)>0\) and \(\sigma<0\), independent of \(d,n,T\).
Further, we fix \(\Wb_{33}=\zero_{d\times d}\) for all \(\tau>0\).
\end{assumption}
Assumption~\ref{ass:init} is a technical assumption that poses constraints on the initialization and on freezing one of the blocks in the \(\Wb\) matrix. We note that similar assumptions have been frequently considered in theoretical studies of the training dynamics of transformer models. For example, \citet{huang2025cotgd,wu2026looped} considered almost zero initializations with certain nonzero blocks. \citet{huang2025cotgd,gatmiry2024looped} considered freezing certain blocks in the weight matrices. 

We also make the following assumption specifying hyperparameter scalings considered in this paper.
\begin{assumption}
\label{ass:scaling}
Suppose that \(T=\Theta((\log d)^{1/4})\), \(\eta_0\in[0.4,0.9]\), \(n=\Theta(d)\), and \(d/n \le 0.1\). 
\end{assumption}
We emphasize that Assumption~\ref{ass:scaling} introduces a key difference compared with related work \citep{huang2025cotgd}, where the authors assumed that \(n=\Theta(d\log^5d)\gg d\), implying that the sample covariance matrix \(\Sb = n^{-1}\sum_{i=1}^n \xb_i\xb_i^\top\) concentrates around \(\Ib_d\). Notably, \(\Sb\approx\Ib_d\) is a setting where the Heavy-Ball method and gradient descent achieve the same convergence rate. Hence, to convincingly demonstrate the learning of the Heavy-Ball method, we consider a different regime where $n$ and $d$ are proportional to each other. In this setting, the eigenvalue distribution of $\Sb$ is asymptotically characterized by the Marchenko-Pastur (MP) law~\citep{marchenko1967}.

We now show that, under the initialization above and the two-stage training
procedure, our transformer model discovers the Heavy-Ball method. 
\begin{theorem}
\label{thm:main-result}
For any \(0<\epsilon<1\), suppose that $n,d = \Omega(\exp(\log^4(1/\epsilon)))$.
Further suppose that the training with gradient flow \eqref{eq:GD-gradient-flow} in Stage 1 stops after \(\tau_1 = d^{o(1)}+O(T\log d)\), and the follow-up Stage 2 training with gradient flow \eqref{eq:outcome-gradient-flow} runs for \(\tau-\tau_1 = O(d^{1/\alpha+o(1)}\epsilon^{-1/\alpha}) \),
where \(\alpha>0\) is a constant independent of \(d,n,T\). Then under Assumptions~\ref{ass:init} and~\ref{ass:scaling}, the following results hold:
\begin{itemize}[leftmargin = *,nosep]
\item The training loss in Stage 1 reaches \(\cL_{\rm GD} \le d^{-1/4}\).
\item The training loss in Stage 2 reaches \(\cL-\min \cL \le \epsilon\).
\item Uniformly for \(T/2\le t<T\), it holds that 
\begin{align*}
    \frac{\E\|\wb_{t+1}(\tau)-\wb_t^{\rm HB}\|^2}
      {\E\|\wb_t^{\rm HB}-\wstar\|^2}
 \le\epsilon.
\end{align*}
Here \(\wb_t^{\rm HB}\) denotes the $t$-th iterate of the Heavy-Ball method with
parameters \(\{(\alpha_t^\star,\beta_t^\star)\}_{t=0}^{T-1}\).
\end{itemize}
\end{theorem}
Theorem~\ref{thm:main-result} shows that the training in both Stage 1 and Stage 2 converges despite the nonconvex loss landscape. Moreover, the last conclusion of Theorem~\ref{thm:main-result} shows that  transformer produces CoT steps whose $(t+1)$-th step is close to the $t$-th iteration of the Heavy-Ball method. This demonstrates that the  algorithm learned by the transformer is essentially equivalent to the Heavy-Ball method. Notably, the two-stage training procedure never involves supervision from any Heavy-Ball teacher. Our result thus indicates that transformers with CoT are capable of discovering unseen algorithms through proper training, providing complementary findings to \citet{huang2025cotgd} where the authors demonstrated the capability of transformers to recover the gradient-descent teacher algorithm.

Theorem~\ref{thm:main-result} is under the assumption that $n,d = \Omega(\exp(\log^4(1/\epsilon)))$. This condition is super-polynomial but subexponential in $1/\epsilon$. While such a super-polynomial assumption is relatively strong, the last conclusion in Theorem~\ref{thm:main-result} with $\epsilon=0.5$ can already imply that $\E\|\wb_{t+1}(\tau)-\wstar\|^2 = O(\E\|\wb_t^{\rm HB}-\wstar\|^2)$. Therefore, Theorem~\ref{thm:main-result} demonstrates that transformers with CoT can discover an algorithm with an accelerated convergence rate under mild assumptions.

\section{Proof sketch}
\label{sec:proof-sketch}
In this section, we outline the proof of Theorem~\ref{thm:main-result}. Since we consider the setting where  \(n=\Theta(d)\) and the eigenvalues of \(\Sb\) follow the MP law, 
even in Stage 1, our analysis is significantly different from the prior work \citep{huang2025cotgd}.  Stage 2 is also more challenging without teacher forcing. To address these challenges, we develop a proof
technique based on auxiliary parameters.

We first identify the parameter structure preserved throughout training.
\begin{lemma}\label{lem:invariant-parameters}
Throughout the two-stage training, the parameters \(\Vb(\tau),\Wb(\tau)\) retain the form
\begingroup
\small
\[
\setlength{\arraycolsep}{2pt}
 \begin{aligned}
 \Vb(\tau)&=\begin{pmatrix}
 \zero&\zero&\zero&\zero&\zero\\
 \zero&0&\zero&0&\zero\\
 v_{31}(\tau)\Ib_d&\zero&v_{33}(\tau)\Ib_d&\zero&\zero\\
 \zero&0&\zero&0&\zero\\
 \zero&\zero&\zero&\zero&\zero
 \end{pmatrix},\quad
 \Wb(\tau)=\begin{pmatrix}
 \zero&\zero&w_{13}(\tau)\Ib_d&\zero&\zero\\
 \zero&0&\zero&w_{24}(\tau)&\wb_{25}(\tau)\\
 \zero&\zero&\zero&\zero&\zero\\
 \zero&0&\zero&w_{44}(\tau)&\wb_{45}(\tau)\\
 \zero&\zero&\zero&\wb_{54}(\tau)&\Wb_{55}(\tau)
 \end{pmatrix}.
 \end{aligned}
\]
\endgroup
\end{lemma}
Lemma~\ref{lem:invariant-parameters} shows that our parameters remain on a
manifold throughout two-stage training. By this lemma, we can reduce the
analysis of the full matrices \(\Vb,\Wb\) to that of
\[
 \bar{\btheta}(\tau)=\bigl(\Vb_{31}(\tau),\Wb_{13}(\tau),\Vb_{33}(\tau),
 w_{24}(\tau),\wb_{25}(\tau),w_{44}(\tau),\wb_{45}(\tau),
 \wb_{54}(\tau),\Wb_{55}(\tau)\bigr).
\]
This conclusion also allows us to establish the following lemma.
\begin{lemma}\label{lem:polynomial-output}
Throughout training in Stages 1 and 2, the \(t\)-th CoT iterate
\(\wb_t(\tau)\) has the form
\(
 \wb_t(\tau)=p_{t,\bar{\btheta}(\tau)}(\Sb)\wb^\star,
\)
where \(p_{t,\bar{\btheta}(\tau)}\) is a polynomial of degree at most \(t\) with zero constant term.
\end{lemma}
Lemma~\ref{lem:polynomial-output} is central to our analysis. Since we
consider the challenging setting where \(n=\Theta(d)\), the sample covariance
matrix \(\Sb\) cannot be approximated by the identity matrix, and the study of
the training dynamics requires precise characterization of their dependence
on \(\Sb\). Lemma~\ref{lem:polynomial-output} provides this characterization.
Recall that the losses in both stages are defined by the distance from
\(\wb_t\) to \(\wb^\star\) or \(\wb_t^{\rm GD}\). By
Lemma~\ref{lem:polynomial-output}, we can express the losses in the two stages
as functions of matrix polynomials in \(\Sb\).
We now give the two-stage analysis.

\noindent\textbf{Stage 1.} 
Thanks to Lemma~\ref{lem:polynomial-output}, we can give
the following decomposition of the training loss.
\begin{lemma}\label{lem:stage1-loss}
The population loss can be written as
\begin{equation*}
 \cL_{\rm GD}(\tau)=\E\left[\begin{aligned}
 &A(\Sb)(\widetilde\theta_1(\tau)-\eta_0)^2
 +B(\Sb)(\widetilde\theta_2(\tau)+\eta_0)^2\\
 &\quad+\norm{
 \Cb_1(\Sb)(\widetilde\theta_2(\tau)+\eta_0)
 +\Cb_2(\Sb)\operatorname{vec}(\widetilde{\bTheta}_3(\tau)-\eta_0\eb_1(\one_T-\eb_0)^{\mathsf T})}^2
\end{aligned}\right],
\end{equation*}
where \(A,B \in \R\), \(\Cb_1\in \R^{(T-1)d^2 \times 1}\),
\(\Cb_2 \in \R^{(T-1)d^2 \times T^2}\) are explicit functions of \(\Sb\), and
\begin{align*}
    &\widetilde\theta_1:=v_{31}(w_{24}+\wb_{25}\eb_0),
 \qquad \widetilde\theta_2:=v_{31}w_{13},\qquad \widetilde{\bTheta}_3 \in \R^{T\times T},&&\\
 &(\widetilde{\bTheta}_3)_{1t}
 :=v_{31}(w_{24}+\wb_{25}\eb_t) +\eta_0v_{33}\bigl(w_{44}+\wb_{45}\eb_t
 +\eb_1^{\mathsf T}\wb_{54}+\eb_1^{\mathsf T}\Wb_{55}\eb_t\bigr), &&
 1\le t<T,\\
 &(\widetilde{\bTheta}_3)_{\ell t}
 :=v_{33}\bigl(w_{44}+\wb_{45}\eb_t
 +\eb_\ell^{\mathsf T}\wb_{54}+\eb_\ell^{\mathsf T}\Wb_{55}\eb_t\bigr),&&
 2\le\ell\le t<T.
\end{align*}
\end{lemma}
Lemma~\ref{lem:stage1-loss} implies that, as gradient flow minimizes \(\cL_{\rm GD}(\tau)\), we have 
\(\widetilde\theta_1(\tau)\to\eta_0\), \(\widetilde\theta_2(\tau)\to-\eta_0\), and
\(\widetilde{\bTheta}_3(\tau)\to\eta_0\eb_1(\one_T-\eb_0)^{\mathsf T}\).
This further motivates us to use $\widetilde\theta_1, \widetilde\theta_2,\widetilde{\bTheta}_3 $ as auxiliary parameters 
to study convergence. We emphasize that the training dynamics we analyze
are still gradient flow in the original parameters; the auxiliary
parameters are only a tool for the proof.

The next lemma gives a PL inequality in these auxiliary parameters.
\begin{lemma}\label{lem:auxiliary-pl}
There exists a constant $c_1>0$ such that 
\begin{equation*}
 \left|\partial_{\widetilde\theta_1(\tau)}\cL_{\rm GD}(\tau)\right|^2
 +\frac1d\left|\partial_{\widetilde\theta_2(\tau)}\cL_{\rm GD}(\tau)\right|^2
 +\norm{\nabla_{\widetilde{\bTheta}_3(\tau)}\cL_{\rm GD}(\tau)}_F^2
 \ge c_1\frac{d/n}{T}\cL_{\rm GD}(\tau).
\end{equation*}
\end{lemma}

Notably, this inequality concerns auxiliary parameters, whereas convergence
requires a PL inequality for the original parameters. Therefore, we define
\[
 \widetilde{\btheta}(\tau):=
 \bigl(\widetilde\theta_1(\tau),\widetilde\theta_2(\tau),
 \widetilde{\bTheta}_3(\tau)\bigr),
 \qquad \Jb_1(\bar{\btheta}(\tau)):=\frac{\partial\widetilde{\btheta}(\tau)}{\partial\bar{\btheta}(\tau)}.
\] Then we have the following lemma.
\begin{lemma}\label{lem:stage1-jacobian}
There is a time \(\tau_0=d^{o(1)}\) and a constant \(c_2>0\) such that, for
every \(\tau\ge\tau_0\),
\begin{equation*}
 \Jb_1(\bar{\btheta}(\tau))\Jb_1(\bar{\btheta}(\tau))^{\mathsf T}
 \succeq c_2\,\diag(1,d^{-1},\Ib).
\end{equation*}
\end{lemma}
Note that the inequality in Lemma~\ref{lem:stage1-jacobian} holds only after a time \(\tau_0=d^{o(1)}\).
This highlights the challenge of the problem: the PL condition is established
only in the later phase of training.

By the chain rule,
 \(\nabla_{\bar{\btheta}(\tau)}\cL_{\rm GD}(\tau)=\Jb_1(\bar{\btheta}(\tau))^{\mathsf T}
\nabla_{\widetilde{\btheta}(\tau)}\cL_{\rm GD}(\tau)\).
Combining Lemmas~\ref{lem:auxiliary-pl} and~\ref{lem:stage1-jacobian}
therefore gives the PL inequality with respect to $\bar{\btheta}$. This further leads to the following inequality:
\begin{equation*}
 -\frac{\dd}{\dd\tau}\cL_{\rm GD}(\tau)
 =\norm{\nabla_{\bar{\btheta}(\tau)}\cL_{\rm GD}(\tau)}^2
 \ge c_1c_2\frac{d/n}{T}\cL_{\rm GD}(\tau),\qquad \tau\ge\tau_0.
\end{equation*}
Solving this inequality proves the first conclusion of
Theorem~\ref{thm:main-result}: it takes a training time
\(
 \tau=d^{o(1)}+O((nT/d)\log d)
\)
to reach \(\cL_{\rm GD}\le d^{-1/4}\).

\noindent\textbf{Stage 2.}
Our Stage 2 analysis has two parts: (i) convergence analysis and
(ii) comparison with the Heavy-Ball method. We first discuss part (i) convergence analysis. We have the following lemma.
\begin{lemma}\label{lem:stage2-lemma1}
In Stage 2, the training loss can be written as
\begin{equation*}
\cL(\tau)=\frac{1}{2T}\mathbb{E}\sum_{t=1}^{T}
 \|p_{t,\bar{\btheta}(\tau)}(\Sb)-\Ib_d\|_F^2,
\end{equation*}
where \(p_{t,\bar{\btheta}(\tau)}(\Sb)\) is a polynomial defined recursively by
\begin{align*}
 &p_{0,\bar{\btheta}(\tau)}(\Sb)=\zero_{d\times d},\qquad p_{1,\bar{\btheta}(\tau)}(\Sb)=\widetilde\theta_1(\tau)\Sb,\\
 &p_{t+1,\bar{\btheta}(\tau)}(\Sb)=(\Ib_d+\widetilde\theta_2(\tau)\Sb)p_{t,\bar{\btheta}(\tau)}(\Sb)
 +\sum_{1\le\ell\le t} p_{\ell,\bar{\btheta}(\tau)}(\Sb)\,
 (\widetilde{\bTheta}_4(\tau))_{\ell t},
 \quad 1\le t<T.
\end{align*}
Here, \(\widetilde\theta_1(\tau)\) and \(\widetilde\theta_2(\tau)\)  are defined in Lemma~\ref{lem:stage1-loss}, and \(\widetilde{\bTheta}_4 \in \R^{T \times T}\) is given as
\begin{align*}
 (\widetilde{\bTheta}_4)_{1t}
 &:=\widetilde\theta_1^{-1}v_{31}(w_{24}+\wb_{25}\eb_t)
 +v_{33}\bigl(w_{44}+\wb_{45}\eb_t+\eb_1^{\mathsf T}\wb_{54}+\eb_1^{\mathsf T}\Wb_{55}\eb_t\bigr),
 &&1\le t<T,\\
 (\widetilde{\bTheta}_4)_{\ell t}
 &:=v_{33}\bigl(w_{44}+\wb_{45}\eb_t
 +\eb_\ell^{\mathsf T}\wb_{54}+\eb_\ell^{\mathsf T}\Wb_{55}\eb_t\bigr),
 &&2\le\ell\le t<T.
\end{align*}
\end{lemma}

Lemma~\ref{lem:stage2-lemma1} gives the loss function based on the matrix polynomials $p_{t,\bar{\btheta}(\tau)}(\Sb)$. Compared with Lemma~\ref{lem:stage1-loss} in Stage 1, this result is significantly more complicated, as $p_{t,\bar{\btheta}(\tau)}(\Sb)$ is defined recursively. This is due to the fact that the transformer takes its own outputs in the previous reasoning steps in Stage 2. Despite this more challenging setting, we can still show the PL inequality for these auxiliary parameters $\widetilde\theta_1, \widetilde\theta_2,\widetilde{\bTheta}_4 $. 
Write \(\cL^*\) for the minimum of \(\cL\). Then we have the following result.
\begin{lemma}\label{lem:stage2-pl-inequality}
In Stage 2, there exists a constant $c_3>0$ such that the following PL inequality holds:
\begin{equation*}
 \frac1d\left|\partial_{\widetilde\theta_1(\tau)}\cL(\tau)\right|^2
 +\frac1d\left|\partial_{\widetilde\theta_2(\tau)}\cL(\tau)\right|^2
 +\frac1{\sqrt d}\left\|\nabla_{\widetilde{\bTheta}_4(\tau)}\cL(\tau)\right\|_F^2
 \ge c_3\frac{d/n}{T}(\cL(\tau)-\cL^*).
\end{equation*}
\end{lemma}
Similar to Stage 1, we aim to translate the PL condition from the auxiliary parameters to $\bar{\btheta}$. 
Define
\[
 \widetilde{\btheta}'(\tau):=
 \bigl(\widetilde\theta_1(\tau),\widetilde\theta_2(\tau),
 \widetilde{\bTheta}_4(\tau)\bigr),
 \qquad \Jb_2(\bar{\btheta}(\tau)):=\frac{\partial\widetilde{\btheta}'(\tau)}{\partial\bar{\btheta}(\tau)}.
\]
Then we have the following lemma.
\begin{lemma}\label{lem:stage2-jacobian}
There exists a constant $c_4>0$ such that for all \(\tau\ge\tau_1\),
\begin{equation*}
 \Jb_2(\bar{\btheta}(\tau))\Jb_2(\bar{\btheta}(\tau))^{\mathsf T}
 \succeq\frac{c_4}{d^{o(1)}+\frac{d/n}{T}\tau}
 \diag\!\left(d^{-1},d^{-1},d^{-1/2}\Ib\right).
\end{equation*}
\end{lemma}
Combining Lemmas~\ref{lem:stage2-pl-inequality} and~\ref{lem:stage2-jacobian} gives
\begin{equation*}
 -\frac{\dd}{\dd\tau}(\cL(\tau)-\cL^*)
 =\|\nabla_{\bar{\btheta}(\tau)}\cL(\tau)\|^2
 \ge\frac{c_3c_4(d/n)}{T\left(d^{o(1)}+\frac{d/n}{T}\tau\right)}
 (\cL(\tau)-\cL^*).
\end{equation*}
Solving this inequality gives, for some constant \(\alpha>0\) and every
\(0<\epsilon<1\), \(\cL(\tau)-\cL^*\le\epsilon\) after time
\(O\!\left(d^{o(1)}(\epsilon/d)^{-1/\alpha}\right)\). This proves the second conclusion of
Theorem~\ref{thm:main-result}.

We now explain the proof of the last conclusion of
Theorem~\ref{thm:main-result} and discuss the comparison with the Heavy-Ball method. 
By Lemma~\ref{lem:polynomial-output}, the iterates given by the trained transformer can be expressed by a matrix polynomial of $\Sb$. This motivates us to give the following definition.
Define
\[
 p_t^{\star}:=\operatorname*{arg\,min}_{\substack{p(0)=0\\\deg p\le t}}
 \E\norm{p(\Sb)\wb^\star-\wb^\star}^2, \qquad 0\le t<T.
\]
In fact, recent works have pointed out that \(p_t^{\star}\) is related to the Heavy-Ball method under the MP law
\citep{pedregosa2020spectral,scieur2020polyak}. However, these results mostly focus on the asymptotic case when \(n,d \to \infty\). Here we give the following lemma under the finite \(n,d\) setting.
\begin{lemma}\label{lem:model-optimal-polynomial-comparison}
For every \(0\le t<T\),
\[
 p_t^{\star}(\Sb)\wb^\star=\wb_t^{\mathrm{HB}}.
\]
Moreover, there exists a constant $c_5 >0$ such that, uniformly for \(T/2\le t<T\),
\[
 \E\norm{\wb_t^{\mathrm{HB}}-\wb^\star}^2
 \ge c_5d\left(\frac dn\right)^t.
\]
\end{lemma}
Now we relate the algorithm learned by the transformer model to the Heavy-Ball method by comparing \(p_{t,\bar{\btheta}(\tau)}\) with \(p_t^{\star}\).
Based on Lemma~\ref{lem:stage2-pl-inequality}, we can further show that
\(\widetilde{\btheta}'(\tau)\to\widetilde{\btheta}'^{\star}\) and \(|\widetilde\theta_2^{\star}|<1\). By the definition of \(p_{t,\bar{\btheta}(\tau)}\) in
Lemma~\ref{lem:stage2-lemma1}, we see that \(p_{t,\bar{\btheta}(\tau)}\) can also be equivalently defined by \(\widetilde{\btheta}'(\tau)\). Therefore, we introduce a slight abuse of notation and denote by  \(p_{t,\widetilde{\btheta}'^{\star}}\) the training limit of \(p_{t,\bar{\btheta}(\tau)}\).
By  Lemma~\ref{lem:polynomial-output}, \(p_{t,\widetilde{\btheta}'^{\star}}\) further determines a unique \(\wb_t^{\star}\)
 by
\[
 \wb_t^{\star}:=p_{t,\widetilde{\btheta}'^{\star}}(\Sb)\wb^\star,\qquad 0\le t\le T.
\]
It follows from
Lemma~\ref{lem:stage2-lemma1} that the highest-degree coefficient of
\(p_{t+1,\widetilde{\btheta}'^{\star}}\) is
\(\widetilde\theta_1^{\star}(\widetilde\theta_2^{\star})^t\).
For fixed \(\widetilde\theta_1^{\star}\) and \(\widetilde\theta_2^{\star}\),
the coefficients of degrees \(1,\ldots,t\) can be freely adjusted
through \(\widetilde{\bTheta}_4\). Hence
\(p_{t+1,\widetilde{\btheta}'^{\star}}\) is the optimal degree-\((t+1)\)
polynomial with a certain highest-degree coefficient. Since
\(|\widetilde\theta_2^{\star}|<1\), the highest-degree coefficient decays
exponentially, suggesting that for late \(t\),
\(p_{t+1,\widetilde{\btheta}'^{\star}}\) approaches \(p_t^\star\). Motivated by this, we have the following lemma.
\begin{lemma}\label{lem:optimal-polynomial-hb-comparison}
There exists a constant $c_6 >0$ such that, uniformly for \(T/2\le t<T\),
\begin{equation*}
\begin{aligned}
 \E\norm{\wb_{t+1}^{\star}-\wb_t^{\mathrm{HB}}}^2
 &=\E\norm{[p_{t+1,\widetilde{\btheta}'^{\star}}(\Sb)-p_t^{\star}(\Sb)]\wb^\star}^2
 \le c_6 d\exp\!\left[-2(\log d)^{1/4}\right] \left(\frac dn\right)^t
 .
\end{aligned}
\end{equation*}

\end{lemma}
Note that \(\wb_t(\tau)\to \wb_t^{\star}\) as \(\cL(\tau)\to\cL^*\).
By choosing a sufficiently large \(\tau\), we can make \(\wb_t(\tau)\)
sufficiently close to \(\wb_t^{\star}\) so that, uniformly for \(T/2\le t<T\),
\[
\E\norm{\wb_{t+1}(\tau)-\wb_{t+1}^{\star}}^2
\le \frac{c_5\epsilon}{4}d\left(\frac dn\right)^T.
\]
Combining Lemmas~\ref{lem:model-optimal-polynomial-comparison} and
\ref{lem:optimal-polynomial-hb-comparison} with the triangle inequality gives,
uniformly for \(T/2\le t<T\),
\[
 \frac{\E\norm{\wb_{t+1}(\tau)-\wb_t^{\mathrm{HB}}}^2}
      {\E\norm{\wb_t^{\mathrm{HB}}-\wb^\star}^2}
 \le \epsilon/2+\frac{2c_6}{c_5} \exp\!\left[-2(\log d)^{1/4}\right]
 \le\epsilon
\]
for \(n,d=\Omega(\exp(\log^4(1/\epsilon)))\) and \(0<\epsilon<1\), the preceding estimates prove
Theorem~\ref{thm:main-result}.

\section{Experiment}
\label{sec:experiment}

In this section, we introduce our experimental setup to empirically examine our theoretical results.
In all the experiments, we test two settings with 
\((d,n)=(16,32)\) and
\((12,24)\) respectively. We set the maximum CoT steps as \(T=24\), and let the learning rate of the GD teacher in Stage 1 be \(\eta_0=0.5\). 
The entries of \(\Vb\) and \(\Wb\) are initialized independently from \(N(0,10^{-4})\). We use float64 online Adam(W) with zero weight decay, learning rate \( 10^{-6}\), and batch size \(6144\) to train the models. The in-context weight prediction data are generated exactly following the definitions in Subsection~\ref{subsec:setting1}.

Several aspects of our experiments differ from the theoretical setting. First of all, we consider Gaussian random initialization for the weight matrices $\Vb$ and $\Wb$, without freezing any blocks or entries. Moreover, in Stage 1, we remove teacher forcing and define the loss using iterates generated autoregressively by the transformer itself. In addition, we train the model using online stochastic Adam rather than gradient descent/flow. Finally, while our theory assumes \(d/n \le0.1\), this condition is not satisfied in our experimental setup. We remark that strictly following the theoretical setting yields similar experimental results. The above modifications are included mainly to empirically show that our theoretical findings can capture more general training behavior in practice.

We evaluate several methods for in-context weight prediction. Besides transformers trained with the two-stage procedure, we consider transformers trained using only Stage 1 or only Stage 2. We also examine a transformer variant without positional encodings. We compare these models with a well-tuned time-varying Heavy-Ball method and the gradient-descent teacher used in Stage 1.

Figure~\ref{fig:typical-spectral} presents experiment results in two settings, \(d=16,n=32\) and \(d=12,n=24\). In both settings, the transformer trained with the two-stage procedure closely matches the Heavy-Ball method, consistent with Theorem~\ref{thm:main-result}. In contrast, both (i) the transformer without positional encodings, trained under the same two-stage procedure, and (ii) the transformer trained only in the first stage with gradient-descent supervision perform similarly to gradient descent. This further demonstrates the necessity of the two-stage training procedure together with positional encodings for the CoT steps. The second-stage-only transformer is not stable and therefore is not included in Figure~\ref{fig:typical-spectral}.


\begingroup
\setlength{\intextsep}{4pt}
\begin{figure}[!ht]
 \centering
 \begin{minipage}{0.98\textwidth}
   \centering
   \begin{minipage}[c]{0.035\textwidth}
     \centering
     \hspace*{0.012\textwidth}\rotatebox[origin=c]{90}{\(\frac{1}{2}\|\wb_t-\wb^\star\|^2\)}
   \end{minipage}%
   \hspace{0.004\textwidth}%
   \begin{minipage}[c]{0.925\textwidth}
     \centering
     {\small\makebox[\linewidth][c]{%
       \textcolor[HTML]{D87519}{\(-\!{\scriptstyle\blacksquare}\!-\)}\,Heavy-Ball\hspace{0.24em}%
       \textcolor[HTML]{B83B36}{\(-\!\times\!-\)}\,GD\hspace{0.24em}%
       \textcolor[HTML]{168A68}{\(-\!\bullet\!-\)}\,Two-stage TF\hspace{0.24em}%
       \textcolor[HTML]{7253A5}{\(-\!\star\!-\)}\,First-stage-only TF\hspace{0.24em}%
       \textcolor[HTML]{3569A8}{\(-\!\blacktriangle\!-\)}\,Two-stage TF no PE}}\\
     \includegraphics[width=\linewidth]{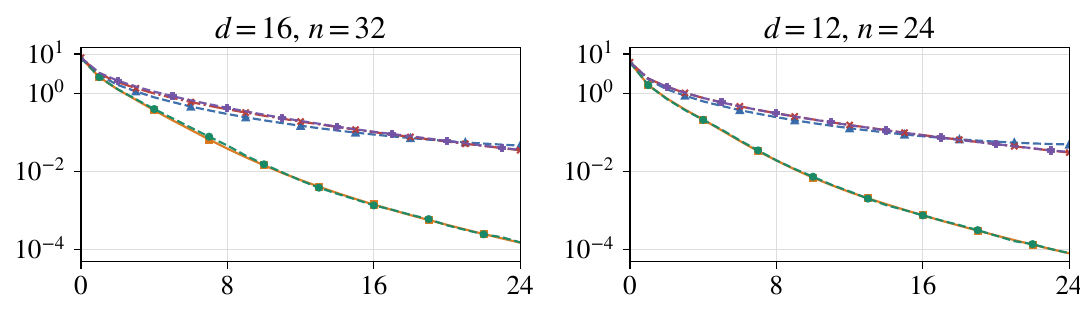}\\[-0.55em]
     \makebox[\linewidth][l]{\hspace*{0.03\linewidth}%
       \makebox[0.5\linewidth]{\(\mathrm{Iteration}\ t\)}%
       \makebox[0.5\linewidth]{\(\mathrm{Iteration}\ t\)}}
   \end{minipage}
 \end{minipage}
 \vspace{-0.4em}
 \caption{In-context weight prediction performance of several models and baseline algorithms. }
 \label{fig:typical-spectral}
\end{figure}
\endgroup

Figure~\ref{fig:operator-heatmaps and one-stage failure} further shows example heatmaps of
the weight matrices learned via two-stage training, along with evidence that training a randomly initialized transformer directly in the second stage is unstable. The learned parameter structure is consistent with Lemma~\ref{lem:invariant-parameters}:
the dominant blocks are
\(\Wb_{13},\Vb_{31},\Vb_{33},w_{24},\Wb_{55}\), while the remaining entries
are small. \(\Vb_{31},\Wb_{13},\Vb_{33}\) are approximately
diagonal, whereas \(\Wb_{55}\) has clear nonzero off-diagonal entries. By
contrast, direct second-stage training from random initialization, tested for
\(d\in\{12,16\}\) with \(n=32\), consistently leads to nonfinite
gradients, with smaller learning rates only delaying failure. 



\begin{figure}[!ht]
 \centering
 \vspace{-1.0em}
 \includegraphics[width=\textwidth]{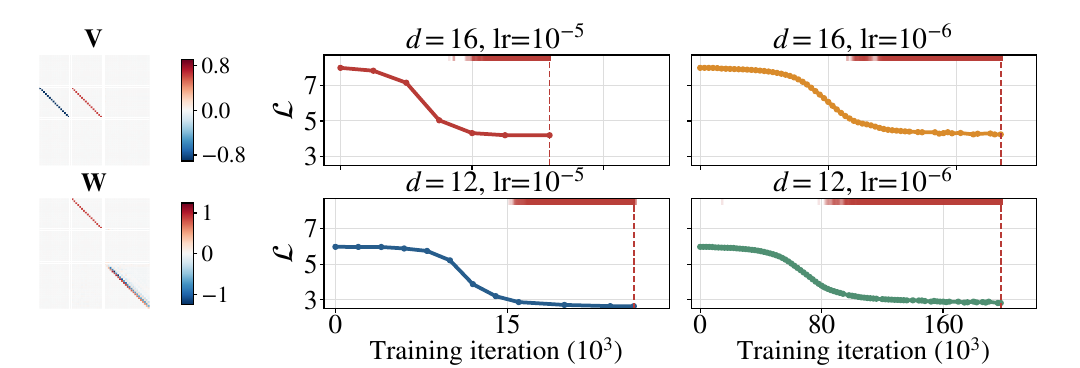}
 \vspace{-1.5em}
 \caption{Examples of the  heatmaps of learned weight matrices  after the two-stage training for \(d=16,n=32\) (left), together with the progress of second-stage-only training for \(n=32\) under four different combinations of dimension \(d\) and learning rate (right). For second-stage-only training, nonfinite gradient values are marked in red. When nonfinite gradients occur, we do not stop training immediately; instead, we allow the model to retry using a newly generated independent in-context dataset. Training is terminated only after nonfinite gradients are observed for 20 consecutive retries. Dashed lines indicate termination.}
 \label{fig:operator-heatmaps and one-stage failure}
 \vspace{-0.5em}
\end{figure}


\section{Conclusion}
In this paper, we study in-context weight prediction with transformers using CoT reasoning. We focus on whether transformers with CoT can go beyond mimicking a teacher algorithm and discover new algorithms. We show that a two-stage training paradigm enables such discovery: the first stage supervises the CoT trajectory with a gradient-descent teacher, while the second stage trains the model using only the ground-truth regression weights. The learned model discovers the Heavy-Ball method, achieving acceleration over the gradient-descent teacher used in the first stage.

Our current analysis is limited to linear self-attention for in-context linear regression. Extending the findings to more practical softmax attention and more general learning tasks is an important future work direction. Moreover, our current setup considers training with gradient flow. Extending the training methods to other training algorithms, including stochastic gradient descent, AdamW, and reinforcement learning-based training techniques, is also an interesting future work direction.
It is also interesting to extend our findings to other types of positional encodings.

\appendix

\section{Related Works}
\label{app:related}

\noindent\textbf{Expressive power and generalization of transformers as algorithm executors.}
A line of recent works studies transformers as algorithm executors, aiming to characterize what algorithms can be implicitly implemented by transformers through in-context learning. \citet{li2023transformersalgorithms} provide a theoretical framework for understanding transformers as algorithm executors and establish generalization and stability properties for in-context learning. \citet{chen2024varyingdepth} investigate how transformer depth affects sequence learning capabilities, showing that while shallow transformers can achieve memorization, more sophisticated reasoning and generalization behaviors require deeper architectures. \citet{cao2025mle} study transformer expressivity in Bayesian network sequence modeling and show that transformers can implement in-context maximum likelihood estimation and autoregressive sampling. \citet{shi2025hierarchicalvit} characterize the approximation and estimation capabilities of vision transformers for hierarchical compositional models.

Recent works further explore the connection between transformer computation and classical statistical algorithms. \citet{mousavi2026outperform} study the statistical advantages of transformers over feedforward and recurrent architectures by analyzing their sample complexity under dynamic sparse retrieval settings. \citet{li2026contexthijacking} analyze the robustness of transformers in the presence of context hijacking for linear classification tasks. \citet{li2026unlabeled} study when unlabeled data can provably improve in-context learning performance and show that deeper or looped transformers can implicitly construct iterative estimators related to classical semi-supervised learning algorithms. \citet{wu2026moirai} investigate transformer-based time-series forecasting through the case study of Moirai and provide theoretical insights into how transformers perform sequential prediction tasks.

\noindent\textbf{Training guarantees of transformers.}
A complementary line of work studies whether gradient-based training can enable transformers to recover nontrivial algorithmic structures and computational mechanisms.
\citet{jelassi2022vision} show that vision transformers can provably learn spatial structures under appropriate training settings.
\citet{nichani2024causal} investigate how gradient descent enables transformers to recover latent causal structures, while \citet{wang2024sparse} demonstrate that transformers can learn sparse token selection mechanisms that are difficult for fully connected networks to recover.
\citet{chen2024induction} characterize the emergence and training dynamics of induction heads, revealing how attention mechanisms learn retrieval-like behaviors from data.
\citet{li2024one} prove that a one-layer transformer can learn one-nearest-neighbor prediction in context, and \citet{wang2025learning} study the learning of compositional functions with transformers under easy-to-hard data curricula.

Recent works further extend these training guarantees to broader algorithmic settings.
\citet{chen2025provably} establish provable learning guarantees for multi-head attention layers.
\citet{gozeten2025test} show that test-time training can improve transformers as in-context learners.
\citet{shi2026towards} analyze how transformers learn random walk structures, while \citet{zhang2026teacher} study teacher-student settings and prove that gradient-trained transformers can learn classes of teacher models.
\citet{wu2026looped} and \citet{zhang2026normalized} further investigate looped transformers and normalized gradient descent, providing theoretical guarantees for iterative optimization procedures implemented by transformers.

\noindent\textbf{Studies on chain-of-thought reasoning.}
Recent works also study chain-of-thought (CoT) reasoning as a mechanism that enables transformers to perform multi-step computation.
\citet{wei2022chain} introduced the CoT prompting method to enhance the reasoning capability of language models by generating intermediate reasoning steps.
Before the formalization of CoT, \citet{nye2021scratchpads} demonstrated that allowing models to produce intermediate computations through scratchpads can substantially improve multi-step problem solving.
Subsequent works proposed various extensions of CoT and scratchpad prompting to improve reasoning efficiency and reliability~\citep{wang2022self,yao2023tree}.

Recent studies have further investigated the mechanisms underlying CoT from empirical, interpretability, and theoretical perspectives.
\citet{wang2023towards,saparov2022language} empirically analyzed the reasoning behaviors induced by CoT, while \citet{wu2023analyzing,hou2023towards} examined how intermediate reasoning emerges from transformer computation through mechanistic analysis.
On the theoretical side, \citet{merrill2024expressive,feng2023towards} characterized the expressive power of transformers with CoT, showing that intermediate reasoning steps can substantially extend their computational capabilities.
\citet{hu2026unveiling} further investigated the statistical foundations of CoT prompting.

Going beyond these studies, \citet{huang2025cotgd} provided the first theoretical analysis of CoT training dynamics in an in-context learning setting, showing that transformers trained with CoT supervision can learn to implement multi-step gradient descent.
\citet{yang2025symbolic} showed that transformers with CoT can learn to implement path finding algorithms in trees.
More recently, \citet{yang2026agentic} showed that transformers with CoT can be trained by policy gradient under a depth-wise curriculum to perform randomized depth-first search.
Our work builds upon this line of research by studying whether CoT trajectory-level training can enable transformers to discover accelerated optimization algorithms beyond the supervised teacher algorithm.

\section{Proof of Stage 1 convergence}
\label{app:invariant}

We use lower-case unbolded letters for scalars, lower-case bold letters for
vectors, and upper-case bold letters for matrices.  For a positive integer
\(m\), let \([m]=\{1,\ldots,m\}\).  We use \(\Ib_m\) for the
\(m\times m\) identity matrix, \(\zero_m\) and \(\one_m\) for the zero
and all-ones vectors in \(\R^m\), and \(\zero_{m\times k}\) for the
\(m\times k\) zero matrix.  For a matrix \(\Ab\), \(\Ab_{[:,-1]}\)
denotes its last column, \(\Ab_{[i:j,-1]}\) denotes rows \(i\) through \(j\)
of that column, and \(\operatorname{vec}(\Ab)\) stacks its columns into a
vector.  For an effective array, entries outside the displayed active index
ranges are defined to be zero.  The arrays retain their stated ambient
dimensions; sums and derivatives below are restricted to the displayed
active ranges.  We write \(\Vb_{ij}\) and \(\Wb_{ij}\) for the blocks of
the parameter matrices \(\Vb\) and \(\Wb\), with the block sizes specified
below.  We use \(\|\vb\|=\|\vb\|_2\) for the Euclidean norm of a vector,
and \(\|\Ab\|_2\) and \(\|\Ab\|_F\) for the spectral and Frobenius norms
of a matrix, respectively.  The notation \(\ind\{\cdot\}\) denotes the indicator.

\newpage
The output is
\(\wb_{t+1}:=\FLSA(\Zb_t^{\rm GD};\Vb,\Wb)_{[d+2:2d+1,-1]}\).
Before imposing any zero-gradient identities, direct multiplication with the
single block \(\Wb_{33}\) fixed to zero gives
\begin{align}
 \wb_{t+1}
 ={}&\wb_t^{\rm GD}
 +\Vb_{31}\sum_{i=1}^n\frac{\xb_i}{\sqrt n}
       \left(\begin{pmatrix}\xb_i/\sqrt n\\y_i/\sqrt n\\\zero_d\\0\\\zero_T\end{pmatrix}^{\!\mathsf T}
       \Wb\begin{pmatrix}\zero_d\\0\\\wb_t^{\rm GD}\\1\\\eb_t\end{pmatrix}\right)\notag\\
 &+\Vb_{32}\sum_{i=1}^n\frac{y_i}{\sqrt n}
       \left(\begin{pmatrix}\xb_i/\sqrt n\\y_i/\sqrt n\\\zero_d\\0\\\zero_T\end{pmatrix}^{\!\mathsf T}
       \Wb\begin{pmatrix}\zero_d\\0\\\wb_t^{\rm GD}\\1\\\eb_t\end{pmatrix}\right)\notag\\
 &+\sum_{\ell=0}^t
   (\Vb_{33}\wb_\ell^{\rm GD}+\Vb_{34}+\Vb_{35}\eb_\ell)
   \left(\begin{pmatrix}\zero_d\\0\\\wb_\ell^{\rm GD}\\1\\\eb_\ell\end{pmatrix}^{\!\mathsf T}
       \Wb\begin{pmatrix}\zero_d\\0\\\wb_t^{\rm GD}\\1\\\eb_t\end{pmatrix}\right).
\label{eq:stage1-direct-retained}
\end{align}
For the data-token terms, the block multiplication is
\begin{align}
&\Vb_{31}\sum_{i=1}^n\frac{\xb_i}{\sqrt n}
       \left(\begin{pmatrix}\xb_i/\sqrt n\\y_i/\sqrt n\\\zero_d\\0\\\zero_T\end{pmatrix}^{\!\mathsf T}
       \Wb\begin{pmatrix}\zero_d\\0\\\wb_t^{\rm GD}\\1\\\eb_t\end{pmatrix}\right)\notag\\
={}&\Vb_{31}\sum_{i=1}^n\frac{\xb_i}{\sqrt n}
 \left[\frac{\xb_i^{\mathsf T}}{\sqrt n}
       (\Wb_{13}\wb_t^{\rm GD}+\Wb_{14}+\Wb_{15}\eb_t)
 +\frac{y_i}{\sqrt n}
       (\Wb_{23}\wb_t^{\rm GD}+w_{24}+\wb_{25}\eb_t)\right]\notag\\
={}&\frac1n\sum_{i=1}^n\Vb_{31}\xb_i
 \left[\xb_i^{\mathsf T}
       (\Wb_{13}\wb_t^{\rm GD}+\Wb_{14}+\Wb_{15}\eb_t)
 +y_i(\Wb_{23}\wb_t^{\rm GD}+w_{24}+\wb_{25}\eb_t)\right],
\label{eq:stage1-data-v31-product}
\end{align}
\begin{align}
&\Vb_{32}\sum_{i=1}^n\frac{y_i}{\sqrt n}
       \left(\begin{pmatrix}\xb_i/\sqrt n\\y_i/\sqrt n\\\zero_d\\0\\\zero_T\end{pmatrix}^{\!\mathsf T}
       \Wb\begin{pmatrix}\zero_d\\0\\\wb_t^{\rm GD}\\1\\\eb_t\end{pmatrix}\right)\notag\\
={}&\frac1n\sum_{i=1}^n\Vb_{32}y_i
 \left[\xb_i^{\mathsf T}
       (\Wb_{13}\wb_t^{\rm GD}+\Wb_{14}+\Wb_{15}\eb_t)
 +y_i(\Wb_{23}\wb_t^{\rm GD}+w_{24}+\wb_{25}\eb_t)\right].
\label{eq:stage1-data-v32-product}
\end{align}
For the teacher-weight columns, the corresponding multiplication is
\begin{align}
&\sum_{\ell=0}^t
   (\Vb_{33}\wb_\ell^{\rm GD}+\Vb_{34}+\Vb_{35}\eb_\ell)
   \left(\begin{pmatrix}\zero_d\\0\\\wb_\ell^{\rm GD}\\1\\\eb_\ell\end{pmatrix}^{\!\mathsf T}
       \Wb\begin{pmatrix}\zero_d\\0\\\wb_t^{\rm GD}\\1\\\eb_t\end{pmatrix}\right)\notag\\
={}&\sum_{\ell=0}^t
   (\Vb_{33}\wb_\ell^{\rm GD}+\Vb_{34}+\Vb_{35}\eb_\ell)
 \Bigl[(\wb_\ell^{\rm GD})^{\mathsf T}(\Wb_{34}+\Wb_{35}\eb_t)
 +\Wb_{43}\wb_t^{\rm GD}+w_{44}+\wb_{45}\eb_t\notag\\
 &\hspace{46mm}
 +\eb_\ell^{\mathsf T}(\Wb_{53}\wb_t^{\rm GD}+\wb_{54}+\Wb_{55}\eb_t)\Bigr].
\label{eq:stage1-history-product}
\end{align}
Combining these three block products gives
\begin{align}
 \wb_{t+1}
={}&\wb_t^{\rm GD}
 +\frac1n\sum_{i=1}^n(\Vb_{31}\xb_i+\Vb_{32}y_i)
 \Bigl[\xb_i^{\mathsf T}(\Wb_{13}\wb_t^{\rm GD}+\Wb_{14}+\Wb_{15}\eb_t)\notag\\
&\hspace{27mm}+y_i(\Wb_{23}\wb_t^{\rm GD}+w_{24}+\wb_{25}\eb_t)\Bigr]\notag\\
&+\sum_{\ell=0}^t(\Vb_{33}\wb_\ell^{\rm GD}+\Vb_{34}+\Vb_{35}\eb_\ell)
 \Bigl[(\wb_\ell^{\rm GD})^{\mathsf T}(\Wb_{34}+\Wb_{35}\eb_t)\notag\\
&\hspace{27mm}+\Wb_{43}\wb_t^{\rm GD}+w_{44}+\wb_{45}\eb_t\notag\\
&\hspace{43mm}+\eb_\ell^{\mathsf T}(\Wb_{53}\wb_t^{\rm GD}+\wb_{54}+\Wb_{55}\eb_t)\Bigr].
\label{eq:full-open-update}
\end{align}

\subsection{Stage 1 invariant parameter structure}
\label{app:invariant-reduced}

\begin{lemma}\label{lem:stage1-invariant}
During Stage 1 training, the parameter matrices retain the following form,
which is the Stage-1 part of Lemma~\ref{lem:invariant-parameters}:
\begin{equation}
\Vb=
\begin{pmatrix}
\zero&\zero&\zero&\zero&\zero\\
\zero&0&\zero&0&\zero\\
v_{31}\Ib_d&\zero&v_{33}\Ib_d&\zero&\zero\\
\zero&0&\zero&0&\zero\\
\zero&\zero&\zero&\zero&\zero
\end{pmatrix},\qquad
\Wb=
\begin{pmatrix}
\zero&\zero&w_{13}\Ib_d&\zero&\zero\\
\zero&0&\zero&w_{24}&\wb_{25}\\
\zero&\zero&\zero&\zero&\zero\\
\zero&0&\zero&w_{44}&\wb_{45}\\
\zero&\zero&\zero&\wb_{54}&\Wb_{55}
\end{pmatrix}.
\label{eq:VW}
\end{equation}
Here $B$ denotes any trainable zero block displayed in~\eqref{eq:VW}.
\end{lemma}

\begin{proof}[Proof of Lemma~\ref{lem:stage1-invariant}]
\begingroup
\setlength{\parskip}{4pt}
\setlength{\abovedisplayskip}{5pt}
\setlength{\belowdisplayskip}{5pt}
\setlength{\abovedisplayshortskip}{3pt}
\setlength{\belowdisplayshortskip}{5pt}
\setlength{\jot}{2pt}
\allowdisplaybreaks[1]
\noindent\textup{(i)} Since $\Zb_t^{\rm GD}$ is independent of the parameters,
differentiating \eqref{eq:GD-guide-loss} using \eqref{eq:full-open-update}
at \eqref{eq:VW} gives
\[
 \nabla_{\Vb_{ij}}\cL_{\rm GD}=0\quad(i\ne3),
 \qquad
 \nabla_{\Wb_{ij}}\cL_{\rm GD}=0\quad(j\in\{1,2\}).
\]
For $\Vb_{32},\Vb_{34},\Vb_{35}$, we obtain
\begin{align*}
 \nabla_{\Vb_{32}}\cL_{\rm GD}
 &=\frac1T\sum_{t=0}^{T-1}\E\!\left[
 (\wb_{t+1}-\wb_{t+1}^{\rm GD})\left\{
 (\wstar)^{\mathsf T}\Sb\Wb_{13}\wb_t^{\rm GD}
 +(w_{24}+\wb_{25}\eb_t)(\wstar)^{\mathsf T}\Sb\wstar\right\}\right]
 =\zero_d,\\
 \nabla_{\Vb_{34}}\cL_{\rm GD}
 &=\frac1T\sum_{t=0}^{T-1}\E\left[
 (\wb_{t+1}-\wb_{t+1}^{\rm GD})\sum_{\ell=0}^t
 \left(w_{44}+\wb_{45}\eb_t+\eb_\ell^{\mathsf T}\wb_{54}
                 +\eb_\ell^{\mathsf T}\Wb_{55}\eb_t\right)\right]
 =\zero_d,\\
 \nabla_{\Vb_{35}}\cL_{\rm GD}
 &=\frac1T\sum_{t=0}^{T-1}\E\left[
 (\wb_{t+1}-\wb_{t+1}^{\rm GD})\sum_{\ell=0}^t
 \left(w_{44}+\wb_{45}\eb_t+\eb_\ell^{\mathsf T}\wb_{54}
                 +\eb_\ell^{\mathsf T}\Wb_{55}\eb_t\right)\eb_\ell^{\mathsf T}\right]
 =\zero_{d\times T}.
\end{align*}
Similarly,
\begin{align*}
 \nabla_{\Wb_{14}}\cL_{\rm GD}
 &=\frac1T\sum_{t=0}^{T-1}\E\left[
 \Sb\Vb_{31}^{\mathsf T}(\wb_{t+1}-\wb_{t+1}^{\rm GD})\right]
 =\zero_d,\\
 \nabla_{\Wb_{15}}\cL_{\rm GD}
 &=\frac1T\sum_{t=0}^{T-1}\E\left[
 \Sb\Vb_{31}^{\mathsf T}(\wb_{t+1}-\wb_{t+1}^{\rm GD})\eb_t^{\mathsf T}\right]
 =\zero_{d\times T},\\
 \nabla_{\Wb_{23}}\cL_{\rm GD}
 &=\frac1T\sum_{t=0}^{T-1}\E\left[
 (\wstar)^{\mathsf T}\Sb\Vb_{31}^{\mathsf T}
 (\wb_{t+1}-\wb_{t+1}^{\rm GD})(\wb_t^{\rm GD})^{\mathsf T}\right]
 =\zero_{1\times d},\\
 \nabla_{\Wb_{34}}\cL_{\rm GD}
 &=\frac1T\sum_{t=0}^{T-1}\E\left[
 \sum_{\ell=0}^t\wb_\ell^{\rm GD}(\wb_\ell^{\rm GD})^{\mathsf T}
 \Vb_{33}^{\mathsf T}(\wb_{t+1}-\wb_{t+1}^{\rm GD})\right]
 =\zero_d,\\
 \nabla_{\Wb_{35}}\cL_{\rm GD}
 &=\frac1T\sum_{t=0}^{T-1}\E\left[
 \sum_{\ell=0}^t\wb_\ell^{\rm GD}(\wb_\ell^{\rm GD})^{\mathsf T}
 \Vb_{33}^{\mathsf T}(\wb_{t+1}-\wb_{t+1}^{\rm GD})\eb_t^{\mathsf T}\right]
 =\zero_{d\times T},\\
 \nabla_{\Wb_{43}}\cL_{\rm GD}
 &=\frac1T\sum_{t=0}^{T-1}\E\left[
 \sum_{\ell=0}^t(\wb_\ell^{\rm GD})^{\mathsf T}
 \Vb_{33}^{\mathsf T}(\wb_{t+1}-\wb_{t+1}^{\rm GD})(\wb_t^{\rm GD})^{\mathsf T}\right]
 =\zero_{1\times d},\\
 \nabla_{\Wb_{53}}\cL_{\rm GD}
 &=\frac1T\sum_{t=0}^{T-1}\E\left[
 \sum_{\ell=0}^t\eb_\ell(\wb_\ell^{\rm GD})^{\mathsf T}
 \Vb_{33}^{\mathsf T}(\wb_{t+1}-\wb_{t+1}^{\rm GD})(\wb_t^{\rm GD})^{\mathsf T}\right]
 =\zero_{T\times d}.
\end{align*}

Note that
\[
\wb_t^{\rm GD}=\left(\Ib_d-(\Ib_d-\eta_0\Sb)^t\right)\wstar
\]
for every $t$, and \eqref{eq:full-open-update} shows that
$\wb_{t+1}$ is also odd in $\wstar$. Thus, every expression inside the
expectations above is odd in $\wstar$. Since $\wstar\sim N(0,\Ib_d)$,
all these gradients are zero. Consequently,
\[
\frac{\dd B}{\dd\tau}=-\nabla_B\cL_{\rm GD}=0,
\qquad B(\tau)=B(0)=0
\]
for every zero block $B$ in \eqref{eq:VW}.

\noindent\textup{(ii)} For every $\Ub\in\R^{d\times d}$ with
$\Ub^\top\Ub=\Ub\Ub^\top=\Ib_d$,
\[
 (\xb_i,\wstar,\wb_t^{\rm GD})
 \mapsto(\Ub\xb_i,\Ub\wstar,\Ub\wb_t^{\rm GD}),
\]
and
\begin{align*}
 &(\Vb_{31},\Wb_{13},\Vb_{33})
 \mapsto(\Ub\Vb_{31}\Ub^\top,\Ub\Wb_{13}\Ub^\top,\Ub\Vb_{33}\Ub^\top),\\
 &\cL_{\rm GD}(\Vb_{31},\Wb_{13},\Vb_{33})
 \mapsto\cL_{\rm GD}(\Ub\Vb_{31}\Ub^\top,\Ub\Wb_{13}\Ub^\top,\Ub\Vb_{33}\Ub^\top),\\
 &\left.\nabla_{\Ab}\cL_{\rm GD}\right|_{(\Vb_{31},\Wb_{13},\Vb_{33})}
 \mapsto\Ub\left.\nabla_{\Ab}\cL_{\rm GD}\right|_{(\Vb_{31},\Wb_{13},\Vb_{33})}\Ub^\top,
 \quad\Ab\in\{\Vb_{31},\Wb_{13},\Vb_{33}\}.
\end{align*}
The loss is unchanged by this transformation.
Since $\Ub\Ab(0)\Ub^\top=\Ab(0)$ for all three blocks,
uniqueness of the gradient-flow ODE gives
\[
 \Ub\Ab(\tau)\Ub^\top=\Ab(\tau)
 \qquad(\Ub^\top\Ub=\Ib_d,\ \tau\ge0).
\]
Fix $\tau$. Taking $\Ub$ to flip the sign of coordinate $i$ gives, for $j\ne i$,
\[
 (\Ab)_{ij}=(\Ub\Ab\Ub^\top)_{ij}=-(\Ab)_{ij},
 \qquad (\Ab)_{ij}=0.
\]
Thus $\Ab$ is diagonal. Taking $\Ub$ to exchange coordinates $i$ and $j$ gives
\[
 (\Ab)_{ii}=(\Ub\Ab\Ub^\top)_{ii}=(\Ab)_{jj}.
\]
Hence
\[
 \Vb_{31}(\tau)=v_{31}(\tau)\Ib_d,\qquad
 \Wb_{13}(\tau)=w_{13}(\tau)\Ib_d,\qquad
 \Vb_{33}(\tau)=v_{33}(\tau)\Ib_d.
\]
\endgroup
\end{proof}

\subsection{Proofs of Lemmas 4.3--4.5}
\label{app:GD-guide}

Throughout the appendices, we use the same scaling assumptions:
\begin{equation}
 n=\Theta(d),\qquad \frac dn\le\frac1{10},\qquad
 T=\Theta\!\left((\log d)^{1/4}\right),                         \label{eq:common-growing-regime}
\end{equation}
and, for all sufficiently large \(d\),
\begin{equation}
 \left(\frac dn\right)^2T
 \ge 1024(\log d)^{1/4}.                    \label{eq:common-tail-margin}
\end{equation}
\subsection{Proof of Lemma~\ref{lem:stage1-loss}}

For the spectral form of the loss $\mathcal L_{\rm GD}$, let \(\nu_{d,n}\) be the expected
spectral law of \(\Sb\), and let \(\mu_{d/n}\) be the Marchenko-Pastur law.
We use the polynomial inner products and norms
\begin{align*}
 \E_{\lambda\sim\nu_{d,n}}f(\lambda)
 &=\frac1d\E\,\tr f(\Sb),\\
 \langle f,g\rangle_{d,n}
 &=\frac1d\E\,\tr\!\left(f(\Sb)g(\Sb)\right),
 &\|R\|_{d,n}^2&=\langle R,R\rangle_{d,n},\\
 \langle f,g\rangle_{\mathrm{MP}}
 &=\E_{\lambda\sim\mu_{d/n}}[f(\lambda)g(\lambda)],
 &\|R\|_{\mathrm{MP}}^2&=\langle R,R\rangle_{\mathrm{MP}}.
\end{align*}

\begin{proof}[Proof of Lemma~\ref{lem:stage1-loss}]
By the definition of the Stage-1 loss~\eqref{eq:GD-guide-loss},
\begin{equation}
 \mathcal L_{\rm GD}
 =\frac1{2T}\E\sum_{t=0}^{T-1}
 \left\|
 \FLSA(\Zb_t^{\rm GD};\Vb,\Wb)_{[d+2:2d+1,-1]}
 -\wb_{t+1}^{\rm GD}
 \right\|^2 .
 \label{eq:guide-loss-from-definition}
\end{equation}
Under the parameter structure in Lemma~\ref{lem:stage1-invariant},
the update~\eqref{eq:update} and the identity
\(\Xb\yb^\top=\Sb\wstar\) give
\begin{equation*}
\begin{aligned}
 \FLSA(\Zb_0^{\rm GD};\Vb,\Wb)_{[d+2:2d+1,-1]}
 &=\widetilde\theta_1\Sb\wstar,\\
 \FLSA(\Zb_t^{\rm GD};\Vb,\Wb)_{[d+2:2d+1,-1]}
 &=\wb_t^{\rm GD}+\widetilde\theta_2\Sb\wb_t^{\rm GD}
   +(\widetilde{\bTheta}_3)_{1t}\Sb\wstar\\
 &\quad+\sum_{\ell=2}^{t}(\widetilde{\bTheta}_3)_{\ell t}\wb_\ell^{\rm GD},
 \qquad 1\le t<T.
\end{aligned}
\end{equation*}
The GD teacher follows
\begin{equation*}
 \wb_{t+1}^{\rm GD}=\wb_t^{\rm GD}
 +\eta_0\Sb(\wstar-\wb_t^{\rm GD}).
\end{equation*}
Therefore, the residual in~\eqref{eq:guide-loss-from-definition} is
\begin{equation*}
\begin{aligned}
 \FLSA(\Zb_0^{\rm GD};\Vb,\Wb)_{[d+2:2d+1,-1]}-\wb_1^{\rm GD}
 &= (\widetilde\theta_1-\eta_0)\Sb\wstar,\\
 \FLSA(\Zb_t^{\rm GD};\Vb,\Wb)_{[d+2:2d+1,-1]}-\wb_{t+1}^{\rm GD}
 &= (\widetilde\theta_2+\eta_0)\Sb\wb_t^{\rm GD}
   +\bigl((\widetilde{\bTheta}_3)_{1t}-\eta_0\bigr)\Sb\wstar\\
 &\quad+\sum_{\ell=2}^{t}(\widetilde{\bTheta}_3)_{\ell t}\wb_\ell^{\rm GD},
 \qquad 1\le t<T.
\end{aligned}
\end{equation*}
For \(p_t^{\rm GD}(\lambda):=1-(1-\eta_0\lambda)^t\),
\(\wb_t^{\rm GD}=p_t^{\rm GD}(\Sb)\wstar\). Thus, after conditioning on $\Sb$ and
using \(\E[\wstar(\wstar)^{\mathsf T}\mid\Sb]=\Ib_d\),
\begin{equation}
\begin{aligned}
 A_0(\lambda)&=(\widetilde\theta_1-\eta_0)\lambda,\\
 A_t(\lambda)&=(\widetilde\theta_2+\eta_0)\lambda p_t^{\rm GD}(\lambda)
 +\bigl((\widetilde{\bTheta}_3)_{1t}-\eta_0\bigr)\lambda
 +\sum_{\ell=2}^{t}(\widetilde{\bTheta}_3)_{\ell t}p_\ell^{\rm GD}(\lambda),
 \qquad 1\le t<T,\\
 \mathcal L_{\rm GD}
 &=\frac1{2T}\sum_{t=0}^{T-1}\E\,\tr\bigl(A_t(\Sb)^2\bigr).
\end{aligned}
\label{eq:guide-auxiliary-loss}
\end{equation}
Let $P_t$ be the orthogonal projection in $\langle\cdot,\cdot\rangle_{d,n}$ onto
\(\operatorname{span}\{\lambda,p_2^{\rm GD},\ldots,p_t^{\rm GD}\}\). Thus, $P_t$
separates the degree-$(t+1)$ component $\lambda p_t^{\rm GD}$ from the lower-degree
span $\operatorname{span}\{\lambda,p_2^{\rm GD},\ldots,p_t^{\rm GD}\}$. By orthogonality,
\begin{equation}
 A_t=(\widetilde\theta_2+\eta_0)\bigl(\lambda p_t^{\rm GD}-P_t[\lambda p_t^{\rm GD}]\bigr)+P_tA_t.
\label{eq:guide-orthogonal-decomposition}
\end{equation}
The two summands in~\eqref{eq:guide-orthogonal-decomposition} are orthogonal
in $\langle\cdot,\cdot\rangle_{d,n}$.
For $1\le t<T$, direct projection of the expression for $A_t$ in
\eqref{eq:guide-auxiliary-loss} gives
\begin{equation*}
\begin{aligned}
(P_tA_t)(\Sb)
={}&(\widetilde\theta_2+\eta_0)(P_t[\lambda p_t^{\rm GD}])(\Sb)
 +\bigl((\widetilde{\bTheta}_3)_{1t}-\eta_0\bigr)\Sb\\
&+\sum_{\ell=2}^{t}(\widetilde{\bTheta}_3)_{\ell t}p_\ell^{\rm GD}(\Sb).
\end{aligned}
\end{equation*}
Substituting this expansion and the orthogonality in
\eqref{eq:guide-orthogonal-decomposition} into
\eqref{eq:guide-auxiliary-loss} gives
\begin{equation*}
\begin{aligned}
\mathcal L_{\rm GD}
={}&\E\Bigg[\frac{\|\Sb\|_F^2}{2T}(\widetilde\theta_1-\eta_0)^2
 +\frac1{2T}\sum_{t=1}^{T-1}
 \bigl\|(\widetilde\theta_2+\eta_0)
 (\lambda p_t^{\rm GD}-P_t[\lambda p_t^{\rm GD}])(\Sb)\bigr\|_F^2\\
&\quad+\frac1{2T}\sum_{t=1}^{T-1}
 \Bigl\|(\widetilde\theta_2+\eta_0)(P_t[\lambda p_t^{\rm GD}])(\Sb)
 +\bigl((\widetilde{\bTheta}_3)_{1t}-\eta_0\bigr)\Sb\\
&\hspace{42mm}
 +\sum_{\ell=2}^{t}(\widetilde{\bTheta}_3)_{\ell t}p_\ell^{\rm GD}(\Sb)\Bigr\|_F^2\Bigg].
\end{aligned}
\end{equation*}
Here $t=1,\ldots,T-1$ indexes blocks of $d^2$ rows. To write the
expanded expression compactly, define
\begin{equation*}
\begin{aligned}
 &A(\Sb):=\frac{\|\Sb\|_F^2}{2T},\qquad
 B(\Sb):=\frac1{2T}\sum_{t=1}^{T-1}
 \bigl\|(\lambda p_t^{\rm GD}-P_t[\lambda p_t^{\rm GD}])(\Sb)\bigr\|_F^2,\\
 &[\Cb_1(\Sb)]_t:=\frac1{\sqrt{2T}}
 \operatorname{vec}\bigl((P_t[\lambda p_t^{\rm GD}])(\Sb)\bigr),\\
 &[\Cb_2(\Sb)]_{t,(\ell,s)}
 :=\frac{\ind\{t=s\}}{\sqrt{2T}}
 \begin{cases}
 \operatorname{vec}(\Sb),&\ell=1,\ 1\le s<T,\\
 \operatorname{vec}(p_\ell^{\rm GD}(\Sb)),&2\le\ell\le s<T,\\
 \zero_{d^2},&\text{otherwise},
 \end{cases}
 \qquad 0\le\ell,s<T.
\end{aligned}
\end{equation*}
Here $\Cb_2(\Sb)$ is the vertical stacking of the block rows
$[\Cb_2(\Sb)]_t$, $1\le t<T$. Equivalently, for
$\Delta\in\R^{T\times T}$,
\begin{equation*}
 \bigl[\Cb_2(\Sb)\operatorname{vec}(\Delta)\bigr]_t
 =\frac1{\sqrt{2T}}\left(
 \Delta_{1t}\operatorname{vec}(\Sb)
 +\sum_{\ell=2}^t\Delta_{\ell t}
   \operatorname{vec}\bigl(p_\ell^{\rm GD}(\Sb)\bigr)\right),
 \qquad 1\le t<T.
\end{equation*}
Therefore, the Stage-1 loss can be written as
\begin{equation}
\begin{aligned}
\mathcal L_{\rm GD}
={}&\underbrace{\E\bigl[A(\Sb)(\widetilde\theta_1-\eta_0)^2\bigr]}_{\mathcal L_{\rm GD,A}}\\
&+\underbrace{\E\bigl[B(\Sb)(\widetilde\theta_2+\eta_0)^2\bigr]}_{\mathcal L_{\rm GD,B}}\\
&+\underbrace{\E\Bigl\|\Cb_1(\Sb)(\widetilde\theta_2+\eta_0)
 +\Cb_2(\Sb)\operatorname{vec}
 \bigl(\widetilde{\bTheta}_3-\eta_0\eb_1(\one_T-\eb_0)^{\mathsf T}\bigr)\Bigr\|^2}_{\mathcal L_{\rm GD,C}}.
\end{aligned}
\label{eq:guide-loss-decomposition}
\end{equation}

\end{proof}

By Lemma~\ref{lem:stage1-invariant}, the gradient with respect to the trainable
blocks of $(\Vb,\Wb)$ reduces to $\nabla_{\bar{\btheta}}\mathcal L_{\rm GD}$.
Hence,
\begin{align}
 -\frac{\mathrm d}{\mathrm d\tau}\mathcal L_{\rm GD}
 &=\left\|\nabla_{\bar{\btheta}}\mathcal L_{\rm GD}\right\|^2\notag\\
 &=\left\|\nabla_{\Vb_{31}}\mathcal L_{\rm GD}\right\|_F^2
 +\left\|\nabla_{\Wb_{13}}\mathcal L_{\rm GD}\right\|_F^2
 +\left\|\nabla_{\Vb_{33}}\mathcal L_{\rm GD}\right\|_F^2\notag\\
 &\quad+\left|\partial_{w_{24}}\mathcal L_{\rm GD}\right|^2
 +\left\|\nabla_{\wb_{25}}\mathcal L_{\rm GD}\right\|^2
 +\left|\partial_{w_{44}}\mathcal L_{\rm GD}\right|^2\notag\\
 &\quad+\left\|\nabla_{\wb_{45}}\mathcal L_{\rm GD}\right\|^2
 +\left\|\nabla_{\wb_{54}}\mathcal L_{\rm GD}\right\|^2
 +\left\|\nabla_{\Wb_{55}}\mathcal L_{\rm GD}\right\|_F^2\notag\\
 &=\frac1d\left(
 \left|\partial_{v_{31}}\mathcal L_{\rm GD}\right|^2
 +\left|\partial_{w_{13}}\mathcal L_{\rm GD}\right|^2
 +\left|\partial_{v_{33}}\mathcal L_{\rm GD}\right|^2\right)\notag\\
 &\quad+\left|\partial_{w_{24}}\mathcal L_{\rm GD}\right|^2
 +\left\|\nabla_{\wb_{25}}\mathcal L_{\rm GD}\right\|^2
 +\left|\partial_{w_{44}}\mathcal L_{\rm GD}\right|^2\notag\\
 &\quad+\left\|\nabla_{\wb_{45}}\mathcal L_{\rm GD}\right\|^2
 +\left\|\nabla_{\wb_{54}}\mathcal L_{\rm GD}\right\|^2
 +\left\|\nabla_{\Wb_{55}}\mathcal L_{\rm GD}\right\|_F^2\notag\\
 &\quad\left(\partial_{v_{31}}\mathcal L_{\rm GD}
 =\left\langle\nabla_{\Vb_{31}}\mathcal L_{\rm GD},\Ib_d\right\rangle_F,
 \quad\nabla_{\Vb_{31}}\mathcal L_{\rm GD}
 =\frac{\partial_{v_{31}}\mathcal L_{\rm GD}}d\Ib_d\right).
                                                               \label{eq:guide-energy}
\end{align}

\begin{lemma}
\label{lem:guide-invariants}
Along the Stage-1 flow, for every training time $\tau$,
\begin{align}
 v_{31}(\tau)^2-w_{13}(\tau)^2
 -\frac{w_{24}(\tau)^2+\|\wb_{25}(\tau)\|^2}{d}
 &=-\frac{\sigma^2}{d},                                  \label{eq:guide-spatial-invariant}\\
 d v_{33}(\tau)^2-|w_{44}(\tau)|^2-\|\wb_{45}(\tau)\|^2
 -\|\wb_{54}(\tau)\|^2-\|\Wb_{55}(\tau)\|_F^2
 &=d v_{33}(0)^2,                                        \label{eq:guide-history-invariant}\\
 \sum_{t=0}^{T-1}(\wb_{25}(\tau))_t-w_{24}(\tau)
 &=-\sigma.                                             \label{eq:guide-positional-kernel}
\end{align}
Consequently, \(v_{33}(\tau)\ge v_{33}(0)>0\).
\end{lemma}

\begin{proof}
\noindent\textup{(i)} Define
\[
 I_1(\tau):=v_{31}(\tau)^2-w_{13}(\tau)^2
 -\frac{w_{24}(\tau)^2+\|\wb_{25}(\tau)\|^2}{d}.
\]
For fixed $\tau$, consider the rescaling
\[
 (v_{31}(\tau),w_{13}(\tau),w_{24}(\tau),\wb_{25}(\tau))\mapsto
 (e^\zeta v_{31}(\tau),e^{-\zeta}w_{13}(\tau),
 e^{-\zeta}w_{24}(\tau),e^{-\zeta}\wb_{25}(\tau)).
\]
This rescaling preserves the products
\begin{align*}
 (e^\zeta v_{31}(\tau))(e^{-\zeta}w_{13}(\tau))
 &=v_{31}(\tau)w_{13}(\tau),\\
 (e^\zeta v_{31}(\tau))
 \bigl(e^{-\zeta}w_{24}(\tau)+e^{-\zeta}\wb_{25}(\tau)\eb_t\bigr)
 &=v_{31}(\tau)\bigl(w_{24}(\tau)+\wb_{25}(\tau)\eb_t\bigr),
 \quad 0\le t<T.
\end{align*}
Hence $\widetilde\theta_1$, $\widetilde\theta_2$, and
$\widetilde{\bTheta}_3$ are unchanged by the rescaling, so
\eqref{eq:guide-loss-decomposition} shows that the loss is unchanged as well.
Differentiating the loss along this rescaling at $\zeta=0$ gives
\begin{align*}
 0&=\left.\frac{\mathrm d}{\mathrm d\zeta}\mathcal L_{\rm GD}
 \right|_{\zeta=0}\\
 &=v_{31}(\tau)\partial_{v_{31}}\mathcal L_{\rm GD}(\tau)
 -w_{13}(\tau)\partial_{w_{13}}\mathcal L_{\rm GD}(\tau)\\
 &\quad-w_{24}(\tau)\partial_{w_{24}}\mathcal L_{\rm GD}(\tau)
 -\bigl\langle\wb_{25}(\tau),
 \nabla_{\wb_{25}}\mathcal L_{\rm GD}(\tau)\bigr\rangle.
\end{align*}
By~\eqref{eq:GD-gradient-flow} and the scalar gradient relations in
\eqref{eq:guide-energy},
\begin{align*}
 \dot v_{31}(\tau)&=-\frac1d\partial_{v_{31}}\mathcal L_{\rm GD}(\tau),
 &\dot w_{13}(\tau)&=-\frac1d\partial_{w_{13}}\mathcal L_{\rm GD}(\tau),\\
 \dot w_{24}(\tau)&=-\partial_{w_{24}}\mathcal L_{\rm GD}(\tau),
 &\dot{\wb}_{25}(\tau)&=-\nabla_{\wb_{25}}\mathcal L_{\rm GD}(\tau).
\end{align*}
Consequently,
\begin{align*}
 \frac{\mathrm d}{\mathrm d\tau}I_1(\tau)
 &=2v_{31}(\tau)\dot v_{31}(\tau)
 -2w_{13}(\tau)\dot w_{13}(\tau)\\
 &\quad-\frac2d\Bigl(w_{24}(\tau)\dot w_{24}(\tau)
 +\bigl\langle\wb_{25}(\tau),\dot{\wb}_{25}(\tau)\bigr\rangle\Bigr)\\
 &=-\frac2d\Bigl(v_{31}(\tau)\partial_{v_{31}}\mathcal L_{\rm GD}(\tau)
 -w_{13}(\tau)\partial_{w_{13}}\mathcal L_{\rm GD}(\tau)\\
 &\qquad\quad-w_{24}(\tau)\partial_{w_{24}}\mathcal L_{\rm GD}(\tau)
 -\bigl\langle\wb_{25}(\tau),
 \nabla_{\wb_{25}}\mathcal L_{\rm GD}(\tau)\bigr\rangle\Bigr)\\
 &=0.
\end{align*}
Thus $I_1(\tau)=I_1(0)=-\sigma^2/d$, which proves
\eqref{eq:guide-spatial-invariant}.

\noindent\textup{(ii)} The same argument applies to the rescaling
\begingroup\small
\thinmuskip=1mu \thickmuskip=3mu
\[
 (v_{33}(\tau),w_{44}(\tau),\wb_{45}(\tau),\wb_{54}(\tau),\Wb_{55}(\tau))
 \mapsto(e^\zeta v_{33}(\tau),e^{-\zeta}w_{44}(\tau),
 e^{-\zeta}\wb_{45}(\tau),e^{-\zeta}\wb_{54}(\tau),e^{-\zeta}\Wb_{55}(\tau)).
\]
\endgroup
Each product with $v_{33}$ in $\widetilde{\bTheta}_3$ is unchanged, so the
same argument gives~\eqref{eq:guide-history-invariant}.
It follows that $v_{33}(\tau)^2\ge v_{33}(0)^2$; continuity and
$v_{33}(0)>0$ then give $v_{33}(\tau)\ge v_{33}(0)>0$.

\noindent\textup{(iii)} Let
\[
 I_3(\tau):=\sum_{t=0}^{T-1}(\wb_{25}(\tau))_t-w_{24}(\tau).
\]
For fixed $\tau$, consider the translation
\[
 (w_{24}(\tau),\wb_{25}(\tau))
 \mapsto(w_{24}(\tau)+\zeta,\wb_{25}(\tau)-\zeta\one_T^{\mathsf T}).
\]
For every $0\le t<T$, it preserves
\[
 w_{24}(\tau)+\zeta+
 (\wb_{25}(\tau)-\zeta\one_T^{\mathsf T})\eb_t
 =w_{24}(\tau)+\wb_{25}(\tau)\eb_t.
\]
Thus $\widetilde\theta_1$, $\widetilde\theta_2$, and
$\widetilde{\bTheta}_3$, and hence the loss, are unchanged.
Differentiating along this translation at $\zeta=0$ gives
\[
 0=\left.\frac{\mathrm d}{\mathrm d\zeta}\mathcal L_{\rm GD}\right|_{\zeta=0}
 =\partial_{w_{24}}\mathcal L_{\rm GD}(\tau)
 -\sum_{t=0}^{T-1}(\nabla_{\wb_{25}}\mathcal L_{\rm GD}(\tau))_t.
\]
Substituting the gradient flow~\eqref{eq:GD-gradient-flow} into the
derivative of $I_3(\tau)$ now gives
\begin{align*}
 \frac{\mathrm d}{\mathrm d\tau}I_3(\tau)
 &=\sum_{t=0}^{T-1}(\dot{\wb}_{25}(\tau))_t-\dot w_{24}(\tau)\\
 &=-\sum_{t=0}^{T-1}(\nabla_{\wb_{25}}\mathcal L_{\rm GD}(\tau))_t
 +\partial_{w_{24}}\mathcal L_{\rm GD}(\tau)\\
 &=0.
\end{align*}
Thus $I_3(\tau)=I_3(0)=-\sigma$, which proves
\eqref{eq:guide-positional-kernel}.
\end{proof}

\begin{lemma}
\label{lem:guide-auxiliary-nondegenerate}
For any $\widetilde\theta_2,\widetilde{\bTheta}_3$ as in
Lemma~\ref{lem:stage1-loss}, for all sufficiently large $d$,
\begin{equation}
\begin{aligned}
 &\E\Bigl[B(\Sb)(\widetilde\theta_2+\eta_0)^2
 +\Bigl\|\Cb_1(\Sb)(\widetilde\theta_2+\eta_0)\\
 &\qquad\quad+\Cb_2(\Sb)\operatorname{vec}
 \bigl(\widetilde{\bTheta}_3-\eta_0\eb_1(\one_T-\eb_0)^{\mathsf T}\bigr)
 \Bigr\|^2\Bigr]\\
 &\quad\ge c\frac{d/n}{T}\Bigl(
 d(\widetilde\theta_2+\eta_0)^2
 +\bigl\|\widetilde{\bTheta}_3-\eta_0\eb_1(\one_T-\eb_0)^{\mathsf T}\bigr\|_F^2
 \Bigr).
\end{aligned}
\label{eq:guide-auxiliary-nondegenerate}
\end{equation}
\end{lemma}

\begin{proof}
By the definitions of $B,\Cb_1,\Cb_2$ and the expansion in
\eqref{eq:guide-auxiliary-loss},
\[
 \mathcal L_{\rm GD,B}=(\widetilde\theta_2+\eta_0)^2\E B(\Sb),
 \qquad
 \mathcal L_{\rm GD,C}
 =\frac1{2T}\sum_{t=1}^{T-1}\E\tr\bigl((P_tA_t)(\Sb)^2\bigr).
\]
Since $P_1$ projects onto $\operatorname{span}\{\lambda\}$,
\begin{align}
 \E B(\Sb)
 &=\frac1{2T}\sum_{t=1}^{T-1}\E\bigl\|
 (\lambda p_t^{\rm GD}-P_t[\lambda p_t^{\rm GD}])(\Sb)\bigr\|_F^2\notag\\
 &\ge\frac1{2T}\E\bigl\|
 (\lambda p_1^{\rm GD}-P_1[\lambda p_1^{\rm GD}])(\Sb)\bigr\|_F^2\notag\\
 &=\frac{\eta_0^2}{2T}\min_{b\in\R}
 \E\|\Sb^2-b\Sb\|_F^2
 \qquad(p_1^{\rm GD}(\lambda)=\eta_0\lambda)\notag\\
 &=\frac{\eta_0^2}{2T}\min_{b\in\R}\Bigl\{
 \E\tr(\Sb^4)-2b\E\tr(\Sb^3)+b^2\E\tr(\Sb^2)\Bigr\}\notag\\
 &=\frac{\eta_0^2}{2T}\left(
 \E\tr(\Sb^4)-\frac{(\E\tr(\Sb^3))^2}{\E\tr(\Sb^2)}\right)
 \qquad\left(b=\frac{\E\tr(\Sb^3)}{\E\tr(\Sb^2)}\right)\notag\\
 &\ge c\frac{d(d/n)}{T}.
 \label{eq:guide-auxiliary-gap}
\end{align}
where the last line follows from Lemma~\ref{lem:finite-wishart-quadratic-gap}
and $\eta_0\ge0.4$.

For the remaining coefficients, the projection of
\eqref{eq:guide-auxiliary-loss} gives
\[
 \bigl((\widetilde{\bTheta}_3)_{1t}-\eta_0\bigr)\Sb
 +\sum_{\ell=2}^t(\widetilde{\bTheta}_3)_{\ell t}p_\ell^{\rm GD}(\Sb)
 =(P_tA_t)(\Sb)
 -(\widetilde\theta_2+\eta_0)(P_t[\lambda p_t^{\rm GD}])(\Sb).
\]
We first consider general coefficients $u_1,\ldots,u_t\in\R$ for
$1\le t<T$, and write
\[
 f(\lambda)=u_1\lambda+\sum_{\ell=2}^t u_\ell p_\ell^{\rm GD}(\lambda)
 =\sum_{k=1}^t c_k\lambda^k.
\]
With $z=1-\eta_0\lambda$, the two representations give
\begin{align*}
 f\left(\frac{1-z}{\eta_0}\right)
 &=\frac{u_1}{\eta_0}(1-z)+\sum_{\ell=2}^t u_\ell(1-z^\ell)\\
 &=\sum_{k=1}^t c_k\eta_0^{-k}
   \sum_{j=0}^k(-1)^j\binom{k}{j}z^j.
\end{align*}
Comparing the coefficients of $z,z^2,\ldots,z^t$ yields
\[
 u_1=\sum_{k=1}^t k\eta_0^{1-k}c_k,
 \qquad
 u_\ell=(-1)^{\ell+1}\sum_{k=\ell}^t
 \binom{k}{\ell}\eta_0^{-k}c_k,\quad 2\le\ell\le t.
\]
Since $\eta_0^{-1}\le2.5$ and $k,\binom{k}{\ell}\le2^t$,
Cauchy--Schwarz gives
\[
 \sum_{\ell=1}^t u_\ell^2
 \le t\sum_{k=1}^t\left[k^2\eta_0^{2-2k}
 +\sum_{\ell=2}^k\binom{k}{\ell}^2\eta_0^{-2k}\right]c_k^2
 \le t^2 25^t\sum_{k=1}^t c_k^2.
\]
The finite-Wishart moment bound in Lemma~\ref{lem:finite-wishart-moment-lower}
therefore implies
\begin{align}
 \frac1d\E\tr\bigl(f(\Sb)^2\bigr)
 &\ge\frac{(d/n)^t}{2(t+1)16^t}\sum_{k=1}^t c_k^2\notag\\
 &\ge\frac{(d/n)^t}{2t^2(t+1)400^t}\sum_{\ell=1}^t u_\ell^2\notag\\
 &\ge\exp\{-CT\log(e+T)\}(d/n)^{CT}\sum_{\ell=1}^t u_\ell^2.
\end{align}
We then apply this estimate column by column with
\[
 (u_1,\ldots,u_t)^{\mathsf T}
 =\bigl[\widetilde{\bTheta}_3-\eta_0\eb_1(\one_T-\eb_0)^{\mathsf T}\bigr]_{1:t,t},
\]
and sum over $1\le t<T$. Splitting the squared norm gives
\begin{align*}
 &\bigl\|\widetilde{\bTheta}_3-\eta_0\eb_1(\one_T-\eb_0)^{\mathsf T}\bigr\|_F^2\\
 &\quad\le\exp\{CT\log(e+T)\}(n/d)^{CT}\frac1d
 \sum_{t=1}^{T-1}\E\tr\!\Bigl[\bigl(
 (P_tA_t)(\Sb)-(\widetilde\theta_2+\eta_0)(P_t[\lambda p_t^{\rm GD}])(\Sb)
 \bigr)^2\Bigr]\\
 &\quad\le\exp\{CT\log(e+T)\}(n/d)^{CT}
 \frac2d\sum_{t=1}^{T-1}\left[
 \E\bigl\|(P_tA_t)(\Sb)\bigr\|_F^2
 +(\widetilde\theta_2+\eta_0)^2
 \E\bigl\|(P_t[\lambda p_t^{\rm GD}])(\Sb)\bigr\|_F^2\right]\\
 &\quad=\exp\{CT\log(e+T)\}(n/d)^{CT}
 \left[\frac2d\sum_{t=1}^{T-1}\E\bigl\|(P_tA_t)(\Sb)\bigr\|_F^2
 +2(\widetilde\theta_2+\eta_0)^2
 \sum_{t=1}^{T-1}\|P_t[\lambda p_t^{\rm GD}]\|_{d,n}^2\right]\\
 &\quad=\exp\{CT\log(e+T)\}(n/d)^{CT}
 \left[\frac{4T}{d}\mathcal L_{\rm GD,C}
 +2(\widetilde\theta_2+\eta_0)^2
 \sum_{t=1}^{T-1}\|P_t[\lambda p_t^{\rm GD}]\|_{d,n}^2\right]\\
 &\quad\le\exp\{CT\log(e+T)\}(n/d)^{CT}
 \left[\frac{4T}{d}\mathcal L_{\rm GD,C}
 +CT(\widetilde\theta_2+\eta_0)^2\right].
\end{align*}
Here the loss decomposition~\eqref{eq:guide-loss-decomposition} gives
\[
 \sum_{t=1}^{T-1}\E\bigl\|(P_tA_t)(\Sb)\bigr\|_F^2
 =\sum_{t=1}^{T-1}\E\tr\bigl((P_tA_t)(\Sb)^2\bigr)
 =2T\mathcal L_{\rm GD,C},
\]
and the last inequality uses $\|P_t[\lambda p_t^{\rm GD}]\|_{d,n}\le
\|\lambda p_t^{\rm GD}\|_{d,n}$ and Lemma~\ref{lem:finite-teacher-trace-estimates}.
Substituting the lower bound for $\E B(\Sb)$ yields
\begin{align*}
 &\frac{d/n}{T}
 \bigl\|\widetilde{\bTheta}_3-\eta_0\eb_1(\one_T-\eb_0)^{\mathsf T}\bigr\|_F^2\\
 &\quad\le\frac{CT}{d}\exp\{CT\log(e+T)\}(n/d)^{CT}
 \bigl(\mathcal L_{\rm GD,B}+\mathcal L_{\rm GD,C}\bigr)\\
 &\quad\le\mathcal L_{\rm GD,B}+\mathcal L_{\rm GD,C}
 \qquad\bigl(n=\Theta(d),\ T=\Theta((\log d)^{1/4})\bigr).
\end{align*}
Combining this estimate with
$d(d/n)(\widetilde\theta_2+\eta_0)^2/T\le C\mathcal L_{\rm GD,B}$
proves~\eqref{eq:guide-auxiliary-nondegenerate}.
\end{proof}

\begin{lemma}
\label{lem:guide-early-energy}
Under~\eqref{eq:warm-init}, for every deterministic
$H=d^{o(1)}$, uniformly for $0\le\tau\le H$,
\begin{align}
 \sup_{0\le r\le\tau}
 \bigl(|v_{31}(r)|+|w_{13}(r)|+|v_{33}(r)-v_{33}(0)|\bigr)
 &\le C\sqrt{\frac{\tau}{T}},\notag\\
 \sup_{0\le r\le\tau}\|\Wb_{55}(r)\|_F
 &\le C\sqrt{\frac{d\tau}{T}},\notag\\
 \int_0^\tau\bigl(|\dot v_{31}(r)|+|\dot v_{33}(r)|\bigr)\,\mathrm dr
 &\le C\sqrt{\frac{\tau}{T}},\notag\\
 \int_0^\tau\left|\left(\frac{v_{31}}{v_{33}}\right)'(r)\right|\,\mathrm dr
 &\le C\left(\sqrt{\frac{\tau}{T}}+\frac{\tau}{T}\right).
 \label{eq:guide-early-energy}
\end{align}
\end{lemma}

\begin{proof}
At initialization,
$A_t(\lambda)=-\eta_0\lambda(1-\eta_0\lambda)^t$.  Lemma~\ref{lem:finite-teacher-transfer} gives the bound needed below:
\begin{equation}
\begin{aligned}
&\frac{\mathcal L_{\rm GD}(0)}d
 =\frac{\eta_0^2}{2T}\sum_{t=0}^{T-1}
   \frac1d\E\tr\!\left[\Sb^2(\Ib-\eta_0\Sb)^{2t}\right]\\
&\le \frac{C}{T}\sum_{t=0}^{T-1}(17/20)^{2t}+d^{-1+o(1)}\\
&\le \frac{C}{T}.
\end{aligned}
\label{eq:guide-early-initial-loss}
\end{equation}
Indeed, Lemma~\ref{lem:finite-teacher-transfer} is applied with
$(a,b)=(2t,2)$, and the
summed error remains $d^{-1+o(1)}$; since
$T=\Theta((\log d)^{1/4})$, this is $o(T^{-1})$.  Thus
\begin{equation}
 \mathcal L_{\rm GD}(0)\le \frac{Cd}{T}.
 \label{eq:guide-early-initial-loss-raw}
\end{equation}
Write
\[
 \bar{\btheta}=(\Vb_{31},\Wb_{13},\Vb_{33},w_{24},\wb_{25},
 w_{44},\wb_{45},\wb_{54},\Wb_{55})
\]
for the active scalar, vector, and matrix coordinates.  By
~\eqref{eq:guide-energy},
\begin{equation*}
\begin{aligned}
&-\frac{\mathrm d}{\mathrm d\tau}\mathcal L_{\rm GD}
 =\|\nabla_{\bar{\btheta}}\mathcal L_{\rm GD}\|^2\\
&=d\bigl(\dot v_{31}^{\,2}+\dot w_{13}^{\,2}+\dot v_{33}^{\,2}\bigr)
 +\dot w_{24}^{\,2}+\|\dot{\wb}_{25}\|^2
 +\dot w_{44}^{\,2}+\|\dot{\wb}_{45}\|^2\\
&\qquad+\|\dot{\wb}_{54}\|^2+\|\dot{\Wb}_{55}\|_F^2.
\end{aligned}
\end{equation*}
Consequently, for $0\le\tau\le H$,
\begin{equation*}
\begin{aligned}
&\int_0^\tau d\bigl(\dot v_{31}^{\,2}+\dot w_{13}^{\,2}+\dot v_{33}^{\,2}\bigr)\,\mathrm dr\\
&\quad+\int_0^\tau\bigl(\dot w_{24}^{\,2}+\|\dot{\wb}_{25}\|^2+\dot w_{44}^{\,2}
 +\|\dot{\wb}_{45}\|^2+\|\dot{\wb}_{54}\|^2
 +\|\dot{\Wb}_{55}\|_F^2\bigr)\,\mathrm dr\\
&=\mathcal L_{\rm GD}(0)-\mathcal L_{\rm GD}(\tau)
 \le\frac{Cd}{T},\\
&\int_0^\tau\bigl(\dot v_{31}^{\,2}+\dot w_{13}^{\,2}+\dot v_{33}^{\,2}\bigr)\,\mathrm dr
 \le\frac CT,\\
&\int_0^\tau\|\dot{\Wb}_{55}\|_F^2\,\mathrm dr
 \le\frac{Cd}{T}.
\end{aligned}
\end{equation*}

\noindent\textup{(i)} By Cauchy--Schwarz and the initialization,
\begin{equation*}
\begin{aligned}
&\sup_{0\le r\le\tau}|v_{31}(r)|
 \le |v_{31}(0)|+\sqrt{\tau}\left(\int_0^\tau\dot v_{31}^{\,2}\,\mathrm dr\right)^{1/2}
 \le C\sqrt{\frac{\tau}{T}},\\
&\sup_{0\le r\le\tau}|w_{13}(r)|
 \le |w_{13}(0)|+\sqrt{\tau}\left(\int_0^\tau\dot w_{13}^{\,2}\,\mathrm dr\right)^{1/2}
 \le C\sqrt{\frac{\tau}{T}},\\
&\sup_{0\le r\le\tau}|v_{33}(r)-v_{33}(0)|
 \le\sqrt{\tau}\left(\int_0^\tau\dot v_{33}^{\,2}\,\mathrm dr\right)^{1/2}
 \le C\sqrt{\frac{\tau}{T}}.
\end{aligned}
\end{equation*}

\noindent\textup{(ii)} Similarly,
\begin{equation*}
\begin{aligned}
&\sup_{0\le r\le\tau}\|\Wb_{55}(r)\|_F
 \le \|\Wb_{55}(0)\|_F
 +\sqrt{\tau}\left(\int_0^\tau\|\dot{\Wb}_{55}\|_F^2\,\mathrm dr\right)^{1/2}
 \le C\sqrt{\frac{d\tau}{T}},\\
&\int_0^\tau\bigl(|\dot v_{31}|+|\dot v_{33}|\bigr)\,\mathrm dr\\
&\quad\le\sqrt{\tau}\left[
 \left(\int_0^\tau\dot v_{31}^{\,2}\,\mathrm dr\right)^{1/2}
 +\left(\int_0^\tau\dot v_{33}^{\,2}\,\mathrm dr\right)^{1/2}\right]
 \le C\sqrt{\frac{\tau}{T}}.
\end{aligned}
\end{equation*}

\noindent\textup{(iii)} Finally, $v_{33}(r)\ge v_{33}(0)>0$ by
~\eqref{eq:guide-history-invariant}, and hence
\begin{equation*}
\begin{aligned}
&\int_0^\tau\left|\left(\frac{v_{31}}{v_{33}}\right)'\right|\,\mathrm dr\\
&=\int_0^\tau
 \left|\frac{\dot v_{31}v_{33}-v_{31}\dot v_{33}}{v_{33}^2}\right|\,\mathrm dr\\
&\le C\int_0^\tau|\dot v_{31}|\,\mathrm dr
 +C\sup_{0\le r\le\tau}|v_{31}(r)|
       \int_0^\tau|\dot v_{33}|\,\mathrm dr\\
&\le C\sqrt{\frac{\tau}{T}}+C\frac{\tau}{T}
 \qquad\bigl(v_{33}(r)\ge v_{33}(0)>0\bigr).
\end{aligned}
\end{equation*}
\end{proof}

\begin{lemma}
\label{lem:guide-zero-integrated-forcing}
Under~\eqref{eq:warm-init}, for every deterministic interval
\([0,H]\) with $H=d^{o(1)}$, uniformly for $0\le s\le\tau\le H$,
\begin{equation}
 \begin{aligned}
 &\left|w_{13}(\tau)-w_{13}(s)+\frac{2\mathbb E[B(\Sb)]}{d}
       \int_s^\tau v_{31}(r)(\widetilde\theta_2(r)+\eta_0)\,\mathrm dr\right|\\
 &\quad\le\frac{\exp\{CT\log(e+T)\}(n/d)^{CT}}{\sqrt d}
                    (1+H/T)^{3/2}
 =d^{-1/2+o(1)}.
 \end{aligned}
 \label{eq:guide-zero-b-integrated}
\end{equation}
\end{lemma}

\begin{proof}
Since the projection $P_t$ enters the calculations below, we record its
coefficients explicitly before estimating them. For $1\le t<T$, define
the projection coefficients $u_{1t},\ldots,u_{tt}$ by
\begin{equation}
 P_t[\lambda p_t^{\rm GD}]
 =u_{1t}\lambda+\sum_{\ell=2}^t u_{\ell t}p_\ell^{\rm GD}.
\end{equation}
The residual is orthogonal to the $t$-dimensional projection space, so
\begin{align*}
 \left\langle\lambda p_t^{\rm GD}-P_t[\lambda p_t^{\rm GD}],\lambda\right\rangle_{d,n}&=0,\\
 \left\langle\lambda p_t^{\rm GD}-P_t[\lambda p_t^{\rm GD}],p_j^{\rm GD}\right\rangle_{d,n}&=0,
 \qquad 2\le j\le t.
\end{align*}
Substituting the projection expansion into these identities gives
\begin{align*}
 \left\langle u_{1t}\lambda+\sum_{\ell=2}^t u_{\ell t}p_\ell^{\rm GD},\lambda\right\rangle_{d,n}
 &=\left\langle\lambda p_t^{\rm GD},\lambda\right\rangle_{d,n},\\
 \left\langle u_{1t}\lambda+\sum_{\ell=2}^t u_{\ell t}p_\ell^{\rm GD},p_j^{\rm GD}\right\rangle_{d,n}
 &=\left\langle\lambda p_t^{\rm GD},p_j^{\rm GD}\right\rangle_{d,n},
 \qquad 2\le j\le t.
\end{align*}
Thus the projection conditions give the explicit system
\[
\begin{pmatrix}
 \langle\lambda,\lambda\rangle_{d,n}
   &\langle p_2^{\rm GD},\lambda\rangle_{d,n}&\cdots&\langle p_t^{\rm GD},\lambda\rangle_{d,n}\\
 \langle\lambda,p_2^{\rm GD}\rangle_{d,n}
   &\langle p_2^{\rm GD},p_2^{\rm GD}\rangle_{d,n}&\cdots&\langle p_t^{\rm GD},p_2^{\rm GD}\rangle_{d,n}\\
 \vdots&\vdots&\ddots&\vdots\\
 \langle\lambda,p_t^{\rm GD}\rangle_{d,n}
   &\langle p_2^{\rm GD},p_t^{\rm GD}\rangle_{d,n}&\cdots&\langle p_t^{\rm GD},p_t^{\rm GD}\rangle_{d,n}
\end{pmatrix}
\begin{pmatrix}u_{1t}\\u_{2t}\\\vdots\\u_{tt}\end{pmatrix}
=
\begin{pmatrix}
 \langle\lambda p_t^{\rm GD},\lambda\rangle_{d,n}\\
 \langle\lambda p_t^{\rm GD},p_2^{\rm GD}\rangle_{d,n}\\
 \vdots\\
 \langle\lambda p_t^{\rm GD},p_t^{\rm GD}\rangle_{d,n}
\end{pmatrix}.
\]
By Lemma~\ref{lem:finite-teacher-gram-lower}, the Gram matrix in this
display is positive definite, so the system is nonsingular.  Since
$p_1^{\rm GD}=\eta_0\lambda$ and $0.4\le\eta_0\le0.9$, replacing $p_1^{\rm GD}$ by
$\lambda$ changes the first column only by a bounded factor.  The coefficient
bound in Lemma~\ref{lem:finite-teacher-gram-coefficients}, applied to the
displayed projection identity, gives,
column by column,
\begin{align*}
 \sum_{\ell=1}^{t}u_{\ell t}^2
 &\le \exp\{CT\log(e+T)\}(n/d)^{CT}
 \sum_{\ell=1}^{t}
 \bigl|\langle P_t[\lambda p_t^{\rm GD}],p_{\ell}^{\rm GD}\rangle_{d,n}\bigr|^2\\
 &\le \exp\{CT\log(e+T)\}(n/d)^{CT}
 \sum_{\ell=1}^{t}
 \|P_t[\lambda p_t^{\rm GD}]\|_{d,n}^2\|p_\ell^{\rm GD}\|_{d,n}^2\\[-2pt]
 &\le C t\exp\{CT\log(e+T)\}(n/d)^{CT}
 \|P_t[\lambda p_t^{\rm GD}]\|_{d,n}^2.
\end{align*}
After enlarging $C$, the factor $t\le T$ is absorbed into
$\exp\{CT\log(e+T)\}$.  Summing this estimate and using
Lemma~\ref{lem:finite-teacher-trace-estimates} yields
\begin{align}
 \sum_{t=1}^{T-1}\sum_{\ell=1}^t u_{\ell t}^2
 &\le\exp\{CT\log(e+T)\}(n/d)^{CT}
 \sum_{t=1}^{T-1}\|P_t[\lambda p_t^{\rm GD}]\|_{d,n}^2\notag\\
 &\le\exp\{CT\log(e+T)\}(n/d)^{CT}
 \sum_{t=1}^{T-1}\|\lambda p_t^{\rm GD}\|_{d,n}^2\notag\\
 &\le CT\exp\{CT\log(e+T)\}(n/d)^{CT}.
 \label{eq:guide-zero-projection-coefficient-bound}
\end{align}
For each $1\le t<T$, the definitions of $\Cb_1,\Cb_2$ give
\begin{align*}
 [\Cb_1(\Sb)]_t
 &=\frac1{\sqrt{2T}}\operatorname{vec}\left(
 u_{1t}\Sb+\sum_{\ell=2}^t u_{\ell t}p_\ell^{\rm GD}(\Sb)\right)\\
 &=\sum_{\ell=1}^t u_{\ell t}[\Cb_2(\Sb)]_{t,(\ell,t)}.
\end{align*}
For $1\le\ell\le t<T$, direct differentiation gives
{\small
\begin{align*}
 (\nabla_{\widetilde{\bTheta}_3}\mathcal L_{\rm GD,C})_{\ell t}
 =2\E\Bigl\langle [\Cb_1(\Sb)]_t(\widetilde\theta_2+\eta_0)
 \bigl[\Cb_2(\Sb)\operatorname{vec}
 (\widetilde{\bTheta}_3-\eta_0\eb_1(\one_T-\eb_0)^{\mathsf T})\bigr]_t,
 [\Cb_2(\Sb)]_{t,(\ell,t)}\Bigr\rangle.
\end{align*}
}
Since
$\partial_{w_{13}}\widetilde\theta_1=0$,
$\partial_{w_{13}}\widetilde\theta_2=v_{31}$, and
$\partial_{w_{13}}\widetilde{\bTheta}_3=0$, the chain rule and the preceding
block identities give
\begin{align*}
 \partial_{w_{13}}\mathcal L_{\rm GD,C}
 &=2v_{31}\sum_{t=1}^{T-1}\E\Bigl\langle [\Cb_1(\Sb)]_t(\widetilde\theta_2+\eta_0)
 \bigl[\Cb_2(\Sb)\operatorname{vec}
 (\widetilde{\bTheta}_3-\eta_0\eb_1(\one_T-\eb_0)^{\mathsf T})\bigr]_t,
 [\Cb_1(\Sb)]_t\Bigr\rangle\\
 &=v_{31}\sum_{t=1}^{T-1}\sum_{\ell=1}^t u_{\ell t}
 (\nabla_{\widetilde{\bTheta}_3}\mathcal L_{\rm GD,C})_{\ell t}.
\end{align*}
Consequently,
\begin{align*}
 \partial_{w_{13}}\mathcal L_{\rm GD}
 &=2v_{31}\E[B(\Sb)](\widetilde\theta_2+\eta_0)
 +\partial_{w_{13}}\mathcal L_{\rm GD,C}\\
 &=2v_{31}\E[B(\Sb)](\widetilde\theta_2+\eta_0)
 +v_{31}\sum_{t=1}^{T-1}\sum_{\ell=1}^t u_{\ell t}
 (\nabla_{\widetilde{\bTheta}_3}\mathcal L_{\rm GD,C})_{\ell t}.
\end{align*}
For $1\le\ell\le t<T$, the chain rule for the auxiliary
parameterization gives
\begin{align*}
 \nabla_{(\Wb_{55})_{\ell t}}\mathcal L_{\rm GD,C}
 &=\eta_0^{\ind\{\ell=1\}}v_{33}
 (\nabla_{\widetilde{\bTheta}_3}\mathcal L_{\rm GD,C})_{\ell t},\\
 (\dot{\Wb}_{55})_{\ell t}
 &=-\nabla_{(\Wb_{55})_{\ell t}}\mathcal L_{\rm GD,C}\\
 &=-\eta_0^{\ind\{\ell=1\}}v_{33}
 (\nabla_{\widetilde{\bTheta}_3}\mathcal L_{\rm GD,C})_{\ell t}.
\end{align*}
Substituting this identity into
$\dot w_{13}=-\partial_{w_{13}}\mathcal L_{\rm GD}/d$ yields
\begin{equation}
 \dot w_{13}
 =-\frac{2v_{31}}d\E[B(\Sb)](\widetilde\theta_2+\eta_0)
 +\frac{v_{31}}{d v_{33}}\sum_{t=1}^{T-1}\sum_{\ell=1}^t
 \frac{u_{\ell t}}{\eta_0^{\ind\{\ell=1\}}}
 (\dot{\Wb}_{55})_{\ell t}.
 \label{eq:guide-zero-b-primitive}
\end{equation}
By Lemma~\ref{lem:guide-early-energy}, the four bounds in
~\eqref{eq:guide-early-energy} hold on $[0,H]$.
Integrating~\eqref{eq:guide-zero-b-primitive} by parts therefore gives
\begin{align*}
 &w_{13}(\tau)-w_{13}(s)+\frac{2\E[B(\Sb)]}d
 \int_s^\tau v_{31}(r)(\widetilde\theta_2(r)+\eta_0)\,\mathrm dr\\
 &\quad=\frac1d\left[\frac{v_{31}}{v_{33}}
 \sum_{t=1}^{T-1}\sum_{\ell=1}^t
 \frac{u_{\ell t}(\Wb_{55})_{\ell t}}
 {\eta_0^{\ind\{\ell=1\}}}\right]_s^\tau\\
 &\qquad-\frac1d\int_s^\tau\left(\frac{v_{31}}{v_{33}}\right)'
 \sum_{t=1}^{T-1}\sum_{\ell=1}^t
 \frac{u_{\ell t}(\Wb_{55})_{\ell t}}
 {\eta_0^{\ind\{\ell=1\}}}\,\mathrm dr.
\end{align*}
Taking absolute values and applying Cauchy--Schwarz gives
\begin{align*}
 &\left|\frac1d\left[\frac{v_{31}}{v_{33}}
 \sum_{t=1}^{T-1}\sum_{\ell=1}^t
 \frac{u_{\ell t}(\Wb_{55})_{\ell t}}
 {\eta_0^{\ind\{\ell=1\}}}\right]_s^\tau
 -\frac1d\int_s^\tau\left(\frac{v_{31}}{v_{33}}\right)'
 \sum_{t=1}^{T-1}\sum_{\ell=1}^t
 \frac{u_{\ell t}(\Wb_{55})_{\ell t}}
 {\eta_0^{\ind\{\ell=1\}}}\,\mathrm dr\right|\\
 &\le \frac{C}{d}
 \left(\sum_{t=1}^{T-1}\sum_{\ell=1}^t u_{\ell t}^2\right)^{1/2}
 \Biggl[
 \sup_{0\le r\le H}\left|\frac{v_{31}(r)}{v_{33}(r)}\right|
 \bigl(\|\Wb_{55}(\tau)\|_F+\|\Wb_{55}(s)\|_F\bigr)\\
 &\hspace{42mm}
 +\int_0^H\left|\left(\frac{v_{31}}{v_{33}}\right)'(r)\right|\,\mathrm dr
 \sup_{0\le r\le H}\|\Wb_{55}(r)\|_F\Biggr]\\
 &\le \frac{C}{\sqrt d}
 \left(\sum_{t=1}^{T-1}\sum_{\ell=1}^t u_{\ell t}^2\right)^{1/2}
 \left[\frac HT+\left(\sqrt{H/T}+H/T\right)\sqrt{H/T}\right]\\
 &\le \frac{C}{\sqrt d}
 \left(\sum_{t=1}^{T-1}\sum_{\ell=1}^t u_{\ell t}^2\right)^{1/2}
 (1+H/T)^{3/2}\\
 &\le \frac{\exp\{CT\log(e+T)\}(n/d)^{CT}}{\sqrt d}
 (1+H/T)^{3/2},
 \qquad\bigl(\eta_0^{-1}\le2.5,\ v_{33}(r)\ge v_{33}(0)>0\bigr).
\end{align*}
This proves Lemma~\ref{lem:guide-zero-integrated-forcing}.
\end{proof}

\begin{lemma}
\label{lem:guide-zero-entry}
Under \eqref{eq:warm-init}, there is a time \(\tau_0=d^{o(1)}\) such that
\begin{equation}
 v_{31}<0,\qquad w_{13}>0,\qquad
 v_{31}(w_{24}+\wb_{25}\eb_0)\ge\frac{\eta_0}{2},\qquad
 -v_{31}w_{13}\ge\frac{\eta_0}{2}.
 \label{eq:guide-zero-entry}
\end{equation}
At that time, \(\mathcal L_{\rm GD}<9\eta_0^2\mathbb E[B(\Sb)]/16\).  Consequently, for all \(\tau\ge\tau_0\),
\begin{equation}
 -v_{31}w_{13}>\frac{\eta_0}{4},\qquad |v_{31}|\ge C^{-1}>0.
 \label{eq:guide-zero-no-return}
\end{equation}
\end{lemma}

\begin{proof}
By the finite-Wishart identity
\(\|\lambda\|_{d,n}^2=d^{-1}\E\tr(\Sb^2)=1+(d+1)/n\) in
Lemma~\ref{lem:finite-wishart-second-moment},
\begin{equation}
\begin{aligned}
&1\le\|\lambda\|_{d,n}^2
 =1+\frac{d+1}{n}\le C,\\
&c\frac dn\le\sum_{t=1}^{T-1}
 \bigl\|\lambda p_t^{\rm GD}-P_t[\lambda p_t^{\rm GD}]\bigr\|_{d,n}^2
 \le2\frac dn.
\end{aligned}
\label{eq:guide-zero-spectral-scales}
\end{equation}
The lower bound in the second line is estimated in~
\eqref{eq:guide-auxiliary-gap}.  The upper bound is
Lemma~\ref{lem:finite-teacher-projection-upper}.
The loss decomposition reads
\begin{align*}
 \frac{\mathcal L_{\rm GD}}d
 &=\frac{\|\lambda\|_{d,n}^2}{2T}
 \bigl(v_{31}(w_{24}+\wb_{25}\eb_0)-\eta_0\bigr)^2\\
 &\quad+\frac1{2T}\sum_{t=1}^{T-1}
 \bigl\|\lambda p_t^{\rm GD}-P_t[\lambda p_t^{\rm GD}]\bigr\|_{d,n}^2
 (v_{31}w_{13}+\eta_0)^2\\
 &\quad+\frac1{2T}\sum_{t=1}^{T-1}
 \bigl\|P_tA_t\bigr\|_{d,n}^2.
\end{align*}

\noindent\textup{(i) Dynamics of $v_{31}$.}
Set
\[
 \tau_b=d^{-1}\exp\{\tfrac14\sqrt{\log d}\}=d^{-1+o(1)}.
\]
At initialization,
\(v_{31}(0)=w_{13}(0)=0\), \(\Wb_{55}(0)=0\),
\(w_{24}(0)+(\wb_{25}(0))_t=\sigma<0\), and \(v_{33}(0)>0\).
On the region
\begin{align*}
 |v_{31}|+|w_{13}|&\le d^{-1/3},\\
 \left\|(w_{24}+\wb_{25}\eb_t)_{t=0}^{T-1}-\sigma\one_T\right\|_2&\le\frac12,
 \qquad v_{33}\le2v_{33}(0),
\end{align*}
Lemma~\ref{lem:guide-early-energy} gives the needed early-energy bounds on
$[0,\tau_b]$.  For each $0\le t<T$, the gradient-flow equation first gives
\begin{align*}
 \frac{\mathrm d}{\mathrm d\tau}(w_{24}+\wb_{25}\eb_t)
 &=-\partial_{w_{24}}\mathcal L_{\rm GD}
   -(\nabla_{\wb_{25}}\mathcal L_{\rm GD})_t.
\end{align*}
The chain rule for
$\widetilde\theta_1=v_{31}(w_{24}+\wb_{25}\eb_0)$ and the first row of
$\widetilde{\bTheta}_3$ gives
\begin{align*}
 \partial_{w_{24}}\mathcal L_{\rm GD}
 &=v_{31}\partial_{\widetilde\theta_1}\mathcal L_{\rm GD}
   +v_{31}\sum_{s=1}^{T-1}
     \partial_{(\widetilde{\bTheta}_3)_{1s}}\mathcal L_{\rm GD},\\
 (\nabla_{\wb_{25}}\mathcal L_{\rm GD})_t
 &=v_{31}\ind\{t=0\}\partial_{\widetilde\theta_1}\mathcal L_{\rm GD}
   +v_{31}\ind\{1\le t<T\}
     \partial_{(\widetilde{\bTheta}_3)_{1t}}\mathcal L_{\rm GD},
 \qquad 0\le t<T.
\end{align*}
The loss decomposition gives the parameter derivatives
{\scriptsize
\begin{equation*}
\begin{aligned}
 \partial_{\widetilde\theta_1}\mathcal L_{\rm GD}
 &=2\E[A(\Sb)](\widetilde\theta_1-\eta_0),\\
 (\nabla_{\widetilde{\bTheta}_3}\mathcal L_{\rm GD})_{\ell t}
 &=2\E\Bigl\langle
 [\Cb_1(\Sb)]_t(\widetilde\theta_2+\eta_0)\\[-2pt]
 &\hspace{17mm}+\bigl[\Cb_2(\Sb)\operatorname{vec}
 (\widetilde{\bTheta}_3-\eta_0\eb_1(\one_T-\eb_0)^{\mathsf T})\bigr]_t,
 [\Cb_2(\Sb)]_{t,(\ell,t)}\Bigr\rangle,
 \qquad 1\le\ell\le t<T.
\end{aligned}
\end{equation*}
By Cauchy--Schwarz,
\begin{align*}
 |\partial_{\widetilde\theta_1}\mathcal L_{\rm GD}|^2
 &=4\E[A(\Sb)]\,\mathcal L_{\rm GD,A}
 \le Cd\,\mathcal L_{\rm GD,A},\\
 \|\nabla_{\widetilde{\bTheta}_3}\mathcal L_{\rm GD}\|_F^2
 &=4\sum_{t=1}^{T-1}\sum_{\ell=1}^{t}
 \left|\E\Bigl\langle
 [\Cb_1(\Sb)]_t(\widetilde\theta_2+\eta_0)
 +\bigl[\Cb_2(\Sb)\operatorname{vec}
 (\widetilde{\bTheta}_3-\eta_0\eb_1(\one_T-\eb_0)^{\mathsf T})\bigr]_t,
 [\Cb_2(\Sb)]_{t,(\ell,t)}\Bigr\rangle\right|^2\\
 &\le Cd\,\E\Bigl\|\Cb_1(\Sb)(\widetilde\theta_2+\eta_0)
 +\Cb_2(\Sb)\operatorname{vec}
 (\widetilde{\bTheta}_3-\eta_0\eb_1(\one_T-\eb_0)^{\mathsf T})\Bigr\|^2\\
 &=Cd\bigl(\mathcal L_{\rm GD,B}+\mathcal L_{\rm GD,C}\bigr).
\end{align*}
}
Here the coefficient bounds use Lemma~\ref{lem:finite-teacher-trace-estimates}
and the loss decomposition~\eqref{eq:guide-loss-decomposition}.  Hence,
\begin{align}
 \left\|\left(\frac{\mathrm d}{\mathrm d\tau}
 (w_{24}+\wb_{25}\eb_t)\right)_{t=0}^{T-1}\right\|_2^2
 &=\left\|\left(\partial_{w_{24}}\mathcal L_{\rm GD}
 + (\nabla_{\wb_{25}}\mathcal L_{\rm GD})_t\right)_{t=0}^{T-1}\right\|_2^2\notag\\
 &\le C v_{31}^2T\left(
 |\partial_{\widetilde\theta_1}\mathcal L_{\rm GD}|^2
 +\|\nabla_{\widetilde{\bTheta}_3}\mathcal L_{\rm GD}\|_F^2\right)\notag\\
 &\le C v_{31}^2Td\mathcal L_{\rm GD}.\notag
\end{align}
Taking square roots and using
$\mathcal L_{\rm GD}(\tau)\le\mathcal L_{\rm GD}(0)\le Cd/T$ gives
\begin{align}
 \left\|\left(\frac{\mathrm d}{\mathrm d\tau}
 (w_{24}+\wb_{25}\eb_t)\right)_{t=0}^{T-1}\right\|_2
 &=\left\|\left(\partial_{w_{24}}\mathcal L_{\rm GD}
 + (\nabla_{\wb_{25}}\mathcal L_{\rm GD})_t\right)_{t=0}^{T-1}\right\|_2\notag\\
 &\le C|v_{31}|\sqrt{Td\mathcal L_{\rm GD}}\notag\\
 &\le C|v_{31}|\sqrt{Td(Cd/T)}
 \le Cd|v_{31}|.
 \label{eq:guide-zero-kernel-velocity}
\end{align}
Integrating the two scalar gradient-flow equations from the initialization and
using this estimate gives, for $0\le\tau\le\tau_b$,
\begin{align}
 |v_{31}(\tau)|
 &=\left|\int_0^\tau\dot v_{31}(r)\,\mathrm dr\right|
 \le \frac{C\tau}{\sqrt T}
 \le C\frac{\tau_b}{\sqrt T}=d^{-1+o(1)},\notag\\
 |w_{13}(\tau)|
 &=\left|\int_0^\tau\dot w_{13}(r)\,\mathrm dr\right|
 \le \frac C{\sqrt T}\int_0^\tau |v_{31}(r)|\,\mathrm dr
 \le C\frac{\tau_b^2}{T}=d^{-2+o(1)},\notag\\
 \left\|(w_{24}+\wb_{25}\eb_t)_{t=0}^{T-1}-\sigma\one_T\right\|_2
 &=\left\|\int_0^\tau
 \left(\frac{\mathrm d}{\mathrm dr}
 (w_{24}+\wb_{25}\eb_t)\right)_{t=0}^{T-1}\,\mathrm dr\right\|_2\notag\\
 &\le Cd\int_0^\tau |v_{31}(r)|\,\mathrm dr
 \le Cd\frac{\tau_b^2}{\sqrt T}=d^{-1+o(1)},\notag\\
 |v_{33}(\tau)-v_{33}(0)|
 &=\left|\int_0^\tau\dot v_{33}(r)\,\mathrm dr\right|
 \le\int_0^\tau|\dot v_{33}(r)|\,\mathrm dr
 \le C\sqrt{\frac{\tau_b}{T}}=d^{-1/2+o(1)}.
 \label{eq:guide-zero-bootstrap}
\end{align}
These estimates keep the trajectory in the displayed region.  For
$1\le t<T$, the functions $p_t^{\rm GD}$, the projections $P_t$, and the spectral
variable $\lambda$ are fixed along the flow.  Differentiating the expression
for $A_t$ in~\eqref{eq:guide-auxiliary-loss} therefore gives
\begin{align*}
 \frac{\mathrm d}{\mathrm d\tau}A_t(\lambda)
 ={}&(v_{31}w_{13})'\lambda p_t^{\rm GD}(\lambda)
 +\bigl((\widetilde{\bTheta}_3)_{1t}\bigr)'\lambda
 +\sum_{\ell=2}^{t}\bigl((\widetilde{\bTheta}_3)_{\ell t}\bigr)'p_\ell^{\rm GD}(\lambda)\\
 ={}&(v_{31}w_{13})'\lambda p_t^{\rm GD}(\lambda)
 +\bigl(v_{31}(w_{24}+\wb_{25}\eb_t)\bigr)'\lambda\\
 &+\partial_{v_{33}}A_t(\lambda)\dot v_{33}
 +\partial_{w_{44}}A_t(\lambda)\dot w_{44}
 +\left\langle\nabla_{\wb_{45}}A_t(\lambda),\dot{\wb}_{45}\right\rangle\\
 &+\left\langle\nabla_{\wb_{54}}A_t(\lambda),\dot{\wb}_{54}\right\rangle
 +\left\langle\nabla_{\Wb_{55}}A_t(\lambda),\dot{\Wb}_{55}\right\rangle_F.
\end{align*}
Indeed, the first two terms are obtained from
$(\widetilde\theta_2+\eta_0)\lambda p_t^{\rm GD}$ and
$((\widetilde{\bTheta}_3)_{1t}-\eta_0)\lambda$, respectively, using
$\widetilde\theta_2=v_{31}w_{13}$ and the definition of the first row of
$\widetilde{\bTheta}_3$; the remaining terms are the chain-rule derivatives
through $v_{33},w_{44},\wb_{45},\wb_{54},\Wb_{55}$.  Applying $P_t$ and
using $P_t\lambda=\lambda$ gives
{\small
\begin{align*}
 \frac{\mathrm d}{\mathrm d\tau}(P_tA_t)
 ={}&\frac{\mathrm d}{\mathrm d\tau}(v_{31}w_{13})P_t(\lambda p_t^{\rm GD})
 +\frac{\mathrm d}{\mathrm d\tau}\bigl[v_{31}(w_{24}+\wb_{25}\eb_t)\bigr]\lambda\\
 &+P_t\partial_{v_{33}}A_t\frac{\mathrm d v_{33}}{\mathrm d\tau}
 +P_t\partial_{w_{44}}A_t\frac{\mathrm d w_{44}}{\mathrm d\tau}\notag\\
 &+P_t\left\langle\nabla_{\wb_{45}}A_t,
       \frac{\mathrm d\wb_{45}}{\mathrm d\tau}\right\rangle
 +P_t\left\langle\nabla_{\wb_{54}}A_t,
       \frac{\mathrm d\wb_{54}}{\mathrm d\tau}\right\rangle\\
 &+P_t\left\langle\nabla_{\Wb_{55}}A_t,
       \frac{\mathrm d\Wb_{55}}{\mathrm d\tau}\right\rangle_F.
\end{align*}
}
Consequently, the chain rule, the gradient-flow equations, and
\eqref{eq:guide-energy} give
{\scriptsize
\begin{align}
 \frac{\mathrm d}{\mathrm d\tau}
 \sum_{t=1}^{T-1}\frac1d\E\bigl\|(P_tA_t)(\Sb)\bigr\|_F^2
 ={}&2\sum_{t=1}^{T-1}
 \frac1d\E\left\langle (P_tA_t)(\Sb),
 \frac{\mathrm d}{\mathrm d\tau}\bigl[(P_tA_t)(\Sb)\bigr]\right\rangle_F\notag\\
 ={}&-2T\left|\partial_{v_{33}}\frac{\mathcal L_{\rm GD}}d\right|^2
 -2dT\left|\partial_{w_{44}}\frac{\mathcal L_{\rm GD}}d\right|^2\notag\\
 &-2dT\left(\left\|\nabla_{\wb_{45}}\frac{\mathcal L_{\rm GD}}d\right\|^2
 +\left\|\nabla_{\wb_{54}}\frac{\mathcal L_{\rm GD}}d\right\|^2
 +\left\|\nabla_{\Wb_{55}}\frac{\mathcal L_{\rm GD}}d\right\|_F^2\right)\notag\\
 &+2\sum_{t=1}^{T-1}\frac1d\E\Bigl\langle (P_tA_t)(\Sb),
 \Bigl((v_{31}w_{13})'P_t(\lambda p_t^{\rm GD})
 +\bigl(v_{31}(w_{24}+\wb_{25}\eb_t)\bigr)'\lambda\Bigr)(\Sb)
 \Bigr\rangle_F.
\end{align}
}
For example, the $v_{33}$ term in the first line above is
\begin{align*}
 2\sum_{t=1}^{T-1}\left\langle P_tA_t,
 P_t\partial_{v_{33}}A_t\,\dot v_{33}\right\rangle_{d,n}
 &=2T\partial_{v_{33}}\!\left(\frac{\mathcal L_{\rm GD}}d\right)\dot v_{33}\\
 &=-2T\left|\partial_{v_{33}}\frac{\mathcal L_{\rm GD}}d\right|^2,
\end{align*}
while the four ordinary block terms are, respectively,
\begin{align*}
 2\sum_{t=1}^{T-1}\left\langle P_tA_t,
 P_t\partial_{w_{44}}A_t\,\dot w_{44}\right\rangle_{d,n}
 &=-2dT\left|\partial_{w_{44}}\frac{\mathcal L_{\rm GD}}d\right|^2,\\
 2\sum_{t=1}^{T-1}\left\langle P_tA_t,
 P_t\nabla_{\wb_{45}}A_t\,\dot{\wb}_{45}\right\rangle_{d,n}
 &=-2dT\left\|\nabla_{\wb_{45}}\frac{\mathcal L_{\rm GD}}d\right\|^2,\\
 2\sum_{t=1}^{T-1}\left\langle P_tA_t,
 P_t\nabla_{\wb_{54}}A_t\,\dot{\wb}_{54}\right\rangle_{d,n}
 &=-2dT\left\|\nabla_{\wb_{54}}\frac{\mathcal L_{\rm GD}}d\right\|^2,\\
 2\sum_{t=1}^{T-1}\left\langle P_tA_t,
 P_t\nabla_{\Wb_{55}}A_t\,\dot{\Wb}_{55}\right\rangle_{d,n}
 &=-2dT\left\|\nabla_{\Wb_{55}}\frac{\mathcal L_{\rm GD}}d\right\|_F^2.
\end{align*}
For $1\le\ell\le t<T$, differentiating the coefficient multiplying
$p_\ell^{\rm GD}$ in the auxiliary parameterization gives
\begin{align*}
 \nabla_{(\Wb_{55})_{\ell t}}\frac{\mathcal L_{\rm GD}}d
 &=\frac{v_{33}}T\left\langle P_tA_t,p_\ell^{\rm GD}\right\rangle_{d,n},\\
 \left\|\nabla_{\Wb_{55}}\frac{\mathcal L_{\rm GD}}d\right\|_F^2
 &=\frac{v_{33}^2}{T^2}\sum_{t=1}^{T-1}\sum_{\ell=1}^{t}
 \left|\left\langle P_tA_t,p_\ell^{\rm GD}\right\rangle_{d,n}\right|^2\\
 &\ge\frac{v_{33}(0)^2}{T^2}
 \exp\{-CT\log(e+T)\}(d/n)^{CT}
 \sum_{t=1}^{T-1}\|P_tA_t\|_{d,n}^2,
\end{align*}
where the last inequality uses $v_{33}\ge v_{33}(0)>0$ from
Lemma~\ref{lem:guide-invariants} and
Lemma~\ref{lem:finite-teacher-gram-lower}.  Substituting this bound into the preceding differential identity and
discarding the other negative squares yields
\begin{align*}
 \frac{\mathrm d}{\mathrm d\tau}
 \sum_{t=1}^{T-1}\|P_tA_t\|_{d,n}^2
 &\le-\frac{2d\,v_{33}(0)^2}{T}
 \exp\{-CT\log(e+T)\}(d/n)^{CT}
 \sum_{t=1}^{T-1}\|P_tA_t\|_{d,n}^2\\
 &\quad+2\sum_{t=1}^{T-1}\left\langle P_tA_t,
 (v_{31}w_{13})'P_t(\lambda p_t^{\rm GD})
 +\bigl(v_{31}(w_{24}+\wb_{25}\eb_t)\bigr)'\lambda
 \right\rangle_{d,n}.
\end{align*}
For the first forcing term, Cauchy--Schwarz and
Lemma~\ref{lem:finite-teacher-trace-estimates} give
\begin{align*}
 &\frac{|(v_{31}w_{13})'|\displaystyle\sum_{t=1}^{T-1}
 \bigl|\langle P_tA_t,P_t(\lambda p_t^{\rm GD})\rangle_{d,n}\bigr|}
 {\left(\displaystyle\sum_{t=1}^{T-1}\|P_tA_t\|_{d,n}^2\right)^{1/2}}\\
 &\quad\le |(v_{31}w_{13})'|
 \left(\sum_{t=1}^{T-1}\|P_t(\lambda p_t^{\rm GD})\|_{d,n}^2\right)^{1/2}\\
 &\quad\le |(v_{31}w_{13})'|
 \left(\sum_{t=1}^{T-1}\|\lambda p_t^{\rm GD}\|_{d,n}^2\right)^{1/2}
 \le C\sqrt T\,|(v_{31}w_{13})'|.
\end{align*}
For the second forcing term, the same argument, together with
$\|\lambda\|_{d,n}\le C$ from Lemma~\ref{lem:finite-wishart-second-moment}, gives
\begin{align*}
 &\frac{\displaystyle\sum_{t=1}^{T-1}
 \left|\left\langle P_tA_t,
 \bigl(v_{31}(w_{24}+\wb_{25}\eb_t)\bigr)'\lambda\right\rangle_{d,n}\right|}
 {\left(\displaystyle\sum_{t=1}^{T-1}\|P_tA_t\|_{d,n}^2\right)^{1/2}}\\
 &\quad\le \|\lambda\|_{d,n}
 \left(\sum_{t=1}^{T-1}
 \left|\bigl(v_{31}(w_{24}+\wb_{25}\eb_t)\bigr)'\right|^2\right)^{1/2}\\
 &\quad\le C\left\|\left(v_{31}(w_{24}+\wb_{25}\eb_t)\right)'_{1\le t<T}\right\|_2.
\end{align*}
Thus, with the explicit damping coefficient displayed above,
\begin{equation}
\begin{aligned}
 \nabla_\tau\left(\sum_{t=1}^{T-1}\|P_tA_t\|_{d,n}^2\right)^{1/2}
 &\le -a_d\left(\sum_{t=1}^{T-1}\|P_tA_t\|_{d,n}^2\right)^{1/2}\\
 &\quad+C\left(\sqrt T\,|(v_{31}w_{13})'|
 +\left\|\left(v_{31}(w_{24}+\wb_{25}\eb_t)\right)'_{1\le t<T}\right\|_2\right),
\end{aligned}
\label{eq:guide-zero-projection-gradient}
\end{equation}
where
\[
 a_d:=\frac dT\exp\{-CT\log(e+T)\}(d/n)^{CT}.
\]
By Grönwall's inequality applied to~\eqref{eq:guide-zero-projection-gradient},
{\footnotesize
\hfuzz=26pt
\begin{align}
 \left(\sum_{t=1}^{T-1}\|P_tA_t(\tau_b)\|_{d,n}^2\right)^{1/2}
 &\le e^{-a_d\tau_b}
 \left(\sum_{t=1}^{T-1}\|P_tA_t(0)\|_{d,n}^2\right)^{1/2}
 \notag\\
 &\quad+C\int_0^{\tau_b}e^{-a_d(\tau_b-r)}\Biggl[\sqrt T\,|(v_{31}w_{13})'(r)|
 \notag\\
 &\hspace{25mm}+\left\|\left(v_{31}(w_{24}+\wb_{25}\eb_t)\right)'_{1\le t<T}\right\|_2\Biggr]\,\mathrm dr
 \notag\\
 &\le d^{-20}+C\sqrt T\int_0^{\tau_b}|(v_{31}w_{13})'|\,\mathrm dr
 +C\int_0^{\tau_b}
 \left\|\left(v_{31}(w_{24}+\wb_{25}\eb_t)\right)'_{1\le t<T}\right\|_2\,\mathrm dr
 \notag\\
 &\le d^{-20}
 +C\sqrt T\left[
 \sup_{0\le r\le\tau_b}|v_{31}(r)|\int_0^{\tau_b}|\dot w_{13}(r)|\,\mathrm dr
 +\sup_{0\le r\le\tau_b}|w_{13}(r)|\int_0^{\tau_b}|\dot v_{31}(r)|\,\mathrm dr\right]
 \notag\\
 &\quad+C\int_0^{\tau_b}|v_{31}(r)|\left\|\partial_r
 (w_{24}+\wb_{25}\eb_t)_{0\le t<T}\right\|_2\,\mathrm dr
 \notag\\
 &\quad+C\int_0^{\tau_b}|\dot v_{31}(r)|\left\|(w_{24}+\wb_{25}\eb_t)_{0\le t<T}\right\|_2\,\mathrm dr
 \notag\\
 &\le d^{-20}+C\frac{\tau_b^3}{T}
 +C\int_0^{\tau_b}\left(d|v_{31}(r)|^2+\sqrt T\,|\dot v_{31}(r)|\right)\,\mathrm dr
 \notag\\
 &\le d^{-20}+C\frac{\tau_b^3}{T}
 +C\frac{d\tau_b^3}{T}+C\tau_b
 \notag\\
 &\le d^{-20}+\frac Td\exp\{CT\log(e+T)\}
 \left(\frac nd\right)^{CT}.
 \label{eq:guide-zero-projection-bound}
\end{align}
}
The fourth displayed inequality in~\eqref{eq:guide-zero-projection-bound} uses~\eqref{eq:guide-zero-kernel-velocity} and~\eqref{eq:guide-zero-bootstrap}.  On the displayed region,
\[
 \left\|(w_{24}+\wb_{25}\eb_t)_{0\le t<T}\right\|_2\le C\sqrt T.
\]
Integrating the same bound once more gives
\begin{align}
 \int_0^{\tau_b}\left(\sum_{t=1}^{T-1}\|P_tA_t\|_{d,n}^2\right)^{1/2}\mathrm dr
 &\le \tau_b d^{-20}
 +C\tau_b\sqrt T\int_0^{\tau_b}|(v_{31}w_{13})'|\,\mathrm dr
 \notag\\
 &\quad+C\tau_b\int_0^{\tau_b}
 \left\|\left(v_{31}(w_{24}+\wb_{25}\eb_t)\right)'_{1\le t<T}\right\|_2\,\mathrm dr
 \notag\\
 &\le \frac Td\exp\{CT\log(e+T)\}(n/d)^{CT}(1+\tau_b).
 \label{eq:guide-zero-projection-integral}
\end{align}
The exact $v_{31}$ equation, obtained from
\eqref{eq:guide-loss-decomposition} and the chain rule, is
\begin{equation}
\begin{aligned}
 \dot v_{31}={}&-\frac{\|\lambda\|_{d,n}^2}{T}
 (w_{24}+\wb_{25}\eb_0)\{v_{31}(w_{24}+\wb_{25}\eb_0)-\eta_0\}\\
 &-\frac{w_{13}}{T}\sum_{t=1}^{T-1}
 \|\lambda p_t^{\rm GD}-P_t[\lambda p_t^{\rm GD}]\|_{d,n}^2
 (v_{31}w_{13}+\eta_0)\\
 &-\frac1T\sum_{t=1}^{T-1}
 \left\langle P_tA_t,
 (w_{24}+\wb_{25}\eb_t)\lambda+w_{13}P_t(\lambda p_t^{\rm GD})
 \right\rangle_{d,n}.
\end{aligned}
\label{eq:guide-zero-v31-flow}
\end{equation}
The first term in~\eqref{eq:guide-zero-v31-flow} is bounded by
\begin{equation}
\begin{aligned}
 &-\frac{\|\lambda\|_{d,n}^2}{T}(w_{24}+\wb_{25}\eb_0)
 \{v_{31}(w_{24}+\wb_{25}\eb_0)-\eta_0\}\\
 &=-\frac{\|\lambda\|_{d,n}^2}{T}
 |w_{24}+\wb_{25}\eb_0|
 \{\eta_0-v_{31}(w_{24}+\wb_{25}\eb_0)\}\\
 &\le-\frac{c}{T}
 \qquad\text{on the displayed region}.
\end{aligned}
\label{eq:guide-zero-v31-main-drift}
\end{equation}
Here we used~\eqref{eq:guide-zero-spectral-scales} and
\(\eta_0\ge0.4\).  For the second term,~\eqref{eq:guide-zero-spectral-scales} and~\eqref{eq:guide-zero-bootstrap} give
\begin{align}
 &\int_0^{\tau_b}\left|\frac{w_{13}}{T}\sum_{t=1}^{T-1}
 \|\lambda p_t^{\rm GD}-P_t[\lambda p_t^{\rm GD}]\|_{d,n}^2
 (v_{31}w_{13}+\eta_0)\right|\,\mathrm dr
 \notag\\
 &\quad\le \frac{C}{T}\sum_{t=1}^{T-1}
 \|\lambda p_t^{\rm GD}-P_t[\lambda p_t^{\rm GD}]\|_{d,n}^2
 \int_0^{\tau_b}|w_{13}(r)|\,\mathrm dr
 \notag\\
 &\quad\le \frac{C}{T}\frac dn\,
 \tau_b\sup_{0\le r\le\tau_b}|w_{13}(r)|
 \le C\frac dn\frac{\tau_b^3}{T^2}
 =o\!\left(\frac{\tau_b}{T}\right),
 \notag\\[4pt]
 &\int_0^{\tau_b}\left|\frac1T\sum_{t=1}^{T-1}
 \left\langle P_tA_t,
 (w_{24}+\wb_{25}\eb_t)\lambda+w_{13}P_t(\lambda p_t^{\rm GD})
 \right\rangle_{d,n}\right|\,\mathrm dr
 \notag\\
 &\quad\le \frac1T\int_0^{\tau_b}
 \left(\sum_{t=1}^{T-1}\|P_tA_t\|_{d,n}^2\right)^{1/2}
 \left(\sum_{t=1}^{T-1}
 \|(w_{24}+\wb_{25}\eb_t)\lambda+w_{13}P_t(\lambda p_t^{\rm GD})\|_{d,n}^2\right)^{1/2}\,\mathrm dr
 \notag\\
 &\quad\le \frac{C}{\sqrt T}
 \int_0^{\tau_b}
 \left(\sum_{t=1}^{T-1}\|P_tA_t\|_{d,n}^2\right)^{1/2}\,\mathrm dr
 \notag\\
 &\quad\le \frac{C}{\sqrt T}\frac Td
 \exp\{CT\log(e+T)\}\left(\frac nd\right)^{CT}(1+\tau_b)
 \notag\\
 &=\frac{C\sqrt T}{d}
 \exp\{CT\log(e+T)\}\left(\frac nd\right)^{CT}(1+\tau_b)
 =o\!\left(\frac{\tau_b}{T}\right).
 \label{eq:guide-zero-v31-errors}
\end{align}
Here the middle inequality uses the displayed region,
\eqref{eq:guide-zero-spectral-scales},
\eqref{eq:guide-zero-projection-coefficient-bound}, and
\eqref{eq:guide-zero-projection-integral}.
Consequently, integrating~\eqref{eq:guide-zero-v31-flow} and using
~\eqref{eq:guide-zero-v31-main-drift} and
~\eqref{eq:guide-zero-v31-errors} gives
\begin{align}
 v_{31}(\tau_b)
 &=\int_0^{\tau_b}\dot v_{31}(r)\,\mathrm dr
 \le-\frac{\tau_b}{CT}+o\!\left(\frac{\tau_b}{T}\right)
 \le-\frac{\tau_b}{2CT}
 \le-\frac1{Cd},
 \label{eq:guide-zero-v31-at-anchor}\\
 v_{31}(\tau_b)(w_{24}+\wb_{25}\eb_0)\big|_{\tau_b}
 &=|v_{31}(\tau_b)|\,|w_{24}(\tau_b)+\wb_{25}(\tau_b)\eb_0|\notag\\
 &\ge |v_{31}(\tau_b)|
 \left(|\sigma|-\left\|(w_{24}+\wb_{25}\eb_t)_{t=0}^{T-1}
 -\sigma\one_T\right\|_2\right)\notag\\
 &\ge |v_{31}(\tau_b)|\bigl(|\sigma|-1/2\bigr)
 \ge\frac1{Cd},
 \label{eq:guide-zero-product-at-anchor}\\
 |v_{31}(\tau_b)w_{13}(\tau_b)|
 &\le\sup_{0\le r\le\tau_b}|v_{31}(r)|
 \sup_{0\le r\le\tau_b}|w_{13}(r)|
\notag\\
 &\le C\frac{\tau_b}{\sqrt T}\frac{\tau_b^2}{T}
 =C\frac{\tau_b^3}{T^{3/2}}
 =d^{-3+o(1)},
 \label{eq:guide-zero-product-small}\\
 \sum_{t=1}^{T-1}\|P_tA_t(\tau_b)\|_{d,n}^2
 &\le2\left(\sum_{t=1}^{T-1}\|P_tA_t(\tau_b)\|_{d,n}^2\right)^{1/2}\!^2
 \notag\\
 &\le2d^{-40}
 +2\left[\frac Td\exp\{CT\log(e+T)\}
 \left(\frac nd\right)^{CT}\right]^2
 \le d^{-2+o(1)}.
 \label{eq:guide-zero-projection-at-anchor}
\end{align}
Here the first line uses~\eqref{eq:guide-zero-v31-main-drift} and
~\eqref{eq:guide-zero-v31-errors}; the second uses the third bound in
~\eqref{eq:guide-zero-bootstrap} and the initialization; the third uses
~\eqref{eq:guide-zero-bootstrap}; and the fourth uses
~\eqref{eq:guide-zero-projection-bound}.
Using~\eqref{eq:guide-loss-decomposition}, we further obtain; all
quantities in the next display are evaluated at $\tau_b$:
\begin{align}
 &\frac{\eta_0^2}{2T}\left(\|\lambda\|_{d,n}^2
 +\sum_{t=1}^{T-1}\|\lambda p_t^{\rm GD}-P_t[\lambda p_t^{\rm GD}]\|_{d,n}^2\right)
 -\frac{\mathcal L_{\rm GD}(\tau_b)}d\notag\\
 &=\frac1{2T}\Biggl[
 \|\lambda\|_{d,n}^2\left\{\eta_0^2
 -\bigl(v_{31}(w_{24}+\wb_{25}\eb_0)-\eta_0\bigr)^2\right\}\notag\\
 &\qquad+\sum_{t=1}^{T-1}\|\lambda p_t^{\rm GD}-P_t[\lambda p_t^{\rm GD}]\|_{d,n}^2
 \left\{\eta_0^2-(v_{31}w_{13}+\eta_0)^2\right\}
 -\sum_{t=1}^{T-1}\|P_tA_t(\tau_b)\|_{d,n}^2\Biggr]
 \notag\\
 &=\frac1{2T}\Biggl[
 2\eta_0\|\lambda\|_{d,n}^2v_{31}(w_{24}+\wb_{25}\eb_0)
 -\|\lambda\|_{d,n}^2\{v_{31}(w_{24}+\wb_{25}\eb_0)\}^2\notag\\
 &\qquad-2\eta_0\sum_{t=1}^{T-1}\|\lambda p_t^{\rm GD}-P_t[\lambda p_t^{\rm GD}]\|_{d,n}^2v_{31}w_{13}
 -\sum_{t=1}^{T-1}\|\lambda p_t^{\rm GD}-P_t[\lambda p_t^{\rm GD}]\|_{d,n}^2(v_{31}w_{13})^2\notag\\
 &\qquad-\sum_{t=1}^{T-1}\|P_tA_t(\tau_b)\|_{d,n}^2\Biggr]
 \notag\\
 &\ge\frac1{2T}\Biggl[
 \frac c d-C|v_{31}(w_{24}+\wb_{25}\eb_0)|^2
 -C\frac dn|v_{31}w_{13}|
 -C\frac dn|v_{31}w_{13}|^2
 -\sum_{t=1}^{T-1}\|P_tA_t(\tau_b)\|_{d,n}^2\Biggr]
 \notag\\
 &\ge\frac1{2T}\left[\frac{c}{d}-Cd^{-2+o(1)}
 -C\frac dn d^{-3+o(1)}-C\frac dn d^{-6+o(1)}-d^{-2+o(1)}\right]
 \notag\\
 &\ge\frac1{CdT}.
 \label{eq:guide-zero-loss-gap}
\end{align}
Here the fourth line uses~\eqref{eq:guide-zero-spectral-scales} and
~\eqref{eq:guide-zero-product-at-anchor}, while the fifth line uses
~\eqref{eq:guide-zero-product-small} and
~\eqref{eq:guide-zero-projection-at-anchor}.
At any point with $v_{31}=0$, the loss decomposition and
~\eqref{eq:guide-zero-loss-gap} give
\begin{align}
 \frac{\mathcal L_{\rm GD}(\tau)}d
 &\ge\frac{\eta_0^2}{2T}\left(\|\lambda\|_{d,n}^2
 +\sum_{t=1}^{T-1}\|\lambda p_t^{\rm GD}-P_t[\lambda p_t^{\rm GD}]\|_{d,n}^2\right)\notag\\
 &\ge\frac{\mathcal L_{\rm GD}(\tau_b)}d+\frac1{CdT}
 \ge\frac{\mathcal L_{\rm GD}(\tau)}d+\frac1{CdT},
 \label{eq:guide-zero-v31-barrier}
\end{align}
\raggedbottom
Since the loss is nonincreasing, continuity implies
\[
 v_{31}(\tau)<0,\qquad \tau\ge\tau_b.
\]

\noindent\textup{(ii) Dynamics of $w_{13}$.}
Before $-v_{31}w_{13}$ reaches $3\eta_0/4$, the integrated identity from Lemma~\ref{lem:guide-zero-integrated-forcing}, together with~\eqref{eq:guide-zero-v31-barrier}, gives
\begin{equation}
\begin{aligned}
&\left|w_{13}(\tau)-w_{13}(\tau_b)
 -\frac1T\sum_{t=1}^{T-1}
 \|\lambda p_t^{\rm GD}-P_t[\lambda p_t^{\rm GD}]\|_{d,n}^2
 \int_{\tau_b}^{\tau}(-v_{31})(v_{31}w_{13}+\eta_0)\,\mathrm dr\right|\\
&\hspace{35mm}\le d^{-1/2+o(1)},\\
&-v_{31}w_{13}<\frac{3\eta_0}{4}
 \quad\Longrightarrow\quad
 v_{31}w_{13}+\eta_0
 =\eta_0-(-v_{31}w_{13})\ge\frac{\eta_0}{4},\\
&w_{13}(\tau)\ge w_{13}(\tau_b)-d^{-1/2+o(1)}
 \ge-d^{-1/2+o(1)},\\
&-d^{-1/2+o(1)}\le-v_{31}(\tau)w_{13}(\tau)<\frac{3\eta_0}{4},\\
&|v_{31}w_{13}+\eta_0|
 \le\eta_0+d^{-1/2+o(1)}=(1+o(1))\eta_0,\\
&\frac1{2T}\sum_{t=1}^{T-1}
 \|\lambda p_t^{\rm GD}-P_t[\lambda p_t^{\rm GD}]\|_{d,n}^2
 (v_{31}w_{13}+\eta_0)^2\\
&\le\frac{(1+o(1))\eta_0^2}{2T}
 \sum_{t=1}^{T-1}\|\lambda p_t^{\rm GD}-P_t[\lambda p_t^{\rm GD}]\|_{d,n}^2
 \le(1+o(1))\frac{\eta_0^2}{T}\frac dn .
\end{aligned}
\label{eq:guide-zero-w13-prethreshold}
\end{equation}
Here the last inequality uses the upper bound in
~\eqref{eq:guide-zero-spectral-scales}.

By~\eqref{eq:guide-energy} and~\eqref{eq:guide-early-initial-loss-raw},
\begin{equation}
 \mathcal L_{\rm GD}(\tau)\le\mathcal L_{\rm GD}(0)\le\frac{Cd}{T},
 \qquad \tau\ge0.
 \label{eq:guide-zero-loss-monotonicity}
\end{equation}
We first prove that $v_{31}(\tau)\ne0$ for every $\tau\ge\tau_b$.
By~\eqref{eq:guide-zero-loss-gap} and loss monotonicity, for every such
$\tau$,
\begin{equation}
\begin{aligned}
&\frac{\eta_0^2}{2T}\left(\|\lambda\|_{d,n}^2
 +\sum_{t=1}^{T-1}\|\lambda p_t^{\rm GD}-P_t[\lambda p_t^{\rm GD}]\|_{d,n}^2\right)
 -\frac{\mathcal L_{\rm GD}(\tau)}d\\
&=\frac{\eta_0^2}{2T}\left(\|\lambda\|_{d,n}^2
 +\sum_{t=1}^{T-1}\|\lambda p_t^{\rm GD}-P_t[\lambda p_t^{\rm GD}]\|_{d,n}^2\right)
 -\frac{\mathcal L_{\rm GD}(\tau_b)}d
 +\frac{\mathcal L_{\rm GD}(\tau_b)-\mathcal L_{\rm GD}(\tau)}d\\
&\ge\frac1{CdT}>0.
\end{aligned}
\label{eq:guide-zero-deficit-positive}
\end{equation}
Using~\eqref{eq:guide-loss-decomposition}, expand the same deficit after
multiplying by $T$:
\begin{equation}
\begin{aligned}
&\frac{\eta_0^2}{2}\left(\|\lambda\|_{d,n}^2
 +\sum_{t=1}^{T-1}\|\lambda p_t^{\rm GD}-P_t[\lambda p_t^{\rm GD}]\|_{d,n}^2\right)
 -T\frac{\mathcal L_{\rm GD}(\tau)}d\\
&=\frac12\Biggl[
 \|\lambda\|_{d,n}^2\left\{\eta_0^2
 -\bigl(v_{31}(w_{24}+\wb_{25}\eb_0)-\eta_0\bigr)^2\right\}\\
&\qquad+\sum_{t=1}^{T-1}\|\lambda p_t^{\rm GD}-P_t[\lambda p_t^{\rm GD}]\|_{d,n}^2
 \left\{\eta_0^2-(v_{31}w_{13}+\eta_0)^2\right\}
 -\sum_{t=1}^{T-1}\|P_tA_t\|_{d,n}^2\Biggr]\\
&=\frac12\Biggl[
 2\eta_0\|\lambda\|_{d,n}^2v_{31}(w_{24}+\wb_{25}\eb_0)
 -\|\lambda\|_{d,n}^2\{v_{31}(w_{24}+\wb_{25}\eb_0)\}^2\\
&\qquad-2\eta_0\sum_{t=1}^{T-1}
 \|\lambda p_t^{\rm GD}-P_t[\lambda p_t^{\rm GD}]\|_{d,n}^2v_{31}w_{13}
 -\sum_{t=1}^{T-1}
 \|\lambda p_t^{\rm GD}-P_t[\lambda p_t^{\rm GD}]\|_{d,n}^2(v_{31}w_{13})^2\\
&\qquad-\sum_{t=1}^{T-1}\|P_tA_t\|_{d,n}^2\Biggr]\\
&\le C\left(|v_{31}(w_{24}+\wb_{25}\eb_0)|
 +\sum_{t=1}^{T-1}\|\lambda p_t^{\rm GD}-P_t[\lambda p_t^{\rm GD}]\|_{d,n}^2|v_{31}w_{13}|\right).
\end{aligned}
\label{eq:guide-zero-deficit-upper}
\end{equation}
\flushbottom
By~\eqref{eq:guide-spatial-invariant}, the required product bounds are
\begin{equation}
\begin{aligned}
 w_{24}^2+\|\wb_{25}\|^2
 &=d v_{31}^2-d w_{13}^2+\sigma^2
 \le d v_{31}^2+\sigma^2,\\
 |w_{24}+\wb_{25}\eb_0|^2
 &\le2\bigl(w_{24}^2+|\wb_{25}\eb_0|^2\bigr)\\
 &\le2\bigl(w_{24}^2+\|\wb_{25}\|^2\bigr)
 \le2(\sigma^2+d v_{31}^2),\\
 |v_{31}(w_{24}+\wb_{25}\eb_0)|
 &\le\sqrt2|v_{31}|\sqrt{\sigma^2+d v_{31}^2},\\
 w_{13}^2
 &=v_{31}^2-\frac{w_{24}^2+\|\wb_{25}\|^2}{d}+\frac{\sigma^2}{d}
 \le v_{31}^2+\frac{\sigma^2}{d},\\
 |v_{31}w_{13}|
 &\le|v_{31}|\sqrt{v_{31}^2+\sigma^2/d}.
\end{aligned}
\label{eq:guide-zero-spatial-product-bounds}
\end{equation}
Substituting~\eqref{eq:guide-zero-spatial-product-bounds} into
~\eqref{eq:guide-zero-deficit-upper}, and using
~\eqref{eq:guide-zero-spectral-scales}, gives
\begin{equation}
\begin{aligned}
&\frac{\eta_0^2}{2}\left(\|\lambda\|_{d,n}^2
 +\sum_{t=1}^{T-1}\|\lambda p_t^{\rm GD}-P_t[\lambda p_t^{\rm GD}]\|_{d,n}^2\right)
 -T\frac{\mathcal L_{\rm GD}(\tau)}d\\
&\le C|v_{31}|\sqrt{\sigma^2+d v_{31}^2}
 +C\left(\sum_{t=1}^{T-1}
 \|\lambda p_t^{\rm GD}-P_t[\lambda p_t^{\rm GD}]\|_{d,n}^2\right)
 |v_{31}|\sqrt{v_{31}^2+\sigma^2/d}\\
&\le C\bigl(|v_{31}|+\sqrt d\,v_{31}^2\bigr).
\end{aligned}
\label{eq:guide-zero-deficit-mobility-upper}
\end{equation}
Indeed, $d/n\le1/10$ and $\sigma$ is fixed, so the second term is bounded by
$C(|v_{31}|+v_{31}^2)$ and the first by
$C(|v_{31}|+\sqrt d\,v_{31}^2)$.
The elementary implication
\[
x\le C(y+\sqrt d\,y^2)
\quad\Longrightarrow\quad
 y^2\ge C^{-1}\min\{x^2,x/\sqrt d\}
\]
now applies with $x$ equal to the left-hand side of
~\eqref{eq:guide-zero-deficit-mobility-upper} and $y=|v_{31}|$.  Thus
\begin{equation}
\begin{aligned}
 v_{31}^2
&\ge C^{-1}\min\Biggl\{
 \left[\frac{\eta_0^2}{2}\left(\|\lambda\|_{d,n}^2
 +\sum_{t=1}^{T-1}\|\lambda p_t^{\rm GD}-P_t[\lambda p_t^{\rm GD}]\|_{d,n}^2\right)
 -T\frac{\mathcal L_{\rm GD}(\tau)}d\right]^2,\\
&\hspace{25mm}\frac1{\sqrt d}\left[\frac{\eta_0^2}{2}\left(\|\lambda\|_{d,n}^2
 +\sum_{t=1}^{T-1}\|\lambda p_t^{\rm GD}-P_t[\lambda p_t^{\rm GD}]\|_{d,n}^2\right)
 -T\frac{\mathcal L_{\rm GD}(\tau)}d\right]
 \Biggr\}.
\end{aligned}
\label{eq:guide-zero-deficit-mobility}
\end{equation}
The common expression
\begin{equation*}
 \frac{\eta_0^2}{2}\left(\|\lambda\|_{d,n}^2
 +\sum_{t=1}^{T-1}\|\lambda p_t^{\rm GD}-P_t[\lambda p_t^{\rm GD}]\|_{d,n}^2\right)
 -T\frac{\mathcal L_{\rm GD}(\tau)}d
\end{equation*}
appearing in both entries of the minimum in~\eqref{eq:guide-zero-deficit-mobility}
is at least $1/(Cd)$ by~\eqref{eq:guide-zero-deficit-positive}.  Therefore the
right-hand side of~\eqref{eq:guide-zero-deficit-mobility} is strictly positive,
and hence $v_{31}(\tau)\ne0$ for every $\tau\ge\tau_b$.

By the chain rule,~\eqref{eq:guide-energy}, the coefficient estimate in
Lemma~\ref{lem:guide-auxiliary-nondegenerate}, and
Lemma~\ref{lem:finite-teacher-gram-lower}, the relevant raw-coordinate derivatives
are as follows.  Since $(\wb_{25})_0$ enters only through
$\widetilde\theta_1=v_{31}(w_{24}+\wb_{25}\eb_0)$,
\begin{align*}
 \partial_{(\wb_{25})_0}\mathcal L_{\rm GD}
 &=\partial_{\widetilde\theta_1}\mathcal L_{\rm GD}
   \partial_{(\wb_{25})_0}\!\bigl[v_{31}(w_{24}+\wb_{25}\eb_0)\bigr]
   =v_{31}\partial_{\widetilde\theta_1}\mathcal L_{\rm GD},\\
 \partial_{\widetilde\theta_1}\mathcal L_{\rm GD}
 &=\frac1T\E\tr\bigl(A_0(\Sb)\Sb\bigr)
   =\frac1T\E\tr\bigl((\widetilde\theta_1-\eta_0)\Sb^2\bigr)\\
 &=\frac dT\|\lambda\|_{d,n}^2
   \bigl(v_{31}(w_{24}+\wb_{25}\eb_0)-\eta_0\bigr).
\end{align*}
For $1\le\ell\le t<T$, the coefficient directions satisfy
\begin{align*}
 \partial_{(\Wb_{55})_{1t}}\mathcal L_{\rm GD}
 &=\eta_0v_{33}\frac1T\E\tr\bigl((P_tA_t)(\Sb)\Sb\bigr)
   =\frac{d v_{33}}T\langle P_tA_t,p_1^{\rm GD}\rangle_{d,n},\\
 \partial_{(\Wb_{55})_{\ell t}}\mathcal L_{\rm GD}
 &=v_{33}\frac1T\E\tr\bigl((P_tA_t)(\Sb)p_\ell^{\rm GD}(\Sb)\bigr)
   =\frac{d v_{33}}T\langle P_tA_t,p_\ell^{\rm GD}\rangle_{d,n},
 \qquad 2\le\ell\le t.
\end{align*}
Hence, by Lemma~\ref{lem:finite-teacher-gram-coefficients},
\begin{align*}
 \frac1d\left\|\nabla_{\Wb_{55}}\mathcal L_{\rm GD}\right\|_F^2
 &=\frac{d v_{33}^2}{T^2}
   \sum_{t=1}^{T-1}\sum_{\ell=1}^{t}
   \left|\langle P_tA_t,p_\ell^{\rm GD}\rangle_{d,n}\right|^2\\
 &\ge\frac{d v_{33}(0)^2}{T^2}
   \exp\{-CT\log(e+T)\}(d/n)^{CT}
   \sum_{t=1}^{T-1}\|P_tA_t\|_{d,n}^2.
\end{align*}
Therefore,
\begin{equation}
\begin{aligned}
-\frac{\mathrm d}{\mathrm d\tau}\frac{\mathcal L_{\rm GD}}d
&=\frac1d\|\nabla_{\bar{\btheta}}\mathcal L_{\rm GD}\|^2\\
&\ge\frac1d\left|\partial_{(\wb_{25})_0}\mathcal L_{\rm GD}\right|^2
 +\frac1d\left\|\nabla_{\Wb_{55}}\mathcal L_{\rm GD}\right\|_F^2\\
&=\frac{d v_{31}^2\|\lambda\|_{d,n}^4}{T^2}
 \bigl(v_{31}(w_{24}+\wb_{25}\eb_0)-\eta_0\bigr)^2\\
&\quad+\frac{d v_{33}(0)^2}{T^2}
 \exp\{-CT\log(e+T)\}(d/n)^{CT}
 \sum_{t=1}^{T-1}\|P_tA_t\|_{d,n}^2\\
&\ge c\exp\{-CT\log(e+T)\}(d/n)^{CT}\min\{d v_{31}^2,d\}\\
&\quad\times\left[\frac{\|\lambda\|_{d,n}^2}{2T^2}
 \bigl(v_{31}(w_{24}+\wb_{25}\eb_0)-\eta_0\bigr)^2
 +\frac1{2T^2}\sum_{t=1}^{T-1}\|P_tA_t\|_{d,n}^2\right]\\
&\ge c\exp\{-CT\log(e+T)\}(d/n)^{CT}\min\{d v_{31}^2,d\}\\
&\quad\times\left[\frac{\|\lambda\|_{d,n}^2}{2T}
 \bigl(v_{31}(w_{24}+\wb_{25}\eb_0)-\eta_0\bigr)^2
 +\frac1{2T}\sum_{t=1}^{T-1}\|P_tA_t\|_{d,n}^2\right]\\
&=c\exp\{-CT\log(e+T)\}(d/n)^{CT}\min\{d v_{31}^2,d\}\\
&\quad\times\left[\frac{\mathcal L_{\rm GD}}d
 -\frac1{2T}\sum_{t=1}^{T-1}
 \|\lambda p_t^{\rm GD}-P_t[\lambda p_t^{\rm GD}]\|_{d,n}^2
 (v_{31}w_{13}+\eta_0)^2\right].
\end{aligned}
\label{eq:guide-zero-dissipation}
\end{equation}
The inequality changing $T^{-2}$ to $T^{-1}$ uses $\|\lambda\|_{d,n}^2\ge1$,
$v_{33}\ge v_{33}(0)>0$, and absorbs one factor $T^{-1}$ into
$\exp\{-CT\log(e+T)\}$.  The final equality is the loss decomposition
~\eqref{eq:guide-loss-decomposition}.
As long as
$\mathcal L_{\rm GD}/d\ge\eta_0^2\|\lambda\|_{d,n}^2/(8T)$, the prethreshold
bound~\eqref{eq:guide-zero-w13-prethreshold} yields
\begin{equation}
\begin{aligned}
&\frac{\mathcal L_{\rm GD}}d
 -\frac1{2T}\sum_{t=1}^{T-1}
 \|\lambda p_t^{\rm GD}-P_t[\lambda p_t^{\rm GD}]\|_{d,n}^2
 (v_{31}w_{13}+\eta_0)^2\\
&\ge\frac{\eta_0^2\|\lambda\|_{d,n}^2}{8T}
 -(1+o(1))\frac{\eta_0^2}{T}\frac dn\\
&\ge\frac{c}{T},
\end{aligned}
\label{eq:guide-zero-dissipation-terms}
\end{equation}
where the last line uses~\eqref{eq:guide-zero-spectral-scales},
$\|\lambda\|_{d,n}^2\ge1$, and $d/n\le1/10$.
Combining~\eqref{eq:guide-zero-deficit-mobility} and
~\eqref{eq:guide-zero-dissipation-terms} with
~\eqref{eq:guide-zero-dissipation} gives
\begin{equation}
\begin{aligned}
-\frac{\mathrm d}{\mathrm d\tau}\frac{\mathcal L_{\rm GD}}d
&\ge\frac{c}{T}\exp\{-CT\log(e+T)\}(d/n)^{CT}\\
&\quad\times\min\Biggl\{
 d\left[\frac{\eta_0^2}{2}\left(\|\lambda\|_{d,n}^2
 +\sum_{t=1}^{T-1}\|\lambda p_t^{\rm GD}-P_t[\lambda p_t^{\rm GD}]\|_{d,n}^2\right)
 -T\frac{\mathcal L_{\rm GD}}d\right]^2,\\
&\hspace{21mm}\sqrt d\left[\frac{\eta_0^2}{2}\left(\|\lambda\|_{d,n}^2
 +\sum_{t=1}^{T-1}\|\lambda p_t^{\rm GD}-P_t[\lambda p_t^{\rm GD}]\|_{d,n}^2\right)
 -T\frac{\mathcal L_{\rm GD}}d\right]
 \Biggr\}.
\end{aligned}
\label{eq:guide-zero-deficit-growth}
\end{equation}
The expression in~\eqref{eq:guide-zero-deficit-growth}
consists of a $\tau$-independent term and
$-T\mathcal L_{\rm GD} /d$.  Its derivative is therefore
$-T\,\mathrm d(\mathcal L_{\rm GD}/d)/\mathrm d\tau$, and
~\eqref{eq:guide-zero-deficit-growth} gives a scalar ODE inequality with the
two rates $d x^2$ and $\sqrt d\,x$, after removing its common factor.
Starting from~\eqref{eq:guide-zero-loss-gap}, solving first with the $d x^2$
rate from $1/(Cd)$ to $d^{-1/2}$ gives
\begin{equation}
\int_{1/(Cd)}^{d^{-1/2}}\frac{\mathrm dx}{d x^2}
=\frac1d\left(Cd-\sqrt d\right)\le C.
\end{equation}
Once the quantity reaches $d^{-1/2}$, solving with the $\sqrt d\,x$
rate up to $C$ gives
\begin{equation}
\int_{d^{-1/2}}^{C}\frac{\mathrm dx}{\sqrt d\,x}
=\frac1{\sqrt d}\log(C\sqrt d)\le C d^{-1/2}\log d.
\end{equation}
Restoring the common factor in~\eqref{eq:guide-zero-deficit-growth}, these
two time contributions are bounded by
\begin{equation}
 T\exp\{CT\log(e+T)\}(n/d)^{CT}
 \quad\text{and}\quad
 T\exp\{CT\log(e+T)\}(n/d)^{CT}d^{-1/2}\log d,
\end{equation}
respectively.  Thus the total time is
\begin{equation}
 T\exp\{CT\log(e+T)\}(n/d)^{CT}
 \bigl(1+d^{-1/2}\log d\bigr)
 =d^{o(1)}.
\label{eq:guide-zero-deficit-time}
\end{equation}
Therefore, before time $H/2$, either
\begin{equation}
 -v_{31}w_{13}=\frac{3\eta_0}{4},
\label{eq:guide-zero-product-threshold}
\end{equation}
or
\begin{equation}
\begin{aligned}
&\frac{\mathcal L_{\rm GD}}d\le\frac{\eta_0^2\|\lambda\|_{d,n}^2}{8T},\\
&\frac{\eta_0}{2}
 \le v_{31}(w_{24}+\wb_{25}\eb_0)
 \le\frac{3\eta_0}{2},\\
&-v_{31}\ge C^{-1}d^{-1/4}.
\end{aligned}
\label{eq:guide-zero-anchor-sublevel}
\end{equation}
Indeed, the first term in~\eqref{eq:guide-loss-decomposition} gives
\begin{equation}
\begin{aligned}
\left|v_{31}(w_{24}+\wb_{25}\eb_0)-\eta_0\right|^2
&\le\frac{2T}{\|\lambda\|_{d,n}^2}\frac{\mathcal L_{\rm GD}}d\\
&\le\frac{\eta_0^2}{4},\\
\left|v_{31}(w_{24}+\wb_{25}\eb_0)\right|^2
&\le2d v_{31}^2\left(v_{31}^2+\frac{\sigma^2}{d}\right),\\
\frac{\eta_0^2}{4}
&\le2d v_{31}^2\left(v_{31}^2+\frac{\sigma^2}{d}\right),
\end{aligned}
\label{eq:guide-zero-anchor-consequences}
\end{equation}
which gives the second and third lines of~\eqref{eq:guide-zero-anchor-sublevel}.

If~\eqref{eq:guide-zero-product-threshold} has not occurred, let $s$ be the first time at which
~\eqref{eq:guide-zero-anchor-sublevel} holds.  Then Lemma~\ref{lem:guide-zero-integrated-forcing}
and~\eqref{eq:guide-zero-w13-prethreshold} give, for every later time before~\eqref{eq:guide-zero-product-threshold},
\begin{equation}
\begin{aligned}
w_{13}(\tau)-w_{13}(s)
&\ge\frac1T\sum_{t=1}^{T-1}
 \|\lambda p_t^{\rm GD}-P_t[\lambda p_t^{\rm GD}]\|_{d,n}^2
 \int_s^\tau(-v_{31})(v_{31}w_{13}+\eta_0)\,\mathrm dr
 -d^{-1/2+o(1)}\\
&\ge\frac{c}{T}d^{-1/4}
 \sum_{t=1}^{T-1}\|\lambda p_t^{\rm GD}-P_t[\lambda p_t^{\rm GD}]\|_{d,n}^2(\tau-s)
 -d^{-1/2+o(1)}.
\end{aligned}
\label{eq:guide-zero-w13-linear-growth}
\end{equation}
Since $w_{13}(s)\ge-d^{-1/2+o(1)}$, after an additional time at most
\begin{equation}
 C T\left(\sum_{t=1}^{T-1}
 \|\lambda p_t^{\rm GD}-P_t[\lambda p_t^{\rm GD}]\|_{d,n}^2\right)^{-1}
 =d^{o(1)},
\end{equation}
we obtain
\begin{equation}
\begin{aligned}
 w_{13}(\tau)&\ge4d^{-1/4}-d^{-1/2+o(1)}\ge3d^{-1/4},\\
 -v_{31}(\tau)&\ge\sqrt{w_{13}(\tau)^2-\sigma^2/d}
 \ge\frac12w_{13}(\tau).
\end{aligned}
\label{eq:guide-zero-w13-anchor-growth}
\end{equation}
Let $\tau_1$ be the first time at which
$w_{13}(\tau_1)\ge3d^{-1/4}$.  On the interval before~\eqref{eq:guide-zero-product-threshold},
Lemma~\ref{lem:guide-zero-integrated-forcing} and~\eqref{eq:guide-zero-w13-anchor-growth}
give
\begin{equation}
\begin{aligned}
w_{13}(\tau)-w_{13}(\tau_1)
&\ge\frac1T\sum_{t=1}^{T-1}
 \|\lambda p_t^{\rm GD}-P_t[\lambda p_t^{\rm GD}]\|_{d,n}^2
 \int_{\tau_1}^{\tau}(-v_{31})(v_{31}w_{13}+\eta_0)\,\mathrm dr
 -d^{-1/2+o(1)}\\
&\ge\frac{c}{T}\sum_{t=1}^{T-1}
 \|\lambda p_t^{\rm GD}-P_t[\lambda p_t^{\rm GD}]\|_{d,n}^2
 \int_{\tau_1}^{\tau}w_{13}(r)\,\mathrm dr
 -d^{-1/2+o(1)}.
\end{aligned}
\label{eq:guide-zero-w13-exponential-integral}
\end{equation}
Gronwall's inequality, using $w_{13}(\tau_1)\ge3d^{-1/4}$ and the error in
~\eqref{eq:guide-zero-w13-exponential-integral}, yields
\begin{equation}
\begin{aligned}
 w_{13}(\tau)
&\ge\left(2d^{-1/4}-d^{-1/2+o(1)}\right)
 \exp\left\{\frac{c}{T}\sum_{t=1}^{T-1}
 \|\lambda p_t^{\rm GD}-P_t[\lambda p_t^{\rm GD}]\|_{d,n}^2(\tau-\tau_1)\right\}\\
&\qquad-d^{-1/2+o(1)}.
\end{aligned}
\label{eq:guide-zero-w13-exponential-growth}
\end{equation}
Consequently, $w_{13}$ reaches $2$ after at most
\begin{equation}
 C T\left(\sum_{t=1}^{T-1}
 \|\lambda p_t^{\rm GD}-P_t[\lambda p_t^{\rm GD}]\|_{d,n}^2\right)^{-1}\log d
 =d^{o(1)}
\end{equation}
additional time.  At that time,
\begin{equation}
 -v_{31}w_{13}
 \ge w_{13}\sqrt{w_{13}^2-\sigma^2/d}
 \ge2\sqrt{4-\sigma^2/d}>\frac{3\eta_0}{4},
\end{equation}
which contradicts the assumption that~\eqref{eq:guide-zero-product-threshold} has not occurred.
Therefore~\eqref{eq:guide-zero-product-threshold} occurs at a time $\tau_*\le H/2$.

\noindent\textup{(iii) No exit.}
At $\tau_*$,~\eqref{eq:guide-spatial-invariant} gives
\begin{equation}
\begin{aligned}
\left(\frac{3\eta_0}{4}\right)^2
&=v_{31}(\tau_*)^2w_{13}(\tau_*)^2\\
&\le v_{31}(\tau_*)^2
 \left(v_{31}(\tau_*)^2+\frac{\sigma^2}{d}\right),
\end{aligned}
\end{equation}
and hence $|v_{31}(\tau_*)|\ge C^{-1}$.  By
~\eqref{eq:guide-early-energy}, for
$\tau_*\le\tau\le\tau_*+\tau_b$,
\begin{equation}
\begin{aligned}
 |v_{31}(\tau)-v_{31}(\tau_*)|
&\le\sqrt{\tau_b}\left(\int_{\tau_*}^{\tau_*+\tau_b}\dot v_{31}^{\,2}\,\mathrm dr\right)^{1/2}
 \le C\sqrt{\frac{\tau_b}{T}},\\
 |w_{13}(\tau)-w_{13}(\tau_*)|
&\le\sqrt{\tau_b}\left(\int_{\tau_*}^{\tau_*+\tau_b}\dot w_{13}^{\,2}\,\mathrm dr\right)^{1/2}
 \le C\sqrt{\frac{\tau_b}{T}},\\
 -v_{31}(\tau)w_{13}(\tau)
&=\frac{3\eta_0}{4}+o(1),
 \qquad |v_{31}(\tau)|\ge\frac1{2C}.
\end{aligned}
\label{eq:guide-zero-no-exit-short-layer}
\end{equation}
Thus the loss term minus the projection-error term in~\eqref{eq:guide-zero-dissipation} satisfies
\begin{equation}
\begin{aligned}
&\frac{\mathcal L_{\rm GD}(\tau)}d
 -\frac1{2T}\sum_{t=1}^{T-1}
 \|\lambda p_t^{\rm GD}-P_t[\lambda p_t^{\rm GD}]\|_{d,n}^2
 (v_{31}w_{13}+\eta_0)^2\\
&\ge\frac{\mathcal L_{\rm GD}(\tau)}d
 -\left(\frac{\eta_0^2}{32}+o(1)\right)
 \frac1T\sum_{t=1}^{T-1}
 \|\lambda p_t^{\rm GD}-P_t[\lambda p_t^{\rm GD}]\|_{d,n}^2.
\end{aligned}
\label{eq:guide-zero-no-exit-terms}
\end{equation}
Since $|v_{31}|\ge1/(2C)$ on this interval,~\eqref{eq:guide-zero-dissipation} and
~\eqref{eq:guide-zero-no-exit-terms} give
\begin{equation}
\begin{aligned}
-\frac{\mathrm d}{\mathrm d\tau}\frac{\mathcal L_{\rm GD}}d
&\ge c\exp\{-CT\log(e+T)\}(d/n)^{CT}d\\
&\quad\times\left[\frac{\mathcal L_{\rm GD}}d
 -\left(\frac{\eta_0^2}{32}+o(1)\right)
 \frac1T\sum_{t=1}^{T-1}
 \|\lambda p_t^{\rm GD}-P_t[\lambda p_t^{\rm GD}]\|_{d,n}^2\right].
\end{aligned}
\label{eq:guide-zero-no-exit-dissipation}
\end{equation}
Let $\tau_0=\tau_*+\tau_b$.  Integrating the preceding linear inequality gives
\begin{equation}
\begin{aligned}
\frac{\mathcal L_{\rm GD}(\tau_0)}d
&\le\left(\frac{\eta_0^2}{32}+o(1)\right)
 \frac1T\sum_{t=1}^{T-1}
 \|\lambda p_t^{\rm GD}-P_t[\lambda p_t^{\rm GD}]\|_{d,n}^2\\
&\quad+\exp\left\{-c d\tau_b\exp\{-CT\log(e+T)\}(d/n)^{CT}\right\}
 \frac{\mathcal L_{\rm GD}(\tau_*)}d\\
&\le\left(\frac{\eta_0^2}{32}+o(1)\right)
 \frac1T\sum_{t=1}^{T-1}
 \|\lambda p_t^{\rm GD}-P_t[\lambda p_t^{\rm GD}]\|_{d,n}^2\\
&\quad+o\left(\frac1T\sum_{t=1}^{T-1}
 \|\lambda p_t^{\rm GD}-P_t[\lambda p_t^{\rm GD}]\|_{d,n}^2\right)\\
&<\frac{9\eta_0^2}{32T}
 \sum_{t=1}^{T-1}\|\lambda p_t^{\rm GD}-P_t[\lambda p_t^{\rm GD}]\|_{d,n}^2.
\end{aligned}
\label{eq:guide-zero-no-exit-loss}
\end{equation}
Here $d\tau_b\exp\{-CT\log(e+T)\}(d/n)^{CT}\ge20\log d$ for large $d$,
by the definition of $\tau_b$.
Since
\begin{equation}
 \mathbb E B(\Sb)=\frac d{2T}\sum_{t=1}^{T-1}
 \|\lambda p_t^{\rm GD}-P_t[\lambda p_t^{\rm GD}]\|_{d,n}^2,
\end{equation}
~\eqref{eq:guide-zero-no-exit-loss} gives
\begin{equation}
 \mathcal L_{\rm GD}(\tau_0)
 <\frac{9\eta_0^2}{16}\mathbb E B(\Sb).
\label{eq:guide-zero-no-exit-loss-raw}
\end{equation}
The first term of~\eqref{eq:guide-loss-decomposition} then gives
\begin{equation}
\begin{aligned}
\left|v_{31}(w_{24}+\wb_{25}\eb_0)-\eta_0\right|^2
&\le\frac{2T}{\|\lambda\|_{d,n}^2}
 \frac{\mathcal L_{\rm GD}(\tau_0)}d\\
&<\frac{9\eta_0^2}{16}
 \frac{\sum_{t=1}^{T-1}\|\lambda p_t^{\rm GD}-P_t[\lambda p_t^{\rm GD}]\|_{d,n}^2}
 {\|\lambda\|_{d,n}^2}\\
&\le\frac{9\eta_0^2}{16}\frac{2d/n}{1}
 <\frac{\eta_0^2}{4},
\end{aligned}
\label{eq:guide-zero-no-exit-first-term}
\end{equation}
The second inequality uses~\eqref{eq:guide-zero-spectral-scales}, and the last uses
$d/n\le1/10$.  Since $v_{31}<0$ by~\eqref{eq:guide-zero-v31-barrier},
\begin{equation}
\begin{aligned}
 w_{13}(\tau_0)&=\frac{-v_{31}(\tau_0)w_{13}(\tau_0)}{-v_{31}(\tau_0)}>0,\\
 v_{31}(\tau_0)(w_{24}+\wb_{25}\eb_0)
 &\ge\eta_0-\left|v_{31}(w_{24}+\wb_{25}\eb_0)-\eta_0\right|\ge\frac{\eta_0}{2},\\
 \widetilde\theta_2(\tau_0)=v_{31}(\tau_0)w_{13}(\tau_0)
 &=-\frac{3\eta_0}{4}+o(1)\le-\frac{\eta_0}{2}.
\end{aligned}
\label{eq:guide-zero-no-exit-anchor}
\end{equation}

If $-v_{31}w_{13}\le\eta_0/4$ at a later time, then
\begin{equation}
\begin{aligned}
 v_{31}w_{13}+\eta_0
 &=\eta_0-(-v_{31}w_{13})\ge\frac{3\eta_0}{4},\\
 \mathcal L_{\rm GD}
&\ge\frac d{2T}\sum_{t=1}^{T-1}
 \|\lambda p_t^{\rm GD}-P_t[\lambda p_t^{\rm GD}]\|_{d,n}^2
 (v_{31}w_{13}+\eta_0)^2\\
&\ge\frac{9\eta_0^2}{16}\mathbb E B(\Sb)
 >\mathcal L_{\rm GD}(\tau_0),
\end{aligned}
\label{eq:guide-zero-no-return-loss}
\end{equation}
contradicting loss monotonicity.  Therefore, for every $\tau\ge\tau_0$,
\begin{equation}
\begin{aligned}
-v_{31}(\tau)w_{13}(\tau)&>\frac{\eta_0}{4},\\
\left(\frac{\eta_0}{4}\right)^2
&<v_{31}(\tau)^2w_{13}(\tau)^2\\
&\le v_{31}(\tau)^2
 \left(v_{31}(\tau)^2+\frac{\sigma^2}{d}\right),\\
|v_{31}(\tau)|&\ge C^{-1}>0.
\end{aligned}
\label{eq:guide-zero-no-return-lower}
\end{equation}
This proves~\eqref{eq:guide-zero-entry} and
~\eqref{eq:guide-zero-no-return}.
\end{proof}

\subsection{Proof of Lemma~\ref{lem:auxiliary-pl}}

\begin{proof}
The first term in~\eqref{eq:guide-loss-decomposition} depends only on
\(\widetilde\theta_1\).  The moment estimate for \(A(\Sb)\) gives
\[
 |\partial_{\widetilde\theta_1}\mathcal L_{\rm GD}|^2
 \ge c\frac{d/n}{T}\mathcal L_{\rm GD,A}.
\]
For the remaining two terms, the bound~\eqref{eq:guide-auxiliary-nondegenerate}
and the decomposition~\eqref{eq:guide-loss-decomposition} give
\begin{equation*}
 \frac1d|\partial_{\widetilde\theta_2}\mathcal L_{\rm GD}|^2
 +\|\nabla_{\widetilde{\bTheta}_3}\mathcal L_{\rm GD}\|_F^2
 \ge c\frac{d/n}{T}
 \bigl(\mathcal L_{\rm GD,B}+\mathcal L_{\rm GD,C}\bigr).
\end{equation*}

Combining the two estimates proves
\[
 |\partial_{\widetilde\theta_1}\mathcal L_{\rm GD}|^2
 +\frac1d|\partial_{\widetilde\theta_2}\mathcal L_{\rm GD}|^2
 +\|\nabla_{\widetilde{\bTheta}_3}\mathcal L_{\rm GD}\|_F^2
 \ge c\frac{d/n}{T}\mathcal L_{\rm GD}.
\]
\end{proof}

\subsection{Proof of Lemma~\ref{lem:stage1-jacobian}}

\begin{proof}
Lemmas~\ref{lem:guide-zero-entry} and~\ref{lem:guide-invariants} give
$|v_{31}|\ge c$ and $v_{33}\ge c$ for $\tau\ge\tau_0$.
Use the active raw blocks $\bar{\btheta}$ from
Lemma~\ref{lem:stage1-invariant}, with the Frobenius norm on matrix blocks.
For the auxiliary map in Lemma~\ref{lem:stage1-loss},
\begin{equation*}
\begin{aligned}
 &\Jb_1=\nabla_{\bar{\btheta}}\widetilde{\btheta},\qquad
 \delta\widetilde{\btheta}=\Jb_1\delta\bar{\btheta},\qquad
 \widetilde{\btheta}=(\widetilde\theta_1,\widetilde\theta_2,
 \widetilde{\bTheta}_3).
\end{aligned}
\end{equation*}
We construct a right inverse $R$ by specifying
$\delta\bar{\btheta}=R\delta\widetilde{\btheta}$:
\begin{equation*}
\begin{aligned}
 &\delta\Vb_{31}=\delta\Vb_{33}=\zero,\qquad
 \delta w_{24}=\delta w_{44}=0,\qquad
 \delta\wb_{45}=\delta\wb_{54}=\zero,\\
 &\delta\Wb_{13}=\frac{\delta\widetilde\theta_2}{v_{31}}\Ib_d,
 \qquad \delta(\wb_{25})_0=\frac{\delta\widetilde\theta_1}{v_{31}},\\
 &\delta(\wb_{25})_t
 =\frac{\delta(\widetilde{\bTheta}_3)_{1t}}{v_{31}},
 \qquad 1\le t<T,\\
 &\delta(\Wb_{55})_{\ell t}
 =\begin{cases}
  \delta(\widetilde{\bTheta}_3)_{\ell t}/v_{33},
                                  &2\le\ell\le t<T,\\
  0,&\text{otherwise}.
 \end{cases}
\end{aligned}
\end{equation*}
Since $\delta\Wb_{13}=\delta w_{13}\Ib_d$ and
$\delta v_{31}=\delta v_{33}=0$, differentiating the two scalar products gives
\begin{equation*}
\begin{aligned}
 &\langle\nabla_{\bar{\btheta}}\widetilde\theta_1,
       \delta\bar{\btheta}\rangle
 =(w_{24}+\wb_{25}\eb_0)\delta v_{31}
   +v_{31}(\delta w_{24}+\delta\wb_{25}\eb_0)\\
 &\quad=v_{31}\delta(\wb_{25})_0
 =v_{31}\frac{\delta\widetilde\theta_1}{v_{31}}
 =\delta\widetilde\theta_1,\\
 &\langle\nabla_{\bar{\btheta}}\widetilde\theta_2,
       \delta\bar{\btheta}\rangle
 =w_{13}\delta v_{31}+v_{31}\delta w_{13}
 =v_{31}\frac{\delta\widetilde\theta_2}{v_{31}}
 =\delta\widetilde\theta_2.
\end{aligned}
\end{equation*}
For the first row and the remaining entries, respectively, the same definitions give
\begin{equation*}
\begin{aligned}
 &\langle\nabla_{\bar{\btheta}}(\widetilde{\bTheta}_3)_{1t},
       \delta\bar{\btheta}\rangle
 =v_{31}\delta(\wb_{25})_t+\eta_0v_{33}\delta(\Wb_{55})_{1t}\\
 &\quad=v_{31}\frac{\delta(\widetilde{\bTheta}_3)_{1t}}{v_{31}}+0
 =\delta(\widetilde{\bTheta}_3)_{1t},\qquad 1\le t<T,\\
 &\langle\nabla_{\bar{\btheta}}(\widetilde{\bTheta}_3)_{\ell t},
       \delta\bar{\btheta}\rangle\\
 &\quad=v_{33}\bigl(\delta w_{44}+\delta\wb_{45}\eb_t
       +\eb_\ell^{\mathsf T}\delta\wb_{54}
       +\eb_\ell^{\mathsf T}\delta\Wb_{55}\eb_t\bigr)\\
 &\quad=v_{33}\delta(\Wb_{55})_{\ell t}
 =v_{33}\frac{\delta(\widetilde{\bTheta}_3)_{\ell t}}{v_{33}}
 =\delta(\widetilde{\bTheta}_3)_{\ell t},\qquad 2\le\ell\le t<T.
\end{aligned}
\end{equation*}
Thus, for every auxiliary variation,
\begin{equation*}
 \Jb_1R\delta\widetilde{\btheta}
 =\delta\widetilde{\btheta},\qquad \Jb_1R=\Ib.
\end{equation*}
Moreover, $\|\Ib_d\|_F^2=d$ and the lower bounds on
$|v_{31}|,v_{33}$ give
\begin{equation*}
\begin{aligned}
 &\|R\delta\widetilde{\btheta}\|^2
 =\|\delta\Wb_{13}\|_F^2
   +\|\delta\wb_{25}\|^2+\|\delta\Wb_{55}\|_F^2\\
 &=\frac{d}{v_{31}^2}|\delta\widetilde\theta_2|^2
   +\frac1{v_{31}^2}\left[|\delta\widetilde\theta_1|^2
      +\sum_{t=1}^{T-1}|\delta(\widetilde{\bTheta}_3)_{1t}|^2\right]\\
 &\quad+\frac1{v_{33}^2}\sum_{t=1}^{T-1}\sum_{\ell=2}^{t}
       |\delta(\widetilde{\bTheta}_3)_{\ell t}|^2\\
 &\le C\left[|\delta\widetilde\theta_1|^2
       +d|\delta\widetilde\theta_2|^2
       +\|\delta\widetilde{\bTheta}_3\|_F^2\right].
\end{aligned}
\end{equation*}
Equivalently, $\|R\,\diag(1,d^{-1/2},\Ib)\|_{\rm op}^2\le C$.
To obtain the matrix lower bound, use $\Jb_1R=\Ib$ and Cauchy--Schwarz:
\begin{equation*}
\begin{aligned}
 &|\delta\widetilde\theta_1|^2+d^{-1}|\delta\widetilde\theta_2|^2
       +\|\delta\widetilde{\bTheta}_3\|_F^2\\
 &=\left\langle\delta\widetilde{\btheta},
       \Jb_1R\bigl(\delta\widetilde\theta_1,
       d^{-1}\delta\widetilde\theta_2,
       \delta\widetilde{\bTheta}_3\bigr)\right\rangle\\
 &=\left\langle\Jb_1^{\mathsf T}\delta\widetilde{\btheta},
       R\bigl(\delta\widetilde\theta_1,
       d^{-1}\delta\widetilde\theta_2,
       \delta\widetilde{\bTheta}_3\bigr)\right\rangle\\
 &\le\|\Jb_1^{\mathsf T}\delta\widetilde{\btheta}\|
       \left\|R\bigl(\delta\widetilde\theta_1,
       d^{-1}\delta\widetilde\theta_2,
       \delta\widetilde{\bTheta}_3\bigr)\right\|\\
 &\le\sqrt C\,\|\Jb_1^{\mathsf T}\delta\widetilde{\btheta}\|
       \left[|\delta\widetilde\theta_1|^2
       +d^{-1}|\delta\widetilde\theta_2|^2
       +\|\delta\widetilde{\bTheta}_3\|_F^2\right]^{1/2}.
\end{aligned}
\end{equation*}
Consequently,
\begin{equation*}
\begin{aligned}
 &(\delta\widetilde{\btheta})^{\mathsf T}
       \Jb_1\Jb_1^{\mathsf T}\delta\widetilde{\btheta}
 =\|\Jb_1^{\mathsf T}\delta\widetilde{\btheta}\|^2\\
 &\ge C^{-1}\left[|\delta\widetilde\theta_1|^2
       +d^{-1}|\delta\widetilde\theta_2|^2
       +\|\delta\widetilde{\bTheta}_3\|_F^2\right],\\
 &\Jb_1\Jb_1^{\mathsf T}\succeq c\,\diag(1,d^{-1},\Ib).
\end{aligned}
\end{equation*}
\end{proof}

\begin{lemma}
\label{lem:guide-PL}
For \(\tau\ge\tau_0\),
\begin{equation}
 -\frac{\mathrm d}{\mathrm d\tau}\mathcal L_{\rm GD}
 \ge c\frac{d/n}{T}\mathcal L_{\rm GD}.
 \label{eq:guide-PL-raw}
\end{equation}
\end{lemma}
\begin{proof}
By the chain rule and Lemmas~\ref{lem:auxiliary-pl}--\ref{lem:stage1-jacobian},
\[
 -\frac{\mathrm d}{\mathrm d\tau}\mathcal L_{\rm GD}
 =\|\nabla_{\bar{\btheta}}\mathcal L_{\rm GD}\|^2
 \ge c\left[
 |\partial_{\widetilde\theta_1}\mathcal L_{\rm GD}|^2
 +d^{-1}|\partial_{\widetilde\theta_2}\mathcal L_{\rm GD}|^2
 +\|\nabla_{\widetilde{\bTheta}_3}\mathcal L_{\rm GD}\|_F^2\right]
 \ge c\frac{d/n}{T}\mathcal L_{\rm GD}.
\]
\end{proof}

\begin{proof}[Proof of Theorem~\ref{thm:main-result}]
Solving Lemma~\ref{lem:guide-PL} with~\eqref{eq:guide-early-initial-loss-raw}
and $\tau_0=d^{o(1)}$ from Lemma~\ref{lem:guide-zero-entry} gives
\begin{equation*}
 \mathcal L_{\rm GD}(\tau)
 \le\frac{Cd}{T}
 \exp\!\left[-c\frac{d/n}{T}(\tau-\tau_0)\right],
 \qquad\tau\ge\tau_0.
\end{equation*}
Thus $\mathcal L_{\rm GD}\le\epsilon_{\rm GD}=d^{-1/4}$ is reached within
$d^{o(1)}+O(T\log d)$ time.
\end{proof}

\section{Handoff from Stage 1 to Stage 2}
\label{app:handoff}

Under \eqref{eq:common-growing-regime}, Lemma~\ref{lem:guide-PL} gives a
Stage-1 time \(\tau_1=d^{o(1)}+O(T\log d)\) at which
\(\mathcal L_{\rm GD}(\tau_1)\le d^{-1/4}\).

\subsection{The Stage-1 endpoint}
\begin{lemma}
\label{lem:guide-coefficient-control}
At every Stage-1 point,
\begin{align}
 &|v_{31}(w_{24}+\wb_{25}\eb_0)-\eta_0|^2
 +|v_{31}w_{13}+\eta_0|^2  +\sum_{t=1}^{T-1}
 |(\widetilde{\bTheta}_3)_{1t}-\eta_0|^2
 +\sum_{t=1}^{T-1}\sum_{\ell=2}^{t}|(\widetilde{\bTheta}_3)_{\ell t}|^2\notag\\
 &\qquad\qquad\qquad\qquad\qquad\qquad\qquad \le
 \exp\{CT\log(e+T)\}(n/d)^{CT}
 \frac{\mathcal L_{\rm GD}}d
 =d^{o(1)}\frac{\mathcal L_{\rm GD}}d .              \label{eq:guide-coefficient-bound}
\end{align}
\end{lemma}

\begin{proof}
By \eqref{eq:guide-auxiliary-loss},
\[
 \widetilde\theta_1-\eta_0=[\lambda]A_0,
 \qquad \widetilde\theta_2+\eta_0=\eta_0^{-1}[\lambda^2]A_1.
\]
For \(1\le t<T\), the same identity gives
\[
 A_t-(\widetilde\theta_2+\eta_0)\lambda p_t^{\rm GD}
 =\bigl((\widetilde{\bTheta}_3)_{1t}-\eta_0\bigr)\lambda
 +\sum_{\ell=2}^t(\widetilde{\bTheta}_3)_{\ell t}p_\ell^{\rm GD}.
\]
The coefficient inversion proved in Lemma~\ref{lem:guide-auxiliary-nondegenerate}
therefore yields
\begin{align*}
 &| (\widetilde{\bTheta}_3)_{1t}-\eta_0|^2
 +\sum_{\ell=2}^t|(\widetilde{\bTheta}_3)_{\ell t}|^2\\
 &\quad\le t^2 25^t\sum_{k=1}^t
 \left|[\lambda^k]\bigl(A_t-(\widetilde\theta_2+\eta_0)
 \lambda p_t^{\rm GD}\bigr)\right|^2\\
 &\quad\le 2t^2 25^t\left(
 \sum_{k=1}^{t+1}|[\lambda^k]A_t|^2
 +4^t|\widetilde\theta_2+\eta_0|^2\right).
\end{align*}
(Here \(\sum_{k=1}^t|[\lambda^k](\lambda p_t^{\rm GD})|^2
 \le\sum_{k=0}^t\binom tk^2\le4^t\), since \(\eta_0\le1\).)
Summing over \(t\) and using the first two identities gives
\begin{align*}
 &|\widetilde\theta_1-\eta_0|^2+|\widetilde\theta_2+\eta_0|^2
 +\sum_{t=1}^{T-1}|(\widetilde{\bTheta}_3)_{1t}-\eta_0|^2
 +\sum_{t=1}^{T-1}\sum_{\ell=2}^t|(\widetilde{\bTheta}_3)_{\ell t}|^2\\
 &\quad\le\exp\{CT\log(e+T)\}
 \sum_{t=0}^{T-1}\sum_{k=1}^{t+1}|[\lambda^k]A_t|^2.
\end{align*}
Finally, \eqref{eq:finite-wishart-moment-lower} in
Lemma~\ref{lem:finite-wishart-moment-lower} implies
\begin{align*}
 \sum_{t=0}^{T-1}\sum_{k=1}^{t+1}|[\lambda^k]A_t|^2
 &\le 2(T+1)16^T(n/d)^T\frac1d
 \sum_{t=0}^{T-1}\E\tr\bigl(A_t(\Sb)^2\bigr)\\
 &=4T(T+1)16^T(n/d)^T\frac{\mathcal L_{\rm GD}}d.
\end{align*}
Substitution proves \eqref{eq:guide-coefficient-bound}, with
\(\exp\{CT\log(e+T)\}(n/d)^{CT}=d^{o(1)}\) by
\eqref{eq:common-growing-regime}.
\end{proof}

\begin{lemma}
\label{lem:guide-handoff-raw-bound}
At any Stage-1 time satisfying
\(\mathcal L_{\rm GD}\le d^{-1/4}\),
\begin{align}
 |v_{31}|+|w_{13}|+|v_{33}|&\le d^{o(1)},\notag\\
 |w_{24}|+\|\wb_{25}\|
 +\sum_{B\in\{44,45,54,55\}}\|\Wb_B\|_F&\le d^{o(1)},\notag\\
 \max_{1\le\ell\le t<T}
 v_{33}|w_{44}+\wb_{45}\eb_t+\eb_\ell^{\mathsf T}\wb_{54}
 +\eb_\ell^{\mathsf T}\Wb_{55}\eb_t|&\le d^{o(1)} .
                                                               \label{eq:guide-first-hitting-raw-bound}
\end{align}
\end{lemma}

\begin{proof}
First consider the first time the loss reaches \(d^{-1/4}\).
Lemmas~\ref{lem:guide-zero-entry} and~\ref{lem:guide-PL} bound this time by
\[
 d^{o(1)}+C\frac{T}{d/n}\log(d^{5/4})=d^{o(1)}.
\]
The scalar energy identity gives, up to this time,
\[
 \sup_\tau(|v_{31}|+|w_{13}|+|v_{33}|)
 +\int(|\dot v_{31}|+|\dot w_{13}|+|\dot v_{33}|)\,\mathrm d\tau
 \le d^{o(1)},\qquad v_{33}\ge v_{33}(0)>0.
\]
For \(1\le\ell\le t<T\), the two coefficient maps are
\begin{equation}
\begin{aligned}
 &\bigl[\ind_{\{\ell=1\}}(1\ \ \eb_t^{\mathsf T})\bigr]_{\ell,t}
       \begin{pmatrix}w_{24}\\\wb_{25}^{\mathsf T}\end{pmatrix}
 =\bigl(\ind_{\{\ell=1\}}(w_{24}+\wb_{25}\eb_t)\bigr)_{\ell,t},\\
 &\bigl[(1\ \ \eb_t^{\mathsf T}\ \ \eb_\ell^{\mathsf T}\
                    (\eb_t\otimes\eb_\ell)^{\mathsf T})\bigr]_{\ell,t}
       \begin{pmatrix}w_{44}\\\wb_{45}^{\mathsf T}\\
                       \wb_{54}\\\operatorname{vec}(\Wb_{55})\end{pmatrix}\\
 &\qquad=\bigl(w_{44}+\wb_{45}\eb_t+\eb_\ell^{\mathsf T}\wb_{54}
                       +\eb_\ell^{\mathsf T}\Wb_{55}\eb_t\bigr)_{\ell,t}.
\end{aligned}
\label{eq:raw-fixed-coefficient-maps}
\end{equation}
Their second row Gram matrix satisfies
\[
 \bigl[(1+\ind_{\{\ell=j\}})(1+\ind_{\{t=s\}})\bigr]_{(\ell,t),(j,s)}
 \succeq\Ib,\qquad
 \left\|\bigl[(1+\ind_{\{\ell=j\}})(1+\ind_{\{t=s\}})\bigr]\right\|_2
 \le CT^2.
\]
Define the fixed matrix
\begin{equation}
\begin{aligned}
 \Kb:={}&
 \left[\Ib-\frac12\begin{pmatrix}1\\\eb_0\end{pmatrix}
                    (1\ \ \eb_0^{\mathsf T})\right]
 \bigl[\ind_{\{\ell=1\}}\begin{pmatrix}1\\\eb_t\end{pmatrix}
                         \bigr]_{\ell,t}\\
 &\quad\cdot
 \bigl[(1+\ind_{\{\ell=j\}})(1+\ind_{\{t=s\}})\bigr]_{(\ell,t),(j,s)}^{-1}\\
 &\quad\cdot
 \bigl[(1\ \ \eb_s^{\mathsf T}\ \ \eb_j^{\mathsf T}\
                    (\eb_s\otimes\eb_j)^{\mathsf T})\bigr]_{j,s},
 \qquad \|\Kb\|_2\le CT^2.
\end{aligned}
\label{eq:raw-block-matrix}
\end{equation}
The chain rule  of
\(w_{44},\wb_{45},\wb_{54},\Wb_{55}\) give
\begin{align}
 &\begin{pmatrix}w_{44}\\\wb_{45}^{\mathsf T}\\
                       \wb_{54}\\\operatorname{vec}(\Wb_{55})\end{pmatrix}
 \in\operatorname{range}
 \bigl[(1\ \ \eb_t^{\mathsf T}\ \ \eb_\ell^{\mathsf T}\
                    (\eb_t\otimes\eb_\ell)^{\mathsf T})\bigr]_{\ell,t}^{\mathsf T},
 \notag\\
 &\sum_{B\in\{44,45,54,55\}}\|\Wb_B\|_F^2
 \le\sum_{1\le\ell\le t<T}
 \left|w_{44}+\wb_{45}\eb_t+\eb_\ell^{\mathsf T}\wb_{54}
                       +\eb_\ell^{\mathsf T}\Wb_{55}\eb_t\right|^2,
 \label{eq:raw-four-block-norm}\\
 &\left[\Ib-\frac12\begin{pmatrix}1\\\eb_0\end{pmatrix}
                    (1\ \ \eb_0^{\mathsf T})\right]
 \begin{pmatrix}\dot w_{24}\\\dot\wb_{25}^{\mathsf T}\end{pmatrix}
 =\frac{v_{31}}{\eta_0v_{33}}\Kb
 \begin{pmatrix}\dot w_{44}\\\dot\wb_{45}^{\mathsf T}\\
                \dot\wb_{54}\\\operatorname{vec}(\dot\Wb_{55})\end{pmatrix}.
 \label{eq:raw-block-velocity}
\end{align}
By the definition of $\widetilde{\bTheta}_3$ in Lemma~\ref{lem:stage1-loss} and
\(v_{31}(w_{24}+\wb_{25}\eb_0)=\widetilde\theta_1\),
\begin{equation}
\begin{aligned}
 &\Kb\begin{pmatrix}w_{44}\\\wb_{45}^{\mathsf T}\\
                       \wb_{54}\\\operatorname{vec}(\Wb_{55})\end{pmatrix}
 +\frac{v_{31}}{\eta_0v_{33}}\Kb\Kb^{\mathsf T}
 \left[\Ib-\frac12\begin{pmatrix}1\\\eb_0\end{pmatrix}
                    (1\ \ \eb_0^{\mathsf T})\right]
 \begin{pmatrix}w_{24}\\\wb_{25}^{\mathsf T}\end{pmatrix}\\
 &\quad=\Kb
 \begin{pmatrix}
 \zero_{2T+1}\\
 \displaystyle\operatorname{vec}\!\left(
 \left[\frac{(\widetilde{\bTheta}_3)_{\ell t}
                  -\frac12\widetilde\theta_1\ind_{\{\ell=1\}}}
              {\eta_0^{\ind_{\{\ell=1\}}}v_{33}}\right]_{\ell,t}\right)
 \end{pmatrix}.
\end{aligned}
\label{eq:raw-block-affine}
\end{equation}
Equation~\eqref{eq:guide-coefficient-bound} implies
\[
 \left\|\Kb
 \begin{pmatrix}
 \zero_{2T+1}\\
 \displaystyle\operatorname{vec}\!\left(
 \left[\frac{(\widetilde{\bTheta}_3)_{\ell t}
                  -\frac12\widetilde\theta_1\ind_{\{\ell=1\}}}
              {\eta_0^{\ind_{\{\ell=1\}}}v_{33}}\right]_{\ell,t}\right)
 \end{pmatrix}\right\|_2\le d^{o(1)}.
\]
Differentiating \eqref{eq:raw-block-affine} and using
\eqref{eq:raw-block-velocity} yields
\begin{align*}
 &\frac{\mathrm d}{\mathrm d\tau}
 \left\{\left(\Ib+\frac{v_{31}^2}{\eta_0^2v_{33}^2}
                     \Kb\Kb^{\mathsf T}\right)^{1/2}
 \left[\Ib-\frac12\begin{pmatrix}1\\\eb_0\end{pmatrix}
                    (1\ \ \eb_0^{\mathsf T})\right]
 \begin{pmatrix}w_{24}\\\wb_{25}^{\mathsf T}\end{pmatrix}\right\}\\
 &\qquad=\frac{v_{31}}{\eta_0v_{33}}
 \left(\Ib+\frac{v_{31}^2}{\eta_0^2v_{33}^2}
                     \Kb\Kb^{\mathsf T}\right)^{-1/2}
 \Kb\frac{\mathrm d}{\mathrm d\tau}
 \begin{pmatrix}
 \zero_{2T+1}\\
 \displaystyle\operatorname{vec}\!\left(
 \left[\frac{(\widetilde{\bTheta}_3)_{\ell t}
                  -\frac12\widetilde\theta_1\ind_{\{\ell=1\}}}
              {\eta_0^{\ind_{\{\ell=1\}}}v_{33}}\right]_{\ell,t}\right)
 \end{pmatrix}.
\end{align*}
Integration by parts gives
\begin{equation}
\begin{aligned}
 &\left[
 \left(\Ib+\frac{v_{31}^2}{\eta_0^2v_{33}^2}
                     \Kb\Kb^{\mathsf T}\right)^{1/2}
 \left[\Ib-\frac12\begin{pmatrix}1\\\eb_0\end{pmatrix}
                    (1\ \ \eb_0^{\mathsf T})\right]
 \begin{pmatrix}w_{24}\\\wb_{25}^{\mathsf T}\end{pmatrix}\right.\\
 &\qquad\left.
 -\frac{v_{31}}{\eta_0v_{33}}
 \left(\Ib+\frac{v_{31}^2}{\eta_0^2v_{33}^2}
                     \Kb\Kb^{\mathsf T}\right)^{-1/2}
 \Kb\begin{pmatrix}
 \zero_{2T+1}\\
 \displaystyle\operatorname{vec}\!\left(
 \left[\frac{(\widetilde{\bTheta}_3)_{\ell t}
                  -\frac12\widetilde\theta_1\ind_{\{\ell=1\}}}
              {\eta_0^{\ind_{\{\ell=1\}}}v_{33}}\right]_{\ell,t}\right)
 \end{pmatrix}\right]_0^\tau\\
 &\quad=-\int_0^\tau
 \left(\frac{\mathrm d}{\mathrm ds}\frac{v_{31}}{\eta_0v_{33}}\right)
 \left(\Ib+\frac{v_{31}^2}{\eta_0^2v_{33}^2}
                     \Kb\Kb^{\mathsf T}\right)^{-3/2}
 \Kb\begin{pmatrix}
 \zero_{2T+1}\\
 \displaystyle\operatorname{vec}\!\left(
 \left[\frac{(\widetilde{\bTheta}_3)_{\ell t}
                  -\frac12\widetilde\theta_1\ind_{\{\ell=1\}}}
              {\eta_0^{\ind_{\{\ell=1\}}}v_{33}}\right]_{\ell,t}\right)
 \end{pmatrix}\,\mathrm ds .
\end{aligned}
\label{eq:raw-block-integration-by-parts}
\end{equation}
Since
\[
 \left\|\left(\Ib+\frac{v_{31}^2}{\eta_0^2v_{33}^2}
                 \Kb\Kb^{\mathsf T}\right)^{-\alpha}\right\|_2\le1
 \quad(\alpha\ge0),\qquad
 \sup_\tau\left|\frac{v_{31}}{\eta_0v_{33}}\right|
 +\int\left|\frac{\mathrm d}{\mathrm d\tau}
                  \frac{v_{31}}{\eta_0v_{33}}\right|\,\mathrm d\tau
 \le d^{o(1)},
\]
we obtain
\begin{equation}
 \sup_\tau\left\|
 \left[\Ib-\frac12\begin{pmatrix}1\\\eb_0\end{pmatrix}
                    (1\ \ \eb_0^{\mathsf T})\right]
 \begin{pmatrix}w_{24}(\tau)\\\wb_{25}(\tau)^{\mathsf T}\end{pmatrix}
 \right\|_2\le d^{o(1)}.
\label{eq:raw-cancellation-primitive}
\end{equation}
At the first hitting time, \eqref{eq:guide-coefficient-bound} and
\eqref{eq:guide-spatial-invariant} give
\[
 |v_{31}|\ge c,\qquad
 \left\|\frac{\widetilde\theta_1}{2v_{31}}
              \begin{pmatrix}1\\\eb_0\end{pmatrix}\right\|_2\le d^{o(1)}.
\]
Thus \eqref{eq:raw-four-block-norm} and the definition of $\widetilde{\bTheta}_3$ in Lemma~\ref{lem:stage1-loss} imply
\[
 |w_{24}|+\|\wb_{25}\|_2
 +\sum_{B\in\{44,45,54,55\}}\|\Wb_B\|_F\le d^{o(1)}.
\]
In fact, the coefficient
estimate and the two identities in \eqref{eq:guide-spatial-invariant}--\eqref{eq:guide-history-invariant}
also give, at this endpoint,
\begin{equation}
 v_{31}^2=\eta_0+o(1),\qquad w_{13}^2=\eta_0+o(1),\qquad
 v_{33}^2=v_{33}(0)^2+d^{-1+o(1)}.
 \label{eq:guide-balanced-endpoint}
\end{equation}
The signs are \(v_{31}<0<w_{13}\).

For any subsequent time \(\tau\), the definition of \eqref{eq:outcome-gradient-flow} and
Lemma~\ref{lem:guide-PL} give
\[
 \int_\tau^\infty
 \left(\|\dot \Vb(s)\|_F^2+\|\dot \Wb(s)\|_F^2\right)^{1/2}\,\mathrm ds
 \le C\sqrt{\frac{T}{d/n}\mathcal L_{\rm GD}(\tau)}
 \le d^{-1/8+o(1)}.
\]
This follows by integrating
\(-\mathcal L_{\rm GD}'/\sqrt{\mathcal L_{\rm GD}}\) and using
Lemma~\ref{lem:guide-PL}.  Thus the raw bounds remain valid at every later
Stage-1 time, and the balance identities again give
\eqref{eq:guide-balanced-endpoint}.
\end{proof}

\subsection{The Stage-2 initial loss}
The auxiliary parameters
\(\widetilde\theta_1,\widetilde\theta_2,\widetilde{\bTheta}_4\) and the
rollout polynomials \(p_t\) are those in Lemma~\ref{lem:stage2-lemma1}.
The following lemma compares this autoregressive rollout with the GD
teacher trajectory.

\begin{lemma}
\label{lem:direct-stage2-entry}
Let \(\tau_1\) be a Stage-1 time at which
\(\mathcal L_{\rm GD}(\tau_1)\le d^{-1/4}\), and keep exactly the same raw
parameters as the initial parameters of Stage~2, resetting only the time
origin.  Then \(\Wb_{33}=\zero_{d\times d}\),
scalar-identity invariance is preserved, and
\(v_{33}(\tau_1)\ge v_{33}(0)>0\). At this endpoint,
\begin{equation}
\begin{aligned}
 &(\widetilde\theta_1-\eta_0)^2+(\widetilde\theta_2+\eta_0)^2\\
 &\quad+\sum_{t=1}^{T-1}
 \left|\widetilde\theta_1(\widetilde{\bTheta}_4)_{1t}-\eta_0\right|^2
 +\sum_{t=1}^{T-1}\sum_{\ell=2}^{t}(\widetilde{\bTheta}_4)_{\ell t}^2
 \le d^{-1/4+o(1)}.
\end{aligned}
\label{eq:direct-stage2-coefficients}
\end{equation}
The autoregressively generated polynomials from the unchanged raw endpoint
therefore satisfy
\begin{equation}
 \frac1{2Td}\sum_{t=1}^{T}\E\,\tr\!\left(
 \left[p_t(\Sb)-\{\Ib-(\Ib-\eta_0\Sb)^t\}\right]^2\right)
 \le d^{-1/4+o(1)}.                                  \label{eq:direct-stage2-rollout}
\end{equation}
Consequently its Stage-2 outcome loss obeys
\begin{equation}
 \sqrt{\frac{\mathcal L(0)}d}
 \le
 \left\{\frac1{2Td}\sum_{t=1}^{T}
 \E\,\tr\!\left[(\Ib-\eta_0\Sb)^{2t}\right]\right\}^{1/2}
 +d^{-1/8+o(1)}.
\label{eq:direct-stage2-loss}
\end{equation}
\end{lemma}
\begin{proof}
Lemma~\ref{lem:guide-coefficient-control}
and \(\mathcal L_{\rm GD}(\tau_1)\le d^{-1/4}\) give
\begin{align}
 &(\widetilde\theta_1-\eta_0)^2+(\widetilde\theta_2+\eta_0)^2\notag\\
 &\quad+\sum_{t=1}^{T-1}
 |(\widetilde{\bTheta}_3)_{1t}-\eta_0|^2
 +\sum_{t=1}^{T-1}\sum_{\ell=2}^{t}(\widetilde{\bTheta}_3)_{\ell t}^2
 \le d^{-5/4+o(1)}.                                  \label{eq:teacher-endpoint-coefficients}
\end{align}
In particular, \(\widetilde\theta_1>0>\widetilde\theta_2\) for sufficiently large \(d\).
By the definitions of \(\widetilde{\bTheta}_3\) and \(\widetilde{\bTheta}_4\),
\begingroup\small
\[
 \widetilde\theta_1(\widetilde{\bTheta}_4)_{1t}-(\widetilde{\bTheta}_3)_{1t}
 =(\widetilde\theta_1-\eta_0)v_{33}
 \bigl(w_{44}+\wb_{45}\eb_t+\eb_1^{\mathsf T}\wb_{54}+\eb_1^{\mathsf T}\Wb_{55}\eb_t\bigr).
\]
\endgroup
Lemma~\ref{lem:guide-handoff-raw-bound} gives
\[
 v_{33}^2\sum_{t=1}^{T-1}
 |w_{44}+\wb_{45}\eb_t+\eb_1^{\mathsf T}\wb_{54}
 +\eb_1^{\mathsf T}\Wb_{55}\eb_t|^2\le d^{o(1)}.
\]
Therefore,
\begin{align*}
&\sum_{t=1}^{T-1}
\left|\widetilde\theta_1(\widetilde{\bTheta}_4)_{1t}-\eta_0\right|^2\\
&\quad\le
2\sum_{t=1}^{T-1}
 |(\widetilde{\bTheta}_3)_{1t}-\eta_0|^2
 +2|\widetilde\theta_1-\eta_0|^2d^{o(1)}
 \le d^{-5/4+o(1)}
 \le d^{-1/4+o(1)}.
\end{align*}
Since \((\widetilde{\bTheta}_4)_{\ell t}=(\widetilde{\bTheta}_3)_{\ell t}\) for \(\ell\ge2\),
\eqref{eq:teacher-endpoint-coefficients} now gives \eqref{eq:direct-stage2-coefficients}.

We compare \(p_t\) with \(1-(1-\eta_0\lambda)^t\).
Subtract at each step the exact GD identity
\[
 \Ib-(\Ib-\eta_0\Sb)^{t+1}
 =(\Ib-\eta_0\Sb)\{\Ib-(\Ib-\eta_0\Sb)^t\}+\eta_0\Sb
\]
from the recurrence in Lemma~\ref{lem:stage2-lemma1}. The difference
satisfies the exact identity
\begin{align}
 &p_{t+1}(\lambda)-\{1-(1-\eta_0\lambda)^{t+1}\}\notag\\
 &=(1+\widetilde\theta_2\lambda)
   \bigl[p_t(\lambda)-\{1-(1-\eta_0\lambda)^t\}\bigr]\notag\\
 &\quad+(\eta_0+\widetilde\theta_2)\lambda
       \{1-(1-\eta_0\lambda)^t\}\notag\\
 &\quad+\bigl[\widetilde\theta_1(\widetilde{\bTheta}_4)_{1t}-\eta_0\bigr]\lambda\notag\\
 &\quad+\sum_{\ell=2}^t(\widetilde{\bTheta}_4)_{\ell t}
    \bigl[p_\ell(\lambda)-\{1-(1-\eta_0\lambda)^\ell\}\bigr]\notag\\
 &\quad+\sum_{\ell=2}^t(\widetilde{\bTheta}_4)_{\ell t}
       \{1-(1-\eta_0\lambda)^\ell\}.                 \label{eq:direct-stage2-difference}
\end{align}
For the base step,
\[
 p_1(\lambda)-\{1-(1-\eta_0\lambda)\}
 =(\widetilde\theta_1-\eta_0)\lambda,\qquad
 \sum_{j=0}^1\left|[\lambda^j]
   \bigl[p_1-\{1-(1-\eta_0\lambda)\}\bigr]\right|^2
 =|\widetilde\theta_1-\eta_0|^2.
\]
For the induction step, by
\eqref{eq:direct-stage2-coefficients}, for all sufficiently large \(d\),
\[
 |\widetilde\theta_2|\le2,\qquad
 \sum_{t=1}^{T-1}\sum_{\ell=2}^t(\widetilde{\bTheta}_4)_{\ell t}^2\le d^{-1/4+o(1)}.
\]
The sum of the absolute coefficients of
\(1-(1-\eta_0\lambda)^\ell\) is at most \(2^\ell+1\).
Multiplication by \(1+\widetilde\theta_2\lambda\) increases a coefficient
norm by at most \(1+|\widetilde\theta_2|\le3\).
Thus \eqref{eq:direct-stage2-difference} and Cauchy--Schwarz give
\begingroup\small
\begin{align*}
 &\left(\sum_{j=0}^{t+1}\left|[\lambda^j]
       \bigl[p_{t+1}-\{1-(1-\eta_0\lambda)^{t+1}\}\bigr]\right|^2\right)^{1/2}\\
 &\quad\le
 \left(3+\sqrt T\left[\sum_{\ell=2}^t
       (\widetilde{\bTheta}_4)_{\ell t}^2\right]^{1/2}\right)
 \max_{1\le s\le t}
 \left(\sum_{j=0}^s\left|[\lambda^j]
       \bigl[p_s-\{1-(1-\eta_0\lambda)^s\}\bigr]\right|^2\right)^{1/2}\\
 &\qquad+C2^T\sqrt T
 \left[(\widetilde\theta_2+\eta_0)^2
       +|\widetilde\theta_1(\widetilde{\bTheta}_4)_{1t}-\eta_0|^2
       +\sum_{\ell=2}^t(\widetilde{\bTheta}_4)_{\ell t}^2\right]^{1/2}.
\end{align*}
\endgroup
Here \(\sqrt T\,d^{-1/8+o(1)}<1\), so the factor multiplying the maximum is
at most four. Iterating this inequality from the base step gives
\begingroup\small
\begin{align*}
 \sum_{t=1}^T
 \sum_{j=0}^t\left|[\lambda^j]
 \bigl[p_t-\{1-(1-\eta_0\lambda)^t\}\bigr]\right|^2
 &\le\exp\{CT\log(e+T)\}\Bigg[
       (\widetilde\theta_1-\eta_0)^2+(\widetilde\theta_2+\eta_0)^2\\
 &\qquad+\sum_{t=1}^{T-1}
       |\widetilde\theta_1(\widetilde{\bTheta}_4)_{1t}-\eta_0|^2
       +\sum_{t=1}^{T-1}\sum_{\ell=2}^t
       (\widetilde{\bTheta}_4)_{\ell t}^2\Bigg]
 \le d^{-1/4+o(1)}.
\end{align*}
\endgroup
Finally, Lemma~\ref{lem:wishart} bounds the norm of the finite-Wishart
moment matrix through degree \(T\) by \(\exp\{CT\log(e+T)\}=d^{o(1)}\).
Therefore
\begingroup\small
\begin{align*}
 &\frac1{2T}\sum_{t=1}^T\E_{\lambda\sim\nu_{d,n}}
       \left[\bigl(p_t(\lambda)-\{1-(1-\eta_0\lambda)^t\}\bigr)^2\right]\\
 &\quad\le\frac{\exp\{CT\log(e+T)\}}{2T}
 \sum_{t=1}^T\sum_{j=0}^t
 \left|[\lambda^j]\bigl[p_t-\{1-(1-\eta_0\lambda)^t\}\bigr]\right|^2
 \le d^{-1/4+o(1)},
\end{align*}
\endgroup
which proves \eqref{eq:direct-stage2-rollout}.

The triangle inequality and \eqref{eq:direct-stage2-rollout} give
\begin{align*}
 \sqrt{\frac{\mathcal L(0)}d}
 &\le\left\{\frac1{2Td}\sum_{t=1}^T
 \E\tr\!\left[(\Ib-\eta_0\Sb)^{2t}\right]\right\}^{1/2}\\
 &\quad+\left\{\frac1{2Td}\sum_{t=1}^T\E\tr\!\left(
 \left[p_t(\Sb)-\{\Ib-(\Ib-\eta_0\Sb)^t\}\right]^2\right)\right\}^{1/2}\\
 &\le\left\{\frac1{2Td}\sum_{t=1}^T
 \E\tr\!\left[(\Ib-\eta_0\Sb)^{2t}\right]\right\}^{1/2}
 +d^{-1/8+o(1)},
\end{align*}
which proves \eqref{eq:direct-stage2-loss}.
\end{proof}

\section{Proof of Stage 2 convergence}
\label{app:auxiliary}

\subsection{Stage 2 part of Lemma~\ref{lem:invariant-parameters}}
\label{app:stage2-invariant}
\label{app:stage2-invariant-reduced}
\begin{proof}[Proof of the Stage 2 part of Lemma~\ref{lem:invariant-parameters}]
At the Stage-1/Stage-2 handoff, \eqref{eq:VW} holds by the
preceding proof. For the autoregressive inputs in Stage 2, the block
multiplication gives
\begin{equation}
\begin{aligned}
 \wb_{t+1}&=\wb_t+\frac1n\sum_{i=1}^n(\Vb_{31}\xb_i+\Vb_{32}y_i)
 \Big[\xb_i^T(\Wb_{13}\wb_t+\Wb_{14}+\Wb_{15}\eb_t)
 \\
 &\qquad+y_i(\Wb_{23}\wb_t+w_{24}+\wb_{25}\eb_t)\Big]\\
 &\quad+\sum_{\ell=0}^t(\Vb_{33}\wb_\ell+\Vb_{34}+\Vb_{35}\eb_\ell)
 \Big[\wb_\ell^T(\Wb_{34}+\Wb_{35}\eb_t)+\Wb_{43}\wb_t+w_{44}+\wb_{45}\eb_t\\
  &\hspace{4em}+\eb_\ell^T(\Wb_{53}\wb_t+\wb_{54}+\Wb_{55}\eb_t)\Big].
\end{aligned}
\label{eq:full-open-retained}
\end{equation}
Substituting \eqref{eq:VW} and using \(\Xb\yb^T=\Sb\wstar\) yields
\begin{equation}
\begin{aligned}
 \wb_{t+1}&=\wb_t+v_{31}w_{13}\Sb\wb_t
 +v_{31}(w_{24}+\wb_{25}\eb_t)\Xb\yb^T\\
 &\quad+v_{33}\sum_{\ell=0}^t
 (w_{44}+\wb_{45}\eb_t+\eb_\ell^T\wb_{54}+\eb_\ell^T\Wb_{55}\eb_t)\wb_\ell.
\end{aligned}
\label{eq:stage2-retained-update}
\end{equation}
\begingroup
\allowdisplaybreaks[1]
\noindent\textup{(i)} Differentiating \eqref{eq:loss} using
\eqref{eq:full-open-retained} at \eqref{eq:VW} gives
\[
 \nabla_{\Vb_{ij}}\cL=0\quad(i\ne3),\qquad
 \nabla_{\Wb_{ij}}\cL=0\quad(j\in\{1,2\}),
\]
since these blocks do not occur in the output recurrence.
To include the dependence through the generated histories, differentiate
\eqref{eq:stage2-retained-update} with respect to \(\wb_{t+1}\), holding
\(\wb_0,\ldots,\wb_t\) fixed and propagating the derivative through
every later iterate. For \(t+1\le s<T\), this gives explicitly
\[
\begin{aligned}
 \frac{\partial\wb_{t+1}}{\partial\wb_{t+1}}&=\Ib_d,\\
 \frac{\partial\wb_{s+1}}{\partial\wb_{t+1}}
 &=(\Ib_d+v_{31}w_{13}\Sb)
       \frac{\partial\wb_s}{\partial\wb_{t+1}}\\
 &\quad+v_{33}\sum_{\ell=t+1}^s
 \bigl(w_{44}+\wb_{45}\eb_s+\eb_\ell^T\wb_{54}
                   +\eb_\ell^T\Wb_{55}\eb_s\bigr)
       \frac{\partial\wb_\ell}{\partial\wb_{t+1}}.
\end{aligned}
\]
For \(\Vb_{32},\Vb_{34},\Vb_{35}\), the chain rule therefore gives
\begin{align*}
 \nabla_{\Vb_{32}}\cL
 &=\frac1T\sum_{t=0}^{T-1}\sum_{s=t+1}^T\E\Bigl[
 \left(\frac{\partial\wb_s}{\partial\wb_{t+1}}\right)^T
 (\wb_s-\wstar)\Bigl\{
 (\wstar)^T\Sb\Wb_{13}\wb_t\\
 &\hspace{12em}
 +(w_{24}+\wb_{25}\eb_t)(\wstar)^T\Sb\wstar\Bigr\}\Bigr]
 =\zero_d,\\
 \nabla_{\Vb_{34}}\cL
 &=\frac1T\sum_{t=0}^{T-1}\sum_{s=t+1}^T\E\Bigl[
 \left(\frac{\partial\wb_s}{\partial\wb_{t+1}}\right)^T
 (\wb_s-\wstar)\\
 &\hspace{4em}\cdot\sum_{\ell=0}^t
 \bigl(w_{44}+\wb_{45}\eb_t+\eb_\ell^T\wb_{54}
                   +\eb_\ell^T\Wb_{55}\eb_t\bigr)\Bigr]
 =\zero_d,\\
 \nabla_{\Vb_{35}}\cL
 &=\frac1T\sum_{t=0}^{T-1}\sum_{s=t+1}^T\E\Bigl[
 \left(\frac{\partial\wb_s}{\partial\wb_{t+1}}\right)^T
 (\wb_s-\wstar)\\
 &\hspace{4em}\cdot\sum_{\ell=0}^t
 \bigl(w_{44}+\wb_{45}\eb_t+\eb_\ell^T\wb_{54}
                   +\eb_\ell^T\Wb_{55}\eb_t\bigr)\eb_\ell^T\Bigr]
 =\zero_{d\times T}.
\end{align*}
Similarly,
\begin{align*}
 \nabla_{\Wb_{14}}\cL
 &=\frac1T\sum_{t=0}^{T-1}\sum_{s=t+1}^T\E\Bigl[
 \Sb\Vb_{31}^T
 \left(\frac{\partial\wb_s}{\partial\wb_{t+1}}\right)^T
 (\wb_s-\wstar)\Bigr]=\zero_d,\\
 \nabla_{\Wb_{15}}\cL
 &=\frac1T\sum_{t=0}^{T-1}\sum_{s=t+1}^T\E\Bigl[
 \Sb\Vb_{31}^T
 \left(\frac{\partial\wb_s}{\partial\wb_{t+1}}\right)^T
 (\wb_s-\wstar)\eb_t^T\Bigr]=\zero_{d\times T},\\
 \nabla_{\Wb_{23}}\cL
 &=\frac1T\sum_{t=0}^{T-1}\sum_{s=t+1}^T\E\Bigl[
 (\wstar)^T\Sb\Vb_{31}^T
 \left(\frac{\partial\wb_s}{\partial\wb_{t+1}}\right)^T
 (\wb_s-\wstar)\wb_t^T\Bigr]=\zero_{1\times d},\\
 \nabla_{\Wb_{34}}\cL
 &=\frac1T\sum_{t=0}^{T-1}\sum_{s=t+1}^T\E\Bigl[
 \sum_{\ell=0}^t\wb_\ell\wb_\ell^T\Vb_{33}^T
 \left(\frac{\partial\wb_s}{\partial\wb_{t+1}}\right)^T
 (\wb_s-\wstar)\Bigr]=\zero_d,\\
 \nabla_{\Wb_{35}}\cL
 &=\frac1T\sum_{t=0}^{T-1}\sum_{s=t+1}^T\E\Bigl[
 \sum_{\ell=0}^t\wb_\ell\wb_\ell^T\Vb_{33}^T
 \left(\frac{\partial\wb_s}{\partial\wb_{t+1}}\right)^T
 (\wb_s-\wstar)\eb_t^T\Bigr]=\zero_{d\times T},\\
 \nabla_{\Wb_{43}}\cL
 &=\frac1T\sum_{t=0}^{T-1}\sum_{s=t+1}^T\E\Bigl[
 \sum_{\ell=0}^t\wb_\ell^T\Vb_{33}^T
 \left(\frac{\partial\wb_s}{\partial\wb_{t+1}}\right)^T
 (\wb_s-\wstar)\wb_t^T\Bigr]=\zero_{1\times d},\\
 \nabla_{\Wb_{53}}\cL
 &=\frac1T\sum_{t=0}^{T-1}\sum_{s=t+1}^T\E\Bigl[
 \sum_{\ell=0}^t\eb_\ell\wb_\ell^T\Vb_{33}^T
 \left(\frac{\partial\wb_s}{\partial\wb_{t+1}}\right)^T
 (\wb_s-\wstar)\wb_t^T\Bigr]=\zero_{T\times d}.
\end{align*}
To verify the zero equalities, first note that \(\wb_0=\zero_d\).
Since \(\Xb\yb^T=\Sb\wstar\), induction in
\eqref{eq:stage2-retained-update} shows that every \(\wb_t\) is odd in
\(\wstar\). The displayed derivative recurrence involves only \(\Sb\)
and the parameters, so every
\(\partial\wb_s/\partial\wb_{t+1}\) is independent of \(\wstar\).
In the three \(\Vb\)-gradients, the factor \(\wb_s-\wstar\) is odd,
whereas each remaining factor is even. The same holds for
\(\Wb_{14},\Wb_{15}\). Each of the other five \(\Wb\)-gradients
contains three factors that are odd in \(\wstar\).
Thus every displayed integrand is odd and has zero expectation under
\(\wstar\sim N(\zero,\Ib_d)\). Consequently,
\[
 \frac{\dd B}{\dd\tau}=-\nabla_B\cL=0,\qquad B(\tau)=B(\tau_1)=0
\]
for every zero block \(B\) in \eqref{eq:VW}.
\endgroup

\begin{samepage}
\noindent\textup{(ii)} For every \(\Ub\in\R^{d\times d}\) with
\(\Ub^\top\Ub=\Ub\Ub^\top=\Ib_d\),
\[
 (\xb_i,\wstar,\wb_t)
 \mapsto(\Ub\xb_i,\Ub\wstar,\Ub\wb_t),
\]
and
\begin{align*}
 &(\Vb_{31},\Wb_{13},\Vb_{33})
 \mapsto(\Ub\Vb_{31}\Ub^\top,\Ub\Wb_{13}\Ub^\top,\Ub\Vb_{33}\Ub^\top),\\
 &\cL(\Vb_{31},\Wb_{13},\Vb_{33})
 \mapsto\cL(\Ub\Vb_{31}\Ub^\top,\Ub\Wb_{13}\Ub^\top,\Ub\Vb_{33}\Ub^\top),\\
 &\left.\nabla_A\cL\right|_{(\Vb_{31},\Wb_{13},\Vb_{33})}
 \mapsto\Ub\left.\nabla_A\cL\right|_{(\Vb_{31},\Wb_{13},\Vb_{33})}\Ub^\top,
 \quad A\in\{\Vb_{31},\Wb_{13},\Vb_{33}\}.
\end{align*}
\end{samepage}
The loss is unchanged by this transformation.
Since \(\Ub A(\tau_1)\Ub^\top=A(\tau_1)\) for all three blocks, uniqueness of the
gradient-flow ODE gives
\[
 \Ub A(\tau)\Ub^\top=A(\tau)\qquad(\Ub^\top\Ub=\Ib_d,\ \tau\ge\tau_1).
\]
Fix \(\tau\). Taking \(\Ub\) to flip the sign of coordinate \(i\) gives,
for \(j\ne i\),
\[
 (A)_{ij}=(\Ub A\Ub^\top)_{ij}=-(A)_{ij},\qquad(A)_{ij}=0.
\]
Thus \(A\) is diagonal. Taking \(\Ub\) to exchange coordinates \(i\)
and \(j\) gives
\[
 (A)_{ii}=(\Ub A\Ub^\top)_{ii}=(A)_{jj}.
\]
Hence
\[
 \Vb_{31}(\tau)=v_{31}(\tau)\Ib_d,\qquad
 \Wb_{13}(\tau)=w_{13}(\tau)\Ib_d,\qquad
 \Vb_{33}(\tau)=v_{33}(\tau)\Ib_d.
\]
\end{proof}
Substituting the parameter form \eqref{eq:VW} into the expansion
\eqref{eq:full-open-update} gives the Stage-1 update
\begin{equation}
\begin{aligned}
 \wb_{t+1}&=\wb_t^{\mathrm{GD}}+v_{31}w_{13}\Sb\wb_t^{\mathrm{GD}}
 +v_{31}(w_{24}+\wb_{25}\eb_t)\Xb\yb^\top\\
 &\quad+v_{33}\sum_{\ell=0}^t
 (w_{44}+\wb_{45}\eb_t+\eb_\ell^\top\wb_{54}+\eb_\ell^\top\Wb_{55}\eb_t)
 \wb_\ell^{\mathrm{GD}}.
\end{aligned}
\label{eq:update}
\end{equation}

\subsection{Stage 2 part of Lemma~\ref{lem:polynomial-output}}
\begin{proof}[Proof of the Stage 2 part of Lemma~\ref{lem:polynomial-output}]
Under the parameter structure in Lemma~\ref{lem:stage1-invariant},
the update \eqref{eq:stage2-retained-update} and the identity
\(\Xb\yb^\top=\Sb\wstar\) give
\[
 \wb_0=\zero_d,\qquad
 \wb_1=v_{31}(w_{24}+\wb_{25}\eb_0)\Sb\wstar.
\]
The autoregressive iterates follow \eqref{eq:stage2-retained-update}.
Therefore, define
\begin{equation}
\begin{aligned}
 p_{0,\bar{\btheta}}(\lambda)&=0,\qquad
 p_{1,\bar{\btheta}}(\lambda)=v_{31}(w_{24}+\wb_{25}\eb_0)\lambda,\\
 p_{t+1,\bar{\btheta}}(\lambda)
 &=(1+v_{31}w_{13}\lambda)p_{t,\bar{\btheta}}(\lambda)
 +v_{31}(w_{24}+\wb_{25}\eb_t)\lambda\\
 &\quad+v_{33}\sum_{\ell=0}^t
 (w_{44}+\wb_{45}\eb_t+\eb_\ell^T\wb_{54}+\eb_\ell^T\Wb_{55}\eb_t)
 p_{\ell,\bar{\btheta}}(\lambda),\qquad1\le t<T.
\end{aligned}
\label{eq:stage2-raw-polynomial-recurrence}
\end{equation}
For \(p_{t,\bar{\btheta}}\) defined above,
\(\wb_t=p_{t,\bar{\btheta}}(\Sb)\wstar\). We prove these claims by
mathematical induction on \(t\). The basis step \(t=0\) follows from
\(p_{0,\bar{\btheta}}=0\). For the induction step, assume that
\(p_{t,\bar{\btheta}}(0)=0\) and \(\deg p_{t,\bar{\btheta}}\le t\) hold through
time \(t\). The recurrence then gives
\[
 p_{t+1,\bar{\btheta}}(0)=0,\qquad
 \deg p_{t+1,\bar{\btheta}}\le t+1.
\]
Indeed, the first term in the recurrence has degree at most \(t+1\), the
forcing term has degree one, and every term in the sum has degree at most
\(\ell\le t\); all terms have zero constant coefficient. Hence the same
claims hold at time \(t+1\), and the induction proves the Stage 2 part of
Lemma~\ref{lem:polynomial-output}.
\end{proof}

\subsection{Proof of Lemma 4.6}

On the invariant subspace in Lemma~\ref{lem:invariant-parameters}, define
\begin{equation}
 \widetilde\theta_1=v_{31}(w_{24}+\wb_{25}\eb_0),\qquad \widetilde\theta_2=v_{31}w_{13}.
                                                               \label{eq:global-auxiliary}
\end{equation}
For \(\widetilde\theta_1\ne0\), define \(\widetilde{\bTheta}_4\in\R^{T\times T}\) by
\begingroup\small
\begin{equation}
 \begin{aligned}
 (\widetilde{\bTheta}_4)_{1t}
 &:=\frac{v_{31}(w_{24}+\wb_{25}\eb_t)}{\widetilde\theta_1}
 +v_{33}\bigl(w_{44}+\wb_{45}\eb_t+\eb_1^{\mathsf T}\wb_{54}+\eb_1^{\mathsf T}\Wb_{55}\eb_t\bigr),
 &&1\le t<T,\\
 (\widetilde{\bTheta}_4)_{\ell t}
 &:=v_{33}\bigl(w_{44}+\wb_{45}\eb_t+\eb_\ell^{\mathsf T}\wb_{54}+\eb_\ell^{\mathsf T}\Wb_{55}\eb_t\bigr),
 &&2\le\ell\le t<T.
 \end{aligned}
 \label{eq:global-theta4}
\end{equation}
\endgroup
All other entries are zero. Thus, by the Stage~2 part of
Lemma~\ref{lem:polynomial-output} and \eqref{eq:stage2-raw-polynomial-recurrence},
writing \(p_t=p_{t,\bar{\btheta}}\),
\begin{equation}
 p_0=0,\quad p_1=\widetilde\theta_1\lambda,\qquad
 p_{t+1}=(1+\widetilde\theta_2\lambda)p_t
 +\sum_{\ell=1}^t p_\ell\,(\widetilde{\bTheta}_4)_{\ell t},\quad 1\le t<T.
 \label{eq:global-recurrence}
\end{equation}
Since \(\wb_t=p_t(\Sb)\wstar\) and
\(\E[\wstar(\wstar)^T\mid\Sb]=\Ib_d\), conditioning on \(\Sb\) gives the loss
formula in Lemma~\ref{lem:stage2-lemma1}.

For \(R(\lambda)=\sum_{j=0}^m r_j\lambda^j\) and \([m]=\{1,\ldots,m\}\), write
\[
\begin{aligned}
 [\lambda^j]R&=r_j,
 &\operatorname{coeff}_{[m]}(R)&=(r_1,\ldots,r_m)^{\mathsf T},\\
 \|R\|_{\mathrm{coeff}}^2&=|r_0|^2+\|\operatorname{coeff}_{[m]}(R)\|_2^2.
\end{aligned}
\]
Coefficient vectors are padded by zeros when needed.

\begin{lemma}\label{lem:global-recurrence}
For nonzero \(\widetilde\theta_1,\widetilde\theta_2\), \(p_t\) has leading term
\(\widetilde\theta_1\widetilde\theta_2^{t-1}\lambda^t\), \(1\le t\le T\).
\end{lemma}

\begin{proof}
The base case is \(p_1=\widetilde\theta_1\lambda\). Suppose the claim holds through
step \(t\). In \eqref{eq:global-recurrence}, every term other than
\(\widetilde\theta_2\lambda p_t\) has degree at most \(t\), so
\begin{equation}
 [\lambda^{t+1}]p_{t+1}=\widetilde\theta_2[\lambda^t]p_t
 =\widetilde\theta_1\widetilde\theta_2^t.
 \label{eq:global-leading}
\end{equation}
This proves the claim by induction.
\end{proof}

On the invariant subspace, the normalized Frobenius metric is
\begin{equation}
\begin{aligned}
 \frac1d\bigl(\|\mathrm d\Vb\|_F^2+\|\mathrm d\Wb\|_F^2\bigr)
 &=(\mathrm d v_{31})^2+(\mathrm d w_{13})^2+(\mathrm d v_{33})^2\\
 &\quad+\frac1d\bigl\{(\mathrm d w_{24})^2+\|\mathrm d\wb_{25}\|_2^2
       +(\mathrm d w_{44})^2+\|\mathrm d\wb_{45}\|_2^2\\
 &\hspace{24mm}+\|\mathrm d\wb_{54}\|_2^2+\|\mathrm d\Wb_{55}\|_F^2\bigr\}.
\end{aligned}\label{eq:normalized-raw-metric}
\end{equation}
Consequently, gradient flow for \(\mathcal L/d\) is the coordinate system
\begin{equation}
\begin{aligned}
\dot v_{31}&=-\nabla_{v_{31}}(\mathcal L/d),&
\dot w_{13}&=-\nabla_{w_{13}}(\mathcal L/d),&
\dot v_{33}&=-\nabla_{v_{33}}(\mathcal L/d),\\
\dot w_{24}&=-d\,\nabla_{w_{24}}(\mathcal L/d),&
\dot \wb_{25}&=-d\,\nabla_{\wb_{25}}(\mathcal L/d),\\
\dot w_{44}&=-d\,\nabla_{w_{44}}(\mathcal L/d),&
\dot \wb_{45}&=-d\,\nabla_{\wb_{45}}(\mathcal L/d),\\
\dot \wb_{54}&=-d\,\nabla_{\wb_{54}}(\mathcal L/d),&
\dot \Wb_{55}&=-d\,\nabla_{\Wb_{55}}(\mathcal L/d).
\end{aligned}                                             \label{eq:normalized-flow}
\end{equation}

\begin{lemma}\label{lem:invariants}
Along the Stage-2 flow, for every training time \(\tau\),
\begin{align}
 v_{31}(\tau)^2-w_{13}(\tau)^2
 -\frac{w_{24}(\tau)^2+\|\wb_{25}(\tau)\|^2}d
 &=-\frac{\sigma^2}d,\label{eq:spatial-invariant}\\
 dv_{33}(\tau)^2-|w_{44}(\tau)|^2-\|\wb_{45}(\tau)\|^2
 -\|\wb_{54}(\tau)\|^2-\|\Wb_{55}(\tau)\|_F^2
 &=d(v_{33}^{\rm init})^2,\label{eq:history-invariant}\\
 \sum_{t=0}^{T-1}(\wb_{25}(\tau))_t-w_{24}(\tau)&=-\sigma.
 \label{eq:positional-kernel-invariant}
\end{align}
Here \(v_{33}^{\rm init}>0\) is the scalar at the original initialization
\eqref{eq:warm-init}. Consequently, \(v_{33}(\tau)\ge v_{33}^{\rm init}>0\).
\end{lemma}

\begin{proof}
The exact identities in Lemma~\ref{lem:guide-invariants} hold here as well:
the same rescalings and translation preserve the Stage-2 recurrence
\eqref{eq:stage2-retained-update}, and the raw parameters are unchanged at the handoff.
\end{proof}

\subsection{Polynomial coefficients and loss decomposition}
\label{app:function}

\begin{lemma}
\label{lem:triangular}
For \(\widetilde\theta_1\widetilde\theta_2\ne0\) and \(1\le t\le T\),
\begin{equation}
 \operatorname{span}\{p_1,\ldots,p_t\}
 =\operatorname{span}\{\lambda,\ldots,\lambda^t\}.
 \label{eq:triangular-span}
\end{equation}
\end{lemma}

\begin{proof}
By Lemma~\ref{lem:global-recurrence} and \(p_j(0)=0\),
\[
 p_1=\widetilde\theta_1\lambda,\quad
 p_2=\widetilde\theta_1\widetilde\theta_2\lambda^2+\cdots,\quad
 \ldots,\quad p_t=\widetilde\theta_1\widetilde\theta_2^{t-1}\lambda^t+\cdots.
\]
Their coefficient matrix in \(\lambda,\ldots,\lambda^t\) is triangular with
nonzero diagonal, hence invertible. This proves \eqref{eq:triangular-span}.
\end{proof}

Building on Lemma~\ref{lem:triangular}, and given that analyzing
\(\mathcal L\) through \(\operatorname{span}\{p_1,\ldots,p_t\}\) is difficult
because \(p_t\) is generated autoregressively, we transform the recurrence
\eqref{eq:global-recurrence} to the basis \(\lambda,\ldots,\lambda^t\).
For \(1\le j\le t<T\),
define the triangular coefficient array \(\widetilde\theta_5\in\R^{T\times T}\) by
\begin{equation}
 (\widetilde\theta_5)_{jt}:=[\lambda^j]p_{t+1}(\lambda).
 \label{eq:theta5-coefficients}
\end{equation}
All other entries are zero. By Lemma~\ref{lem:global-recurrence},
\begin{equation}
 p_{t+1}(\lambda)
 =\widetilde\theta_1\widetilde\theta_2^{\,t}\lambda^{t+1}
 +\sum_{j=1}^{t}(\widetilde\theta_5)_{jt}\lambda^j .
 \label{eq:theta5}
\end{equation}
In this way, we obtain another expression for \(p_t\) in the monomial
basis. By Lemma~\ref{lem:triangular}, \(p_1,\ldots,p_t\) form a basis of
\(\operatorname{span}\{\lambda,\ldots,\lambda^t\}\). Since the leading
terms cancel in \(p_{t+1}-(1+\widetilde\theta_2\lambda)p_t\),
\eqref{eq:global-recurrence} and \eqref{eq:theta5} give equivalent
parameterizations, with a unique correspondence between the active entries
of \(\widetilde{\bTheta}_4\) and \(\widetilde\theta_5\) for fixed nonzero
\(\widetilde\theta_1,\widetilde\theta_2\).
For \(1\le t<T\), let
\begin{equation}
 \Hb_t=\left[\left\langle\lambda^i,\lambda^j\right\rangle_{d,n}\right]_{1\le i,j\le t}.
 \label{eq:theta5-gram}
\end{equation}
The moment bound \eqref{eq:finite-wishart-moment-lower} in
Lemma~\ref{lem:finite-wishart-moment-lower} gives
\[
 \lambda_{\min}(\Hb_t)
 \ge\frac{(d/n)^t}{2(t+1)16^t}>0,\qquad 1\le t<T.
\]
For the rollout polynomials, write
\[
 R_t(\lambda)=p_t(\lambda)-1,\qquad
 \wb_t-\wstar=R_t(\Sb)\wstar.
\]

\begin{lemma}\label{lem:stage2-loss-decomposition}
For nonzero \(\widetilde\theta_1,\widetilde\theta_2\), let
\begin{equation}
 \widetilde\theta_5^\star(\widetilde\theta_1,\widetilde\theta_2)
 :=\argmin_{\widetilde\theta_5}
       \mathcal L(\widetilde\theta_1,\widetilde\theta_2,\widetilde\theta_5).
 \label{eq:theta5-conditional-minimum}
\end{equation}
Define the two nonnegative loss terms by
\begin{equation}
 \mathcal L_A:=\mathcal L-\min_{\widetilde\theta_5}\mathcal L,
 \qquad
 \mathcal L_B:=\min_{\widetilde\theta_5}\mathcal L-\mathcal L^\star.
 \label{eq:stage2-loss-AB-definition}
\end{equation}
Then \(\mathcal L-\mathcal L^\star=\mathcal L_A+\mathcal L_B\), with
\begin{equation}
\begin{aligned}
 \frac{\mathcal L_A}{d}
 &=\frac1{2T}\sum_{t=1}^{T-1}
  \bigl(\widetilde\theta_5-
       \widetilde\theta_5^\star(\widetilde\theta_1,\widetilde\theta_2)\bigr)_{1:t,t}^{\mathsf T}
  \Hb_t
  \bigl(\widetilde\theta_5-
       \widetilde\theta_5^\star(\widetilde\theta_1,\widetilde\theta_2)\bigr)_{1:t,t},\\
 \frac{\mathcal L_B}{d}
 &=\frac1{2T}\left[
  C_\nu(-\widetilde\theta_2)
       \bigl(\widetilde\theta_1-\frac{B_\nu(-\widetilde\theta_2)}{C_\nu(-\widetilde\theta_2)})\bigr)^2
  +\frac{B_\nu(-\widetilde\theta_2^\star)^2}{C_\nu(-\widetilde\theta_2^\star)}
  -\frac{B_\nu(-\widetilde\theta_2)^2}{C_\nu(-\widetilde\theta_2)}
 \right].
\end{aligned}
 \label{eq:stage2-completed-square}
\end{equation}
We write $  \nu=\nu_{d,n} $. Here \(B_\nu,C_\nu\) are given by \eqref{eq:stage2-finite-BC-expectations},
and \(-\widetilde\theta_2^\star\) is the maximizer in
Lemma~\ref{lem:finite-function-unique}. Notably, as $L$ is minimized, $\widetilde\theta_1(\widetilde\theta_2)=\frac{B_\nu(-\widetilde\theta_2)}{C_\nu(-\widetilde\theta_2)}$ for a certain $ \widetilde\theta_2$.
Thus, the global minimizing parameters satisfy
\(\widetilde\theta_1^\star
=\widetilde\theta_1^\star(\widetilde\theta_2^\star)\) and
\(\widetilde\theta_5^\star
=\widetilde\theta_5^\star(\widetilde\theta_1^\star,\widetilde\theta_2^\star)\).
\end{lemma}

\begin{proof}
For \(1\le t\le T\) and \(\ell\in\mathbb R\), let
\[
 R_{t,\ell}^\nu
 =\argmin_{\substack{R(0)=-1,\ \deg R\le t\\{}[\lambda^t]R=\ell}}
      \|R\|_{d,n}^2.
\]
For \(t\ge2\), write each feasible polynomial as
\[
 R(\lambda)=-1+\sum_{i=1}^{t-1}([\lambda^i]R)\lambda^i+\ell\lambda^t.
\]
For \(1\le j<t\),
\begin{align*}
 \nabla_{[\lambda^j]R}\|R\|_{d,n}^2
 &=2\langle R,\lambda^j\rangle_{d,n}\\
 &=2\sum_{i=1}^{t-1}
       \langle\lambda^i,\lambda^j\rangle_{d,n}[\lambda^i]R
       -2\langle\lambda^j,1-\ell\lambda^t\rangle_{d,n}.
\end{align*}
At the minimum these derivatives vanish, so
\begin{align*}
 \sum_{i=1}^{t-1}
       \langle\lambda^i,\lambda^j\rangle_{d,n}
                     [\lambda^i]R_{t,\ell}^\nu
 &=\langle\lambda^j,1-\ell\lambda^t\rangle_{d,n},
       \qquad 1\le j<t,\\
 \Hb_{t-1}\operatorname{coeff}_{[t-1]}(R_{t,\ell}^\nu)
 &=\bigl(\langle\lambda^j,1-\ell\lambda^t\rangle_{d,n}\bigr)_{j=1}^{t-1}.
\end{align*}
(The displayed Gram bound makes the quadratic strictly convex.)
Solving this system gives
\begin{align*}
 \operatorname{coeff}_{[t-1]}(R_{t,\ell}^\nu)
 &=\Hb_{t-1}^{-1}
     \bigl(\langle\lambda^j,1-\ell\lambda^t\rangle_{d,n}\bigr)_{j=1}^{t-1}\\
 &=\Hb_{t-1}^{-1}
     \bigl(\langle\lambda^j,1\rangle_{d,n}\bigr)_{j=1}^{t-1}
    -\ell\Hb_{t-1}^{-1}
     \bigl(\langle\lambda^j,\lambda^t\rangle_{d,n}\bigr)_{j=1}^{t-1}\\
 &=\operatorname{coeff}_{[t-1]}(R_{t,0}^\nu)
   +\ell\left[
       \operatorname{coeff}_{[t-1]}(R_{t,1}^\nu)
       -\operatorname{coeff}_{[t-1]}(R_{t,0}^\nu)\right].
\end{align*}
The constant and leading coefficients agree as well. Therefore
\begin{equation}
 R_{t,\ell}^\nu
 =R_{t,0}^\nu+\ell(R_{t,1}^\nu-R_{t,0}^\nu),\qquad
 \langle R_{t,\ell}^\nu,\lambda^j\rangle_{d,n}=0\quad(1\le j<t).
 \label{eq:stage2-projection-affine}
\end{equation}
(For \(t=1\), \(R_{1,\ell}^\nu=\ell\lambda-1\).)
For degree \(t+1\) and \(\ell=\widetilde\theta_1\widetilde\theta_2^t\),
this gives
\begin{equation}
 \Hb_t\bigl(\widetilde\theta_5^\star
                  (\widetilde\theta_1,\widetilde\theta_2)\bigr)_{1:t,t}
 =\bigl(\langle\lambda^j,
       1-\widetilde\theta_1\widetilde\theta_2^t\lambda^{t+1}
           \rangle_{d,n}\bigr)_{j=1}^t.
 \label{eq:stage2-conditional-normal}
\end{equation}
Thus
\begin{equation}
\begin{aligned}
 p_{t+1}-1
 &=R_{t+1,\widetilde\theta_1\widetilde\theta_2^t}^\nu
   +\sum_{j=1}^t
       \bigl(\widetilde\theta_5-
           \widetilde\theta_5^\star(\widetilde\theta_1,\widetilde\theta_2)\bigr)_{jt}
            \lambda^j,\\
 \left\langle R_{t+1,\widetilde\theta_1\widetilde\theta_2^t}^\nu,
              \lambda^j\right\rangle_{d,n}
 &=0,\qquad 1\le j\le t<T.
\end{aligned}
 \label{eq:stage2-conditional-orthogonality}
\end{equation}
Squaring and summing therefore yields
\begin{equation}
\begin{aligned}
 \frac{\mathcal L}{d}
 &=\frac1{2T}\sum_{t=1}^T\|p_t-1\|_{d,n}^2\\
 &=\frac{\mathcal L_A}{d}
   +\frac1{2T}\sum_{t=1}^T
           \|R_{t,\widetilde\theta_1\widetilde\theta_2^{t-1}}^\nu\|_{d,n}^2\\
 &=\frac{\mathcal L_A}{d}
   +\frac1{2T}\left[
      \sum_{t=1}^T\|R_{t,0}^\nu\|_{d,n}^2
      -2\widetilde\theta_1B_\nu(-\widetilde\theta_2)
      +\widetilde\theta_1^2C_\nu(-\widetilde\theta_2)\right],
\end{aligned}
 \label{eq:stage2-loss-after-projection}
\end{equation}
where the affine formula \eqref{eq:stage2-projection-affine} gives
\begin{equation}
\begin{aligned}
 B_\nu(-\widetilde\theta_2)
 &=-\sum_{t=1}^T\widetilde\theta_2^{t-1}
          \langle R_{t,0}^\nu,R_{t,1}^\nu-R_{t,0}^\nu\rangle_{d,n},\\
 C_\nu(-\widetilde\theta_2)
 &=\sum_{t=1}^T\widetilde\theta_2^{2t-2}
           \|R_{t,1}^\nu-R_{t,0}^\nu\|_{d,n}^2
 \ge\|\lambda\|_{d,n}^2>0.
\end{aligned}
 \label{eq:stage2-finite-BC-expectations}
\end{equation}
Completing the scalar square gives
\[
 \frac{\mathcal L-\mathcal L_A}{d}
 =\frac1{2T}\left[
     \sum_{t=1}^T\|R_{t,0}^\nu\|_{d,n}^2
     -\frac{B_\nu(-\widetilde\theta_2)^2}{C_\nu(-\widetilde\theta_2)}
     +C_\nu(-\widetilde\theta_2)
       \bigl(\widetilde\theta_1-\widetilde\theta_1^\star(\widetilde\theta_2)\bigr)^2
   \right].
\]
Lemma~\ref{lem:finite-function-unique} identifies the unique maximizer,
and Lemma~\ref{lem:auxiliary-surjective} realizes its parameters. Hence
\begin{equation}
 \frac{\mathcal L^\star}{d}
 =\frac1{2T}\left[
       \sum_{t=1}^T\|R_{t,0}^\nu\|_{d,n}^2
       -\frac{B_\nu(-\widetilde\theta_2^\star)^2}
              {C_\nu(-\widetilde\theta_2^\star)}\right].
 \label{eq:stage2-loss-minimum}
\end{equation}
Subtracting proves \eqref{eq:stage2-completed-square}.
Minimizing its scalar square gives the following identities used below.
In particular, with \(\widetilde\theta_2\) fixed,
\begin{equation}
\begin{aligned}
 \min_{\widetilde\theta_1,\widetilde\theta_5}\mathcal L
 &=\mathcal L^\star+\min_{\widetilde\theta_1}\mathcal L_B,\\
 \frac{2T}{d}\min_{\widetilde\theta_1}\mathcal L_B
 &=\frac{B_\nu(-\widetilde\theta_2^\star)^2}{C_\nu(-\widetilde\theta_2^\star)}
    -\frac{B_\nu(-\widetilde\theta_2)^2}{C_\nu(-\widetilde\theta_2)},\\
 \frac{\mathcal L_A+\mathcal L_B-\min_{\widetilde\theta_1}\mathcal L_B}{d}
 &=\frac{\mathcal L_A}{d}
   +\frac{C_\nu(-\widetilde\theta_2)}{2T}
      \bigl(\widetilde\theta_1-\widetilde\theta_1^\star(\widetilde\theta_2)\bigr)^2.
\end{aligned}
 \label{eq:stage2-conditional-AB}
\end{equation}
\end{proof}

\subsection{Finite-Wishart comparison}
For \(\nu=\nu_{d,n}\) and \(\mu=\mu_{d/n}\), the following estimates
compare polynomial expectations under the two distributions. Conditioning
on \(\Sb\) gives
\begin{equation}
 \frac{\mathcal L}{d}
 =\frac1{2T}\sum_{t=1}^T
   \frac1d\E\,\tr\!\left((p_t(\Sb)-\Ib_d)^2\right).
 \label{eq:finite-loss}
\end{equation}
\begin{lemma}\label{lem:finite-transfer}
Under \eqref{eq:common-growing-regime}, for all sufficiently large \(d\)
and all polynomials \(f,g\) of degree at most \(T\),
\begin{equation}
\begin{aligned}
 &|\langle f,g\rangle_{d,n}-\langle f,g\rangle_{\mathrm{MP}}|\le\exp\{CT\log(e+T)\}\left(\frac dn\right)^{-T}
       \left(\frac1d+\frac1n\right)
       \|f\|_{\mathrm{MP}}\|g\|_{\mathrm{MP}}.
\end{aligned}
\label{eq:finite-inner-product-transfer}
\end{equation}
In particular,
\begin{equation}
 \frac12\|f\|_{\mathrm{MP}}^2\le\|f\|_{d,n}^2
 \le\frac32\|f\|_{\mathrm{MP}}^2.             \label{eq:finite-norm-equivalence}
\end{equation}
\end{lemma}

\begin{proof}
Expanding \(fg\) and applying Lemma~\ref{lem:wishart} to each moment
of order at most \(2T\), followed by Cauchy--Schwarz, gives
\[
\begin{aligned}
 &|\langle f,g\rangle_{d,n}-\langle f,g\rangle_{\mathrm{MP}}|\\
 &\quad\le\frac{\exp\{CT\log(e+T)\}}n
       \left(\sum_{i=0}^T|[\lambda^i]f|\right)
       \left(\sum_{j=0}^T|[\lambda^j]g|\right)\\
 &\quad\le\frac{(T+1)\exp\{CT\log(e+T)\}}n
       \|f\|_{\mathrm{coeff}}\|g\|_{\mathrm{coeff}}\\
 &\quad\le\frac{(T+1)^2\,16^T\exp\{CT\log(e+T)\}}n
       \left(\frac dn\right)^{-T}
       \|f\|_{\mathrm{MP}}\|g\|_{\mathrm{MP}}\\
 &\quad\le\exp\{CT\log(e+T)\}\left(\frac dn\right)^{-T}
       \left(\frac1d+\frac1n\right)
       \|f\|_{\mathrm{MP}}\|g\|_{\mathrm{MP}}.
\end{aligned}
\]
The third inequality uses Lemma~\ref{lem:mp-coefficient-norm}
for both \(f\) and \(g\); in the last inequality we enlarge the absolute
constant \(C\). This proves \eqref{eq:finite-inner-product-transfer}.
Under \eqref{eq:common-growing-regime},
\[
 \exp\{CT\log(e+T)\}\left(\frac dn\right)^{-T}
       \left(\frac1d+\frac1n\right)
 =d^{-1+o(1)}\le\frac12.
\]
Setting \(g=f\) in \eqref{eq:finite-inner-product-transfer} therefore gives
\[
 -\frac12\|f\|_{\mathrm{MP}}^2
 \le\|f\|_{d,n}^2-\|f\|_{\mathrm{MP}}^2
 \le\frac12\|f\|_{\mathrm{MP}}^2,
\]
which is \eqref{eq:finite-norm-equivalence}.
\end{proof}

For later use, \eqref{eq:common-growing-regime} also gives the following
numerical estimate. For every fixed absolute \(C\),
\[
 T\log(e+T)+6T\log(n/d)=o(\log d),\qquad
 \log\frac{(d/n)^{3/2}}{8192T}=O(\log\log d).
\]
Hence, for all sufficiently large \(d\),
\begin{equation}
 \exp\{CT\log(e+T)\}\left(\frac dn\right)^{-6T}
       \left(\frac1d+\frac1n\right)
 =d^{-1+o(1)}\le\frac{(d/n)^{3/2}}{8192T}.
 \label{eq:growing-transfer-budget}
\end{equation}

\begin{lemma}
\label{lem:projection-transfer}
Assume \eqref{eq:common-growing-regime} and write
\(\nu=\nu_{d,n}\), \(\mu=\mu_{d/n}\).
For \(1\le t\le T\), \(\ell\in\R\), and
\(\sigma\in\{\nu,\mu\}\), let \(R_{t,\ell}^\sigma\) be the unique minimizer
of
\[
 \min\left\{\|R\|_\sigma^2:
 R(0)=-1,\ \deg R\le t,\ [\lambda^t]R=\ell\right\}.
\]
For each fixed \(A\ge0\) and all sufficiently large \(d\), uniformly for \(|\ell|\le A\),
\begin{equation}
\begin{aligned}
 &\|R_{t,\ell}^\nu-R_{t,\ell}^\mu\|_{\mathrm{coeff}}
 +\left|\|R_{t,\ell}^\nu\|_{d,n}^2-\|R_{t,\ell}^\mu\|_{\mathrm{MP}}^2\right|\\
 &\quad\le C_A\exp\{CT\log(e+T)\}\left(\frac dn\right)^{-6T}
       \left(\frac1d+\frac1n\right).
\end{aligned}
\label{eq:fixed-leading-projection-transfer}
\end{equation}
The constants \(C\) and \(C_A\) are independent of \(d,n,T,d/n\), with
\(C_A\) depending only on \(A\).
\end{lemma}

\begin{proof}
The polynomial \(R(\lambda)=-1+\ell\lambda^t\) is admissible under both
distributions. Using it as a comparison polynomial, minimality and
\eqref{eq:finite-norm-equivalence} give
\[
\begin{aligned}
 \|R_{t,\ell}^\nu\|_{\mathrm{MP}}^2
 &\le2\|R_{t,\ell}^\nu\|_{d,n}^2
 \le2\|R_{t,\ell}^\mu\|_{d,n}^2
 \le3\|R_{t,\ell}^\mu\|_{\mathrm{MP}}^2,\\
 \|R_{t,\ell}^\mu\|_{\mathrm{MP}}
 &\le\|-1+\ell\lambda^t\|_{\mathrm{MP}}
 \le1+|\ell|\|\lambda^t\|_{\mathrm{MP}}
 \le1+A(1+\sqrt{d/n})^{2t}\le1+A3^T.
\end{aligned}
\]
Here the MP distribution is supported on
\([(1-\sqrt{d/n})^2,(1+\sqrt{d/n})^2]\), and
\((1+\sqrt{d/n})^2<3\) under \eqref{eq:common-growing-regime}.

For \(t=1\), both minimizers equal \(\ell\lambda-1\). For \(t\ge2\),
the orthogonality argument in the proof of
Lemma~\ref{lem:stage2-loss-decomposition} applies under either distribution:
\[
 \langle R_{t,\ell}^\sigma,\lambda^j\rangle_\sigma=0,
 \qquad \sigma\in\{\nu,\mu\},\quad 1\le j<t.
\]
Both minimizers have constant term \(-1\) and leading coefficient \(\ell\), so
\[
 R_{t,\ell}^\nu-R_{t,\ell}^\mu
 =\sum_{j=1}^{t-1}
   \bigl([\lambda^j]R_{t,\ell}^\nu-[\lambda^j]R_{t,\ell}^\mu\bigr)\lambda^j.
\]
Hence, by orthogonality,
\[
 \langle R_{t,\ell}^\nu,R_{t,\ell}^\nu-R_{t,\ell}^\mu\rangle_{d,n}=0,
 \qquad
 \langle R_{t,\ell}^\mu,R_{t,\ell}^\nu-R_{t,\ell}^\mu
       \rangle_{\mathrm{MP}}=0.
\]
Consequently,
\[
\begin{aligned}
 \|R_{t,\ell}^\nu-R_{t,\ell}^\mu\|_{d,n}^2
 &=\langle R_{t,\ell}^\nu,R_{t,\ell}^\nu-R_{t,\ell}^\mu\rangle_{d,n}
   -\langle R_{t,\ell}^\mu,R_{t,\ell}^\nu-R_{t,\ell}^\mu\rangle_{d,n}\\
 &=-\langle R_{t,\ell}^\mu,R_{t,\ell}^\nu-R_{t,\ell}^\mu\rangle_{d,n}\\
 &=\langle R_{t,\ell}^\mu,R_{t,\ell}^\nu-R_{t,\ell}^\mu\rangle_{\mathrm{MP}}
   -\langle R_{t,\ell}^\mu,R_{t,\ell}^\nu-R_{t,\ell}^\mu\rangle_{d,n}.
\end{aligned}
\]
Now apply Lemma~\ref{lem:finite-transfer} to the last line:
\[
\begin{aligned}
 \frac12\|R_{t,\ell}^\nu-R_{t,\ell}^\mu\|_{\mathrm{MP}}^2
 &\le\|R_{t,\ell}^\nu-R_{t,\ell}^\mu\|_{d,n}^2\\
 &\le\exp\{CT\log(e+T)\}\left(\frac dn\right)^{-T}
       \left(\frac1d+\frac1n\right)
       \|R_{t,\ell}^\mu\|_{\mathrm{MP}}
       \|R_{t,\ell}^\nu-R_{t,\ell}^\mu\|_{\mathrm{MP}}\\
 &\le(1+A3^T)\exp\{CT\log(e+T)\}\left(\frac dn\right)^{-T}
       \left(\frac1d+\frac1n\right)
       \|R_{t,\ell}^\nu-R_{t,\ell}^\mu\|_{\mathrm{MP}}.
\end{aligned}
\]
Thus
\[
 \|R_{t,\ell}^\nu-R_{t,\ell}^\mu\|_{\mathrm{MP}}
 \le2(1+A3^T)\exp\{CT\log(e+T)\}
       \left(\frac dn\right)^{-T}\left(\frac1d+\frac1n\right).
\]
By Lemma~\ref{lem:mp-coefficient-norm},
\[
\begin{aligned}
 \|R_{t,\ell}^\nu-R_{t,\ell}^\mu\|_{\mathrm{coeff}}
 &\le\sqrt{T+1}\,4^T\left(\frac dn\right)^{-T/2}
       \|R_{t,\ell}^\nu-R_{t,\ell}^\mu\|_{\mathrm{MP}}\\
 &\le2\sqrt{T+1}\,4^T(1+A3^T)\exp\{CT\log(e+T)\}
       \left(\frac dn\right)^{-3T/2}\left(\frac1d+\frac1n\right)\\
 &\le C_A\exp\{CT\log(e+T)\}\left(\frac dn\right)^{-6T}
       \left(\frac1d+\frac1n\right).
\end{aligned}
\]

Finally, minimality under each distribution gives
\[
 \|R_{t,\ell}^\nu\|_{d,n}^2\le\|R_{t,\ell}^\mu\|_{d,n}^2,
 \qquad
 \|R_{t,\ell}^\mu\|_{\mathrm{MP}}^2
 \le\|R_{t,\ell}^\nu\|_{\mathrm{MP}}^2.
\]
Subtracting the corresponding terms yields
\[
\begin{aligned}
 \|R_{t,\ell}^\nu\|_{d,n}^2-\|R_{t,\ell}^\nu\|_{\mathrm{MP}}^2
 &\le\|R_{t,\ell}^\nu\|_{d,n}^2-\|R_{t,\ell}^\mu\|_{\mathrm{MP}}^2\\
 &\le\|R_{t,\ell}^\mu\|_{d,n}^2-\|R_{t,\ell}^\mu\|_{\mathrm{MP}}^2.
\end{aligned}
\]
By Lemma~\ref{lem:finite-transfer} and the norm bounds above,
\[
\begin{aligned}
 &\left|\|R_{t,\ell}^\nu\|_{d,n}^2
              -\|R_{t,\ell}^\mu\|_{\mathrm{MP}}^2\right|\\
 &\quad\le
 \max_{\sigma\in\{\nu,\mu\}}
 \left|\|R_{t,\ell}^\sigma\|_{d,n}^2
              -\|R_{t,\ell}^\sigma\|_{\mathrm{MP}}^2\right|\\
 &\quad\le\exp\{CT\log(e+T)\}\left(\frac dn\right)^{-T}
       \left(\frac1d+\frac1n\right)
       \max_{\sigma\in\{\nu,\mu\}}\|R_{t,\ell}^\sigma\|_{\mathrm{MP}}^2\\
 &\quad\le3(1+A3^T)^2\exp\{CT\log(e+T)\}
       \left(\frac dn\right)^{-T}\left(\frac1d+\frac1n\right)\\
 &\quad\le C_A\exp\{CT\log(e+T)\}\left(\frac dn\right)^{-6T}
       \left(\frac1d+\frac1n\right).
\end{aligned}
\]
Combining the coefficient and value estimates, and enlarging \(C_A\),
proves \eqref{eq:fixed-leading-projection-transfer}.
\end{proof}

\subsection{Parameter bounds along Stage 2}
\label{app:compactification}
\label{app:evolution}
\begin{lemma}\label{lem:auxiliary-parameter-bounds}
\label{lem:robust-warm-gap}
Under \eqref{eq:common-growing-regime}, for all sufficiently large \(d\):
\par\medskip\noindent\textbf{(i)}
Starting from the unchanged endpoint in
Lemma~\ref{lem:direct-stage2-entry}, throughout \(0\le\tau\le d^{-1/2}\),
\begin{equation}
 \frac1{100}<\widetilde\theta_1<2,\qquad
 -\widetilde\theta_2=\eta_0+O(d^{-1/4})\in(8/25,23/25),\qquad
 \frac{\mathcal L}{d}\le\frac{49}{100T}.
 \label{eq:stage2-initial-parameter-bounds}
\end{equation}
Moreover, \(c\le |v_{31}|,w_{13},v_{33}\le C\).
\par\medskip\noindent\textbf{(ii)}
Every auxiliary point with \(\widetilde\theta_1\ne0\) and
\(2T\mathcal L/d\le7d/(4n)\) satisfies
\begin{equation}
 \frac35<\widetilde\theta_1<\frac32,\qquad
 \frac1{150}<-\widetilde\theta_2<\frac{10}{3},\qquad
 \frac dn\widetilde\theta_2^2<\frac9{10}<1.
 \label{eq:stage2-sublevel-parameter-bounds}
\end{equation}
\par\medskip\noindent\textbf{(iii)}
The Stage-2 flow from the endpoint of Lemma~\ref{lem:direct-stage2-entry}
exists globally and satisfies, for every \(\tau\ge0\),
\begin{equation}
 \frac1{100}<\widetilde\theta_1<2,\qquad
 \frac1{150}<-\widetilde\theta_2<\frac{10}{3},\qquad
 \frac dn\widetilde\theta_2^2<\frac9{10},\qquad
 \frac{\mathcal L}{d}\le\frac{49}{100T}.
 \label{eq:stage2-global-parameter-bounds}
\end{equation}
Moreover, it reaches
\begin{equation}
 \frac{2T\mathcal L}{d}\le\frac74\frac dn
 \label{eq:finite-warm-gap}
\end{equation}
before \(d^{-1/2}\), and
\begin{equation}
 1+(d/n)^{-1/2}+v_{31}(\tau)^2+v_{33}(\tau)^2
 \le d^{o(1)}+C\frac{d/n}{T}\tau
 \quad(\tau\ge0).
 \label{eq:gauge-growth}
\end{equation}
\end{lemma}

\begin{proof}[Proof of (i) and (ii)]
\par\medskip\noindent\textbf{(i)}
Lemmas~\ref{lem:wishart} and \ref{lem:mp-uniform-bounds} give
\[
\begin{aligned}
 &\sup_{0.4\le\eta_0\le0.9}
 \left|\sum_{t=1}^T\left(
    \|(1-\eta_0\lambda)^t\|_{d,n}^2
   -\|(1-\eta_0\lambda)^t\|_{\mathrm{MP}}^2\right)\right|\\
 &\quad\le\frac{\exp\{CT\log(e+T)\}}n
                  \sum_{t=1}^T(1+0.9)^{2t}=d^{-1+o(1)},\\
 &\sum_{t=1}^T\|(1-\eta_0\lambda)^t\|_{d,n}^2
 \le\sum_{t=1}^{\infty}\|(1-\eta_0\lambda)^t\|_{\mathrm{MP}}^2
       +d^{-1+o(1)}<\frac34+d^{-1+o(1)}.
\end{aligned}
\]
By Lemma~\ref{lem:direct-stage2-entry} and the triangle inequality,
\[
\begin{aligned}
 \frac{2T\mathcal L(0)}d
 &\le\left(
 \sqrt{\sum_{t=1}^T\|(1-\eta_0\lambda)^t\|_{d,n}^2}
       +d^{-1/8+o(1)}\right)^2\\
 &\le\left(\sqrt{3/4+d^{-1+o(1)}}+d^{-1/8+o(1)}\right)^2
 <\frac{49}{50}.
\end{aligned}
\]
Cauchy--Schwarz gives, for \(0\le\tau\le d^{-1/2}\),
\begin{equation}
\begin{aligned}
 \int_0^\tau (|\dot v_{31}|+|\dot w_{13}|+|\dot v_{33}|)\,\mathrm ds
 &\le\sqrt{3\tau}
 \left(\int_0^\tau
       (\dot v_{31}^{\,2}+\dot w_{13}^{\,2}+\dot v_{33}^{\,2})
                      \,\mathrm ds\right)^{1/2}\\
 &\le\sqrt{\frac{3\tau\mathcal L(0)}d}
 \le C d^{-1/4}/\sqrt T.
\end{aligned}
\label{eq:entry-scalar-drift}
\end{equation}
The estimates in Lemma~\ref{lem:guide-handoff-raw-bound},
specifically \eqref{eq:guide-balanced-endpoint},
therefore give \(c\le |v_{31}|,w_{13},v_{33}\le C\) and
\(-\widetilde\theta_2=\eta_0+O(d^{-1/4})\).
Since \(0.4\le\eta_0\le0.9\), this lies in \((8/25,23/25)\).
The first summand of the loss satisfies
\[
\begin{aligned}
 (\widetilde\theta_1-1)^2
 &\le(1+d/n+1/n)\widetilde\theta_1^2-2\widetilde\theta_1+1\\
 &=\|\widetilde\theta_1\lambda-1\|_{d,n}^2
 \le\frac{2T\mathcal L(\tau)}d
 \le\frac{2T\mathcal L(0)}d<\frac{49}{50}
 <(99/100)^2.
\end{aligned}
\]
Thus \(1/100<\widetilde\theta_1<199/100<2\).
Global existence follows from Lemma~\ref{lem:stage1-global-flow}.
Also, for every \(\tau\ge0\),
\[
\begin{aligned}
 1+(d/n)^{-1/2}+v_{31}(\tau)^2+v_{33}(\tau)^2
 &\le1+(d/n)^{-1/2}+2v_{31}(0)^2+2v_{33}(0)^2\\
 &\quad+2\tau\int_0^\tau
                     (\dot v_{31}^{\,2}+\dot v_{33}^{\,2})\,\mathrm ds\\
 &\le d^{o(1)}+2\tau\frac{\mathcal L(0)-\mathcal L(\tau)}d
 \le d^{o(1)}+C\frac{d/n}{T}\tau.
\end{aligned}
\]
(Use Lemma~\ref{lem:guide-handoff-raw-bound} and \(n=\Theta(d)\).)
This proves \eqref{eq:gauge-growth}.

\par\medskip\noindent\textbf{(ii)}
Since \(p_1(\lambda)=\widetilde\theta_1\lambda\) and the first two
MP moments are \(\E[\lambda]=1\), \(\E[\lambda^2]=1+d/n\),
\[
\begin{aligned}
 \|p_1-1\|_{\mathrm{MP}}^2
 &=\E_{\lambda\sim\mu_{d/n}}
       [(\widetilde\theta_1\lambda-1)^2]\\
 &=(1+d/n)\widetilde\theta_1^2-2\widetilde\theta_1+1\\
 &=(1-\widetilde\theta_1)^2+\frac dn\widetilde\theta_1^2.
\end{aligned}
\]
The quadratic coefficient of \(p_2\) is
\(\widetilde\theta_1\widetilde\theta_2\). Minimizing over its linear
coefficient and applying \eqref{eq:conditional-norm} in
Lemma~\ref{lem:projection} with
\((t,\ell)=(2,\widetilde\theta_1\widetilde\theta_2)\) gives
\[
\begin{aligned}
 \|p_2-1\|_{\mathrm{MP}}^2
 &\ge\min_{\ell\in\R}
   \|\ell\lambda+\widetilde\theta_1\widetilde\theta_2\lambda^2-1
     \|_{\mathrm{MP}}^2\\
 &=\frac{d/n}{1+d/n}
       (1-\widetilde\theta_1(-\widetilde\theta_2))^2
       +\left(\frac dn\right)^2\widetilde\theta_1^2\widetilde\theta_2^2.
\end{aligned}
\]
The actual linear coefficient of \(p_2\) need not attain this minimum.
Lemma~\ref{lem:finite-transfer} and the loss formula then give
\begin{align*}
 \|p_1-1\|_{\mathrm{MP}}^2+\|p_2-1\|_{\mathrm{MP}}^2
 &\le\sum_{t=1}^T\|p_t-1\|_{\mathrm{MP}}^2\\
 &\le(1+d^{-1+o(1)})\frac{2T\mathcal L}{d}
 \le\frac95\frac dn.
\end{align*}
For \(\widetilde\theta_1\le3/5\),
\[
\begin{aligned}
 (1-\widetilde\theta_1)^2+\frac dn\widetilde\theta_1^2-\frac95\frac dn
 &\ge\left(1-\frac35\right)^2
       +\frac dn\left(\frac35\right)^2-\frac95\frac dn\\
 &=\frac4{25}-\frac{36}{25}\frac dn
 \ge\frac2{125}>0
 \quad\left(\frac35<\frac1{1+d/n}\right).
\end{aligned}
\]
Thus
\[
 (1-\widetilde\theta_1)^2\le\frac95\frac dn\le\frac9{50}<\frac14,
 \qquad \frac35<\widetilde\theta_1<\frac32.
\]
Subtracting \(\min_{\widetilde\theta_1}\|p_1-1\|_{\mathrm{MP}}^2
=(d/n)/(1+d/n)\) gives
\[
\begin{aligned}
 (1-\widetilde\theta_1(-\widetilde\theta_2))^2
 &\le\frac{1+d/n}{d/n}
       \left(\frac95\frac dn-\frac{d/n}{1+d/n}\right)\\
 &=\frac45+\frac95\frac dn
 \le\frac{49}{50}<\left(\frac{99}{100}\right)^2,\\
 \frac1{100}
 &<\widetilde\theta_1(-\widetilde\theta_2)<\frac{199}{100},\\
 \frac1{150}=\frac{1/100}{3/2}
 &<-\widetilde\theta_2
 <\frac{199/100}{3/5}<\frac{10}{3}.
\end{aligned}
\]
For all \(t\), Lemma~\ref{lem:projection} also gives
\[
\begin{aligned}
 \|p_t-1\|_{\mathrm{MP}}^2
 &\ge(d/n)^t\widetilde\theta_1^2\widetilde\theta_2^{2t-2},\\
 \widetilde\theta_1^2
 \sum_{k=0}^{T-1}\bigl((d/n)\widetilde\theta_2^2\bigr)^k
 &\le\frac nd\sum_{t=1}^T\|p_t-1\|_{\mathrm{MP}}^2\le\frac95,\\
 \sum_{k=0}^{T-1}\bigl((d/n)\widetilde\theta_2^2\bigr)^k
 &<5<\sum_{k=0}^7(9/10)^k
 \quad(T\ge8).
\end{aligned}
\]
Hence \((d/n)\widetilde\theta_2^2<9/10\).

\end{proof}

\subsection{Study of \texorpdfstring{\(\mathcal L_A\) and \(\mathcal L_B\)}{LA and LB}}
The Stage-2 optimization reduces to maximizing the scalar function
\(B_{\nu_{d,n}}^2/C_{\nu_{d,n}}\). We now establish that this reduced
objective has a unique, nondegenerate global maximizer, and characterize
its location and landscape.

\medskip\noindent\textbf{(i) Global uniqueness.}
\begin{lemma}
\label{lem:function}
\label{lem:function-unique}
\label{lem:finite-function-unique}
For all sufficiently large \(d\) under \eqref{eq:common-growing-regime},
the function \(B_{\nu_{d,n}}(-\widetilde\theta_2)^2/
C_{\nu_{d,n}}(-\widetilde\theta_2)\) under the finite-Wishart distribution
\(\nu_{d,n}\) has a unique nondegenerate global maximizer
\(\widetilde\theta_2^\star\in\mathbb R\). The point
\(-\widetilde\theta_2^\star>0\) is the only positive critical point
of \(B_{\nu_{d,n}}^2/C_{\nu_{d,n}}\).
\end{lemma}

\begin{proof}
Write \(\nu=\nu_{d,n}\), \(\mu=\mu_{d/n}\), and
\(S_k=\sum_{r=0}^k(d/n)^r\). The MP counterparts have explicit
coefficients: Lemma~\ref{lem:mp-function-unique} gives
\begin{equation}
 B_\mu(x)=\sum_{k=0}^{T-1}\frac{(d/n)^k}{S_k}x^k,\qquad
 C_\mu(x)=\sum_{k=0}^{T-1}\frac{(d/n)^kS_{k+1}}{S_k}x^{2k}.
 \label{eq:function-BC}
\end{equation}

To control the intersections with a horizontal line, for \(y\in\R\)
we consider
\[
 \frac{B_\nu(x)^2}{C_\nu(x)}=y
 \quad\Longleftrightarrow\quad B_\nu(x)^2-yC_\nu(x)=0.
\]
Multiplication by \((1+x)^2\) preserves the positive roots. For
\(\sigma\in\{\nu,\mu\}\), consider the coefficient ratios
\[
 r_{j,\sigma}
 =\frac{[x^j](1+x)^2B_\sigma(x)^2}
        {[x^j](1+x)^2C_\sigma(x)},\qquad 0\le j\le2T.
\]
The denominator coefficients are positive by
Lemma~\ref{lem:finite-coefficient-ratio-comparison}, and
\[
 [x^j](1+x)^2(B_\nu(x)^2-yC_\nu(x))
 =[x^j](1+x)^2C_\nu(x)\,(r_{j,\nu}-y).
\]
Its bound \eqref{eq:finite-function-ratio-transfer} gives
\[
 \max_{0\le j\le2T}|r_{j,\nu}-r_{j,\mu}|=o((d/n)^T).
\]
Together with the MP ratio gaps and endpoint comparison in
Lemma~\ref{lem:mp-function-unique}, this preserves the ordering
\[
 r_{0,\nu}<r_{1,\nu}<r_{2,\nu}<r_{3,\nu}
 >r_{4,\nu}>\cdots>r_{2T,\nu},\qquad r_{2T,\nu}<r_{0,\nu}.
\]
Descartes' rule gives at most one positive root for \(y\le r_{0,\nu}\)
and at most two otherwise, counting multiplicity. Since the derivative is
positive near zero and negative for large \(x\), there is exactly one
positive critical point, a nondegenerate global maximum.
Coefficient positivity also gives
\[
 |B_\nu(-x)|<B_\nu(x),\qquad C_\nu(-x)=C_\nu(x)>0\quad(x>0),
\]
so the maximum is unique over \(\R\).
\end{proof}

\medskip\noindent\textbf{(ii) Location of the maximizer.}
\begin{lemma}
\label{lem:finite-function-location}
Under \eqref{eq:common-growing-regime}, for all sufficiently large \(d\),
the maximizer in Lemma~\ref{lem:finite-function-unique} satisfies
\begin{equation}
 \frac{1+d/n}{1+d/n+(d/n)^2}
 <-\widetilde\theta_2^\star
 <1-\frac{(d/n)^2}{128}.
 \label{eq:finite-optimum-location}
\end{equation}
\end{lemma}

\begin{proof}
Write \(\nu=\nu_{d,n}\), \(\mu=\mu_{d/n}\).
The bound \eqref{eq:mp-endpoint-derivatives} in
Lemma~\ref{lem:mp-function-unique} gives the MP endpoint derivative signs,
while \eqref{eq:finite-ratio-derivative-transfer} in
Lemma~\ref{lem:finite-objective-comparison}, with \(j=1\), gives
\[
 \sup_{0\le x\le1}
 \left|\left(\frac{B_\nu^2}{C_\nu}\right)'(x)
       -\left(\frac{B_\mu^2}{C_\mu}\right)'(x)\right|
 \le d^{-1+o(1)}.
\]
Since \(d/n=\Theta(1)\) under \eqref{eq:common-growing-regime},
the two derivative signs are preserved:
\[
\begin{aligned}
 \left(\frac{B_\nu^2}{C_\nu}\right)'\!
       \left(\frac{1+d/n}{1+d/n+(d/n)^2}\right)
 &\ge c(d/n)^4-d^{-1+o(1)}>0,\\
 \left(\frac{B_\nu^2}{C_\nu}\right)'\!
       \left(1-\frac{(d/n)^2}{128}\right)
 &\le-c(d/n)^3+d^{-1+o(1)}<0.
\end{aligned}
\]
Continuity and Lemma~\ref{lem:finite-function-unique} place the unique
maximizer between these arguments,
proving \eqref{eq:finite-optimum-location}.
\end{proof}

\medskip\noindent\textbf{(iii) Uniform landscape estimates.}
\begin{lemma}
\label{lem:finite-scalar-landscape}
Under \eqref{eq:common-growing-regime}, write \(\nu=\nu_{d,n}\).
For all sufficiently large \(d\), there are constants \(c,C>0\),
independent of \(d,n,T\), such that
\begin{equation}
\begin{aligned}
 -\left(\frac{B_\nu^2}{C_\nu}\right)''(-\widetilde\theta_2)
 &\ge C^{-1},
 &&|\widetilde\theta_2-\widetilde\theta_2^\star|\le c,\\
 \left|\left(\frac{B_\nu^2}{C_\nu}\right)'(-\widetilde\theta_2)\right|
 &\ge C^{-1},
 &&|\widetilde\theta_2-\widetilde\theta_2^\star|\ge c,\\
 \frac{B_\nu(-\widetilde\theta_2^\star)^2}
      {C_\nu(-\widetilde\theta_2^\star)}
 -\frac{B_\nu(-\widetilde\theta_2)^2}{C_\nu(-\widetilde\theta_2)}
 &\ge C^{-1},
 &&|\widetilde\theta_2-\widetilde\theta_2^\star|\ge c.
\end{aligned}
 \label{eq:finite-scalar-uniform-nondegeneracy}
\end{equation}
These estimates hold on
\[
 \frac1{150}\le-\widetilde\theta_2\le\frac{10}{3},\qquad
 \frac dn(-\widetilde\theta_2)^2\le\frac9{10}.
\]
On the same region,
\begin{equation}
\begin{aligned}
 C^{-1}|\widetilde\theta_2-\widetilde\theta_2^\star|^2
 &\le
 \frac{B_\nu(-\widetilde\theta_2^\star)^2}{C_\nu(-\widetilde\theta_2^\star)}
 -\frac{B_\nu(-\widetilde\theta_2)^2}{C_\nu(-\widetilde\theta_2)}
 \le C|\widetilde\theta_2-\widetilde\theta_2^\star|^2.
\end{aligned}
 \label{eq:finite-scalar-quadratic-gap}
\end{equation}
\begin{equation}
\begin{aligned}
 &(\widetilde\theta_2^\star-\widetilde\theta_2)
       \nabla_{\widetilde\theta_2}
       \frac{B_\nu(-\widetilde\theta_2)^2}{C_\nu(-\widetilde\theta_2)}\\
 &\quad\ge C^{-1}\left[
 \frac{B_\nu(-\widetilde\theta_2^\star)^2}{C_\nu(-\widetilde\theta_2^\star)}
 -\frac{B_\nu(-\widetilde\theta_2)^2}{C_\nu(-\widetilde\theta_2)}\right].
\end{aligned}
 \label{eq:stage2-scalar-direction}
\end{equation}
\end{lemma}

\begin{proof}
Write \(\nu=\nu_{d,n}\), \(\mu=\mu_{d/n}\).
Lemma~\ref{lem:finite-function-unique} gives
\[
 \frac{B_\nu(-\widetilde\theta_2^\star)^2}
      {C_\nu(-\widetilde\theta_2^\star)}
 \ge\frac{B_\nu(-\widetilde\theta_2^{\mathrm{MP}*})^2}
          {C_\nu(-\widetilde\theta_2^{\mathrm{MP}*})}.
\]
Apply the last inequality of \eqref{eq:mp-scalar-uniform-landscape} in
Lemma~\ref{lem:mp-scalar-landscape}, followed by
Lemma~\ref{lem:finite-objective-comparison} with \(j=0\), to obtain
\[
\begin{aligned}
 |\widetilde\theta_2^\star-\widetilde\theta_2^{\mathrm{MP}*}|^2
 &\le C\left[
 \frac{B_\mu(-\widetilde\theta_2^{\mathrm{MP}*})^2}
      {C_\mu(-\widetilde\theta_2^{\mathrm{MP}*})}
 -\frac{B_\mu(-\widetilde\theta_2^\star)^2}
       {C_\mu(-\widetilde\theta_2^\star)}\right]\\
 &\le2C\sup_{0\le x\le1}
       \left|\frac{B_\nu(x)^2}{C_\nu(x)}
             -\frac{B_\mu(x)^2}{C_\mu(x)}\right|
 \le d^{-1+o(1)}.
\end{aligned}
\]
Thus \(|\widetilde\theta_2^\star-\widetilde\theta_2^{\mathrm{MP}*}|<c/2\)
for all sufficiently large \(d\), where \(c\) is from
Lemma~\ref{lem:mp-scalar-landscape}. The triangle inequality gives
\[
\begin{aligned}
 |\widetilde\theta_2-\widetilde\theta_2^\star|\le c
 &\ \Longrightarrow\
 |\widetilde\theta_2-\widetilde\theta_2^{\mathrm{MP}*}|<2c,\\
 |\widetilde\theta_2-\widetilde\theta_2^\star|\ge c
 &\ \Longrightarrow\
 |\widetilde\theta_2-\widetilde\theta_2^{\mathrm{MP}*}|>c/2.
\end{aligned}
\]
Apply the first two inequalities of \eqref{eq:mp-scalar-uniform-landscape}
in Lemma~\ref{lem:mp-scalar-landscape} and
Lemma~\ref{lem:finite-objective-comparison} with \(j=2,1\), respectively.
For the value gap, use the last inequality of
\eqref{eq:mp-scalar-uniform-landscape}, the displayed maximality inequality,
and \eqref{eq:finite-ratio-derivative-transfer} with \(j=0\).
Taking \(d\) sufficiently large and enlarging \(C\), we obtain
\[
\begin{aligned}
 -\left(\frac{B_\nu^2}{C_\nu}\right)''(-\widetilde\theta_2)
 &\ge-\left(\frac{B_\mu^2}{C_\mu}\right)''(-\widetilde\theta_2)
       -d^{-1+o(1)}\ge C^{-1},
 &&|\widetilde\theta_2-\widetilde\theta_2^\star|\le c,\\
 \left|\left(\frac{B_\nu^2}{C_\nu}\right)'(-\widetilde\theta_2)\right|
 &\ge\left|\left(\frac{B_\mu^2}{C_\mu}\right)'(-\widetilde\theta_2)\right|
       -d^{-1+o(1)}\ge C^{-1},
 &&|\widetilde\theta_2-\widetilde\theta_2^\star|\ge c,\\
 \frac{B_\nu(-\widetilde\theta_2^\star)^2}
      {C_\nu(-\widetilde\theta_2^\star)}
 -\frac{B_\nu(-\widetilde\theta_2)^2}{C_\nu(-\widetilde\theta_2)}
 &\ge\frac{B_\mu(-\widetilde\theta_2^{\mathrm{MP}*})^2}
            {C_\mu(-\widetilde\theta_2^{\mathrm{MP}*})}
       -\frac{B_\mu(-\widetilde\theta_2)^2}{C_\mu(-\widetilde\theta_2)}
       -2d^{-1+o(1)}\ge C^{-1},
 &&|\widetilde\theta_2-\widetilde\theta_2^\star|\ge c.
\end{aligned}
\]
For \(|\widetilde\theta_2-\widetilde\theta_2^\star|\le c\), Taylor's formula gives
\begin{align*}
 &\frac{B_\nu(-\widetilde\theta_2^\star)^2}{C_\nu(-\widetilde\theta_2^\star)}
   -\frac{B_\nu(-\widetilde\theta_2)^2}{C_\nu(-\widetilde\theta_2)}\\
 &\quad=|\widetilde\theta_2-\widetilde\theta_2^\star|^2
    \int_0^1(1-s)
       \left[-\left(\frac{B_\nu^2}{C_\nu}\right)''
       \!\left(-\widetilde\theta_2^\star+
                 s(\widetilde\theta_2^\star-\widetilde\theta_2)\right)\right]
       \,\mathrm ds,\\
 &(\widetilde\theta_2^\star-\widetilde\theta_2)
       \nabla_{\widetilde\theta_2}
       \frac{B_\nu(-\widetilde\theta_2)^2}{C_\nu(-\widetilde\theta_2)}\\
 &\quad=|\widetilde\theta_2-\widetilde\theta_2^\star|^2
    \int_0^1
       \left[-\left(\frac{B_\nu^2}{C_\nu}\right)''
       \!\left(-\widetilde\theta_2^\star+
                 s(\widetilde\theta_2^\star-\widetilde\theta_2)\right)\right]
       \,\mathrm ds .
\end{align*}
Here \(C^{-1}\le-(B_\nu^2/C_\nu)''\le C\) on the segment, by
\eqref{eq:finite-scalar-uniform-nondegeneracy} and
Lemma~\ref{lem:finite-BC-comparison}.
For \(c\le|\widetilde\theta_2-\widetilde\theta_2^\star|<4\),
\begin{align*}
 C^{-1}
 &\le
 \frac{B_\nu(-\widetilde\theta_2^\star)^2}{C_\nu(-\widetilde\theta_2^\star)}
 -\frac{B_\nu(-\widetilde\theta_2)^2}{C_\nu(-\widetilde\theta_2)}
 \le C,\\
 (\widetilde\theta_2^\star-\widetilde\theta_2)
       \nabla_{\widetilde\theta_2}
       \frac{B_\nu(-\widetilde\theta_2)^2}{C_\nu(-\widetilde\theta_2)}
 &=|\widetilde\theta_2-\widetilde\theta_2^\star|
       \left|\left(\frac{B_\nu^2}{C_\nu}\right)'(-\widetilde\theta_2)\right|
 \ge cC^{-1}.
\end{align*}
(The derivative points toward the unique maximizer by
Lemma~\ref{lem:finite-function-unique}.)
Enlarging \(C\) proves \eqref{eq:finite-scalar-quadratic-gap}
and \eqref{eq:stage2-scalar-direction}.
\end{proof}

\begin{lemma}\label{lem:warm-gap}
Under \eqref{eq:common-growing-regime}, write \(\nu=\nu_{d,n}\).
For all sufficiently large \(d\),
\begin{equation}
 \sup_{0.4\le-\widetilde\theta_2\le0.9}
 \left[
 \sum_{t=1}^T\|R_{t,0}^\nu\|_{d,n}^2
       -\frac{B_\nu(-\widetilde\theta_2)^2}{C_\nu(-\widetilde\theta_2)}
 \right]
 <\frac{149}{100}\frac dn.
 \label{eq:warm-gap}
\end{equation}
\end{lemma}

\begin{proof}
Write \(\mu=\mu_{d/n}\) and
\(S_k=\sum_{r=0}^k(d/n)^r\). Lemmas~\ref{lem:stage2-loss-decomposition},
\ref{lem:projection-transfer}, and \ref{lem:finite-objective-comparison} give
\begin{align*}
 \sum_{t=1}^T\|R_{t,0}^\nu\|_{d,n}^2
       -\frac{B_\nu(-\widetilde\theta_2)^2}{C_\nu(-\widetilde\theta_2)}
 &=\sum_{k=0}^{T-1}\frac{(d/n)^k}{S_k}
       -\frac{B_\mu(-\widetilde\theta_2)^2}{C_\mu(-\widetilde\theta_2)}
       +O(d^{-1+o(1)})\\
 &\le\sum_{k=0}^{T-1}\frac{(d/n)^k}{S_k}
       -\frac{B_\mu(0.4)^2}{C_\mu(0.4)}
       +O(d^{-1+o(1)})\\
 &\le\frac{36}{25}\frac dn+O(d^{-1+o(1)})
 <\frac{149}{100}\frac dn.
\end{align*}
(The first inequality uses Lemma~\ref{lem:mp-function-unique}:
\(B_\mu^2/C_\mu\) increases on \([0.4,0.9]\), since \(S_1/S_2>0.9\).
The second is \eqref{eq:mp-conditional-uniform-bound} in
Lemma~\ref{lem:mp-uniform-bounds}.)
\end{proof}

\subsection{Coefficient and coordinate estimates}
Lemma~\ref{lem:auxiliary-parameter-bounds}(i)--(ii) gives
\[
 \frac1{100}\le\widetilde\theta_1\le2,\qquad
 \frac1{150}\le-\widetilde\theta_2\le\frac{10}{3},\qquad
 \frac{\mathcal L}{d}\le\frac CT.
\]
These bounds hold along the initial interval \(0\le\tau\le d^{-1/2}\)
and at every low-loss point in part (ii) of that lemma.
Let
\begin{equation}
 \Lambda_T=\exp\{CT^2\log(e+T)\}(n/d)^{CT^2}=d^{o(1)},
 \label{eq:stage2-coefficient-budget}
\end{equation}
where \(C\) is a sufficiently large fixed constant.

\begin{lemma}
\label{lem:stage2-coefficient-coordinate-bounds}
At these points, for all sufficiently large \(d\),
\[
 \sum_{t=1}^T\|p_t\|_{\mathrm{coeff}}^2
 +\|\widetilde{\bTheta}_4\|_F^2\le\Lambda_T,\qquad
 \max_{1\le t<T}
 \left\|\bigl([\lambda^i]p_j\bigr)_{1\le i,j\le t}^{-1}\right\|_2
 \le\Lambda_T.
\]
\end{lemma}

\begin{proof}
By Lemmas~\ref{lem:finite-transfer}
and~\ref{lem:mp-coefficient-norm} and the loss identity,
\[
\begin{aligned}
 \sum_{t=1}^T\|p_t-1\|_{\mathrm{MP}}^2
 &\le2\sum_{t=1}^T\|p_t-1\|_{d,n}^2
   =\frac{4T\mathcal L}{d}\le C,\\
 \sum_{t=1}^T\|p_t\|_{\mathrm{coeff}}^2
 &\le (T+1)16^T(n/d)^T\sum_{t=1}^T\|p_t\|_{\mathrm{MP}}^2\\
 &\le 2(T+1)16^T(n/d)^T
        \left(T+\frac{4T\mathcal L}{d}\right)
 \le\exp\{CT\log(e+T)\}(n/d)^{CT}.
\end{aligned}
\]
For \(1\le t<T\), define
\[
 \Mb_t=
 \begin{pmatrix}
 \vert&\vert&&\vert\\
 \operatorname{coeff}_{[t]}(p_1)&
 \operatorname{coeff}_{[t]}(p_2)&\cdots&
 \operatorname{coeff}_{[t]}(p_t)\\
 \vert&\vert&&\vert
 \end{pmatrix}
 =\bigl([\lambda^i]p_j\bigr)_{1\le i,j\le t}.
\]
The rows correspond to \(\lambda,\lambda^2,\ldots,\lambda^t\).
By Lemma~\ref{lem:global-recurrence}, \(\Mb_t\) is upper triangular
with diagonal \(\widetilde\theta_1\widetilde\theta_2^{j-1}\).
Let
\[
 \Db_t=\diag(\widetilde\theta_1,\widetilde\theta_1\widetilde\theta_2,
                    \ldots,\widetilde\theta_1\widetilde\theta_2^{t-1}),
 \qquad \Nb_t=\Db_t^{-1}\Mb_t-\Ib.
\]
Then \(\Nb_t\) is strictly upper triangular, and hence
\[
 \Nb_t^t=0,\qquad \Mb_t=\Db_t(\Ib+\Nb_t),\qquad
 (\Ib+\Nb_t)\sum_{j=0}^{t-1}(-\Nb_t)^j
     =\Ib-(-\Nb_t)^t=\Ib.
\]
The preceding coefficient estimate and the scalar bounds give
\[
 \|\Mb_t\|_2+\|\Db_t^{-1}\|_2
 \le\exp\{CT\log(e+T)\}(n/d)^{CT}.
\]
Therefore
\begin{equation}
\begin{aligned}
 \Mb_t^{-1}
 &=(\Ib+\Nb_t)^{-1}\Db_t^{-1}
   =\left(\Ib-\Nb_t+\cdots+(-\Nb_t)^{t-1}\right)\Db_t^{-1},\\
 \|\Mb_t^{-1}\|_2
 &\le\|\Db_t^{-1}\|_2\sum_{j=0}^{t-1}\|\Nb_t\|_2^j\\
 &\le t\|\Db_t^{-1}\|_2
        (1+\|\Db_t^{-1}\|_2\|\Mb_t\|_2)^{t-1}\\
 &\le\exp\{CT^2\log(e+T)\}(n/d)^{CT^2}.
\end{aligned}
 \label{eq:quantitative-history-inverse}
\end{equation}
Applying this inverse to the recurrence gives
\[
\begin{aligned}
 (\widetilde{\bTheta}_4)_{1:t,t}
 &=\Mb_t^{-1}\operatorname{coeff}_{[t]}
     [p_{t+1}-(1+\widetilde\theta_2\lambda)p_t],\\
 \|\widetilde{\bTheta}_4\|_F^2
 &\le C(1+\widetilde\theta_2^2)
       \max_{t<T}\|\Mb_t^{-1}\|_2^2
       \sum_{t=1}^T\|p_t\|_{\mathrm{coeff}}^2
 \le\Lambda_T.
\end{aligned}
\]

\end{proof}

\begin{lemma}
\label{lem:stage2-inverse-coordinate}
At the same points, for all sufficiently large \(d\),
\begin{equation}
 \left\|
  \frac{\partial\operatorname{vec}\widetilde{\bTheta}_4}
       {\partial(\widetilde\theta_1,\widetilde\theta_2,
                   \operatorname{vec}\widetilde\theta_5)}
       \right\|_{\mathrm{op}}^2\le\Lambda_T.
 \label{eq:stage2-inverse-coordinate-bound}
\end{equation}
Only the active entries occur in this derivative.
\end{lemma}

\begin{proof}
Use the history-inverse bounds in
Lemma~\ref{lem:stage2-coefficient-coordinate-bounds}. For any coordinate
\(\xi\in\{\widetilde\theta_1,\widetilde\theta_2,
       (\widetilde\theta_5)_{js}:1\le j\le s<T\}\),
the monomial representation gives
\[
\begin{aligned}
 \frac{\partial p_{t+1}}{\partial\widetilde\theta_1}
 &=\widetilde\theta_2^t\lambda^{t+1},&
 \frac{\partial p_{t+1}}{\partial\widetilde\theta_2}
 &=t\widetilde\theta_1\widetilde\theta_2^{t-1}\lambda^{t+1},&
 \frac{\partial p_{t+1}}{\partial(\widetilde\theta_5)_{js}}
 &=\ind_{\{s=t\}}\lambda^j .
\end{aligned}
\]
Also \(\partial p_1/\partial\widetilde\theta_1=\lambda\), and
its other derivatives vanish. Differentiating the recurrence by the
product rule and collecting the coefficients of \(p_1,\ldots,p_t\) gives
\begin{equation}
\begin{aligned}
 \frac{\partial(\widetilde{\bTheta}_4)_{1:t,t}}{\partial\xi}
 =\Mb_t^{-1}\operatorname{coeff}_{[t]}\Bigg[
 &\frac{\partial p_{t+1}}{\partial\xi}
 -(1+\widetilde\theta_2\lambda)\frac{\partial p_t}{\partial\xi}
 -\ind_{\{\xi=\widetilde\theta_2\}}\lambda p_t\\
 &-\sum_{\ell=1}^t(\widetilde{\bTheta}_4)_{\ell t}
                    \frac{\partial p_\ell}{\partial\xi}\Bigg].
\end{aligned}
 \label{eq:full-history-jacobian-inverse}
\end{equation}
The degree-\((t+1)\) coefficients cancel because
\([\lambda^{t+1}]p_{t+1}=\widetilde\theta_2[\lambda^t]p_t\).
The explicit derivatives imply
\[
\begin{aligned}
 \sum_\xi\sum_{t=1}^T
      \left\|\frac{\partial p_t}{\partial\xi}\right\|_{\mathrm{coeff}}^2
 &=1+\sum_{t=1}^{T-1}
       \left(|\widetilde\theta_2|^{2t}
       +t^2\widetilde\theta_1^2|\widetilde\theta_2|^{2t-2}+t\right)\\
 &\le CT^3(10/3)^{2T}.
\end{aligned}
\]
Taking norms and applying Cauchy--Schwarz yields
\[
\begin{aligned}
 \left\|
  \frac{\partial\operatorname{vec}\widetilde{\bTheta}_4}
       {\partial(\widetilde\theta_1,\widetilde\theta_2,
                   \operatorname{vec}\widetilde\theta_5)}
       \right\|_{\mathrm{op}}^2
 &\le\sum_\xi
       \left\|\frac{\partial\widetilde{\bTheta}_4}{\partial\xi}\right\|_F^2\\
 &\le CT\max_{t<T}\|\Mb_t^{-1}\|_2^2
 \Bigg[
 (1+\widetilde\theta_2^2+\|\widetilde{\bTheta}_4\|_F^2)
   \sum_\xi\sum_{t=1}^T
       \left\|\frac{\partial p_t}{\partial\xi}\right\|_{\mathrm{coeff}}^2\\
 &\hspace{25mm}+\sum_{t=1}^T\|p_t\|_{\mathrm{coeff}}^2\Bigg]\\
 &\le\Lambda_T .
\end{aligned}
\]
Only the active entries occur in the derivative and the sum.
\end{proof}

\begin{lemma}
\label{lem:stage2-conditional-minimizer-bounds}
At the same points, for all sufficiently large \(d\),
\begin{equation}
 \max_{1\le t<T}\|\Hb_t^{-1}\|_2
 +\sum_{i=1}^2
   \left\|\frac{\partial\widetilde\theta_5^\star
                 (\widetilde\theta_1,\widetilde\theta_2)}
                {\partial\widetilde\theta_i}\right\|_F^2
 \le\exp\{CT\log(e+T)\}(n/d)^{CT}\le\Lambda_T.
 \label{eq:stage2-normal-equation-bounds}
\end{equation}
The loss gradients satisfy \eqref{eq:stage2-theta5-derivatives}.
\end{lemma}

\begin{proof}
The matrix \(\Hb_t\) depends only on \(\nu_{d,n}\), so differentiating
the normal equations in Lemma~\ref{lem:stage2-loss-decomposition} gives
\begin{equation}
\begin{aligned}
 \Hb_t(\widetilde\theta_5^\star
                   (\widetilde\theta_1,\widetilde\theta_2))_{1:t,t}
 &=\left(\left\langle\lambda^j,
       1-\widetilde\theta_1\widetilde\theta_2^t\lambda^{t+1}
          \right\rangle_{d,n}\right)_{j=1}^t,\\
 \frac{\partial(\widetilde\theta_5^\star
                   (\widetilde\theta_1,\widetilde\theta_2))_{1:t,t}}
      {\partial\widetilde\theta_1}
 &=-\widetilde\theta_2^t\Hb_t^{-1}
       \left(\langle\lambda^j,\lambda^{t+1}\rangle_{d,n}\right)_{j=1}^t,\\
 \frac{\partial(\widetilde\theta_5^\star
                   (\widetilde\theta_1,\widetilde\theta_2))_{1:t,t}}
      {\partial\widetilde\theta_2}
 &=-t\widetilde\theta_1\widetilde\theta_2^{t-1}\Hb_t^{-1}
       \left(\langle\lambda^j,\lambda^{t+1}\rangle_{d,n}\right)_{j=1}^t.
\end{aligned}
 \label{eq:stage2-theta5-normal-equations}
\end{equation}
For \(a\in\R^t\), Lemmas~\ref{lem:finite-transfer}
and~\ref{lem:mp-coefficient-norm} give
\[
\begin{aligned}
 a^{\mathsf T}\Hb_ta
 &=\sum_{j,k=1}^t a_ja_k\E_{\lambda\sim\nu_{d,n}}[\lambda^{j+k}]
   =\left\|\sum_{j=1}^ta_j\lambda^j\right\|_{d,n}^2\\
 &\ge\frac12\left\|\sum_{j=1}^ta_j\lambda^j\right\|_{\mathrm{MP}}^2
 \ge\frac{(d/n)^T}{2(T+1)16^T}\|a\|_2^2,\\
 \|\Hb_t^{-1}\|_2
 &=\lambda_{\min}(\Hb_t)^{-1}
 \le2(T+1)16^T(n/d)^T.
\end{aligned}
\]
By Lemma~\ref{lem:wishart}, the moment vector satisfies
\[
\begin{aligned}
 \left\|\left(\langle\lambda^j,\lambda^{t+1}
                       \rangle_{d,n}\right)_{j=1}^t\right\|_2^2
 &=\sum_{j=1}^t
       \left(\E_{\lambda\sim\nu_{d,n}}[\lambda^{j+t+1}]\right)^2\\
 &\le T\left(\max_{r\le2T}
                    \E_{\lambda\sim\nu_{d,n}}[\lambda^r]\right)^2
 \le\exp\{CT\log(e+T)\}.
\end{aligned}
\]
Furthermore,
\[
 \sum_{i=1}^2
 \left|\frac{\partial(\widetilde\theta_1\widetilde\theta_2^t)}
                   {\partial\widetilde\theta_i}\right|^2
 =|\widetilde\theta_2|^{2t}
       +t^2\widetilde\theta_1^2|\widetilde\theta_2|^{2t-2}
 \le CT^2(10/3)^{2T}.
\]
Substituting these three bounds into the differentiated normal equations,
\[
\begin{aligned}
 &\sum_{i=1}^2
 \left\|\frac{\partial(\widetilde\theta_5^\star
             (\widetilde\theta_1,\widetilde\theta_2))_{1:t,t}}
             {\partial\widetilde\theta_i}\right\|_2^2\\
 &\quad\le \|\Hb_t^{-1}\|_2^2
 \left\|\left(\langle\lambda^j,\lambda^{t+1}\rangle_{d,n}
                          \right)_{j=1}^t\right\|_2^2
 \sum_{i=1}^2
 \left|\frac{\partial(\widetilde\theta_1\widetilde\theta_2^t)}
                   {\partial\widetilde\theta_i}\right|^2\\
 &\quad\le [2(T+1)16^T(n/d)^T]^2
        \exp\{CT\log(e+T)\}\,CT^2(10/3)^{2T}\\
 &\quad\le\exp\{CT\log(e+T)\}(n/d)^{CT}.
\end{aligned}
\]
Summing over \(t\) proves \eqref{eq:stage2-normal-equation-bounds}.
Differentiating \eqref{eq:stage2-completed-square} gives
\begin{equation}
\begin{aligned}
 \left(\nabla_{\widetilde\theta_5}\frac{\mathcal L_A}{d}\right)_{1:t,t}
 &=\left(\nabla_{\widetilde\theta_5}
     \frac{\mathcal L(\widetilde\theta_1,\widetilde\theta_2,
                        \widetilde\theta_5)}d\right)_{1:t,t}\\
 &=\frac1T\Hb_t
     (\widetilde\theta_5-
        \widetilde\theta_5^\star(\widetilde\theta_1,\widetilde\theta_2))_{1:t,t},\\
 \nabla_{\widetilde\theta_i}
     \frac{\mathcal L(\widetilde\theta_1,\widetilde\theta_2,
                        \widetilde\theta_5)}d
 &=\nabla_{\widetilde\theta_i}\frac{\mathcal L_B}{d}
  -\left\langle\frac{\partial\widetilde\theta_5^\star
                          (\widetilde\theta_1,\widetilde\theta_2)}
                         {\partial\widetilde\theta_i},
        \nabla_{\widetilde\theta_5}
          \frac{\mathcal L(\widetilde\theta_1,\widetilde\theta_2,
                            \widetilde\theta_5)}d\right\rangle_F,
 \quad i=1,2.
\end{aligned}
 \label{eq:stage2-theta5-derivatives}
\end{equation}
\end{proof}

\begin{proof}[Proof of Lemma~\ref{lem:auxiliary-parameter-bounds}(iii)]
For \(0\le\tau\le d^{-1/2}\), use
\eqref{eq:stage2-initial-parameter-bounds}, proved in part (i).
Cauchy--Schwarz and \eqref{eq:normalized-flow} give
\begin{align*}
 |w_{24}(\tau)|^2+\|\wb_{25}(\tau)\|_2^2
 &\le2\bigl(|w_{24}(0)|^2+\|\wb_{25}(0)\|_2^2\bigr)
       +2\tau\int_0^\tau
            (|\dot w_{24}|^2+\|\dot\wb_{25}\|_2^2)\,\mathrm ds\\
 &\le d^{o(1)}+2d\tau\,\frac{\mathcal L(0)-\mathcal L(\tau)}d
 \le d^{o(1)}+\frac{Cd\tau}{T},\\
 \sum_{t=1}^{T-1}|w_{24}+\wb_{25}\eb_t|^2
 &\le2T|w_{24}|^2+2\|\wb_{25}\|_2^2
 \le d^{o(1)}+Cd\tau.
\end{align*}
(The initial bound is Lemma~\ref{lem:guide-handoff-raw-bound}.)

By \eqref{eq:global-auxiliary}--\eqref{eq:global-theta4} and the chain rule,
\begin{align*}
 &\nabla_{\widetilde\theta_1}
       \frac{\mathcal L(\widetilde\theta_1,\widetilde\theta_2,
                                      \widetilde\theta_5)}d\\
 &\quad=\frac1{v_{31}}\nabla_{(\wb_{25})_0}\frac{\mathcal L}{d}
  +\frac{v_{31}}{\widetilde\theta_1^2}
       \sum_{t=1}^{T-1}(w_{24}+\wb_{25}\eb_t)
          \nabla_{(\widetilde{\bTheta}_4)_{1t}}\frac{\mathcal L}{d}
  +\left\langle\nabla_{\widetilde\theta_1}\widetilde{\bTheta}_4,
          \nabla_{\widetilde{\bTheta}_4}\frac{\mathcal L}{d}\right\rangle_F,\\
 &\left(\nabla_{\Wb_{55}}\frac{\mathcal L}{d}\right)_{\ell t}
   =v_{33}\nabla_{(\widetilde{\bTheta}_4)_{\ell t}}\frac{\mathcal L}{d},
 \qquad 1\le\ell\le t<T.
\end{align*}
(The derivative of \(\widetilde{\bTheta}_4\) holds
\(\widetilde\theta_2,\widetilde\theta_5\) fixed.)
Lemma~\ref{lem:stage2-inverse-coordinate} therefore gives
\begin{align*}
 &\left|\nabla_{\widetilde\theta_1}
       \frac{\mathcal L(\widetilde\theta_1,\widetilde\theta_2,
                                      \widetilde\theta_5)}d\right|^2\\
 &\quad\le C\left|\nabla_{(\wb_{25})_0}\frac{\mathcal L}{d}\right|^2
   +C\left(\Lambda_T+
            \sum_{t=1}^{T-1}|w_{24}+\wb_{25}\eb_t|^2\right)
        \left\|\nabla_{\widetilde{\bTheta}_4}\frac{\mathcal L}{d}\right\|_F^2,\\
 &\left\|\nabla_{\widetilde\theta_5}\frac{\mathcal L}{d}\right\|_F^2
 \le\Lambda_T
        \left\|\nabla_{\widetilde{\bTheta}_4}\frac{\mathcal L}{d}\right\|_F^2.
\end{align*}
Write \(\nu=\nu_{d,n}\). The normal equations and scalar completion of
the square in Lemma~\ref{lem:stage2-loss-decomposition} give
\begin{align*}
 &\frac{\mathcal L_A+\mathcal L_B-\min_{\widetilde\theta_1}\mathcal L_B}d\\
 &\quad=\frac{\mathcal L_A}{d}
       +\frac{C_\nu(-\widetilde\theta_2)}{2T}
          \bigl(\widetilde\theta_1-
                 \widetilde\theta_1^\star(\widetilde\theta_2)\bigr)^2\\
 &\quad=\frac T2\sum_{t=1}^{T-1}
       \left(\nabla_{\widetilde\theta_5}\frac{\mathcal L}{d}\right)_{1:t,t}^{\mathsf T}
       \Hb_t^{-1}
       \left(\nabla_{\widetilde\theta_5}\frac{\mathcal L}{d}\right)_{1:t,t}\\
 &\qquad+\frac{T}{2C_\nu(-\widetilde\theta_2)}
       \left[
       \nabla_{\widetilde\theta_1}
          \frac{\mathcal L(\widetilde\theta_1,\widetilde\theta_2,
                                         \widetilde\theta_5)}d
       +\left\langle
          \nabla_{\widetilde\theta_1}
              \widetilde\theta_5^\star(\widetilde\theta_1,\widetilde\theta_2),
          \nabla_{\widetilde\theta_5}\frac{\mathcal L}{d}
        \right\rangle_F\right]^2\\
 &\quad\le CT\left|\nabla_{(\wb_{25})_0}\frac{\mathcal L}{d}\right|^2
       +\left(d^{o(1)}
            +CT\sum_{t=1}^{T-1}|w_{24}+\wb_{25}\eb_t|^2\right)
          \left\|\nabla_{\widetilde{\bTheta}_4}\frac{\mathcal L}{d}\right\|_F^2\\
 &\quad\le CT\bigl(d^{o(1)}+d\tau\bigr)
       \left(\left\|\nabla_{\wb_{25}}\frac{\mathcal L}{d}\right\|_2^2
            +\left\|\nabla_{\Wb_{55}}\frac{\mathcal L}{d}\right\|_F^2\right).
\end{align*}
(Use \(C_\nu\ge1\), Lemma~\ref{lem:stage2-conditional-minimizer-bounds},
\(\Lambda_T^2=d^{o(1)}\), and the displacement bound above.)
Thus \eqref{eq:normalized-flow} implies
\begin{equation}
 -\frac{\mathrm d}{\mathrm d\tau}\frac{\mathcal L}{d}
 \ge\frac{cd}{T(d^{o(1)}+d\tau)}
       \frac{\mathcal L_A+\mathcal L_B-\min_{\widetilde\theta_1}\mathcal L_B}d.
 \label{eq:entry-conditional-pl}
\end{equation}
(Each \(d^{o(1)}\) above is a bound independent of \(\tau\).)

Lemmas~\ref{lem:warm-gap} and \ref{lem:finite-BC-comparison} give
\[
 \frac{2T}{d}\left(\mathcal L^\star+\min_{\widetilde\theta_1}\mathcal L_B\right)
 \le\frac{149}{100}\frac dn+Cd^{-1/4}\le\frac85\frac dn.
\]
Before \eqref{eq:finite-warm-gap} is reached,
\begin{align*}
 -\frac{\mathrm d}{\mathrm d\tau}
       \left(\frac{\mathcal L}{d}-\frac4{5T}\frac dn\right)
 &\ge\frac{cd}{T(d^{o(1)}+d\tau)}
       \left(\frac{\mathcal L}{d}-\frac4{5T}\frac dn\right),\\
 \frac{\mathcal L(\tau)}d-\frac4{5T}\frac dn
 &\le\left(\frac{\mathcal L(0)}d-\frac4{5T}\frac dn\right)
       \exp\left\{-\frac cT
          \log\left(1+\frac{d\tau}{d^{o(1)}}\right)\right\}\\
 &\le\frac{49}{100T}
       \exp\left\{-\frac cT
          \log\left(1+\frac{d\tau}{d^{o(1)}}\right)\right\}.
\end{align*}
At \(\tau=d^{-1/2}\), this yields
\[
\begin{aligned}
 \frac{2T\mathcal L(d^{-1/2})}d
 &\le\frac85\frac dn
       +\frac{49}{50}
         \exp\left\{-\left(\frac c2+o(1)\right)\frac{\log d}{T}\right\}\\
 &<\frac74\frac dn
 \quad\left(T=\Theta((\log d)^{1/4}),\ \frac dn=\Theta(1)\right).
\end{aligned}
\]
Hence \eqref{eq:finite-warm-gap} is reached before \(d^{-1/2}\).
Loss monotonicity preserves \eqref{eq:finite-warm-gap} thereafter.
Combining (i) and (ii) proves \eqref{eq:stage2-global-parameter-bounds}.
\end{proof}

\subsection{Proof of Lemma~\ref{lem:stage2-pl-inequality}}
\begin{proof}
Lemma~\ref{lem:auxiliary-parameter-bounds}(iii) gives
\[
 \frac1{100}\le\widetilde\theta_1\le2,\qquad
 \frac1{150}\le-\widetilde\theta_2\le\frac{10}{3},\qquad
 \frac dn\widetilde\theta_2^2\le\frac9{10},\qquad
 \frac{\mathcal L}{d}\le\frac CT.
\]
Write \(\nu=\nu_{d,n}\).

\par\medskip\noindent\textbf{(i)}
We first use the coordinates
\((\widetilde\theta_1,\widetilde\theta_2,\widetilde\theta_5)\).
By Lemma~\ref{lem:finite-BC-comparison},
\begin{equation}
 1\le C_\nu(-\widetilde\theta_2)\le C,\qquad
 \sum_{j=0}^2\left(
       |B_\nu^{(j)}(-\widetilde\theta_2)|
       +|C_\nu^{(j)}(-\widetilde\theta_2)|\right)\le C.
 \label{eq:stage2-BC-derivative-bounds}
\end{equation}
The two bounds \eqref{eq:finite-scalar-quadratic-gap} and
\eqref{eq:stage2-scalar-direction} in
Lemma~\ref{lem:finite-scalar-landscape} imply
\begin{align*}
 &\frac{B_\nu(-\widetilde\theta_2^\star)^2}{C_\nu(-\widetilde\theta_2^\star)}
    -\frac{B_\nu(-\widetilde\theta_2)^2}{C_\nu(-\widetilde\theta_2)}\\
 &\quad\le C|\widetilde\theta_2-\widetilde\theta_2^\star|
       \left|\nabla_{\widetilde\theta_2}
          \frac{B_\nu(-\widetilde\theta_2)^2}{C_\nu(-\widetilde\theta_2)}\right|\\
 &\quad\le\frac12\left[
    \frac{B_\nu(-\widetilde\theta_2^\star)^2}{C_\nu(-\widetilde\theta_2^\star)}
    -\frac{B_\nu(-\widetilde\theta_2)^2}{C_\nu(-\widetilde\theta_2)}\right]
    +C\left|\nabla_{\widetilde\theta_2}
          \frac{B_\nu(-\widetilde\theta_2)^2}{C_\nu(-\widetilde\theta_2)}\right|^2.
\end{align*}
(The second inequality uses the quadratic gap and Cauchy--Schwarz.)
Differentiating \(\mathcal L_B\) in
Lemma~\ref{lem:stage2-loss-decomposition} gives
\begin{align*}
 T\nabla_{\widetilde\theta_1}\frac{\mathcal L_B}{d}
 &=C_\nu(-\widetilde\theta_2)
       \bigl(\widetilde\theta_1-\widetilde\theta_1^\star(\widetilde\theta_2)\bigr),\\
 2T\nabla_{\widetilde\theta_2}\frac{\mathcal L_B}{d}
 &=-2\widetilde\theta_1\nabla_{\widetilde\theta_2}B_\nu(-\widetilde\theta_2)
    +\widetilde\theta_1^2\nabla_{\widetilde\theta_2}C_\nu(-\widetilde\theta_2),\\
 \nabla_{\widetilde\theta_2}
       \frac{B_\nu(-\widetilde\theta_2)^2}{C_\nu(-\widetilde\theta_2)}
 &=-2T\nabla_{\widetilde\theta_2}\frac{\mathcal L_B}{d}\\
 &\quad+\left[
       \bigl(\widetilde\theta_1+\widetilde\theta_1^\star(\widetilde\theta_2)\bigr)
               \nabla_{\widetilde\theta_2}C_\nu(-\widetilde\theta_2)
       -2\nabla_{\widetilde\theta_2}B_\nu(-\widetilde\theta_2)\right]
       \bigl(\widetilde\theta_1-\widetilde\theta_1^\star(\widetilde\theta_2)\bigr).
\end{align*}
The bracket is bounded by \eqref{eq:stage2-BC-derivative-bounds}, so
\begin{align*}
 &|\widetilde\theta_1-\widetilde\theta_1^\star(\widetilde\theta_2)|^2
    +\left|\nabla_{\widetilde\theta_2}
          \frac{B_\nu(-\widetilde\theta_2)^2}{C_\nu(-\widetilde\theta_2)}\right|^2
 \le CT^2\sum_{i=1}^2
       \left|\nabla_{\widetilde\theta_i}\frac{\mathcal L_B}{d}\right|^2,\\
 &\frac{\mathcal L_B}{d}
 \le\frac CT\left[
       |\widetilde\theta_1-\widetilde\theta_1^\star(\widetilde\theta_2)|^2
       +\left|\nabla_{\widetilde\theta_2}
          \frac{B_\nu(-\widetilde\theta_2)^2}{C_\nu(-\widetilde\theta_2)}\right|^2
       \right]
 \le CT\sum_{i=1}^2
       \left|\nabla_{\widetilde\theta_i}\frac{\mathcal L_B}{d}\right|^2.
\end{align*}
For \(\mathcal L_A\), Lemma~\ref{lem:stage2-conditional-minimizer-bounds} gives
\begin{align*}
 \frac{\mathcal L_A}{d}
 &=\frac T2\sum_{t=1}^{T-1}
   \left(\nabla_{\widetilde\theta_5}\frac{\mathcal L}{d}\right)_{1:t,t}^{\mathsf T}
   \Hb_t^{-1}
   \left(\nabla_{\widetilde\theta_5}\frac{\mathcal L}{d}\right)_{1:t,t}
 \le\frac{T\Lambda_T}{2}
       \left\|\nabla_{\widetilde\theta_5}\frac{\mathcal L}{d}\right\|_F^2,\\
 \nabla_{\widetilde\theta_i}\frac{\mathcal L_B}{d}
 &=\nabla_{\widetilde\theta_i}\frac{\mathcal L}{d}
   +\left\langle
       \nabla_{\widetilde\theta_i}
          \widetilde\theta_5^\star(\widetilde\theta_1,\widetilde\theta_2),
       \nabla_{\widetilde\theta_5}\frac{\mathcal L}{d}\right\rangle_F,
       \qquad i=1,2,\\
 \sum_{i=1}^2
       \left|\nabla_{\widetilde\theta_i}\frac{\mathcal L_B}{d}\right|^2
 &\le2\sum_{i=1}^2
       \left|\nabla_{\widetilde\theta_i}\frac{\mathcal L}{d}\right|^2
       +2\Lambda_T
       \left\|\nabla_{\widetilde\theta_5}\frac{\mathcal L}{d}\right\|_F^2.
\end{align*}
Combining these estimates with
\(\mathcal L-\mathcal L^\star=\mathcal L_A+\mathcal L_B\) proves
\begin{equation}
 \sum_{i=1}^2
       \left|\nabla_{\widetilde\theta_i}
          \frac{\mathcal L(\widetilde\theta_1,\widetilde\theta_2,\widetilde\theta_5)}d
       \right|^2
 +\Lambda_T
       \left\|\nabla_{\widetilde\theta_5}\frac{\mathcal L}{d}\right\|_F^2
 \ge\frac cT\,\frac{\mathcal L-\mathcal L^\star}{d}.
 \label{eq:stage2-auxiliary-PL}
\end{equation}

\par\medskip\noindent\textbf{(ii)}
Return to
\((\widetilde\theta_1,\widetilde\theta_2,\widetilde{\bTheta}_4)\).
For \(\xi=\widetilde\theta_1,\widetilde\theta_2\) or an active entry of
\(\widetilde\theta_5\), the chain rule gives
\[
 \nabla_\xi
     \frac{\mathcal L(\widetilde\theta_1,\widetilde\theta_2,\widetilde\theta_5)}d
 =\sum_{i=1}^2\ind_{\{\xi=\widetilde\theta_i\}}
       \nabla_{\widetilde\theta_i}
     \frac{\mathcal L(\widetilde\theta_1,\widetilde\theta_2,\widetilde{\bTheta}_4)}d
  +\left\langle\nabla_\xi\widetilde{\bTheta}_4,
       \nabla_{\widetilde{\bTheta}_4}\frac{\mathcal L}{d}\right\rangle_F.
\]
Cauchy--Schwarz and Lemma~\ref{lem:stage2-inverse-coordinate} imply
\begin{align*}
 \left\|\nabla_{\widetilde\theta_5}\frac{\mathcal L}{d}\right\|_F^2
 &\le\Lambda_T
       \left\|\nabla_{\widetilde{\bTheta}_4}\frac{\mathcal L}{d}\right\|_F^2,\\
 \sum_{i=1}^2
       \left|\nabla_{\widetilde\theta_i}
          \frac{\mathcal L(\widetilde\theta_1,\widetilde\theta_2,\widetilde\theta_5)}d
       \right|^2
 &\le2\sum_{i=1}^2
       \left|\nabla_{\widetilde\theta_i}
          \frac{\mathcal L(\widetilde\theta_1,\widetilde\theta_2,\widetilde{\bTheta}_4)}d
       \right|^2\\
 &\quad+2\Lambda_T
       \left\|\nabla_{\widetilde{\bTheta}_4}\frac{\mathcal L}{d}\right\|_F^2.
\end{align*}
Substitution into \eqref{eq:stage2-auxiliary-PL} gives
\[
 2\sum_{i=1}^2
       \left|\nabla_{\widetilde\theta_i}\frac{\mathcal L}{d}\right|^2
 +(2\Lambda_T+\Lambda_T^2)
       \left\|\nabla_{\widetilde{\bTheta}_4}\frac{\mathcal L}{d}\right\|_F^2
 \ge\frac cT\,\frac{\mathcal L-\mathcal L^\star}{d}.
\]
Since \(2\Lambda_T+\Lambda_T^2=d^{o(1)}\le2\sqrt d\), multiplying by
\(d/2\) yields
\begin{equation}
 \frac1d\sum_{i=1}^2|\nabla_{\widetilde\theta_i}\mathcal L|^2
 +\frac1{\sqrt d}\|\nabla_{\widetilde{\bTheta}_4}\mathcal L\|_F^2
 \ge c\frac{d/n}{T}(\mathcal L-\mathcal L^\star)
 \quad(d/n\le1/10).
 \label{eq:stage2-theta4-PL}
\end{equation}
\end{proof}

\subsection{Raw-parameter estimates}
\label{app:convergence}

\begin{lemma}\label{lem:right-inverse}
Under \eqref{eq:common-growing-regime}, along the Stage-2 flow, every
variation of the active auxiliary parameters has a raw lift satisfying
\begin{align}
 &(\mathrm d v_{31})^2+(\mathrm d w_{13})^2+(\mathrm d v_{33})^2
 +\frac1d\bigl\{(\mathrm d w_{24})^2+\|\mathrm d\wb_{25}\|^2
 +\|\mathrm d w_{44}\|_F^2+\|\mathrm d\wb_{45}\|_F^2\bigr\}\notag\\
 &\quad+\frac1d\bigl\{\|\mathrm d\wb_{54}\|_F^2
                       +\|\mathrm d\Wb_{55}\|_F^2\bigr\}\notag\\
 &\qquad\le C\bigl(1+(d/n)^{-1/2}+v_{31}^2+v_{33}^2\bigr)
 \left\{|\mathrm d\widetilde\theta_1|^2+|\mathrm d\widetilde\theta_2|^2
             +\frac1d\|\mathrm d\widetilde{\bTheta}_4\|_F^2\right\}.
 \label{eq:right-inverse}
\end{align}
\end{lemma}

\begin{proof}
As in the proof of the Stage-1 Jacobian bound, write
\[
 \Jb_2=\nabla_{\bar{\btheta}}
 \bigl(\widetilde\theta_1,\widetilde\theta_2,
        \widetilde{\bTheta}_4\bigr),
 \qquad
 \Jb_2\,\mathrm d\bar{\btheta}
 =\bigl(\mathrm d\widetilde\theta_1,
         \mathrm d\widetilde\theta_2,
         \mathrm d\widetilde{\bTheta}_4\bigr).
\]
For a prescribed auxiliary variation, choose
\begin{equation}
\begin{aligned}
 \mathrm d v_{31}
 &=\frac{v_{31}}{\widetilde\theta_1}\mathrm d\widetilde\theta_1,\qquad
 \mathrm d w_{13}
 =\frac{\mathrm d\widetilde\theta_2
       -(\widetilde\theta_2/\widetilde\theta_1)\mathrm d\widetilde\theta_1}
       {v_{31}},\\
 \mathrm d v_{33}&=\mathrm d w_{24}=\mathrm d\wb_{25}
               =\mathrm d w_{44}=\mathrm d\wb_{45}=\mathrm d\wb_{54}=0,
\end{aligned}
 \label{eq:right-inverse-spatial}
\end{equation}
and, for \(1\le t<T\),
\begin{equation}
\begin{aligned}
 \mathrm d(\Wb_{55})_{\ell t}
 &=\frac{(\mathrm d\widetilde{\bTheta}_4)_{\ell t}}{v_{33}},
 \qquad 1\le\ell\le t.
\end{aligned}
 \label{eq:right-inverse-history}
\end{equation}
All other entries of \(\mathrm d\Wb_{55}\) are zero.
Substitution into \eqref{eq:global-auxiliary} and
\eqref{eq:global-theta4} gives, entry by entry,
\[
\begin{aligned}
 \langle\nabla_{\bar{\btheta}}\widetilde\theta_1,\mathrm d\bar{\btheta}\rangle
 &=\frac{\widetilde\theta_1}{v_{31}}\mathrm d v_{31}
   =\mathrm d\widetilde\theta_1,\\
 \langle\nabla_{\bar{\btheta}}\widetilde\theta_2,\mathrm d\bar{\btheta}\rangle
 &=w_{13}\mathrm d v_{31}+v_{31}\mathrm d w_{13}
   =\mathrm d\widetilde\theta_2,\\
 \langle\nabla_{\bar{\btheta}}(\widetilde{\bTheta}_4)_{1t},
                         \mathrm d\bar{\btheta}\rangle
 &=v_{33}\mathrm d(\Wb_{55})_{1t}
   =(\mathrm d\widetilde{\bTheta}_4)_{1t},\\
 \langle\nabla_{\bar{\btheta}}(\widetilde{\bTheta}_4)_{\ell t},
                         \mathrm d\bar{\btheta}\rangle
 &=v_{33}\mathrm d(\Wb_{55})_{\ell t}
   =(\mathrm d\widetilde{\bTheta}_4)_{\ell t},
 \qquad 2\le\ell\le t<T.
\end{aligned}
\]
Lemma~\ref{lem:auxiliary-parameter-bounds}(iii) gives
\[
 c\le\widetilde\theta_1\le C,\qquad
 c\sqrt{d/n}\le-\widetilde\theta_2\le C.
\]
The inherited invariants imply
\[
\begin{aligned}
 v_{31}^4+\frac{\sigma^2}{d}v_{31}^2
 &\ge\widetilde\theta_2^2\ge c(d/n),
 &|v_{31}|&\ge c(d/n)^{1/4},\\
 v_{33}^2
 &=(v_{33}^{\rm init})^2+
   \frac1d\sum_{B\in\{44,45,54,55\}}\|\Wb_B\|_F^2,
 &v_{33}&\ge v_{33}^{\rm init}\ge c.
\end{aligned}
\]
The first lower bound uses \(\sigma^2/d=o(\sqrt{d/n})\).
Thus
\[
\begin{aligned}
 |\mathrm d v_{31}|^2+|\mathrm d w_{13}|^2
 &\le C\bigl(v_{31}^2+(d/n)^{-1/2}\bigr)
          \bigl(|\mathrm d\widetilde\theta_1|^2
               +|\mathrm d\widetilde\theta_2|^2\bigr),\\
 \frac1d\|\mathrm d\Wb_{55}\|_F^2
 &\le\frac C d\|\mathrm d\widetilde{\bTheta}_4\|_F^2.
\end{aligned}
\]
Thus the displayed variation is a right inverse of the auxiliary map.
The last two estimates prove \eqref{eq:right-inverse}.
\end{proof}

\subsection{Proof of Lemma~\ref{lem:stage2-jacobian}}

\begin{proof}
As in the proof of the Stage-1 Jacobian bound, use the right inverse
constructed in Lemma~\ref{lem:right-inverse}. We have
\[
\begin{aligned}
 &\Jb_2\,\mathrm d\bar{\btheta}
 =(\mathrm d\widetilde\theta_1,\mathrm d\widetilde\theta_2,
              \bigl((\mathrm d\widetilde{\bTheta}_4)_{\ell t}\bigr)_{1\le\ell\le t<T}^{\mathsf T})^{\mathsf T},\\
 &\|\mathrm d\bar{\btheta}\|^2
 \le C\bigl(1+(d/n)^{-1/2}+v_{31}^2+v_{33}^2\bigr)
 \left\{d|\mathrm d\widetilde\theta_1|^2
       +d|\mathrm d\widetilde\theta_2|^2
       +\|\mathrm d\widetilde{\bTheta}_4\|_F^2\right\}.
\end{aligned}
\]
The factors \(d\) come from
\(\|\mathrm d v_{31}\Ib_d\|_F^2=d|\mathrm d v_{31}|^2\) and
\(\|\mathrm d w_{13}\Ib_d\|_F^2=d|\mathrm d w_{13}|^2\).
Since \(\Jb_2\,\mathrm d\bar{\btheta}
=(\mathrm d\widetilde\theta_1,\mathrm d\widetilde\theta_2,
\mathrm d\widetilde{\bTheta}_4)\), the minimum-norm characterization of
\(\Jb_2\Jb_2^{\mathsf T}\) gives
\[
\begin{aligned}
 \Jb_2\Jb_2^{\mathsf T}
 &\succeq\frac{c}{1+(d/n)^{-1/2}+v_{31}^2+v_{33}^2}
             \diag(d^{-1},d^{-1},\Ib)\\
 &\succeq\frac{c}{1+(d/n)^{-1/2}+v_{31}^2+v_{33}^2}
             \diag(d^{-1},d^{-1},d^{-1/2}\Ib).
\end{aligned}
\]
Apply \eqref{eq:gauge-growth} in
Lemma~\ref{lem:auxiliary-parameter-bounds}(iii).
Consequently,
\[
 \Jb_2\Jb_2^{\mathsf T}
 \succeq\frac{c}{d^{o(1)}+\frac{d/n}{T}\tau}
              \diag(d^{-1},d^{-1},d^{-1/2}\Ib).
\]
In the time convention of Lemma~\ref{lem:stage2-jacobian}, replace
the elapsed time by \(\tau-\tau_1\le\tau\).
\end{proof}

\begin{lemma}\label{lem:weighted-pl}
In the fixed-aspect-ratio regime of Theorem~\ref{thm:main-result}, there is a constant
\(c>0\) such that along the Stage-2 flow,
\begin{equation}
 -\frac{\mathrm d}{\mathrm d\tau}
 \frac{\mathcal L}d
 \ge\frac{c(d/n)}
 {T\bigl(1+(d/n)^{-1/2}+v_{31}^2+v_{33}^2\bigr)}
 \left\{\frac{\mathcal L-\mathcal L^*}{d}\right\}.
                                                               \label{eq:weighted-pl}
\end{equation}
\end{lemma}

\begin{proof}
Similar to the proof of Lemma~\ref{lem:guide-PL} in Section~B,
the chain rule, Lemma~\ref{lem:right-inverse}, and Cauchy--Schwarz give
\begin{align*}
 -\frac{\mathrm d}{\mathrm d\tau}\frac{\mathcal L}{d}
 &=\frac1d\|\nabla_{\bar{\btheta}}\mathcal L\|^2\\
 &\ge\frac{c}{1+(d/n)^{-1/2}+v_{31}^2+v_{33}^2}
       \left[
        \sum_{i=1}^2
           \left|\nabla_{\widetilde\theta_i}\frac{\mathcal L}{d}\right|^2
        +d\left\|\nabla_{\widetilde{\bTheta}_4}
                                      \frac{\mathcal L}{d}\right\|_F^2
       \right]\\
 &\ge\frac{c}{1+(d/n)^{-1/2}+v_{31}^2+v_{33}^2}
       \left[
        \sum_{i=1}^2
           \left|\nabla_{\widetilde\theta_i}\frac{\mathcal L}{d}\right|^2
        +\sqrt d\left\|\nabla_{\widetilde{\bTheta}_4}
                                      \frac{\mathcal L}{d}\right\|_F^2
       \right]\\
 &\ge\frac{c(d/n)}
       {T(1+(d/n)^{-1/2}+v_{31}^2+v_{33}^2)}
       \frac{\mathcal L-\mathcal L^\star}{d}.
\end{align*}
(The last line is Lemma~\ref{lem:stage2-pl-inequality}.)
\end{proof}

\begin{proof}[Convergence part of Theorem~\ref{thm:main-result}]
As in the proof of Theorem~\ref{thm:main-result} for Stage 1,
Lemma~\ref{lem:auxiliary-parameter-bounds}(iii) gives global existence,
the parameter bounds, and entry into the low-loss region. Hence
Lemmas~\ref{lem:stage2-coefficient-coordinate-bounds} and
\ref{lem:weighted-pl} apply from Stage-2 time zero.

\par\medskip\noindent\textbf{(i)}
By \eqref{eq:gauge-growth} and \eqref{eq:weighted-pl},
\[
 -\frac{\mathrm d}{\mathrm d\tau}
       \frac{\mathcal L-\mathcal L^\star}{d}
 \ge\frac{c(d/n)/T}{d^{o(1)}+C(d/n)\tau/T}
       \frac{\mathcal L-\mathcal L^\star}{d}.
\]
Since, exactly as in the Stage-1 proof,
\[
 \frac{\mathcal L(0)-\mathcal L^\star}{d}
 \le\frac{\mathcal L(0)}d
 \le\frac{49}{100T}\le C\frac{d/n}{T},
\]
integration gives
\begin{equation}
 \frac{\mathcal L(\tau)-\mathcal L^\star}{d}
 \le C\frac{d/n}{T}
       \left(1+\frac{(d/n)\tau}{T\,d^{o(1)}}\right)^{-\alpha},
 \qquad \alpha>0.
 \label{eq:uniform-polynomial-rate}
\end{equation}

\par\medskip\noindent\textbf{(ii)}
By \eqref{eq:stage2-loss-AB-definition},
\[
 \frac{\mathcal L-\mathcal L^\star}{d}
 =\frac{\mathcal L_A}{d}+\frac{\mathcal L_B}{d}.
\]
Write \(\nu=\nu_{d,n}\), with
\(\widehat R_{T,t}=p_{t,\widetilde{\btheta}'^\star}-1\).
Lemma~\ref{lem:auxiliary-parameter-bounds}(iii), and (ii) applied at the
minimizer using Lemma~\ref{lem:warm-gap}, place both points in
\[
 \frac1{100}\le\widetilde\theta_1\le2,\qquad
 \frac1{150}\le(-\widetilde\theta_2)\le\frac{10}{3},\qquad
 (d/n)(-\widetilde\theta_2)^2\le\frac9{10}.
\]
By \eqref{eq:finite-BC-extended-bounds} in
Lemma~\ref{lem:finite-BC-comparison},
\[
 c\le C_\nu(-\widetilde\theta_2)\le C,\qquad
 \left|\left(\frac{B_\nu}{C_\nu}\right)'(-\widetilde\theta_2)\right|\le C.
\]
Hence
\[
 \left|\widetilde\theta_1^\star(\widetilde\theta_2)
             -\widetilde\theta_1^\star\right|
 =\left|\int_{-\widetilde\theta_2^\star}^{-\widetilde\theta_2}
          \left(\frac{B_\nu}{C_\nu}\right)'(x)\,\mathrm dx\right|
 \le C|\widetilde\theta_2-\widetilde\theta_2^\star|.
\]
The definition of \(\mathcal L_B\) in \eqref{eq:stage2-completed-square}
and the two inequalities of \eqref{eq:finite-scalar-quadratic-gap} in
Lemma~\ref{lem:finite-scalar-landscape} give
\begin{align*}
 &|\widetilde\theta_1-\widetilde\theta_1^\star|^2
  +|\widetilde\theta_2-\widetilde\theta_2^\star|^2\\
 &\quad\le2\left|\widetilde\theta_1-
                  \widetilde\theta_1^\star(\widetilde\theta_2)\right|^2
       +C|\widetilde\theta_2-\widetilde\theta_2^\star|^2
 \le CT\frac{\mathcal L_B}{d},\\
 \frac{\mathcal L_B}{d}
 &\le\frac CT\left\{
 \left|\widetilde\theta_1-
             \widetilde\theta_1^\star(\widetilde\theta_2)\right|^2
 +|\widetilde\theta_2-\widetilde\theta_2^\star|^2\right\}\\
 &\le\frac CT\left\{
 |\widetilde\theta_1-\widetilde\theta_1^\star|^2
 +|\widetilde\theta_2-\widetilde\theta_2^\star|^2\right\}.
\end{align*}

For \(1\le t<T\), subtracting the two instances of
\eqref{eq:stage2-conditional-normal} gives
\[
 \Hb_t\bigl(\widetilde\theta_5^\star
             (\widetilde\theta_1,\widetilde\theta_2)
                    -\widetilde\theta_5^\star\bigr)_{1:t,t}
 =-\bigl(\widetilde\theta_1\widetilde\theta_2^t
       -\widetilde\theta_1^\star(\widetilde\theta_2^\star)^t\bigr)
       \bigl(\langle\lambda^j,\lambda^{t+1}\rangle_{d,n}\bigr)_{j=1}^t.
\]
Thus, for \(1\le j\le t\),
\begin{align*}
 &\left\langle\sum_{i=1}^t
       \bigl(\widetilde\theta_5^\star(\widetilde\theta_1,\widetilde\theta_2)
                   -\widetilde\theta_5^\star\bigr)_{it}\lambda^i
       +\bigl(\widetilde\theta_1\widetilde\theta_2^t
             -\widetilde\theta_1^\star(\widetilde\theta_2^\star)^t\bigr)
                         \lambda^{t+1},\lambda^j\right\rangle_{d,n}=0.
\end{align*}
By \eqref{eq:stage2-completed-square} and
\eqref{eq:finite-function-quadratics}, squaring the coefficient expansion
of \(R_t-\widehat R_{T,t}\) therefore gives
\begin{align*}
 \sum_{t=1}^T\|R_t-\widehat R_{T,t}\|_{d,n}^2
 &=\sum_{t=1}^{T-1}
 \bigl(\widetilde\theta_5-
       \widetilde\theta_5^\star(\widetilde\theta_1,\widetilde\theta_2)\bigr)_{1:t,t}^{\mathsf T}
 \Hb_t
 \bigl(\widetilde\theta_5-
       \widetilde\theta_5^\star(\widetilde\theta_1,\widetilde\theta_2)\bigr)_{1:t,t}\\
 &\quad+\sum_{k=0}^{T-1}c_{k,\nu}
 \bigl(\widetilde\theta_1(-\widetilde\theta_2)^k
       -\widetilde\theta_1^\star(-\widetilde\theta_2^\star)^k\bigr)^2\\
 &=\frac{2T\mathcal L_A}{d}
 +\sum_{k=0}^{T-1}c_{k,\nu}
 \bigl(\widetilde\theta_1(-\widetilde\theta_2)^k
       -\widetilde\theta_1^\star(-\widetilde\theta_2^\star)^k\bigr)^2.
\end{align*}

The coefficient bounds in Lemma~\ref{lem:finite-scalar-comparison}
and \eqref{eq:finite-BC-extended-bounds} give
\begin{align*}
 c_{0,\nu},c_{1,\nu}&\ge c,\qquad
 \sum_{k=0}^{T-1}c_{k,\nu}(-\widetilde\theta_2)^{2k}
 =C_\nu(-\widetilde\theta_2)\le C,\\
 \sum_{k=1}^{T-1}k^2c_{k,\nu}(-\widetilde\theta_2)^{2k-2}
 &=\frac14\left\{C_\nu''(-\widetilde\theta_2)
                 +\frac{C_\nu'(-\widetilde\theta_2)}{-\widetilde\theta_2}\right\}
 \le C.
\end{align*}
The mean-value theorem gives
\begin{align*}
 &\sum_{k=0}^{T-1}c_{k,\nu}
 \bigl(\widetilde\theta_1(-\widetilde\theta_2)^k
       -\widetilde\theta_1^\star(-\widetilde\theta_2^\star)^k\bigr)^2\\
 &\quad\le2C_\nu(-\widetilde\theta_2)
                     |\widetilde\theta_1-\widetilde\theta_1^\star|^2\\
 &\qquad+2(\widetilde\theta_1^\star)^2
           |\widetilde\theta_2-\widetilde\theta_2^\star|^2
       \sum_{k=1}^{T-1}k^2c_{k,\nu}
       \max\{-\widetilde\theta_2,-\widetilde\theta_2^\star\}^{2k-2}\\
 &\quad\le C\left\{
       |\widetilde\theta_1-\widetilde\theta_1^\star|^2
       +|\widetilde\theta_2-\widetilde\theta_2^\star|^2\right\}
 \le CT\frac{\mathcal L_B}{d}.
\end{align*}
Conversely, since
\[
 \widetilde\theta_1(\widetilde\theta_2-\widetilde\theta_2^\star)
 =\widetilde\theta_1\widetilde\theta_2
  -\widetilde\theta_1^\star\widetilde\theta_2^\star
  -\widetilde\theta_2^\star(\widetilde\theta_1-\widetilde\theta_1^\star),
\]
the first two summands imply
\begin{align*}
 \frac{\mathcal L_B}{d}
 &\le\frac CT\left\{
       |\widetilde\theta_1-\widetilde\theta_1^\star|^2
       +|\widetilde\theta_2-\widetilde\theta_2^\star|^2\right\}\\
 &\le\frac CT\left\{
       |\widetilde\theta_1-\widetilde\theta_1^\star|^2
       +|\widetilde\theta_1\widetilde\theta_2
                   -\widetilde\theta_1^\star\widetilde\theta_2^\star|^2\right\}\\
 &\le\frac CT\sum_{k=0}^{T-1}c_{k,\nu}
 \bigl(\widetilde\theta_1(-\widetilde\theta_2)^k
       -\widetilde\theta_1^\star(-\widetilde\theta_2^\star)^k\bigr)^2.
\end{align*}
Together with \(\mathcal L-\mathcal L^\star=\mathcal L_A+\mathcal L_B\),
these estimates yield
\begin{equation}
 \frac cT\sum_{t=1}^T
 \|R_t(\tau)-\widehat R_{T,t}\|_{d,n}^2
 \le\frac{\mathcal L(\tau)-\mathcal L^*}{d}
 \le\frac CT\sum_{t=1}^T
 \|R_t(\tau)-\widehat R_{T,t}\|_{d,n}^2.
 \label{eq:rollout-quadratic-growth}
\end{equation}
Here \(c,C\) may depend on the fixed comparison bounds in \(n=\Theta(d)\),
but not on \(d,n,T\).
Combining this inequality with \eqref{eq:uniform-polynomial-rate} yields
\begin{equation}
 \frac{\mathcal L(\tau)-\mathcal L^*}{d}
 +\sum_{t=1}^T
 \|R_t(\tau)-\widehat R_{T,t}\|_{d,n}^2
 \le C T^C\left(\frac dn\right)^{-C}\frac{d/n}{T}
 \left(1+\frac{(d/n)\tau}{T\,d^{o(1)}}\right)^{-\alpha}.
 \label{eq:main-convergence-rate}
\end{equation}
Solving \eqref{eq:main-convergence-rate} for the required accuracy gives
the Stage-2 time bound in Theorem~\ref{thm:main-result}.
\end{proof}

\section{Comparison with the Heavy-Ball method}
\label{app:hb}

\subsection{Optimality of the Heavy-Ball trajectory}
\begin{lemma}\label{lem:tv-hb}
The jointly optimal trajectory $\wb_t^{\rm HB}$ defined after \eqref{eq:learned-hb-update}
exists and is unique.  Its residual $ R_t^{\rm HB}$ at every time \(t\), which satisfies $\wb_t^{\rm HB}- \wstar= R_t^{\rm HB}(\Sb)\wstar$, is the unique
minimizer of
\[
 \min_{\substack{R(0)=-1\\\deg R\le t}}\|R\|_{d,n}^2.
\]
 Under \eqref{eq:common-growing-regime}, for all sufficiently large \(d\),
uniformly for \(1\le t\le T\),
\begin{equation}
 \frac1d\E\|\wb_t^{\rm HB}-\wstar\|^2
 \ge\frac{1-d/n}{2}\left(\frac dn\right)^t.
 \label{eq:finite-tail-residual-lower}
\end{equation}
\end{lemma}

\begin{proof}
Every trajectory in \eqref{eq:learned-hb-update} has a residual \(R_t\)
with \(R_t(0)=-1\) and \(\deg R_t\le t\).
Use the same definition of \(p_t^\star\) also at \(t=T\).
The finite moment Gram matrices are positive definite, so the minimizers
exist uniquely, and their normal equations are
\begin{equation}
 \langle p_t^\star-1,\lambda p\rangle_{d,n}=0
 \quad(\deg p<t),\qquad
 p_0^\star=0,\qquad p_t^\star(0)-1=-1.
 \label{eq:hb-optimal-normal-equations}
\end{equation}
For \(t\ge1\),
\[
 \deg(p_t^\star-1)<t
 \ \Longrightarrow\
 0=\langle p_t^\star-1,\lambda(p_t^\star-1)\rangle_{d,n}
   =\int_0^\infty\lambda(p_t^\star-1)^2\,\mathrm d\nu_{d,n}>0,
\]
a contradiction.  Hence
\[
 \deg(p_t^\star-1)=t,\qquad
 [\lambda^t](p_t^\star-1)\ne0,\qquad
 \int_0^\infty\lambda(p_t^\star-1)(p_j^\star-1)
                 \,\mathrm d\nu_{d,n}=0\quad(0\le j<t).
\]
Define the HB coefficients by
\begin{equation}
\begin{aligned}
 \alpha_0^{\rm HB}
 &=\frac{\langle1,\lambda\rangle_{d,n}}
         {\langle\lambda,\lambda\rangle_{d,n}}
   =\frac1{1+d/n+1/n},
 &\beta_0^{\rm HB}&=0,\\
 \alpha_t^{\rm HB}
 &=-\frac{[\lambda^{t+1}](p_{t+1}^\star-1)}
          {[\lambda^t](p_t^\star-1)},
 &&1\le t<T,\\
 \beta_t^{\rm HB}
 &=\frac{[\lambda^t](p_{t+1}^\star-1)
          +\alpha_t^{\rm HB}[\lambda^{t-1}](p_t^\star-1)}
         {[\lambda^t](p_t^\star-1)}-1,
 &&1\le t<T.
\end{aligned}
 \label{eq:hb-optimal-coefficients}
\end{equation}
For \(1\le t<T\), the leading terms cancel, and
\begin{align*}
 &\deg\{p_{t+1}^\star-1+\alpha_t^{\rm HB}\lambda(p_t^\star-1)\}\le t,\\
 &\int_0^\infty\lambda
 \{p_{t+1}^\star-1+\alpha_t^{\rm HB}\lambda(p_t^\star-1)\}
                   (p_j^\star-1)\,\mathrm d\nu_{d,n}\\
 &\qquad=\alpha_t^{\rm HB}
      \langle p_t^\star-1,\lambda^2(p_j^\star-1)\rangle_{d,n}=0,
 \qquad 0\le j\le t-2.
\end{align*}
Thus
\[
 p_{t+1}^\star-1+\alpha_t^{\rm HB}\lambda(p_t^\star-1)
 \in\operatorname{span}\{p_t^\star-1,p_{t-1}^\star-1\}.
\]
Evaluating at zero and comparing the degree-\(t\) coefficients gives
\begin{equation}
\begin{aligned}
 p_1^\star(\lambda)&=\alpha_0^{\rm HB}\lambda,\\
 p_{t+1}^\star-1
 &=(1+\beta_t^{\rm HB}-\alpha_t^{\rm HB}\lambda)(p_t^\star-1)
       -\beta_t^{\rm HB}(p_{t-1}^\star-1),
 \qquad 1\le t<T.
\end{aligned}
 \label{eq:hb-optimal-three-term}
\end{equation}
This is the recurrence in \citet[Theorem~3.1]{pedregosa2020spectral}.
Substituting \(\lambda=\Sb\) gives
\[
\begin{aligned}
 p_1^\star(\Sb)\wstar&=\alpha_0^{\rm HB}\Xb\yb^{\mathsf T},\\
 p_{t+1}^\star(\Sb)\wstar
 &=p_t^\star(\Sb)\wstar
   -\alpha_t^{\rm HB}\{\Sb p_t^\star(\Sb)\wstar-\Xb\yb^{\mathsf T}\}\\
 &\quad+\beta_t^{\rm HB}
           \{p_t^\star(\Sb)\wstar-p_{t-1}^\star(\Sb)\wstar\}.
\end{aligned}
\]
Therefore \(p_t^\star(\Sb)\wstar\) is an admissible HB trajectory.
For every admissible residual sequence, \eqref{eq:hb-optimal-normal-equations}
gives the exact identity
\begin{equation}
 \frac d{2T}\sum_{t=1}^T
     \left\{\|R_t\|_{d,n}^2-\|p_t^\star-1\|_{d,n}^2\right\}
 =\frac d{2T}\sum_{t=1}^T
                 \|R_t-(p_t^\star-1)\|_{d,n}^2\ge0.
 \label{eq:hb-joint-optimality}
\end{equation}
Equality holds if and only if \(R_t=p_t^\star-1\) for every \(t\).
Thus the jointly optimal trajectory is unique, and
\eqref{eq:hb-optimal-coefficients} uniquely determines its coefficients.

Finally, \eqref{eq:finite-norm-equivalence} and
\eqref{eq:mp-identities} in Lemma~\ref{lem:mp-family} give
\begin{align*}
 \frac1d\E\|\wb_t^{\rm HB}-\wstar\|^2
 &=\|p_t^\star-1\|_{d,n}^2\\
 &\ge\frac12\min_{\substack{R(0)=-1\\\deg R\le t}}
                 \|R\|_{\mathrm{MP}}^2
 =\frac{(d/n)^t}{2\sum_{r=0}^t(d/n)^r}
 \ge\frac{1-d/n}{2}\left(\frac dn\right)^t.
\end{align*}
\end{proof}

\subsection{The highest-degree coefficient}
\begin{lemma}
\label{lem:finite-hb-defect}
Under \eqref{eq:common-growing-regime}, let \(\widehat R_{T,t}\) be the
finite-Wishart model-optimal residual. For all sufficiently large \(d\),
uniformly for \(0\le t<T\),
\begin{equation}
 \frac{\E\|\widehat R_{T,t+1}(\Sb)\wstar
                         -(\wb_t^{\rm HB}-\wstar)\|^2}
      {\E\|\wb_t^{\rm HB}-\wstar\|^2}
 \le\frac{C}{1-d/n}\left(1-\frac{(d/n)^2}{128}\right)^{2t}.
 \label{eq:finite-optimal-hb-relative}
\end{equation}
\end{lemma}

\begin{proof}
By Lemmas~\ref{lem:global-recurrence}
and~\ref{lem:stage2-loss-decomposition}, the lower-degree coefficients are
optimal with the highest-degree coefficient fixed:
\begin{equation}
 \widehat R_{T,t+1}
 =\argmin_{\substack{R(0)=-1,\ \deg R\le t+1\\
 [\lambda^{t+1}]R=\widetilde\theta_1^\star(\widetilde\theta_2^\star)^t}}
       \|R\|_{d,n}^2.
 \label{eq:finite-optimal-leading-coefficient}
\end{equation}
When the highest-degree coefficient is zero, the same minimization is
over degree at most \(t\). Hence Lemma~\ref{lem:tv-hb} gives
\[
 R_{t+1,0}^{\nu_{d,n}}=p_t^\star-1,\qquad
 (p_t^\star(\Sb)-\Ib_d)\wstar=\wb_t^{\rm HB}-\wstar.
\]
Applying Lemma~\ref{lem:finite-leading-gap} with
\(\ell=\widetilde\theta_1^\star(\widetilde\theta_2^\star)^t\),
\begin{align*}
 &\frac{\E\|\widehat R_{T,t+1}(\Sb)\wstar
                         -(\wb_t^{\rm HB}-\wstar)\|^2}
      {\E\|\wb_t^{\rm HB}-\wstar\|^2}\\
 &\quad=\frac{\|\widehat R_{T,t+1}-(p_t^\star-1)\|_{d,n}^2}
                  {\|p_t^\star-1\|_{d,n}^2}
 =\frac{(\widetilde\theta_1^\star)^2(\widetilde\theta_2^\star)^{2t}
                       c_{t,\nu_{d,n}}}{\|p_t^\star-1\|_{d,n}^2}\\
 &\quad\le\frac{4(\widetilde\theta_1^\star)^2}{1-d/n}
                        |\widetilde\theta_2^\star|^{2t}.
\end{align*}
Lemmas~\ref{lem:finite-function-location}
and~\ref{lem:finite-scalar-comparison} give
\[
 0<\widetilde\theta_1^\star
 \le\frac{2\sum_{k=0}^{T-1}(d/n)^k}{(1+d/n)/2}\le8,
 \qquad |\widetilde\theta_2^\star|<1-\frac{(d/n)^2}{128}<1.
\]
Substitution proves \eqref{eq:finite-optimal-hb-relative}.
\end{proof}

\begin{corollary}
\label{cor:growing-horizon-hb}
Under the assumptions of Lemma~\ref{lem:finite-hb-defect}, for all sufficiently
large \(d\), uniformly for \(T/2\le t<T\),
\begin{equation}
 \frac{\E\|\widehat R_{T,t+1}(\Sb)\wstar
                         -(\wb_t^{\rm HB}-\wstar)\|^2}
      {\E\|\wb_t^{\rm HB}-\wstar\|^2}
 \le\exp\!\left[-2(\log d)^{1/4}\right].
 \label{eq:growing-horizon-HB}
\end{equation}
\end{corollary}

\begin{proof}
By \eqref{eq:common-tail-margin}, for \(t\ge T/2\),
\begin{align*}
 &\frac{C}{1-d/n}\left(1-\frac{(d/n)^2}{128}\right)^{2t}
 \le\frac{C}{1-d/n}\exp\!\left[-\frac{(d/n)^2T}{128}\right]\\
 &\qquad\le\frac{C}{1-d/n}\exp\!\left[-8(\log d)^{1/4}\right]
 \le\exp\!\left[-2(\log d)^{1/4}\right].
\end{align*}
\end{proof}

\subsection{The trained trajectory}

\begin{proof}[Completion of the proof of Theorem~\ref{thm:main-result}]
For every polynomial \(f\), conditioning on \(\Sb\) gives
\(\E\|f(\Sb)\wstar\|^2=d\|f\|_{d,n}^2\).  The exact decomposition is
\begin{equation}
 \wb_{t+1}(\tau)-\wb_t^{\rm HB}
 =\{R_{t+1}(\tau)-\widehat R_{T,t+1}\}(\Sb)\wstar
 +\widehat R_{T,t+1}(\Sb)\wstar-(\wb_t^{\rm HB}-\wstar).
                                                               \label{eq:learned-optimal-hb-difference}
\end{equation}
Equations \eqref{eq:main-convergence-rate} and
\eqref{eq:finite-tail-residual-lower} give, uniformly for \(T/2\le t<T\),
\begin{equation}
 \frac{d\|R_{t+1}(\tau)-\widehat R_{T,t+1}\|_{d,n}^2}
      {\E\|\wb_t^{\rm HB}-\wstar\|^2}
 \le\frac{C}{(1-d/n)(d/n)^{2T+t}}
 \left(1+\frac{(d/n)\tau}{T\,d^{o(1)}}\right)^{-\alpha}.
 \label{eq:learned-optimal-hb-relative}
\end{equation}
Here we used
\(T^{C-1}(d/n)^{1-C}\le C(d/n)^{-2T}\) for large \(d\).
Choose \(\tau\) so that \eqref{eq:learned-optimal-hb-relative} is at most
\(\epsilon/4\) simultaneously for all these indices.  It suffices to take
\[
 \tau=O\!\left(d^{o(1)}\frac{T}{d/n}
 \left[\frac1{(1-d/n)(d/n)^{3T}\epsilon}\right]^{1/\alpha}\right).
\]
Corollary~\ref{cor:growing-horizon-hb} bounds the second term in
\eqref{eq:learned-optimal-hb-difference}, in the same squared relative norm,
by \(\exp[-2(\log d)^{1/4}]\le\epsilon/4\) for sufficiently large \(d\).
The triangle inequality therefore yields
\[
 \frac{\bigl(\E\|\wb_{t+1}(\tau)-\wb_t^{\rm HB}\|^2\bigr)^{1/2}}
      {\bigl(\E\|\wb_t^{\rm HB}-\wstar\|^2\bigr)^{1/2}}
 \le\frac{\sqrt\epsilon}{2}+\frac{\sqrt\epsilon}{2}
 =\sqrt\epsilon.
\]
Squaring proves the claimed comparison.  Under
\eqref{eq:common-growing-regime}, both \(T/(d/n)\) and
\(((1-d/n)(d/n)^{3T})^{-1}\) are \(d^{o(1)}\), so the chosen Stage-2 time
remains \(O(d^{o(1)}\epsilon^{-1/\alpha})\).
\end{proof}

\section{Supporting lemmas}
\label{app:spectral}
In the MP estimates below, \(\mu_{d/n}\) denotes the MP distribution~\citep{marchenko1967} with
\(0<d/n\le1/10\) and support
\[
 [(1-\sqrt{d/n})^2,(1+\sqrt{d/n})^2].
\]
Unless a distribution is indicated explicitly, norms and inner products
are under this distribution.
For
\[
 0<x<(d/n)^{-1/2},
\]
define
\begin{equation}
 B(x)=\sum_{k=0}^{\infty}\frac{\left(\frac{d}{n}\right)^k}{\sum_{r=0}^{k}\left(\frac{d}{n}\right)^r}x^k,\qquad
 C(x)=\sum_{k=0}^{\infty}
 \frac{\left(\frac{d}{n}\right)^k\sum_{r=0}^{k+1}\left(\frac{d}{n}\right)^r}{\sum_{r=0}^{k}\left(\frac{d}{n}\right)^r}x^{2k}.                              \label{eq:infinite-function}
\end{equation}

\begin{lemma}
\label{lem:function-sign-support}
For \(0<d/n\le1/10\), the function \(B(x)^2/C(x)\) in
\eqref{eq:infinite-function} has exactly one critical point
\(x_0\in(0,(d/n)^{-1/2})\).
It satisfies
\begin{equation}
 \frac{1+d/n}{1+d/n+(d/n)^2}<x_0<1,\qquad
 \left.\frac{\mathrm d^2}{\mathrm dx^2}
 \frac{B(x)^2}{C(x)}\right|_{x=x_0}<0.            \label{eq:function-unique-location}
\end{equation}
Moreover,
\begin{equation}
 x_0\le 1-\frac{(d/n)^2}{128}.                       \label{eq:explicit-function-separation}
\end{equation}
\end{lemma}

\begin{proof}
By the definitions in \eqref{eq:infinite-function}, the coefficients of \(B\)
and \(C\) are, respectively,
\[
 \frac{(d/n)^k}{\sum_{r=0}^k(d/n)^r},
 \qquad
 \frac{(d/n)^k\sum_{r=0}^{k+1}(d/n)^r}{\sum_{r=0}^k(d/n)^r}.
\]
Pairing the \((i,j)\) and \((j,i)\) terms in
\(2B'C-BC'\) gives the absolutely convergent identity.  Writing
\(S_m:=\sum_{r=0}^m(d/n)^r\), we have
\begin{equation}
 \frac{x}{2}\{2B'(x)C(x)-B(x)C'(x)\}
 =\sum_{0\le j<i}(i-j)
   \frac{(d/n)^{i+j}x^{i+2j}}{S_iS_j}
   \bigl\{S_{j+1}-S_{i+1}x^{i-j}\bigr\}.
\label{eq:function-pairing}
\end{equation}
Since
\[
 \left(\sum_{r=0}^{m+1}\left(\frac{d}{n}\right)^r\right)^2
 -\left(\sum_{r=0}^{m}\left(\frac{d}{n}\right)^r\right)
  \left(\sum_{r=0}^{m+2}\left(\frac{d}{n}\right)^r\right)
 =\left(\frac{d}{n}\right)^{m+1}>0,
\]
the ratios \(\sum_{r=0}^{m}\left(\frac{d}{n}\right)^r/\sum_{r=0}^{m+1}\left(\frac{d}{n}\right)^r\) strictly increase.  Using this monotonicity in
\eqref{eq:function-pairing}, every summand is nonnegative and the sum is
strictly positive when
\[
 0<x\le\frac{1+\frac dn}{1+\frac dn+(\frac dn)^2}.
\]
Every summand is negative when \(1\le x<\left(\frac{d}{n}\right)^{-1/2}\).  Therefore
\begin{equation}
 \left(\frac{B^2}{C}\right)'(x)>0\quad\left(0<x\le\frac{1+\frac dn}{1+\frac dn+(\frac dn)^2}\right),\qquad
 \left(\frac{B^2}{C}\right)'(x)<0\quad(1\le x<\left(\frac{d}{n}\right)^{-1/2}).                  \label{eq:outer-function-sign}
\end{equation}

It remains to prove that \((B^2/C)'\) crosses zero only once in
\[
 \left(\frac{1+d/n}{1+d/n+(d/n)^2},1\right).
\]
Multiply numerator and denominator by \((1+x)^2\):
\begin{equation}
 (1+x)^2B(x)^2=\sum_{j\ge0}A_jx^j,\qquad
 (1+x)^2C(x)=\sum_{j\ge0}D_jx^j,\qquad
 D_j>0.                              \label{eq:function-ratio-definition}
\end{equation}
Direct convolution gives
\begin{equation}
 [x^k]B(x)^2
 =\frac{(1-\left(\frac{d}{n}\right))^2\left(\frac{d}{n}\right)^k}{1-\left(\frac{d}{n}\right)^{k+2}}
   \sum_{j=1}^{k+1}\frac{1+\left(\frac{d}{n}\right)^j}{1-\left(\frac{d}{n}\right)^j}.     \label{eq:exact-convolution}
\end{equation}
The first differences are
\begin{align}
 \frac{A_1}{D_1}-\frac{A_0}{D_0}&=\frac{\left(\frac{d}{n}\right)}{(1+\left(\frac{d}{n}\right))^2}>0,   \label{eq:function-ratio-first (1,0)}\\    
 \frac{A_2}{D_2}-\frac{A_1}{D_1}
 &=\frac{\left(\frac{d}{n}\right)(-2\left(\frac{d}{n}\right)^5+\left(\frac{d}{n}\right)^4+9\left(\frac{d}{n}\right)^3
       +13\left(\frac{d}{n}\right)^2+8\left(\frac{d}{n}\right)+2)}
 {(1+\left(\frac{d}{n}\right))^2(1+\left(\frac{d}{n}\right)+\left(\frac{d}{n}\right)^2)
       (\left(\frac{d}{n}\right)^3+2\left(\frac{d}{n}\right)^2+3\left(\frac{d}{n}\right)+1)}>0,  \label{eq:function-ratio-first (2,1)}\\
 \frac{A_3}{D_3}-\frac{A_2}{D_2}
 &=\frac{\left(\frac{d}{n}\right)^2(-2\left(\frac{d}{n}\right)^5+3\left(\frac{d}{n}\right)^3+5\left(\frac{d}{n}\right)^2+4\left(\frac{d}{n}\right)+3)}
 {(1+\left(\frac{d}{n}\right)^2)(1+\left(\frac{d}{n}\right)+\left(\frac{d}{n}\right)^2)^2
       (\left(\frac{d}{n}\right)^3+2\left(\frac{d}{n}\right)^2+3\left(\frac{d}{n}\right)+1)}>0.             \label{eq:function-ratio-first}
\end{align}
The last two numerators are positive because
\(2\left(\frac{d}{n}\right)^5<2\left(\frac{d}{n}\right)^3\).

The coefficient argument below proves the remaining comparisons:
\begin{align}
 \frac{A_{2m+1}}{D_{2m+1}}>\frac{A_{2m+2}}{D_{2m+2}}
 &\qquad(m\ge1),                                            \label{eq:odd-even}\\
 \frac{A_{2m}}{D_{2m}}>\frac{A_{2m+1}}{D_{2m+1}}
 &\qquad(m\ge2).
                                                               \label{eq:even-odd}
\end{align}

We prove this by giving the following estimate:
\begin{equation}
 \frac{A_{j+1}}{D_{j+1}}\le\frac34\frac{A_j}{D_j},\qquad j\ge3.
 \label{eq:infinite-coefficient-contraction}
\end{equation}
The terms \((1+(d/n)^r)/(1-(d/n)^r)\) decrease with \(r\). Thus
\eqref{eq:exact-convolution} gives
\begin{equation}
 \frac{[x^{k+1}]B(x)^2}{[x^k]B(x)^2}
 \le\frac dn\frac{k+2}{k+1}
       \frac{1-(d/n)^{k+2}}{1-(d/n)^{k+3}}
 <\frac dn\frac{k+2}{k+1},\qquad k\ge0.
 \label{eq:infinite-convolution-ratio}
\end{equation}
Since
\(A_j=[x^j]B(x)^2+2[x^{j-1}]B(x)^2+[x^{j-2}]B(x)^2\),
applying this bound to the three terms yields
\begin{equation}
 \frac{A_{j+1}}{A_j}\le\frac dn\frac{j}{j-1},\qquad j\ge3.
 \label{eq:infinite-numerator-ratio}
\end{equation}
For the denominator coefficients, use
\(\sum_{r=0}^{m}(d/n)^r=(1-(d/n)^{m+1})/(1-d/n)\) and differentiate this
identity to obtain
\begin{equation}
 \frac{[x^{2m}]C(x)}{[x^{2m-2}]C(x)}
 =\frac dn\left(1-\frac{(d/n)^m}
                         {\left(\sum_{r=0}^{m}(d/n)^r\right)^2}\right),
 \qquad m\ge1.
 \label{eq:infinite-denominator-ratio}
\end{equation}
Here \(D_{2m}=[x^{2m}]C(x)+[x^{2m-2}]C(x)\) and
\(D_{2m+1}=2[x^{2m}]C(x)\), so
\begin{align}
 \frac{D_{2m+2}}{D_{2m+1}}
 &\ge\frac{1+(d/n)\{1-(d/n)^2\}}{2}
 >\frac32\frac dn,\qquad m\ge1,                  \label{eq:infinite-odd-denominator}\\
 \frac{D_{2m+1}}{D_{2m}}
 &\ge\frac{2(d/n)\{1-(d/n)^3\}}{1+d/n}
 >\frac65\frac dn,\qquad m\ge3.                  \label{eq:infinite-even-denominator}
\end{align}
The strict inequalities follow from
\(1-2(d/n)-(d/n)^3\ge71/8000>0\) and
\(2\{1-(d/n)^3\}/(1+d/n)\ge7271/5800>6/5\).
Together with \eqref{eq:infinite-numerator-ratio}, these prove
\eqref{eq:odd-even} and \eqref{eq:even-odd} except for \(m=2\).

For this remaining case, retain the factor
\((1-(d/n)^4)/(1-(d/n)^5)\) in
\eqref{eq:infinite-convolution-ratio} at \(k=2\).
The bounds at \(k=3,4\) are at most \((5/4)(d/n)\), which is smaller.
Applying these bounds to the three terms in \(A_5\) gives
\[
 \frac{A_5}{A_4}
 \le\frac43\frac dn\frac{1-(d/n)^4}{1-(d/n)^5}
 <\frac{D_5}{D_4}.
\]
Indeed, using \eqref{eq:infinite-denominator-ratio} at \(m=2\)
and clearing positive denominators reduces the last strict inequality to
\[
 1-4(d/n)^2-2(d/n)^3+(d/n)^5\ge\frac{31}{4000}>0.
\]
This proves \eqref{eq:even-odd} also at \(m=2\).

The coarser denominator bounds are
\[
\begin{aligned}
 \frac{D_{2m+2}}{D_{2m+1}}&\ge\frac12,\qquad m\ge1,\\
 \frac{D_{2m+1}}{D_{2m}}&\ge2\frac dn\left(1-\frac dn\right),\qquad m\ge2.
\end{aligned}
\]
Substituting these in \eqref{eq:infinite-numerator-ratio} gives, for every \(j\ge3\),
\[
 \frac{A_{j+1}D_j}{A_jD_{j+1}}
 \le\max\left\{3\frac dn,\frac{2}{3(1-d/n)}\right\}
 \le\frac{20}{27}<\frac34,
\]
which proves \eqref{eq:infinite-coefficient-contraction}.

Combining \eqref{eq:function-ratio-first (1,0)}, \eqref{eq:function-ratio-first (2,1)},  \eqref{eq:function-ratio-first} and  \eqref{eq:infinite-coefficient-contraction}, these give
\begin{equation}
 \frac{A_0}{D_0}<\frac{A_1}{D_1}<\frac{A_2}{D_2}<\frac{A_3}{D_3}>\frac{A_4}{D_4}>\frac{A_5}{D_5}>\cdots.    \label{eq:function-ratio-unimodal}
\end{equation}

We use the following coefficient form of Descartes' rule.  If
\(f(x)=\sum_{j\ge0}b_jx^j\) has positive radius of convergence and its
nonzero coefficients have at most \(r\) sign changes, then \(f\) has at most \(r\)
positive zeros, counted with multiplicity.  We prove this by induction on \(r\).
The case \(r=0\) is immediate.  For \(r\ge1\), let \(k\) be the last index
in the first constant-sign block and set \(g(x)=xf'(x)-kf(x)\).  The
coefficients of \(g\) have at most \(r-1\) sign changes.  Rolle's theorem
applied to \(x^{-k}f(x)\) shows that the total positive-zero multiplicity of
\(g\) is at least that of \(f\) minus one, which completes the induction.

For any real \(y\), the coefficients of
\[
 (1+x)^2B(x)^2-y(1+x)^2C(x)
\]
are \(D_j(A_j/D_j-y)\), and \eqref{eq:function-ratio-unimodal}
therefore gives at most two sign changes.  Every horizontal line thus
meets \(B^2/C\) at at most two positive points, counted with multiplicity.
If \(x>0\) is a critical point, set \(y=B(x)^2/C(x)\).  Then the
corresponding numerator and its derivative both vanish at \(x\), so
\(x\) is a zero of multiplicity at least two.  The preceding
multiplicity bound makes it exactly double; hence every positive
critical point is nondegenerate.  Moreover,
\[
 \frac{B(x)^2}{C(x)}>0,\qquad
 \lim_{x\uparrow(d/n)^{-1/2}}\frac{B(x)^2}{C(x)}=0,
\]
since \(B\) converges and \(C\) diverges at the right endpoint.
A local minimum would give another intersection at its level to its right,
contradicting the same multiplicity bound. Thus every critical point is a
nondegenerate maximum. The signs in \eqref{eq:outer-function-sign} then prove
uniqueness and \eqref{eq:function-unique-location}.

We next prove the bound \(x_0\le 1-(d/n)^2/128\).  At \(x=1\), every summand on the right side of
\eqref{eq:function-pairing} is negative.  Keeping only the term
\((i,j)=(1,0)\) therefore gives
\begin{align}
&\left.\sum_{0\le j<i}(i-j)
 \frac{(\frac dn)^i}{\sum_{r=0}^{i}(\frac dn)^r}
 \frac{(\frac dn)^j}{\sum_{r=0}^{j}(\frac dn)^r}x^{i+2j}
 \left\{\sum_{r=0}^{j+1}(\frac dn)^r
 -\sum_{r=0}^{i+1}(\frac dn)^r x^{i-j}\right\}\right|_{x=1}\notag\\
&\quad\le
 \frac{\frac dn}{\sum_{r=0}^{1}(\frac dn)^r}
 \left\{\sum_{r=0}^{1}(\frac dn)^r
 -\sum_{r=0}^{2}(\frac dn)^r\right\}
 =-\frac{(\frac dn)^3}{1+\frac dn}.
 \label{eq:function-sign-at-one}
\end{align}
For \(0\le x\le1\), we bound the derivative of the left-hand side of
\eqref{eq:function-pairing} by differentiating its summands. The bounds
\[
\frac{(d/n)^k}{\sum_{r=0}^{k}(d/n)^r}\le(d/n)^k,
\qquad
\sum_{r=0}^{k}(d/n)^r\le(1-d/n)^{-1}
\]
then show that
\begin{equation}
 \scriptstyle
 \left|\left[\frac{x}{2}\{2B'(x)C(x)-B(x)C'(x)\}\right]'\right|
 \le
 \frac{3}{1-d/n}\sum_{0\le j<i}(i^2-j^2)(d/n)^{i+j}
 =\frac{3(d/n)\{1+2(d/n)+3(d/n)^2\}}
 {(1-d/n)^5(1+d/n)^2}.                          \label{eq:function-derivative-bound}
\end{equation}
To evaluate the sum, set \(k=i-j\). Then
\begin{align*}
 \sum_{0\le j<i}(i^2-j^2)\left(\frac dn\right)^{i+j}
 &=\sum_{k=1}^{\infty}\sum_{j=0}^{\infty}
   (k^2+2jk)\left(\frac dn\right)^{k+2j}\\
 &=\left(\sum_{k=1}^{\infty}k^2\left(\frac dn\right)^k\right)
   \left(\sum_{j=0}^{\infty}\left(\frac dn\right)^{2j}\right)
  +2\left(\sum_{k=1}^{\infty}k\left(\frac dn\right)^k\right)\\
 &\qquad\times\left(\sum_{j=0}^{\infty}j\left(\frac dn\right)^{2j}\right)\\
 &=\frac{\frac dn\left(1+2\frac dn+3(\frac dn)^2\right)}
 {(1-\frac dn)^4(1+\frac dn)^2}.
\end{align*}
The mean-value theorem, \eqref{eq:function-sign-at-one}, and
\eqref{eq:function-derivative-bound} now give
\begin{equation}
\frac{x}{2}\{2B'(x)C(x)-B(x)C'(x)\}|_{x=1-(d/n)^2/128}
 \le-\frac{(d/n)^3}{1+d/n}
 \left\{1-\frac{3(1+2(d/n)+3(d/n)^2)}{128(1-d/n)^5(1+d/n)}\right\}<0.
 \label{eq:explicit-function-negative}
\end{equation}
Indeed, the fraction subtracted in braces is increasing in \(d/n\) for
\(0<d/n\le1/10\).  Evaluating it at \(d/n=1/10\) gives
\[
 \left.\frac{3(1+2(d/n)+3(d/n)^2)}{128(1-d/n)^5(1+d/n)}\right|_{d/n=1/10}
 =\frac{25625}{577368}<1.
\]
The left side of
\eqref{eq:function-pairing} has the sign of \((B^2/C)'\). Since this derivative changes sign
only once, at its unique maximizer \(x_0\),
\begin{equation}
 x_0<1-\frac{\left(\frac{d}{n}\right)^2}{128}.                     \label{eq:explicit-infinite-separation}
\end{equation}

\end{proof}

We retain the shorthand \(R_t^*\) for the MP projection polynomial.  It is
the same object as \(R_{t+1,0}^{\mu}\) in Section~D whenever \(t+1\le T\);
thus \(R_{t,0}^{\mu}=R_{t-1}^*\) for \(1\le t\le T\).
Equivalently, \(R_t^*\) is the degree-at-most-\(t\) polynomial with
\(R_t^*(0)=-1\) minimizing \(\|R_t^*\|_{\mathrm{MP}}\).

\begin{lemma}\label{lem:mp-family}
Let \(\pi_0=1\), \(\pi_1=\lambda-1\), and
\begin{equation}
 \pi_{t+1}=\{\lambda-(1+\left(\frac{d}{n}\right))\}\pi_t-\left(\frac{d}{n}\right)\pi_{t-1}.
                                                               \label{eq:mp-jacobi}
\end{equation}
Then \(\{\pi_t\}\) is the monic orthogonal family for
\(\mu_{d/n}\), and
\begin{equation}
 \norm{R_t^*}^2=\frac{\left(\frac{d}{n}\right)^t}{\sum_{r=0}^{t}\left(\frac{d}{n}\right)^r},
 \qquad \pi_t(0)=(-1)^t,
 \qquad \norm{\pi_t}^2=\left(\frac{d}{n}\right)^t.                    \label{eq:mp-identities}
\end{equation}
Moreover,
\begin{equation}
 R_t^*(\lambda)
 =-\frac{\left(\frac{d}{n}\right)^t}{\sum_{r=0}^{t}\left(\frac{d}{n}\right)^r}
   \sum_{j=0}^t(-1)^j\left(\frac{d}{n}\right)^{-j}\pi_j(\lambda).       \label{eq:Rstar}
\end{equation}
\end{lemma}

\begin{proof}
Section~4 of \citet{kusalik2005orthogonal} with parameter \(n/d\) and its
degree-\(t\) polynomial evaluated at \((n/d)\lambda\) and multiplied by
\((d/n)^t\), gives \eqref{eq:mp-jacobi}, orthogonality, and squared norms
\((d/n)^t\).  The recurrence gives \(\pi_t(0)=(-1)^t\).  For every
\(R(0)=-1\) of degree at most \(t\), write
\(R=\sum_{j=0}^t a_j\pi_j\).  Orthogonality and Cauchy--Schwarz give
\[
 1=\left|\sum_{j=0}^t a_j\pi_j(0)\right|^2
 \le\left(\sum_{j=0}^t|a_j|^2\|\pi_j\|_{\mathrm{MP}}^2\right)
       \left(\sum_{j=0}^t\frac{\pi_j(0)^2}{\|\pi_j\|_{\mathrm{MP}}^2}\right)
 =\|R\|_{\mathrm{MP}}^2
       \sum_{j=0}^t\frac{\pi_j(0)^2}{\|\pi_j\|_{\mathrm{MP}}^2},
\]
with equality only at \eqref{eq:Rstar}.  This proves the remaining identity
in \eqref{eq:mp-identities} and uniqueness.
\end{proof}

\begin{lemma}\label{lem:mp-coefficient-norm}
For every polynomial \(f\) of degree at most \(T\),
\[
 \|f\|_{\mathrm{coeff}}
 \le\sqrt{T+1}\,4^T(d/n)^{-T/2}\|f\|_{\mathrm{MP}}.
\]
\end{lemma}
\begin{proof}
For j=0 and j=1, the coefficient sums of \(\pi_j\) are 1 and 2.
For \(j\ge1\), \eqref{eq:mp-jacobi} gives
\[
\begin{aligned}
 \sum_{k=0}^{j+1}|[\lambda^k]\pi_{j+1}|
 &\le (2+d/n)\sum_{k=0}^{j}|[\lambda^k]\pi_j|
       +(d/n)\sum_{k=0}^{j-1}|[\lambda^k]\pi_{j-1}|\\
 &\le (2+d/n)4^j+(d/n)4^{j-1}\le4^{j+1},
\end{aligned}
\]
where the last inequality uses \(d/n\le1/10\).  Induction therefore gives
\(\sum_{k=0}^{j}|[\lambda^k]\pi_j|\le4^j\) for every j.
By \eqref{eq:mp-identities} and Cauchy--Schwarz,
\begin{equation}
\begin{aligned}
 \|f\|_{\mathrm{coeff}}
 &\le\sum_{j=0}^T4^j\left(\frac dn\right)^{-j}
             |\langle f,\pi_j\rangle_{\mathrm{MP}}|\\
 &\le\sqrt{T+1}\,4^T\left(\frac dn\right)^{-T/2}
             \|f\|_{\mathrm{MP}}.
\end{aligned}
\label{eq:mp-coefficient-norm}
\end{equation}

\end{proof}
\begin{lemma}\label{lem:projection}
For every \(\ell\in\R\),
\begin{equation}
 \argmin_{\substack{\deg R\le t,\ R(0)=-1\\
                  [\lambda^t]R=\ell}}
 \norm R_{\mathrm{MP}}^2
 =\{1+(-1)^t\ell\}R_{t-1}^*+\ell\pi_t,
                                                               \label{eq:conditional-R}
\end{equation}
\begin{equation}
 \min_{\substack{\deg R\le t,\ R(0)=-1\\
                  [\lambda^t]R=\ell}}
 \norm R_{\mathrm{MP}}^2
 =\frac{\left(\frac{d}{n}\right)^{t-1}}{\sum_{r=0}^{t-1}\left(\frac{d}{n}\right)^r}
   \{1+(-1)^t\ell\}^2+\left(\frac{d}{n}\right)^t\ell^2.               \label{eq:conditional-norm}
\end{equation}
\end{lemma}

\begin{proof}
Every admissible \(R\) is uniquely \(q+\ell\pi_t\), where
\(\deg q\le t-1\) and \(q(0)=-1-(-1)^t\ell\).  Lemma~\ref{lem:mp-family}
gives
\[
 \norm q^2\ge
 \{1+(-1)^t\ell\}^2\frac{\left(\frac{d}{n}\right)^{t-1}}{\sum_{r=0}^{t-1}\left(\frac{d}{n}\right)^r},
\]
with equality only at
\(q=\{1+(-1)^t\ell\}R_{t-1}^*\).  Since
\(\ip{q}{\pi_t}=0\),
\begin{align*}
 \norm R^2
 &=\norm q^2+2\ell\ip{q}{\pi_t}+\ell^2\norm{\pi_t}^2\\
 &\ge
 \{1+(-1)^t\ell\}^2\frac{\left(\frac{d}{n}\right)^{t-1}}{\sum_{r=0}^{t-1}\left(\frac{d}{n}\right)^r}
 +\ell^2\left(\frac{d}{n}\right)^t.
\end{align*}
This proves \eqref{eq:conditional-R}-\eqref{eq:conditional-norm}, including
uniqueness.
\end{proof}

The preceding lemmas treat the infinite-horizon function.  We now pass to
the finite polynomial needed below and prove the corresponding coefficient
ratio statement.

\begin{lemma}\label{lem:mp-function-unique}
Let \(0<d/n\le1/10\), \(T\ge8\), and
\(S_k=\sum_{r=0}^k(d/n)^r\). Under the MP distribution \(\mu=\mu_{d/n}\),
the coefficient functions are
\begin{equation}
 B_\mu(x)=\sum_{k=0}^{T-1}\frac{(d/n)^k}{S_k}x^k,\qquad
 C_\mu(x)=\sum_{k=0}^{T-1}\frac{(d/n)^kS_{k+1}}{S_k}x^{2k}.
 \label{eq:finite-function-BC}
\end{equation}
The function \(B_\mu^2/C_\mu\) has a unique nondegenerate global
maximizer \(-\widetilde\theta_2^{\mathrm{MP}*}>0\) over \(\mathbb R\), and
\begin{equation}
 \left(\frac{B_\mu^2}{C_\mu}\right)'\!\left(\frac{S_1}{S_2}\right)
 \ge c(d/n)^4,\qquad
 \left(\frac{B_\mu^2}{C_\mu}\right)'\!
       \left(1-\frac{(d/n)^2}{128}\right)\le-c(d/n)^3,\qquad
 -\widetilde\theta_2^{\mathrm{MP}*}\le1-\frac{(d/n)^2}{128}.
 \label{eq:mp-endpoint-derivatives}
\end{equation}
\end{lemma}

\begin{proof}
By Lemma~\ref{lem:projection}, for \(0\le k<T\),
\begin{equation}
\begin{aligned}
&\min_{\substack{R(0)=-1,\ \deg R\le k+1\\
       [\lambda^{k+1}]R=(-1)^k\widetilde\theta_1(-\widetilde\theta_2)^k}}
       \|R\|_{\mathrm{MP}}^2\\
&\quad=\frac{(d/n)^k}{S_k}\bigl(1-\widetilde\theta_1(-\widetilde\theta_2)^k\bigr)^2
 +(d/n)^{k+1}\widetilde\theta_1^2(-\widetilde\theta_2)^{2k}.
\end{aligned}
 \label{eq:mp-fixed-leading-quadratic}
\end{equation}
Expanding the right-hand side gives the coefficient identities
\begin{equation}
 b_{k,\mu}:=[x^k]B_\mu(x)=\frac{(d/n)^k}{S_k},\qquad
 c_{k,\mu}:=[x^{2k}]C_\mu(x)=\frac{(d/n)^kS_{k+1}}{S_k}.
 \label{eq:finite-function-mp-coefficients}
\end{equation}
Summing the expanded identity over \(0\le k<T\) gives \eqref{eq:finite-function-BC}.
We next prove the coefficient-ratio certificate. Write, only in this proof,
\(b_k=b_{k,\mu}\) and \(c_k=c_{k,\mu}\).  Set
\[
 A_{j,T}:=[x^j](1+x)^2B_\mu(x)^2,\qquad
 D_{j,T}:=[x^j](1+x)^2C_\mu(x).\qquad
\]
For later use, the numerator coefficient is explicitly
\[
 A_{j,T}=\sum_{\substack{a+b=j\\0\le a,b<T}}b_ab_b
 +2\sum_{\substack{a+b=j-1\\0\le a,b<T}}b_ab_b
 +\sum_{\substack{a+b=j-2\\0\le a,b<T}}b_ab_b,
\]
where an empty sum is zero.  Also,
\[
 D_{2m,T}=c_m+c_{m-1},\qquad D_{2m+1,T}=2c_m,
\]
where \(c_{-1}=c_T=0\).
For \(j<T\), the coefficients are identical to those of the infinite coefficient formulas
in \eqref{eq:function-ratio-definition}; hence the exact certificate
\eqref{eq:function-ratio-unimodal} gives
\[
\frac{A_{0,T}}{D_{0,T}}<\frac{A_{1,T}}{D_{1,T}}<\frac{A_{2,T}}{D_{2,T}}<\frac{A_{3,T}}{D_{3,T}}>
\frac{A_{4,T}}{D_{4,T}}>\cdots>\frac{A_{T-1,T}}{D_{T-1,T}}.
\]

For \(b_k=(d/n)^k/S_k\),
\[
\frac{b_{k+1}}{b_k}
=\frac dn\frac{S_k}{S_{k+1}}<\frac dn.
\]
For \(T-1\le j\le2T-2\), the coefficient sums give
\[
 [x^{j+1}]B_\mu(x)^2
 =\sum_{i=j-T+2}^{T-1}b_i b_{j+1-i}
 \le\frac dn\sum_{i=j-T+2}^{T-1}b_i b_{j-i}
 \le\frac dn\,[x^j]B_\mu(x)^2.
\]
Consequently,
\([x^{j+1}]B_\mu(x)^2\le(d/n)[x^j]B_\mu(x)^2\) for
\(T-1\le j\le2T-2\), while \eqref{eq:exact-convolution} gives
\[
\frac{[x^j]B_\mu(x)^2}{[x^{j-1}]B_\mu(x)^2}
 \le\frac dn\frac{j+1}{j},
\qquad 1\le j<T.
\]

At the first two coefficient transitions, \(T\ge8\) and
\(0<d/n\le1/10\) give
\begin{align}
\frac{A_{T,T}}{A_{T-1,T}}
&\le
\frac{1+\frac87(2+d/n)(d/n)}
     {2+\frac6{7(d/n)}}\notag\\
&=\frac{(d/n)\{7+16(d/n)+8(d/n)^2\}}{6+14(d/n)}
\le\frac32\frac dn,                                           \label{eq:finite-function-first-transition}\\
\frac{A_{T+1,T}}{A_{T,T}}
&\le
\frac{1+(2+d/n)(d/n)}
     {2+\frac7{8(d/n)}}\notag\\
&=\frac{8(d/n)(1+d/n)^2}{7+16(d/n)}
\le\frac32\frac dn.                                           \label{eq:finite-function-second-transition}
\end{align}
For the last inequalities, substitute \(d/n\le1/10\) in each numerator
after factoring out \(d/n\), and drop the positive linear term in its
denominator.  To treat both parities of \(T\), observe that for \(4\le m<T\),
\[
\frac{c_m}{c_{m-1}}
=\frac dn\left(1-\frac{(d/n)^m}{S_m^2}\right)
\ge\frac dn\{1-(d/n)^4\}.
\]
Each of the two denominator ratios is either
\((1+c_m/c_{m-1})/2\) or
\(2(c_m/c_{m-1})/(1+c_m/c_{m-1})\), with \(4\le m<T\).  Thus
\[
\min\left\{\frac{D_{T,T}}{D_{T-1,T}},
            \frac{D_{T+1,T}}{D_{T,T}}\right\}
\ge\min\left\{\frac12,
 \frac{2(d/n)\{1-(d/n)^4\}}{1+(d/n)\{1-(d/n)^4\}}\right\}
\ge\frac95\frac dn.
\]
Here \(1/2\ge(9/5)(d/n)\), and the second entry of the minimum is at
least \(2(d/n)(1-10^{-4})/(11/10)>(9/5)(d/n)\).
Both coefficient transitions therefore strictly decrease \(r_{j,T}\),
regardless of the parity of \(T\).

For \(T+1\le j\le2T-2\), the coefficient estimates above give
\[
0<\frac{[x^j]B_\mu(x)^2}{[x^{j-1}]B_\mu(x)^2}\le\frac dn,
\qquad
0<\frac{[x^{j-1}]B_\mu(x)^2}{[x^{j-2}]B_\mu(x)^2}\le\frac dn.
\]
Expanding the three coefficient sums in \(A_{j,T}\) and using these two
inequalities yields
\[
\frac{A_{j+1,T}}{A_{j,T}}
\le
\frac{(2+d/n)(d/n)+1}{(d/n)+2+(d/n)^{-1}}
=\frac dn.
\]
For the denominator coefficients, if \(j=2m+1\), then
\[
\frac{D_{j+1,T}}{D_{j,T}}=\frac{1+c_{m+1}/c_m}{2}>\frac12>\frac dn,
\]
whereas if \(j=2m\), then
\[
\frac{D_{j+1,T}}{D_{j,T}}
=\frac{2c_m}{c_m+c_{m-1}}
>\frac dn,
\]
because \(m\ge4\),
\[
\frac{c_m}{c_{m-1}}
=\frac dn\left(1-\frac{(d/n)^m}{S_m^2}\right)
\ge\frac dn\{1-(d/n)^4\},
\]
and \((2-d/n)(1-(d/n)^4)>1\) on \(0<d/n\le1/10\).
Therefore
\begin{equation}
\frac{A_{0,T}}{D_{0,T}}<\frac{A_{1,T}}{D_{1,T}}<\frac{A_{2,T}}{D_{2,T}}<\frac{A_{3,T}}{D_{3,T}}>
\frac{A_{4,T}}{D_{4,T}}>\cdots>\frac{A_{2T,T}}{D_{2T,T}}.                         \label{eq:finite-function-mp-ratios}
\end{equation}
The exact certificate also gives an absolute \(C\) such that
\begin{equation}
\min_{0\le j<2T}|\frac{A_{j+1,T}}{D_{j+1,T}}-\frac{A_{j+1,T}}{D_{j+1,T}}|
\ge C^{-1}(d/n)^T.                                \label{eq:finite-function-mp-gap}
\end{equation}
For \(0\le j<T\), positivity of the coefficient sums gives
\[
 A_{j,T}\ge\left(1-\frac dn\right)(d/n)^j,\qquad
 D_{j,T}\le2\left(1+\frac dn\right)\le3.
\]
For \(3\le j\le T-2\), \eqref{eq:infinite-coefficient-contraction} therefore gives
\[
 \frac{A_{j,T}}{D_{j,T}}-\frac{A_{j+1,T}}{D_{j+1,T}}\ge\frac14\frac{A_{j,T}}{D_{j,T}}
 \ge\frac1{20}(d/n)^T.
\]
The first three gaps have the same lower bound up to an absolute constant
by \eqref{eq:function-ratio-first}.
For the two coefficient transitions, the bounds in
\eqref{eq:finite-function-first-transition}-\eqref{eq:finite-function-second-transition}
and the common denominator-ratio lower bound give
\[
 \max\left\{\frac{A_{T,T}/D_{T,T}}{A_{T-1,T}/D_{T-1,T}},
             \frac{A_{T+1,T}/D_{T+1,T}}{A_{T-1,T}/D_{T-1,T}}\right\}
 \le\frac{(3/2)(d/n)}{(9/5)(d/n)}=\frac56<\frac9{10}.
\]

For \(T+1\le j\le2T-2\), the estimates above give
\[
 \frac{A_{j+1,T}/A_{j,T}}{D_{j+1,T}/D_{j,T}}
 \le\frac9{10}.
\]
For odd \(j\), this is \(2(d/n)\le9/10\).  For even \(j\), it follows from
\[
 \frac{1+(d/n)\{1-(d/n)^4\}}{2\{1-(d/n)^4\}}
 \le\frac9{10},
\]
because
\(4/5-d/n-(9/5)(d/n)^4+(d/n)^5\) is decreasing and positive at
\(d/n=1/10\).  For the last pair,
\[
 \frac{A_{2T,T}}{A_{2T-1,T}}
 =\frac{b_{T-1}}{2(b_{T-1}+b_{T-2})}\le\frac{d/n}{2},
 \qquad
 \frac{D_{2T,T}}{D_{2T-1,T}}=\frac12,
\]
so \(\frac{A_{2T,T}/D_{2T,T}}{A_{2T-1,T}/D_{2T-1,T}}\le d/n<9/10\) as well.  Thus every tail gap is at least
\[
 \frac1{10}{A_{j,T}/D_{j,T}}\ge\frac1{10}{A_{2T,T}/D_{2T,T}}
 =\frac1{10}\frac{(d/n)^{T-1}}{S_{T-1}S_T}
 \ge\frac1{10}\left(1-\frac dn\right)^2(d/n)^T,
\]
which proves \eqref{eq:finite-function-mp-gap} for all remaining indices.

For every real \(y\), the coefficients of
\((1+x)^2(B_\mu(x)^2-yC_\mu(x))\) are
\(D_{j,T}(A_{j,T}/D_{j,T}-y)\), so \eqref{eq:finite-function-mp-ratios} gives at most two sign changes.
Together with the first and last ratio inequalities and
\(r_{2T,T}<r_{0,T}\), the same horizontal-line and multiplicity argument
as in the proof of Lemma~\ref{lem:function-sign-support} shows that
\(B_\mu^2/C_\mu\) has a unique positive critical point, which is a
nondegenerate global maximum on \((0,\infty)\).  Since the coefficients of
\(B_\mu\) are positive and include both parities,
\[
 |B_\mu(-x)|<B_\mu(x),\qquad C_\mu(-x)=C_\mu(x)>0
 \quad(x>0).
\]
The maximum is consequently unique on \(\mathbb R\).

We now prove the derivative bounds. Pairing terms gives
\begin{equation}
 \frac{x}{2}(2B_\mu'C_\mu-B_\mu C_\mu')
 =\sum_{0\le j<i<T}
 \frac{(i-j)(d/n)^{i+j}x^{i+2j}}{S_iS_j}
       (S_{j+1}-x^{i-j}S_{i+1}).
 \label{eq:finite-function-left-derivative}
\end{equation}
Since
\[
 \frac{S_{k+1}}{S_{k+2}}-\frac{S_k}{S_{k+1}}
 =\frac{(d/n)^{k+1}}{S_{k+1}S_{k+2}}>0,\qquad
 \frac{S_{j+1}}{S_{i+1}}
 =\prod_{k=j+1}^{i}\frac{S_k}{S_{k+1}}
 \ge\left(\frac{S_1}{S_2}\right)^{i-j},
\]
every summand in \eqref{eq:finite-function-left-derivative} is
nonnegative at \(x=S_1/S_2\). More specifically, since \(T\ge3\),
we consider the \((i,j)=(2,0)\) summand and obtain
\[
 \left.\frac{x}{2}(2B_\mu'C_\mu-B_\mu C_\mu')\right|_{x=S_1/S_2}
 \ge 2(d/n)^4\left(\frac{S_1}{S_2}\right)^2\frac{S_1}{S_2^3}
 \quad(S_2^2-S_1S_3=(d/n)^2).
\]
By \eqref{eq:finite-function-BC},
\begin{equation}
 B_\mu(x)\ge1,\qquad
 C_\mu(x)\le S_1\sum_{k\ge0}(d/n)^k=\frac{S_1}{1-d/n},
 \qquad 0\le x\le1,
 \label{eq:mp-BC-unit-bounds}
\end{equation}
which proves the first bound in \eqref{eq:mp-endpoint-derivatives}.
For the second bound, every summand at \(x=1\) is negative, and
the \((1,0)\) summand is \(-(d/n)^3/S_1\). The derivative bound in
the proof of Lemma~\ref{lem:function-sign-support} also bounds this
finite sum, since the same bound was obtained by summing absolute values.
Thus the mean-value theorem gives
\[
\begin{aligned}
 &\left.\frac{x}{2}(2B_\mu'C_\mu-B_\mu C_\mu')\right|_{x=1-(d/n)^2/128}\\
 &\quad\le-\frac{(d/n)^3}{1+d/n}
 \left[1-\frac{3(1+2(d/n)+3(d/n)^2)}
 {128(1-d/n)^5(1+d/n)}\right]\le-c(d/n)^3.
\end{aligned}
\]
Together with \eqref{eq:mp-BC-unit-bounds}, this proves the two
derivative inequalities in \eqref{eq:mp-endpoint-derivatives}.
Their signs and the uniqueness already proved place the maximizer at
\(-\widetilde\theta_2^{\mathrm{MP}*}\le1-(d/n)^2/128\), which proves
the last assertion of that equation.
\end{proof}

Let \(-\widetilde\theta_2^{\mathrm{MP}*}\) denote the positive maximizer of
\(B_\mu^2/C_\mu\) from Lemma~\ref{lem:mp-function-unique}. Set
\[
\widetilde\theta_1^{\mathrm{MP}*}=
\frac{B_\mu(-\widetilde\theta_2^{\mathrm{MP}*})}
     {C_\mu(-\widetilde\theta_2^{\mathrm{MP}*})},
\qquad \mu=\mu_{d/n}.
\]
By Lemma~\ref{lem:projection}, define the corresponding polynomials by
\begin{equation}
 R_{T,t}=(1-\widetilde\theta_1^{\mathrm{MP}*}(-\widetilde\theta_2^{\mathrm{MP}*})^{t-1})R_{t,0}^{\mu}
 +(-1)^{t-1}\widetilde\theta_1^{\mathrm{MP}*}(-\widetilde\theta_2^{\mathrm{MP}*})^{t-1}\pi_t,\qquad1\le t\le T.
                                                               \label{eq:optimal-rollout}
\end{equation}

\begin{lemma}\label{lem:mp-uniform-bounds}
For \(0<d/n\le1/10\), \(T\ge8\), and \(S_k=\sum_{r=0}^k(d/n)^r\),
\begin{equation}
 \sum_{k=0}^{T-1}\frac{(d/n)^k}{S_k}
 -\frac{B_\mu(0.4)^2}{C_\mu(0.4)}
 \le\frac{36}{25}\frac dn,
 \qquad \mu=\mu_{d/n},
 \label{eq:mp-conditional-uniform-bound}
\end{equation}
and
\begin{equation}
 \sup_{0.4\le\eta_0\le0.9}
 \sum_{t=1}^\infty\E_{\lambda\sim\mu_{d/n}}
       [(1-\eta_0\lambda)^{2t}]<\frac34.
 \label{eq:mp-polynomial-expectation-bound}
\end{equation}
\end{lemma}
\begin{proof}
By the coefficient formulas in Lemma~\ref{lem:mp-function-unique},
\[
\begin{aligned}
 &\sum_{k=0}^{T-1}\frac{(d/n)^k}{S_k}
       -\frac{B_\mu(0.4)^2}{C_\mu(0.4)}\\
 &=\min_{\widetilde\theta_1}
   \sum_{k=0}^{T-1}\left[
   \frac{(d/n)^k}{S_k}(1-\widetilde\theta_1\,0.4^k)^2
   +(d/n)^{k+1}\widetilde\theta_1^2\,0.4^{2k}\right]\\
 &\le\frac{d/n}{S_1}
   +\frac{d/n}{S_1}\left(1-\frac{0.4}{S_1}\right)^2
   +\frac{0.4^2(d/n)^2}{S_1^2}
   +S_1\sum_{k=2}^\infty(d/n)^k\\
 &\le\left[\frac{34}{25}
       +\left(\frac4{25}+\frac{11}{9}-\frac35\right)\frac dn\right]\frac dn
 \le\frac{1618}{1125}\frac dn<\frac{36}{25}\frac dn.
\end{aligned}
\]
(The first inequality uses the feasible choice
\(\widetilde\theta_1=1/S_1\) and bounds each \(k\ge2\) summand by
\(S_1(d/n)^k\). The next inequality follows from
\[
 \frac{1+(1-0.4/S_1)^2}{S_1}\le\frac{34}{25}-\frac35\frac dn,
 \qquad
 \frac{0.4^2(d/n)}{S_1^2}
       +\frac{S_1(d/n)}{1-d/n}
 \le\left(\frac4{25}+\frac{11}{9}\right)\frac dn.
\]
For the first bound, the difference between its right and left sides is
\[
 \frac{(d/n)(7+7(d/n)-11(d/n)^2-15(d/n)^3)}
      {25(1+d/n)^3}\ge0.
\]
The second uses \(S_1/(1-d/n)\le11/9\).)

For the second estimate, the MP support gives
\(|1-\eta_0\lambda|<1\) throughout the stated parameter range.
Using the geometric identity and integrating gives
\[
\begin{aligned}
 \sum_{t=1}^\infty\E_{\lambda\sim\mu_{d/n}}
       [(1-\eta_0\lambda)^{2t}]
 &=-1+\frac1{2\eta_0}\E_{\lambda\sim\mu_{d/n}}[\lambda^{-1}]
       +\frac12\E_{\lambda\sim\mu_{d/n}}[(2-\eta_0\lambda)^{-1}]\\
 &=-1+\frac1{2\eta_0(1-d/n)}
       +\frac14\sum_{r=0}^\infty
           \left(\frac{\eta_0}{2}\right)^r
           \E_{\lambda\sim\mu_{d/n}}[\lambda^r].
\end{aligned}
\]
For the displayed expectation, termwise differentiation is justified by
\(\eta_0(1+\sqrt{d/n})^2/2<1\), uniformly on the stated rectangle, and
\begin{equation*}
 \frac{\partial^2}{\partial\eta_0^2}
 \sum_{t=1}^\infty\E[(1-\eta_0\lambda)^{2t}]
 =\sum_{t=1}^\infty 2t(2t-1)\E[\lambda^2(1-\eta_0\lambda)^{2t-2}]\ge0.
\end{equation*}
Thus the maximum in \(\eta_0\in[0.4,0.9]\) is attained at an endpoint.
For \(r\ge1\),
\[
 \E_{\lambda\sim\mu_{d/n}}[\lambda^r]
 =\sum_{j=1}^r\frac1r\binom rj\binom r{j-1}(d/n)^{j-1},\qquad
 \frac{\partial}{\partial(d/n)}\E[\lambda^r]\ge0,
\]
and \(\partial_{(d/n)}[2\eta_0(1-d/n)]^{-1}>0\).  The displayed sum is
therefore increasing in \(d/n\), so its maximum occurs at \(d/n=1/10\)
and \(\eta_0\in\{0.4,0.9\}\). Direct integration gives
\[
 \frac{397}{72}-\frac{\sqrt{1481}}8,\qquad
 \frac{823}{648}-\frac{\sqrt{6961}}{72},
\]
both of which are less than \(3/4\). This proves \eqref{eq:mp-polynomial-expectation-bound}.
\end{proof}
\begin{lemma}
\label{lem:tail-error}
For the polynomials \(R_{T,t}\) defined in \eqref{eq:optimal-rollout},
\begin{equation}
 \frac{\|R_{T,t}-R_{t,0}^{\mu}\|^2}{\|R_{t,0}^{\mu}\|^2}
 =\left(\widetilde\theta_1^{\mathrm{MP}*}\right)^2(-\widetilde\theta_2^{\mathrm{MP}*})^{2t-2}\sum_{r=0}^{t}\left(\frac{d}{n}\right)^r
 \le\frac{(1-\left(\frac{d}{n}\right)^2/128)^{2t-2}}{1-\left(\frac{d}{n}\right)}.         \label{eq:tail-error}
\end{equation}
\end{lemma}

\begin{proof}
Subtracting \(R_{t,0}^{\mu}\) in \eqref{eq:optimal-rollout} and using
\(R_{t,0}^{\mu}\perp\pi_t\) gives the equality in
\eqref{eq:tail-error}.  For \(0\le(-\widetilde\theta_2^{\mathrm{MP}*})\le1\), completion of squares
and telescoping give
\[
\begin{aligned}
 C_\mu(-\widetilde\theta_2^{\mathrm{MP}*})-B_\mu(-\widetilde\theta_2^{\mathrm{MP}*})
 &\ge \frac dn-\frac1{4(d/n)}
 \left(\frac1{1+d/n}-\frac1{\sum_{r=0}^{T}(d/n)^r}\right)\\
 &\ge\frac dn\left(1-\frac1{4(1+d/n)}\right)>0.
\end{aligned}
\]
Thus \(0<\widetilde\theta_1^{\mathrm{MP}*}<1\).  The upper
bound in \eqref{eq:tail-error} now follows from the last inequality in
\eqref{eq:mp-endpoint-derivatives} and
\(\sum_{r=0}^{t}(d/n)^r\le(1-d/n)^{-1}\).
\end{proof}

\begin{lemma}
\label{lem:finite-scalar-comparison}
Assume \eqref{eq:common-growing-regime}, and write
\(\nu=\nu_{d,n}\), \(\mu=\mu_{d/n}\). For \(\sigma\in\{\nu,\mu\}\),
write the existing coefficient functions as
\[
 B_\sigma(x)=\sum_{k=0}^{T-1}b_{k,\sigma}x^k,\qquad
 C_\sigma(x)=\sum_{k=0}^{T-1}c_{k,\sigma}x^{2k}.
\]
For all sufficiently large \(d\), their coefficients are positive and
\begin{equation}
 \max_{0\le k<T}
 (|b_{k,\nu}-b_{k,\mu}|+|c_{k,\nu}-c_{k,\mu}|)
 \le\exp\{CT\log(e+T)\}(n/d)^{6T}
       \left(\frac1d+\frac1n\right)=d^{-1+o(1)}.
 \label{eq:finite-function-coefficient-transfer}
\end{equation}
All constants are independent of \(d,n,T\).
\end{lemma}

\begin{proof}
The formula in the proof of
Lemma~\ref{lem:stage2-loss-decomposition} gives
\begin{equation}
 \min_{\substack{R(0)=-1,\ \deg R\le k+1\\
                 [\lambda^{k+1}]R=(-1)^k\ell}}
       \|R\|_\sigma^2
 =\|R_{k+1,0}^\sigma\|_\sigma^2
       -2b_{k,\sigma}\ell+c_{k,\sigma}\ell^2.
 \label{eq:finite-function-quadratics}
\end{equation}
For the three values \(\ell=-1,0,1\), this quadratic identity is exactly
\begin{equation*}
\begin{aligned}
 \|R_{k+1,-1}^\sigma\|_\sigma^2
 -\|R_{k+1,1}^\sigma\|_\sigma^2&=4b_{k,\sigma},\\
 \|R_{k+1,-1}^\sigma\|_\sigma^2
 +\|R_{k+1,1}^\sigma\|_\sigma^2
 -2\|R_{k+1,0}^\sigma\|_\sigma^2&=2c_{k,\sigma}.
\end{aligned}
\end{equation*}
Lemma~\ref{lem:projection-transfer} compares the three terms on the left,
so these two identities, together with the transfer bound, give
\eqref{eq:finite-function-coefficient-transfer}.  By \eqref{eq:finite-function-mp-coefficients} and
\(1\le S_k\le(1-d/n)^{-1}\),
\[
 b_{k,\mu}=\frac{(d/n)^k}{S_k}\ge(1-d/n)(d/n)^k,\qquad
 c_{k,\mu}=S_{k+1}b_{k,\mu}\ge b_{k,\mu}.
\]
Since the error in \eqref{eq:finite-function-coefficient-transfer}
is \(o((d/n)^T)\), the finite-Wishart coefficients are positive as well.
\end{proof}

\begin{lemma}
\label{lem:finite-coefficient-ratio-comparison}
Under \eqref{eq:common-growing-regime}, write
\(\nu=\nu_{d,n}\), \(\mu=\mu_{d/n}\). For
\[
 r_{j,\sigma}
 =\frac{[x^j](1+x)^2B_\sigma(x)^2}
        {[x^j](1+x)^2C_\sigma(x)},\qquad
 0\le j\le2T,\quad\sigma\in\{\nu,\mu\},
\]
the denominator coefficients are at least \(c(d/n)^T\), and
\begin{equation}
 \max_{0\le j\le2T}|r_{j,\nu}-r_{j,\mu}|
 \le CT\exp\{CT\log(e+T)\}(n/d)^{8T}
       \left(\frac1d+\frac1n\right)=o((d/n)^T).
 \label{eq:finite-function-ratio-transfer}
\end{equation}
Here \(c,C>0\) are independent of \(d,n,T\), and \(d\) is sufficiently large.
\end{lemma}

\begin{proof}
The MP coefficients of \((1+x)^2C_\mu\) are bounded below by
\(c(d/n)^T\), while those of \((1+x)^2B_\mu^2\) are bounded above by \(C\).
Expanding the products and using
\eqref{eq:finite-function-coefficient-transfer} gives
\[
\begin{aligned}
 &\max_{0\le j\le2T}
   \left|[x^j](1+x)^2(B_\nu(x)^2-B_\mu(x)^2)\right|\\
 &\qquad\le CT\exp\{CT\log(e+T)\}(n/d)^{6T}
                \left(\frac1d+\frac1n\right),\\
 &\max_{0\le j\le2T}
   \left|[x^j](1+x)^2(C_\nu(x)-C_\mu(x))\right|\\
 &\qquad\le C\exp\{CT\log(e+T)\}(n/d)^{6T}
                \left(\frac1d+\frac1n\right).
\end{aligned}
\]
The second error is \(o((d/n)^T)\), proving the denominator lower bound
also under \(\nu\). Subtracting the two fractions proves
\eqref{eq:finite-function-ratio-transfer}.
\end{proof}

\begin{lemma}
\label{lem:finite-BC-comparison}
Under \eqref{eq:common-growing-regime}, write
\(\nu=\nu_{d,n}\), \(\mu=\mu_{d/n}\). For all sufficiently large \(d\),
uniformly on \(0\le x\le10/3\), \((d/n)x^2\le9/10\),
\begin{equation}
 C_\sigma(x)\ge1,\qquad
 \max_{0\le j\le2}
 \bigl(|B_\sigma^{(j)}(x)|+|C_\sigma^{(j)}(x)|\bigr)\le C,
 \quad\sigma\in\{\nu,\mu\},
 \label{eq:finite-BC-extended-bounds}
\end{equation}
and
\begin{equation}
 \max_{0\le j\le2}
 \bigl(|B_\nu^{(j)}(x)-B_\mu^{(j)}(x)|
       +|C_\nu^{(j)}(x)-C_\mu^{(j)}(x)|\bigr)
 \le d^{-1+o(1)}.
 \label{eq:finite-scalar-derivative-transfer}
\end{equation}
The constant \(C\) is independent of \(d,n,T\).
\end{lemma}

\begin{proof}
For \(0\le j\le2\), differentiating the coefficient differences in
Lemma~\ref{lem:finite-scalar-comparison} gives
\[
\begin{aligned}
 &|B_\nu^{(j)}(x)-B_\mu^{(j)}(x)|
       +|C_\nu^{(j)}(x)-C_\mu^{(j)}(x)|\\
 &\quad\le CT^3(10/3)^{2T}
       \exp\{CT\log(e+T)\}(n/d)^{6T}
       \left(\frac1d+\frac1n\right)=d^{-1+o(1)}.
\end{aligned}
\]
This proves \eqref{eq:finite-scalar-derivative-transfer}.
The MP formulas give, on the stated region,
\[
 \max_{0\le j\le2}
 \bigl(|B_\mu^{(j)}(x)|+|C_\mu^{(j)}(x)|\bigr)
 \le C\sum_{k\ge0}(k+1)^2(9/10)^k\le C.
\]
Moreover, \eqref{eq:stage2-finite-BC-expectations} and the direct
second-moment identity
\[
 \E_{\lambda\sim\nu}[\lambda^2]
 =\frac1d\E\tr(\Sb^2)=1+\frac{d+1}{n}
 \ge1+\frac dn=\E_{\lambda\sim\mu}[\lambda^2]
\]
give
\[
 C_\nu(x)\ge1+\frac{d+1}{n}\ge1,
 \qquad C_\mu(x)\ge1+\frac dn\ge1.
\]
Together these prove \eqref{eq:finite-BC-extended-bounds}.
\end{proof}

\begin{lemma}
\label{lem:finite-objective-comparison}
Under \eqref{eq:common-growing-regime}, write
\(\nu=\nu_{d,n}\), \(\mu=\mu_{d/n}\). For all sufficiently large \(d\),
\begin{equation}
 \max_{0\le j\le2}\sup_{\substack{0\le x\le10/3\\
                                  (d/n)x^2\le9/10}}
 \left|\left(\frac{B_\nu^2}{C_\nu}\right)^{(j)}(x)
       -\left(\frac{B_\mu^2}{C_\mu}\right)^{(j)}(x)\right|
 \le d^{-1+o(1)}.
 \label{eq:finite-ratio-derivative-transfer}
\end{equation}
\end{lemma}

\begin{proof}
Apply the quotient rule through order two and
Lemma~\ref{lem:finite-BC-comparison}. Its lower bounds
\(C_\nu,C_\mu\ge1\) control the denominators, and its derivative
comparison controls each difference of numerator factors.
\end{proof}

\begin{lemma}
\label{lem:mp-scalar-landscape}
Under \eqref{eq:common-growing-regime}, for all sufficiently large \(d\),
the MP maximizer \(-\widetilde\theta_2^{\mathrm{MP}*}\) in
Lemma~\ref{lem:mp-function-unique} satisfies the following estimates
for fixed \(c,C>0\), depending only on the range of \(d/n\):
\begin{equation}
\begin{aligned}
 -\left(\frac{B_\mu^2}{C_\mu}\right)''(x)
 &\ge C^{-1},
 &&|x+\widetilde\theta_2^{\mathrm{MP}*}|\le2c,\\
 \left|\left(\frac{B_\mu^2}{C_\mu}\right)'(x)\right|
 &\ge C^{-1},
 &&|x+\widetilde\theta_2^{\mathrm{MP}*}|\ge c/2,\\
 \frac{B_\mu(-\widetilde\theta_2^{\mathrm{MP}*})^2}
      {C_\mu(-\widetilde\theta_2^{\mathrm{MP}*})}
 -\frac{B_\mu(x)^2}{C_\mu(x)}
 &\ge C^{-1}|x+\widetilde\theta_2^{\mathrm{MP}*}|^2.
\end{aligned}
 \label{eq:mp-scalar-uniform-landscape}
\end{equation}
These estimates hold on
\(1/150\le x\le10/3\), \((d/n)x^2\le9/10\),
with \(\mu=\mu_{d/n}\).
\end{lemma}

\begin{proof}
Lemmas~\ref{lem:function-sign-support} and~\ref{lem:mp-function-unique}
give the unique nondegenerate maximizers \(x_0\) and
\(-\widetilde\theta_2^{\mathrm{MP}*}\) of \(B^2/C\) and
\(B_\mu^2/C_\mu\), respectively. It remains to prove
\eqref{eq:mp-scalar-uniform-landscape} uniformly under
\eqref{eq:common-growing-regime}.

\medskip\noindent\textup{(i)}
We first use Lemma~\ref{lem:function-sign-support}.
The functions and \(x_0\) are continuous in \(d/n\), which ranges over
a fixed closed subinterval of \((0,1/10]\). Uniqueness of \(x_0\) and
\((B^2/C)''(x_0)<0\) give fixed \(c,m>0\) such that
\[
\begin{aligned}
 -\left(\frac{B^2}{C}\right)''(x)&\ge2m,
 &&|x-x_0|\le3c,\\
 \left|\left(\frac{B^2}{C}\right)'(x)\right|&\ge2m,
 &&|x-x_0|\ge c/4.
\end{aligned}
\]
Here and below, \(x\) is in the stated region, and \(c\) is small enough that
\([x_0-3c,x_0+3c]\) lies in that region.

We then compare the finite and infinite sums. The omitted terms satisfy
\[
 \sum_{k\ge T}(k+1)^2(9/10)^k\le CT^2(9/10)^T,
\]
so \eqref{eq:infinite-function} and \eqref{eq:finite-function-BC} give
\begin{equation}
 \max_{0\le j\le2}
 \left(\left|B_\mu^{(j)}(x)-B^{(j)}(x)\right|
 +\left|C_\mu^{(j)}(x)-C^{(j)}(x)\right|\right)
 \le CT^2(9/10)^T.
 \label{eq:finite-scalar-extended-convergence}
\end{equation}
Since \(T=\Theta((\log d)^{1/4})\), \(T^2(9/10)^T=o(1)\).
Using the bounds on \(B_\mu,B,C_\mu,C\) and their first two derivatives,
together with \(C_\mu,C\ge1\), differentiating the quotients gives
\[
 \max_{0\le j\le2}
 \left|\left(\frac{B_\mu^2}{C_\mu}\right)^{(j)}(x)
       -\left(\frac{B^2}{C}\right)^{(j)}(x)\right|
 \le CT^2(9/10)^T.
\]
For all sufficiently large \(d\), this gives
\[
 \left(\frac{B_\mu^2}{C_\mu}\right)'(x_0-c)>0>
 \left(\frac{B_\mu^2}{C_\mu}\right)'(x_0+c).
\]
By Lemma~\ref{lem:mp-function-unique},
\(|x_0+\widetilde\theta_2^{\mathrm{MP}*}|<c\). The mean-value theorem
and the zero derivative at the finite maximizer then give
\[
 2m|x_0+\widetilde\theta_2^{\mathrm{MP}*}|
 \le\left|\left(\frac{B^2}{C}\right)'
                 (-\widetilde\theta_2^{\mathrm{MP}*})\right|
 \le CT^2(9/10)^T,
 \qquad |x_0+\widetilde\theta_2^{\mathrm{MP}*}|\le c/4.
\]
Consequently,
\[
\begin{aligned}
 |x+\widetilde\theta_2^{\mathrm{MP}*}|\le2c
 &\ \Longrightarrow\ |x-x_0|\le9c/4<3c,\\
 |x+\widetilde\theta_2^{\mathrm{MP}*}|\ge c/2
 &\ \Longrightarrow\ |x-x_0|\ge c/4.
\end{aligned}
\]
Combining the preceding bounds gives
\[
\begin{aligned}
 -\left(\frac{B_\mu^2}{C_\mu}\right)''(x)
 &\ge2m-CT^2(9/10)^T\ge m,
 &&|x+\widetilde\theta_2^{\mathrm{MP}*}|\le2c,\\
 \left|\left(\frac{B_\mu^2}{C_\mu}\right)'(x)\right|
 &\ge2m-CT^2(9/10)^T\ge m,
 &&|x+\widetilde\theta_2^{\mathrm{MP}*}|\ge c/2.
\end{aligned}
\]

\medskip\noindent\textup{(ii)}
We finally use Taylor's formula and
\((B_\mu^2/C_\mu)'(-\widetilde\theta_2^{\mathrm{MP}*})=0\).
For \(|x+\widetilde\theta_2^{\mathrm{MP}*}|\le c\),
\[
\begin{aligned}
 &\frac{B_\mu(-\widetilde\theta_2^{\mathrm{MP}*})^2}
       {C_\mu(-\widetilde\theta_2^{\mathrm{MP}*})}
  -\frac{B_\mu(x)^2}{C_\mu(x)}\\
 &\quad=-(x+\widetilde\theta_2^{\mathrm{MP}*})^2
  \int_0^1(1-s)\left(\frac{B_\mu^2}{C_\mu}\right)''
  \!\left(-\widetilde\theta_2^{\mathrm{MP}*}
      +s(x+\widetilde\theta_2^{\mathrm{MP}*})\right)\,\mathrm ds\\
 &\quad\ge\frac m2|x+\widetilde\theta_2^{\mathrm{MP}*}|^2.
\end{aligned}
\]
For \(|x+\widetilde\theta_2^{\mathrm{MP}*}|>c\),
Lemma~\ref{lem:mp-function-unique} makes the same difference no smaller
than at distance \(c\) from the maximizer toward \(x\). Thus
\[
 \frac{B_\mu(-\widetilde\theta_2^{\mathrm{MP}*})^2}
      {C_\mu(-\widetilde\theta_2^{\mathrm{MP}*})}
 -\frac{B_\mu(x)^2}{C_\mu(x)}
 \ge\frac{mc^2}{2}
 \ge\frac{9mc^2}{200}|x+\widetilde\theta_2^{\mathrm{MP}*}|^2,
\]
since \(|x+\widetilde\theta_2^{\mathrm{MP}*}|\le10/3\).
Taking \(C\ge\max\{2/m,200/(9mc^2)\}\) proves
\eqref{eq:mp-scalar-uniform-landscape}.
\end{proof}

\label{app:wishart-supporting}

The following estimate follows from the exact real Wishart moment expansion
in \citet[Theorem~5.3 and Section~6]{redelmeier2011genus} and the MP moment
formula \citep{marchenko1967}.

\begin{lemma}\label{lem:wishart}
Let \(\Sb=n^{-1}\sum_{a=1}^n\xb_a\xb_a^{\mathsf T}\), where
\(\xb_1,\ldots,\xb_n\) are independent \(N(0,\Ib_d)\) vectors.
For every integer \(k\ge1\) and \(1\le d\le n\),
\begin{equation}
 0\le\frac1d\E\,\tr(\Sb^k)
 -\E_{\lambda\sim\mu_{d/n}}[\lambda^k]
 \le\frac{(2k-1)!!}{n}
 \le\frac{\exp\{Ck\log(e+k)\}}{n}.
 \label{eq:wishart-rate}
\end{equation}
The difference is zero for \(k=0\).
\end{lemma}

\label{app:supporting-training}
All estimates below are uniform over the full interval
\(0.4\le\eta_0\le0.9\). Since \(n=\Theta(d)\),
\[
 \eta_0^T\ge0.4^T=e^{-O(T)}=d^{-o(1)}.
\]

\begin{lemma}
\label{lem:finite-teacher-transfer}
Under~\eqref{eq:common-growing-regime}, the following estimates hold
uniformly for \(0.4\le\eta_0\le0.9\). On the support
\([(1-\sqrt{d/n})^2,(1+\sqrt{d/n})^2]\) of the MP law,
\begin{equation*}
 |1-\eta_0\lambda|
 \le\max\!\left\{
 1-0.4(1-\sqrt{d/n})^2,\,
 0.9(1+\sqrt{d/n})^2-1\right\}
 <\frac{17}{20}.
\end{equation*}
Moreover, for all integers $a,b\ge0$ with $a+b\le2T$,
\begin{equation*}
\left|\frac1d\E\tr
 [\Sb^b(\Ib-\eta_0\Sb)^a]
-\E_{\lambda\sim\mu_{d/n}}
 [\lambda^b(1-\eta_0\lambda)^a]\right|
\le \exp\{CT\log(e+T)\}\left(\frac1d+\frac1n\right)
=d^{-1+o(1)}.
\end{equation*}
\end{lemma}

\begin{proof}
The first estimate follows from $d/n\le1/10$. For the second, expand
\(\Sb^b(\Ib-\eta_0\Sb)^a\) and apply \eqref{eq:wishart-rate}
term by term. Since $0<\eta_0\le1$, the sum of the absolute binomial
coefficients is at most $2^a$, which is absorbed into
\(\exp\{CT\log(e+T)\}\). This proves the displayed transfer estimate.
The projection spaces are independent of the particular allowed step, since
\(\lambda,1-(1-\eta_0\lambda)^2,\ldots,
1-(1-\eta_0\lambda)^t\) span all zero-constant polynomials of degree at most
\(t\). Their triangular change of basis and its inverse have cost at most
\(d^{o(1)}\), because \(\eta_0^T\ge0.4^T=d^{-o(1)}\). Combining
\eqref{eq:finite-norm-equivalence} and \eqref{eq:mp-coefficient-norm}
with this change of basis gives the stated uniform scales.
\end{proof}

\begin{lemma}
\label{lem:finite-wishart-second-moment}
For $1\le d\le n$, the finite-Wishart second moment satisfies
\begin{equation}
 \frac1d\E\tr(\Sb^2)=1+\frac{d+1}{n}.
\label{eq:finite-wishart-second-moment}
\end{equation}
\end{lemma}

\begin{proof}
Wick's formula gives
\[
 \frac1d\E\tr(\Sb^2)=1+\frac{d+1}{n}.
\]
\end{proof}

\begin{lemma}
\label{lem:finite-wishart-quadratic-gap}
For $1\le d\le n$, the finite-Wishart quadratic gap satisfies
\begin{equation}
 \inf_{b\in\R}\frac1d\E\tr\bigl((\Sb^2-b\Sb)^2\bigr)
 =\frac1d\E\tr(\Sb^4)
 -\frac{(\E\tr(\Sb^3))^2}{d\,\E\tr(\Sb^2)}
 \ge\frac dn.
 \label{eq:finite-wishart-quadratic-gap}
\end{equation}
\end{lemma}

\begin{proof}
Wick's formula for the second, third, and fourth moments gives
\begin{align*}
 \frac1d\E\tr(\Sb^2)&=1+\frac{d+1}{n},\\
 \frac1d\E\tr(\Sb^3)&=1+\frac{3(d+1)}n+\frac{d^2+3d+4}{n^2},\\
 \frac1d\E\tr(\Sb^4)&=1+\frac{6(d+1)}n
 +\frac{6d^2+17d+21}{n^2}+\frac{d^3+6d^2+21d+20}{n^3}.
\end{align*}
Substitution yields
\begin{align*}
 &\frac{\E\tr(\Sb^2)\E\tr(\Sb^4)-(\E\tr(\Sb^3))^2}{d^2}
 -\frac dn\frac1d\E\tr(\Sb^2)\\
 &\quad=\left(\frac dn\right)^3
 +\frac{1+4(d/n)+5(d/n)^2+(d/n)^3}{n}
 +\frac{10+17(d/n)+10(d/n)^2}{n^2}\\
 &\qquad+\frac{17(1+d/n)}{n^3}+\frac4{n^4}\ge0.
\end{align*}
Completing the square in $b$ proves~\eqref{eq:finite-wishart-quadratic-gap}.
\end{proof}

\begin{lemma}
\label{lem:finite-wishart-moment-lower}
Under~\eqref{eq:common-growing-regime}, for all sufficiently large $d$,
every $1\le t\le T$ and $u_1,\ldots,u_t\in\R$ satisfy
\begin{equation}
\begin{aligned}
 \frac1d\E\tr\!\left[\left(\sum_{k=1}^t u_k\Sb^k\right)^2\right]
 &=\frac1d\sum_{j,k=1}^t u_ju_k\E\tr(\Sb^{j+k})\\
 &\ge\frac{(d/n)^t}{2(t+1)16^t}\sum_{k=1}^t u_k^2.
\end{aligned}
\label{eq:finite-wishart-moment-lower}
\end{equation}
\end{lemma}

\begin{proof}
For the higher moments, apply~\eqref{eq:finite-norm-equivalence} and
\eqref{eq:mp-coefficient-norm} to $\sum_{k=1}^t u_k\lambda^k$:
\[
 \frac1d\E\tr\!\left[\left(\sum_{k=1}^t u_k\Sb^k\right)^2\right]
 \ge\frac12\left\|\sum_{k=1}^t u_k\lambda^k\right\|_{\mathrm{MP}}^2
 \ge\frac{(d/n)^t}{2(t+1)16^t}\sum_{k=1}^t u_k^2.
\]
Expanding the square gives the finite-moment identity in
\eqref{eq:finite-wishart-moment-lower}.
\end{proof}

\begin{lemma}
\label{lem:finite-teacher-trace-estimates}
Under~\eqref{eq:common-growing-regime}, for all sufficiently large $d$,
every $1\le t<T$ satisfies
\begin{equation}
 \frac1d\E\tr\bigl(\Sb^2p_t^{\rm GD}(\Sb)^2\bigr)\le C.
\end{equation}
\end{lemma}

\begin{proof}
Since $p_t^{\rm GD}(\Sb)=\Ib_d-(\Ib_d-\eta_0\Sb)^t$,
\begin{align*}
 \frac1d\E\tr\bigl(\Sb^2p_t^{\rm GD}(\Sb)^2\bigr)
 &\le\frac2d\E\tr(\Sb^2)
 +\frac2d\E\tr\bigl(\Sb^2(\Ib-\eta_0\Sb)^{2t}\bigr)\\
 &\le2\left(1+\frac{d+1}{n}\right)+C(17/20)^{2t}
 +\exp\{CT\log(e+T)\}\left(\frac1d+\frac1n\right)
 \le C,
\end{align*}
where the second line follows from Lemma~\ref{lem:finite-teacher-transfer}.
This proves the lemma.
\end{proof}

\begin{lemma}
\label{lem:finite-teacher-gram-lower}
Under~\eqref{eq:common-growing-regime}, for every $1\le t<T$ and every
$u_1,\ldots,u_t\in\R$,
\begin{equation*}
 \begin{aligned}
 &\sum_{\ell=1}^{t}\left|
 \left\langle u_1\lambda+\sum_{j=2}^{t}u_jp_j^{\rm GD},p_\ell^{\rm GD}\right\rangle_{d,n}\right|^2\\
 &\qquad\ge \exp\{-CT\log(e+T)\}(d/n)^{CT}
 \left\|u_1\lambda+\sum_{j=2}^{t}u_jp_j^{\rm GD}\right\|_{d,n}^2,
 \qquad p_\ell^{\rm GD}(\lambda)=1-(1-\eta_0\lambda)^\ell.
 \end{aligned}
\end{equation*}
\end{lemma}

\begin{proof}
For \(1\le j\le t\),
\begin{equation*}
 p_j^{\rm GD}(\lambda)=1-(1-\eta_0\lambda)^j
 =\sum_{m=1}^{j}(-1)^{m+1}\binom jm\eta_0^m\lambda^m.
\end{equation*}
Thus \((\lambda,p_2^{\rm GD},\ldots,p_t^{\rm GD})\) and
\((\lambda,\ldots,\lambda^t)\) are triangular bases, and both coefficient
maps have norm at most \(\exp\{CT\log(e+T)\}\), since
\(0.4\le\eta_0\le0.9\).  Lemma~\ref{lem:finite-wishart-moment-lower}
gives, for every coefficient vector in the monomial basis,
\begin{equation*}
 \frac1d\E\tr\left[\left(\sum_{m=1}^t v_m\Sb^m\right)^2\right]
 \ge\frac{(d/n)^t}{2(t+1)16^t}\sum_{m=1}^t v_m^2.
\end{equation*}
Applying the two triangular coefficient bounds to the associated Gram
matrix gives the displayed inequality.
\end{proof}

\begin{lemma}
\label{lem:finite-teacher-gram-coefficients}
Under~\eqref{eq:common-growing-regime}, for every $1\le t<T$ and every
$u_1,\ldots,u_t\in\R$,
\begin{equation*}
 \begin{aligned}
 \sum_{\ell=1}^{t}\left|
 \left\langle u_1\lambda+\sum_{j=2}^{t}u_jp_j^{\rm GD},p_\ell^{\rm GD}\right\rangle_{d,n}\right|^2
 &\ge \exp\{-CT\log(e+T)\}(d/n)^{CT}
 \sum_{j=1}^{t}u_j^2,\\[-1pt]
 &\hspace{35mm}p_\ell^{\rm GD}(\lambda)=1-(1-\eta_0\lambda)^\ell.
 \end{aligned}
\end{equation*}
\end{lemma}

\begin{proof}
With the same triangular coefficient formula,
\begin{equation*}
 \left\|u_1\lambda+\sum_{j=2}^{t}u_jp_j^{\rm GD}\right\|_{d,n}^2
 \ge \exp\{-CT\log(e+T)\}(d/n)^{CT}
       \sum_{j=1}^{t}u_j^2.
\end{equation*}
Combining this inequality with Lemma~\ref{lem:finite-teacher-gram-lower}
yields the displayed bound after enlarging \(C\).
\end{proof}

\begin{lemma}
\label{lem:finite-teacher-projection-upper}
Under~\eqref{eq:common-growing-regime}, for all sufficiently large $d$,
\begin{equation}
 \sum_{t=1}^{T-1}
 \bigl\|\lambda p_t^{\rm GD}-P_t[\lambda p_t^{\rm GD}]\bigr\|_{d,n}^2
 \le 2\frac dn .
\end{equation}
\end{lemma}

\begin{proof}
Put $r=d/n$ and $S_t=\sum_{j=0}^t r^j$.  Since
$p_1^{\rm GD},\ldots,p_t^{\rm GD}$ span the zero-constant polynomials of degree at most $t$,
the finite projection calculation based on
\eqref{eq:finite-function-mp-coefficients} gives
\begin{equation}
 \min_{\substack{q(0)=0\\ \deg q\le t}}
 \bigl\|\lambda p_t^{\rm GD}-q\bigr\|_{\mathrm{MP}}^2
 =\eta_0^{2t}r^t\frac{S_{t+1}}{S_t}.
 \label{eq:finite-teacher-mp-projection}
\end{equation}
Lemma~\ref{lem:finite-teacher-transfer} and the inverse-Gram transfer
therefore give, uniformly over $1\le t<T$,
\begin{align*}
 \sum_{t=1}^{T-1}
 \bigl\|\lambda p_t^{\rm GD}-P_t[\lambda p_t^{\rm GD}]\bigr\|_{d,n}^2
 &\le \sum_{t=1}^{T-1}
       \eta_0^{2t}r^t\frac{S_{t+1}}{S_t}+d^{-1+o(1)}\\
 &\le \frac1{(1-r)^2}\sum_{t=1}^{\infty}r^t+d^{-1+o(1)}\\
 &=\frac{r}{(1-r)^2}+d^{-1+o(1)}\\
 &\le 2r
 =2\frac dn,
\end{align*}
where $r\le1/10$ and $d^{-1+o(1)}=o(r)$ under
\eqref{eq:common-growing-regime}.  This proves the lemma.
\end{proof}

\begin{lemma}
\label{lem:stage1-global-flow}
Under Assumption~\ref{ass:init}, the raw gradient flows in both stages
exist uniquely for all training times \(\tau\ge0\).
During Stage 1, for \(0\le s\le\tau\),
\begin{align*}
 &\|\Vb(\tau)-\Vb(s)\|_F^2+\|\Wb(\tau)-\Wb(s)\|_F^2\\
 &\quad\le(\tau-s)
       \bigl(\mathcal L_{\rm GD}(s)-\mathcal L_{\rm GD}(\tau)\bigr)
 \le(\tau-s)\mathcal L_{\rm GD}(s).
\end{align*}
The same bound holds during Stage 2 with
\(\mathcal L_{\rm GD}\) replaced by \(\mathcal L\).
\end{lemma}

\begin{proof}
Both losses are nonnegative polynomials in the trainable entries
(the rollout length is finite and all Gaussian task moments are finite).
Their gradients are locally Lipschitz, so both flows have unique local
solutions. For Stage 1, Cauchy--Schwarz gives
\begin{align*}
 &\|\Vb(\tau)-\Vb(s)\|_F^2+\|\Wb(\tau)-\Wb(s)\|_F^2\\
 &\quad\le(\tau-s)\int_s^\tau
       \bigl(\|\dot\Vb(r)\|_F^2+\|\dot\Wb(r)\|_F^2\bigr)\,\mathrm dr\\
 &\quad=(\tau-s)
       \bigl(\mathcal L_{\rm GD}(s)-\mathcal L_{\rm GD}(\tau)\bigr)
 \le(\tau-s)\mathcal L_{\rm GD}(s).
\end{align*}
The calculation is identical for Stage 2. At any finite endpoint of a
maximal existence interval, this bound makes the raw trajectory Cauchy
with a finite limit. Local existence at that limit extends the solution.
Thus both maximal intervals are infinite.
\end{proof}

\begin{lemma}
\label{lem:auxiliary-surjective}
For every \(\widetilde\theta_1>0>\widetilde\theta_2\) and finite active array
\(\widetilde{\bTheta}_4\), the recurrence \eqref{eq:global-recurrence}
has a raw representative satisfying \(\Wb_{33}=\zero_{d\times d}\) and
\eqref{eq:spatial-invariant}--\eqref{eq:positional-kernel-invariant}.
In particular, the auxiliary minimum in
Lemma~\ref{lem:stage2-loss-decomposition} is attainable.
\end{lemma}

\begin{proof}
Set the unused blocks, \(w_{44},\wb_{45},\wb_{54}\), and the first row
of \(\Wb_{55}\) to zero. For \(v_{31}<0\), choose
\begin{align*}
 w_{13}
 &=\frac{\widetilde\theta_2}{v_{31}},\\
 w_{24}
 &=\frac{\widetilde\theta_1
          \left(1+\sum_{t=1}^{T-1}(\widetilde{\bTheta}_4)_{1t}\right)
                 +\sigma v_{31}}{v_{31}(T+1)},\\
 (\wb_{25})_0
 &=\frac{\widetilde\theta_1}{v_{31}}-w_{24},&
 (\wb_{25})_t
 &=\frac{\widetilde\theta_1(\widetilde{\bTheta}_4)_{1t}}{v_{31}}
       -w_{24},\qquad 1\le t<T.
\end{align*}
Then
\[
 \sum_{t=0}^{T-1}(\wb_{25})_t-w_{24}=-\sigma,\qquad
 v_{31}(w_{24}+(\wb_{25})_0)=\widetilde\theta_1,\qquad
 \frac{v_{31}(w_{24}+(\wb_{25})_t)}{\widetilde\theta_1}
       =(\widetilde{\bTheta}_4)_{1t}.
\]
Moreover,
\begin{align*}
 w_{24}^2+\|\wb_{25}\|_2^2
 &=\frac{\widetilde\theta_1^2}{v_{31}^2}
   \left[
    1+\sum_{t=1}^{T-1}(\widetilde{\bTheta}_4)_{1t}^2
    -\frac{\left(1+\sum_{t=1}^{T-1}(\widetilde{\bTheta}_4)_{1t}\right)^2}{T+1}
   \right]+\frac{\sigma^2}{T+1},\\
 v_{31}^2-w_{13}^2-\frac{w_{24}^2+\|\wb_{25}\|_2^2}{d}
 &=v_{31}^2-\frac{\widetilde\theta_2^2}{v_{31}^2}
       -\frac{\sigma^2}{d(T+1)}\\
 &\quad-\frac{\widetilde\theta_1^2}{dv_{31}^2}
   \left[
    1+\sum_{t=1}^{T-1}(\widetilde{\bTheta}_4)_{1t}^2
    -\frac{\left(1+\sum_{t=1}^{T-1}(\widetilde{\bTheta}_4)_{1t}\right)^2}{T+1}
   \right].
\end{align*}
(The square bracket is positive by Cauchy--Schwarz.)
As a function of \(v_{31}^2>0\), the last expression increases strictly
from \(-\infty\) to \(+\infty\). Choose its unique value equal to
\(-\sigma^2/d\), which proves \eqref{eq:spatial-invariant}.

Finally, choose \(v_{33}>0\) by
\[
\begin{aligned}
 v_{33}^2
 &=\frac{(v_{33}^{\rm init})^2+
    \sqrt{(v_{33}^{\rm init})^4+
          \frac4d\sum_{1\le t<T}\sum_{\ell=2}^t
                    (\widetilde{\bTheta}_4)_{\ell t}^2}}2,\\
 (\Wb_{55})_{\ell t}
 &=\frac{(\widetilde{\bTheta}_4)_{\ell t}}{v_{33}}
 \quad(2\le\ell\le t<T).
\end{aligned}
\]
Thus
\[
 d v_{33}^2-\|\Wb_{55}\|_F^2
 =d v_{33}^2-\frac1{v_{33}^2}
       \sum_{1\le t<T}\sum_{\ell=2}^t(\widetilde{\bTheta}_4)_{\ell t}^2
 =d(v_{33}^{\rm init})^2.
\]
All three invariants and \eqref{eq:global-auxiliary}--\eqref{eq:global-theta4}
now hold.
\end{proof}

\begin{lemma}
\label{lem:finite-leading-gap}
Under \eqref{eq:common-growing-regime}, write \(\nu=\nu_{d,n}\) and $\mu = \mu_{d,n}$.
For all sufficiently large \(d\), every \(0\le t<T\) and \(\ell\in\R\)
satisfy
\begin{equation}
 \|R_{t+1,\ell}^\nu-R_{t+1,0}^\nu\|_{d,n}^2
 =\ell^2\|R_{t+1,1}^\nu-R_{t+1,0}^\nu\|_{d,n}^2,\qquad
 \frac{\|R_{t+1,\ell}^\nu-R_{t+1,0}^\nu\|_{d,n}^2}
      {\|R_{t+1,0}^\nu\|_{d,n}^2}
 \le\frac{4\ell^2}{1-d/n}.
 \label{eq:finite-leading-gap}
\end{equation}
\end{lemma}

\begin{proof}
By \eqref{eq:stage2-projection-affine} and
\eqref{eq:finite-function-quadratics},

\[
 R_{t+1,\ell}^\nu-R_{t+1,0}^\nu
 =\ell(R_{t+1,1}^\nu-R_{t+1,0}^\nu).
\]
Since \(R_{t+1,0}^\nu\) minimizes over all polynomials of degree at most
\(t\), \eqref{eq:finite-norm-equivalence} and Lemma~\ref{lem:mp-family} imply
\begin{align*}
 \|R_{t+1,0}^\nu\|_{d,n}^2
 =\min_{\substack{R(0)=-1\\\deg R\le t}}\|R\|_{d,n}^2\ge\frac12\min_{\substack{R(0)=-1\\\deg R\le t}}\|R\|_{\mathrm{MP}}^2
 =\frac{(d/n)^t}{2\sum_{r=0}^t(d/n)^r}.
\end{align*}
By Lemma~\ref{lem:finite-scalar-comparison},
\begin{align*}
 |\|R_{t+1,1}^\nu-R_{t+1,0}^\nu\|_{d,n}^2-\|R_{t+1,1}^\mu-R_{t+1,0}^\mu\|_{d,n}^2|
 &\le\exp\{CT\log(e+T)\}(n/d)^{6T}
                   \left(\frac1d+\frac1n\right)
 =o((d/n)^T),\\
 \|R_{t+1,1}^\mu-R_{t+1,0}^\mu\|_{d,n}^2
 &=\frac{(d/n)^t\sum_{r=0}^{t+1}(d/n)^r}
                  {\sum_{r=0}^t(d/n)^r}\ge(d/n)^t.
\end{align*}
Thus, uniformly for \(0\le t<T\),
\begin{align*}
 &\|R_{t+1,1}^\nu-R_{t+1,0}^\nu\|_{d,n}^2\le2\|R_{t+1,1}^\mu-R_{t+1,0}^\mu\|_{d,n}^2
 =\frac{2(d/n)^t\sum_{r=0}^{t+1}(d/n)^r}
                   {\sum_{r=0}^t(d/n)^r}, \\
 &\frac{\|R_{t+1,1}^\nu-R_{t+1,0}^\nu\|_{d,n}^2}{\|R_{t+1,0}^\nu\|_{d,n}^2}
 \le4\sum_{r=0}^{t+1}(d/n)^r\le\frac4{1-d/n}.
\end{align*}
Multiplying by \(\ell^2\) proves \eqref{eq:finite-leading-gap}.
\end{proof}

\begin{lemma}
\label{lem:hb-coefficient-comparison}
Under \eqref{eq:common-growing-regime}, for all sufficiently large \(d\),
uniformly for \(1\le t\le T\),
\begin{equation}
 \frac{\E\|\wb_t^{\rm HB}-\wstar-R_t^*(\Sb)\wstar\|^2}
      {d\|R_t^*\|_{\mathrm{MP}}^2}
 \le\exp\{CT\log(e+T)\}(d/n)^{-2T}
                      \left(\frac1d+\frac1n\right)^2.
 \label{eq:finite-tvhb-mp-transfer}
\end{equation}
Moreover, uniformly for \(0\le t<T\),
\begin{equation}
 |\alpha_t^{\rm HB}-1|+|\beta_t^{\rm HB}-d/n|
 \le C(d/n)^{t+1}
 +\exp\{CT\log(e+T)\}(d/n)^{-2T}
                      \left(\frac1d+\frac1n\right).
 \label{eq:tv-hb}
\end{equation}
\end{lemma}

\begin{proof}
The normal equations \eqref{eq:hb-optimal-normal-equations} and
Lemma~\ref{lem:mp-family} give
\begin{align*}
 \|p_t^\star-1-R_t^*\|_{d,n}^2
 &=\langle R_t^*,p_t^\star-1-R_t^*\rangle_{\mathrm{MP}}
   -\langle R_t^*,p_t^\star-1-R_t^*\rangle_{d,n}.
\end{align*}
Applying \eqref{eq:finite-inner-product-transfer} and
\eqref{eq:finite-norm-equivalence} yields
\[
 \|p_t^\star-1-R_t^*\|_{\mathrm{MP}}
 \le2\exp\{CT\log(e+T)\}(d/n)^{-T}
       \left(\frac1d+\frac1n\right)\|R_t^*\|_{\mathrm{MP}}.
\]
The identity \(\E\|f(\Sb)\wstar\|^2=d\|f\|_{d,n}^2\) proves
\eqref{eq:finite-tvhb-mp-transfer}. By \eqref{eq:mp-coefficient-norm},
\[
 \|p_t^\star-1-R_t^*\|_{\mathrm{coeff}}
 \le\exp\{CT\log(e+T)\}(d/n)^{-2T}
                      \left(\frac1d+\frac1n\right)=d^{-1+o(1)}.
\]
The two highest coefficients in \eqref{eq:Rstar} are
\[
 [\lambda^t]R_t^*=\frac{(-1)^{t+1}}{\sum_{r=0}^t(d/n)^r},\qquad
 \frac{[\lambda^{t-1}]R_t^*}{[\lambda^t]R_t^*}=-t(1+d/n),\qquad
 |[\lambda^t]R_t^*|\ge1-d/n.
\]
Substitution into the coefficient formulas
\eqref{eq:hb-optimal-coefficients} gives
\begin{align*}
 &\left|\alpha_t^{\rm HB}
      -\frac{\sum_{r=0}^t(d/n)^r}{\sum_{r=0}^{t+1}(d/n)^r}\right|
 +\left|\beta_t^{\rm HB}
      -\frac{(d/n)\sum_{r=0}^{t-1}(d/n)^r}
                   {\sum_{r=0}^{t+1}(d/n)^r}\right|\\
 &\quad\le\exp\{CT\log(e+T)\}(d/n)^{-2T}
                         \left(\frac1d+\frac1n\right).
\end{align*}
(For \(t=0\), the empty sum is zero and the first line of
\eqref{eq:hb-optimal-coefficients} applies.) Finally,
\[
 \left|\frac{\sum_{r=0}^t(d/n)^r}{\sum_{r=0}^{t+1}(d/n)^r}-1\right|
 +\left|\frac{(d/n)\sum_{r=0}^{t-1}(d/n)^r}
                   {\sum_{r=0}^{t+1}(d/n)^r}-\frac dn\right|
 =\frac{(2+d/n)(d/n)^{t+1}}{\sum_{r=0}^{t+1}(d/n)^r}
 \le3(d/n)^{t+1},
\]
which proves \eqref{eq:tv-hb}.
\end{proof}

\bibliography{trajectory_hb_references}

\bibliographystyle{ims}

\end{document}